\documentclass[12pt]{puthesis}
\newcommand{\proquestmode}{}

\RequirePackage{amsthm,amsmath,amssymb}
\RequirePackage[round]{natbib}

\RequirePackage{bm}
\RequirePackage{booktabs}
\RequirePackage{subcaption}
\RequirePackage[chapter]{algorithm}
\RequirePackage[noend]{algpseudocode}

\usepackage{nicefrac}
\usepackage{microtype}  
\usepackage{forest}
\forestset{sn edges/.style={for tree={edge={->}}}} 
\usepackage{tikz}
\usetikzlibrary{positioning,angles,quotes}
\usepackage[cmintegrals]{newtxmath}
\usepackage[cal=euler]{mathalfa}
\usepackage{enumitem}
\usepackage{caption}
\usepackage{bm}
\usepackage{wrapfig}
\usepackage{mathtools}
\usepackage[none]{hyphenat}

\newtheorem{manualtheoreminner}{Proposition}
\newenvironment{manualtheorem}[1]{%
  \renewcommand\themanualtheoreminner{#1}%
  \manualtheoreminner
}{\endmanualtheoreminner}

\newtheorem{prop}{Proposition}

\usepackage{xr}

\DeclarePairedDelimiter\abs{\lvert}{\rvert}
\DeclarePairedDelimiter\norm{\lVert}{\rVert}%
\makeatletter
\let\oldabs\abs
\def\abs{\@ifstar{\oldabs}{\oldabs*}}
\let\oldnorm\norm
\def\norm{\@ifstar{\oldnorm}{\oldnorm*}}
\makeatother

\usepackage{bibentry} 
\usepackage{url}     
\usepackage{amsfonts}       
\usepackage{microtype}      % microtypography
\usepackage{xcolor}         % colors

\usepackage{graphicx}
\usepackage{mathtools}

\usepackage{thmtools}
\usepackage{thm-restate}
\usepackage{enumitem}

\newcommand{\ctrain}{\gC_{\text{train}}}

\newcommand{\po}{\text{po}}
\DeclareMathOperator\supp{supp}
\newcommand\Tstrut{\rule{0pt}{2.6ex}}         % = `top' strut
\usepackage{appendix}
\AtBeginEnvironment{subappendices}{%
\chapter*{Appendix}
\addcontentsline{toc}{chapter}{Appendix}
\counterwithin{figure}{section}
\counterwithin{table}{section}
}

\usepackage{chngcntr}

\usepackage{macros}

\usepackage{url}
  
\title{Learning Algorithms for Intelligent Agents and Mechanisms}

\submitted{September 2022}  
\copyrightyear{2022}  
\author{Jad Rahme}
\adviser{Ryan P. Adams} 
\departmentprefix{Program in}

\usepackage{graphicx}
\usepackage{verbatim}
\usepackage{multirow}
\usepackage{longtable}
\usepackage{booktabs}
\ifdefined\printmode
\usepackage{url}
\else

\ifdefined\proquestmode
\usepackage[hidelinks]{hyperref}
\hypersetup{bookmarksnumbered}
\makeatletter
\hypersetup{pdftitle=\@title,pdfauthor=\@author}
\makeatother

\else

\usepackage[hidelinks]{hyperref}
\hypersetup{colorlinks,bookmarksnumbered}
\makeatletter
\hypersetup{pdftitle=\@title,pdfauthor=\@author}
\makeatother
\fi

\abstract{
The ability to learn from past experiences and adapt one's behavior accordingly within an environment or context to achieve a certain goal is a characteristic of a truly intelligent entity.  Developing efficient, robust, and reliable learning algorithms towards that end is an active area of research and a major step towards achieving artificial general intelligence. In this thesis, we research learning algorithms for optimal decision making in two different contexts, Reinforcement Learning in Part~\ref{part:rl} and Auction Design in Part~\ref{part:auctions}.  

Reinforcement learning (RL) is an area of machine learning that is concerned with how an agent should act in an environment in order to maximize its cumulative reward over time. In Chapter~\ref{chap:Zlearning}, inspired by statistical physics, we develop a novel approach to RL that not only learns optimal policies with enhanced desirable properties but also sheds new light on maximum entropy RL. In Chapter~\ref{chap:GeneralizationRL}, we tackle the generalization problem in RL using a Bayesian perspective. We show that imperfect knowledge of the environment's dynamics effectively turn a fully-observed Markov Decision Process (MDP) into a Partially Observed MDP (POMDP) that we call the Epistemic POMDP. Informed by this observation, we develop a new policy learning algorithm LEEP which has improved generalization properties.

An auction is the process of organizing the buying and selling of products and services that is of great practical importance. Designing an incentive compatible, individually rational auction that maximizes revenue is a challenging and intractable problem. Recently, a deep learning based approach was proposed to learn optimal auctions from data. While successful, this approach suffers from a few limitations, including sample inefficiency, lack of generalization to new auctions, and training difficulties.  In Chapter~\ref{chap:EquivariantNet}, we construct a symmetry preserving neural network architecture, EquivariantNet, suitable for anonymous auctions. EquivariantNet is not only more sample efficient but is also able to learn auction rules that generalize well to other settings.  In Chapter~\ref{chap:ALGnet}, we propose a novel formulation of the auction learning problem as a two player game. The resulting learning algorithm, ALGNet,  is easier to train, more reliable and better suited for non stationary settings. 
}

\acknowledgements{
First of all, I would like to thank my adviser, Ryan Adams, for his guidance and support throughout my doctoral studies at Princeton. I'm grateful for his availability, flexibility, and generosity, especially when it comes to sharing research ideas and insights on a wide range of topics. 

Throughout my PhD, I was fortunate to have the support of many Princeton faculty and administrative staff. I'm grateful to Matt Weinberg for his guidance and mentorship. His knowledge and expertise in mechanism design were critical when it came to bringing the second part of this thesis to fruition. I would also like to thank Peter Ramadge, Szymon Rusinkiewicz, Karthik Narasimhan for completing my thesis committee. I extend my gratitude to the supportive PACM department and Davis International Center.

I was also fortunate to take part in many opportunities outside of Princeton.  I would like to express my gratitude to Sergey Levine for inviting me to visit his lab at UC~Berkeley during the 2020-2021 academic year. His extensive knowledge and perspective on reinforcement learning helped me gain novel insights in the field. I'm also grateful for the summer internship opportunities I had at Quantlab, Susquehanna, and D.E.~Shaw~\&~Co. Despite being mostly virtual due to the pandemic, these experiences were very enriching and helped me acquire valuable knowledge, both theoretical and practical. 

I would also like to thank my co-authors and collaborators including: Samy Jelassi, Aviral Kumar, Benjamin Eysenbach, Bianca Dumitrascu, Dibya Ghosh, Joan Bruna and Amy Zhang. I extend my gratitude to my colleagues in the Laboratory for Intelligent Probabilistic Systems (LIPS) at Princeton.

Last but not least, I would like to thank my parents, brother, sister, grandparents, family members, and Iryna for their never-ending support.
}

\dedication{To my parents.}

\fi

\begin{document}
\nobibliography*

\makefrontmatter
\chapter{Introduction}

\section{Overview}
Reinforcement learning (RL) is an area of machine learning that is concerned with how an agent should act in an environment in order to maximize its cumulative reward over time. Whenever the agent takes an action, the state of the environment changes and the agent is rewarded accordingly.
Based on observations of the dynamics, 
 the agent gains a better understanding of its environment, which enables it to refine its strategy, also called {\it policy}, by choosing actions that are increasingly closer to optimality.
The difficulty in reinforcement learning is that the dynamics of the environment and their associated rewards (the rules of the game) are unknown to the agent in advance; they can only be inferred through trial and error.
This translates in practice to a tension between two types of behavior: exploration and exploitation.
Exploration consists of trying new strategies with the hope of finding better ones or gaining a better knowledge of the rules of the game.
Exploitation on the other hand is a more conservative attitude that consists of accumulating rewards via strategies that are already known to be good.
Optimally balancing these two behaviors is at the core of RL and remains a major open question to this day. 

Recent advancements on this challenging problem resulted in successes on tasks that were thought to be out of reach for our current technology.
One of the most notable examples is AlphaGo Zero \citep{alphaGoZero}, a computer program that achieved super-human performance in the traditional board game of Go through self-play.
The scope of reinforcement learning is, however, not limited to games; RL offers a very general framework to reason about a broad range of problems such as robotic manipulation and dexterity \citep{openAIdexterity}, data center cooling \citep{dataCenterCooling}, and optimizing chemical reactions \citep{zhou2017optimizing}.
Recently  RL has been successfully applied on various physics problems including optimal jet grooming \citep{jetgrooming}, quantum state preparation \citep{bukov2018reinforcement,bukov2018reinforcement2,albarran2018measurement}, quantum gate design \citep{niu2019universal} and quantum error correction \citep{fosel2018reinforcement}, often outperforming previous optimization methods.

Many approaches could be taken to tackle a reinforcement learning problem. Most of the successful ones however end up learning a quantity called a value function in one way or another. A value function is a function that takes a state of the environment and returns the expected cumulative rewards an agent can achieve starting from that state. A more technical definition of the value function can be found in Section~\ref{sec:RLintro}. In a game of chess or Go, a value function could for example look at the state of the board and compute the probability that Black wins.

Rewards and value functions have a very similar flavor to energies - they are extensive quantities and the agent is trying to find a path that maximizes them. Many natural phenomena can be understood via an extremization principle. For example, in classical mechanics or electrodynamics, the \emph{principle of least action} dictates that a mass or light will follow the path that minimizes a physical quantity called the \emph{action}.
Similarly, in thermodynamics, a system with many degrees of freedom---such as a gas---will explore its configuration space in search of a configuration that minimizes its free energy. In RL, value functions are often treated as the central object of study.
This stands in contrast to statistical physics formulations of such problems in which a quantity called the \emph{partition function} is the primary abstraction, from which all the relevant thermodynamic quantities---average energy, entropy, heat capacity---can be derived.
A natural question to ask is whether 
there exists a theoretical framework for reinforcement learning that is centered on a partition function, in which value functions can be interpreted via average energies?

This question is explored in Chapter~\ref{chap:Zlearning}. Inspired by the construction of partition functions in Statistical Physics, we construct a partition function for every state of the environment from the ensemble of possible trajectories spanning from that state. Although value functions can be derived from these partition functions and interpreted via average energies, we show that our purely partition function based approach can form the basis of alternative dynamic programming approaches. 

Compared to classical reinforcement learning methods, our approach has three main benefits. First, in deterministic environments, partition functions obey linear Bellman equations allowing direct solutions that were unavailable for the nonlinear equations associated with the use of traditional value functions. 
Second, our approach is able to treat all rewards equally over time, which contrasts with traditional approaches that need to discount future rewards in order to get well defined Bellman equations. 
Third, our approach learns policies that are qualitatively different from the ones found by classical RL algorithms. These policies not only optimize for energy (the sum of future rewards) but also take {\it entropy} into account, favoring states from which many good outcomes are possible. To illustrate that point, let's consider a simple setting in which an agent is trying to go from point A to point B. At point A, two actions are possible: going up or going down. If the agent chooses up then there is only one path that leads to B, but if he chooses down then there are 99 valid paths that lead to B. Let's further assume that all these paths are equally good. Traditional RL approaches will not have a preference between going up or going down as both of them lead to B. Our approach however will prefer going down and even prescribes choosing the down action 99 times more frequently than the up action. Such policies are desirable for their exploratory and robustness properties.

 From this example, we can see that our statistical physics based approach to reinforcement learning naturally leads our agent to select action in a stochastic way - this is referred to as a stochastic policy. In contrast, a policy is deterministic if the agent always selects the same action given a certain state. If an environment is known and Markovian, commonly referred to as a Markov Decision Process (MDP), one could prove that there always exists a deterministic policy that acts optimally in that environment. As a result, one could think that learning a deterministic policy is optimal \citep{sutton2018reinforcement}. However, when the environment is not fully known, this is not the case and deterministic policies can then fail in a miserable way. 

To illustrate that, consider a simple task in which an agent is first presented with an image of an animal and then has identify it with as few guesses as possible. In this example, a policy is just a mapping from images to guesses or labels. During the training phase, the agent is presented with images of dogs, cats and other animals from the training set and learns how to identify them. During the testing phase the agent is presented with new images of the same animals. In the hypothetical case where the training set is exhaustive, containing every picture of every animal from every angle, background and lighting condition, a deterministic policy learned on the training set will perform equally as good on the testing set. In real life however, data is limited and it is very unlikely that the agent will be able to learn a perfect classifier that correctly generalizes to new images with perfect accuracy. As a result, during the testing phase, the agent will at some point encounter an image that he is unable to classify correctly, not only in his first guess but in all the subsequent infinite number of available guesses as well because the policy is deterministic.  This is problematic, especially given that even a completely random guessing policy will eventually guess the correct label.

This simple toy experiment seems to indicate that there is a benefit in learning a stochastic policy to improve generalization. 
In fact, many successful RL algorithms encourage the agent to learn a stochastic policy by explicitly regularizing the entropy of the policy in the optimization objective.
Adding uniform randomness everywhere is, however, sub optimal.
The randomness in the agent's policy should reflect his uncertainty and confidence about his decision making. 
In Chapter~\ref{chap:GeneralizationRL} we study the generalization problem in Reinforcement Learning from a Bayesian perspective by modeling the uncertainty the agent has about the environment. We show how incomplete and imperfect knowledge of the environment implicitly turns a fully observed and Markovian environment (MDP), into a Partially Observed MDP (POMDP) \citep{sondik1971optimal} which we call the Epistemic POMDP. This novel point of view allows us to derive a new RL algorithm, LEEP, with improved generalization properties.

Having robust, data efficient, and reliable learning algorithms is a necessary ingredient to solve real world problems with RL. 
An equally important and crucial ingredient is having a good reward function. After all, this is the quantity that RL algorithms are trying to optimize. 
Many real world problems do not come with natural reward functions - these are usually crafted by humans. 
Even in situations where there is a natural reward function to optimize, there is still some benefit in engineering a new reward function that makes the optimization problem easier to solve. 
For instance in the game of chess, a natural reward function is to reward an agent with a~+1 for win,~0 for a draw and~-1 for a loss. 
This is an example of a sparse reward function because the agent only gets rewarded at the end of the game without intermediate feedback. 
Sparse reward functions are hard to optimize and it is sometimes helpful to introduce intermediate rewards that are positive when the agent captures an opponents' piece and negative when they lose one. 
While it's helpful to introduce intermediate rewards and more generally handcraft a reward function, reward shaping introduces human bias into the problem and in many cases, the policy that the agent learns will exploit the reward function in ways that the designer did not foresee. 
In some cases, the policy that maximizes the human engineered reward function does not solve the original task. One example of that is a game called Coast Runners where boats compete to finish a race as quickly as possible. To help them, intermediate targets were designed along the racetrack that not only help them get speed boosts but also reward them with extra bonus points when they are hit. 
It turned out that these intermediate targets disincentivized the agent from learning to win the race. The highest possible score is achieved by ignoring the race and focusing on hitting these intermediate targets \citep{ clark2016faulty}.

Unintended negative consequences of an incentive are not restricted to games or RL, they can be found in many real life societal policies as well. The Cobra effect \citep{siebert2001kobra} is a historical anecdote used to illustrate perverse incentives that presumably occurred in India during the British rule. The British government, concerned about the increasing number of venomous cobras in Delhi, offered a bounty incentive for every dead cobra in hopes of reducing their numbers. This policy was very successful at reducing the number of cobras until the people realized that they could significantly increase their income by breeding cobras. The policy was subsequently canceled and the final situation was worse than the starting point. 
%[cite more examples?]

Designing reward functions that are robust to these perverse incentives is not a very well studied subject within reinforcement learning. It is however a central theme in Mechanism Design, a sub field of Game Theory. Game Theory studies the emergent macroscopic behavior resulting from known microscopic interactions between agents. Mechanism Design goes the other way - it starts from a desirable macroscopic behavior and then tries to design mechanisms and incentives, or the rules of the game, that would result in this global behavior, assuming individuals act rationally. That's why it's sometimes referred to as reverse game theory. Mechanism design has been applied to many fields from economics and politics  to many problems such as market design, auction theory, social choice theory, voting systems, networked-systems, and many others. In this thesis, we will focus on Auction Theory, and while some of the ideas and methods we introduce are specific to that domain, others are more generally applicable to other domains in mechanism design.

An auction is a process of organizing the buying and selling of products and services that is of great practical importance in many private and public sectors. Examples include the sales of treasury bills by the US government, radio wave frequencies by the FCC, art by Christie’s, or ads by Google. A simple auction model goes as follows: at the start of the auction, each one of the bidders place a bid on each one of the items. All of these bids are then collected by the Auctioneer who decides the item allocation as well as the amount each bidder has to pay for their participation in the auction. While any mapping from bids to allocations and from bids to payments could constitute a valid auction mechanism, in practice, we prefer auctions that verify some desirable properties. 

The first desirable property is called {\it incentive compatibility}. An incentive compatible auction mechanism is one where the utility of each bidder is maximized by bidding truthfully on each one of the items. This means that the optimal bid on an item is exactly the amount of money that the bidder is willing to pay for the item. If an auction is not incentive compatible, bidders could strategically choose their bids and potentially bid untruthfully to maximize the value they get from participating in the auction. Enforcing incentive compatibility disincentivises any strategic behavior and as a result levels the playing field for all the bidders, irrespective of their experience, motivation, and means. Incentive compatible auctions are sometimes referred to as strategy-proof auctions. 

The second desirable property is called {\it individual rationality}. An individually rational auction is one where a bidder is never worse off after participating in the auction as long as their bid is truthful. For instance, this means that a bidder will not be charged for participating in the auction if she didn't end up getting any of the items. More generally, the value of the items that a truthful bidder gets is always greater than or equal to the amount she has to pay to the auctioneer. Individual rationality encourages participation in the auction. 

How to design an incentive compatible, individually rational auction mechanism that maximizes revenue for the auctioneer? Despite its apparent simplicity, this problem turns out to be surprisingly hard. In the case where there is a single item for sale, the solution is known from Myerson's seminal piece of work \citep{myerson1981optimal}. Beyond the single item setting, the problem is not completely resolved even for auctions as simple as two bidders and two items, despite forty years of mathematical research. Another line of work to confront this theoretical hurdle consists in building automated methods to find the optimal auction, typically by framing the problem as a linear program. However, this approach suffers from severe scalability issues as the number of constraints and variables grows exponentially with the number of bidders and items \citep{guo2010computationally}. 

A recent line of work initiated by \citet{dutting2017optimal} leverages the expressivity and scalability of neural networks to go beyond the limitations of linear programs. Their idea is to parameterize the allocation and payment functions with deep neural networks and use gradient descent to learn an incentive compatible, individually rational auction mechanism that maximizes revenue. Their algorithm, RegretNet is capable of finding near-optimal results in several known settings and obtaining new mechanisms in unknown cases. While very successful, RegretNet suffers from two weakness. In practice, the algorithm is sample inefficient and hard to train.

RegretNet can require a large number of samples to learn an optimal auction. Furthermore, it is incapable of generalizing to new auctions with a different number of bidders and items. An optimal mechanism learned by RegretNet for an auction consisting of $n$ bidders and $m$ items can only be used on auctions with $n$ bidders and $m$ items because the neural network expects inputs of a specific dimension. If an additional bidder were to join or leave the auction or a new item was added or removed, then we would have to build and train a new auction mechanism from scratch. In Chapter~\ref{chap:EquivariantNet}, we tackle these two issues in the case of symmetric auctions.  These
are auctions which are invariant to the relabeling of the items or bidders. More specifically, such auctions are anonymous (in that they can be executed without any information about the bidders, or labeling them) and item-symmetric (in that it only matters what bids are made for an item, and not its a priori label). We prove that these auctions always admit optimal allocation and payment functions that are equivariant and then proceed to build equivariant neural network architectures that respect this symmetry. Our algorithm, EquivariantNet, is more sample efficient than RegretNet and can also generalize to different auctions. 

The loss function in RegretNet is non stationary. It depends on several hyperparameters whose values change over time according to a predefined schedule. This makes RegretNet hard to train in practice. We also observe experimentally that the algorithm is very sensitive to the choice of these hyperparameters, converging to a suboptimal mechanism when these are not picked appropriately. Furthermore, these hyperparameters are setting dependent, their values depend on the number of bidders and items in the auction, and are currently found through an expensive hyperpameter search. In Chapter~\ref{chap:ALGnet}, we construct a novel, stationary, and hyperparameter-free loss function inspired by recent theoretical results from Auction Theory and propose a novel formulation of the auction learning problem as a two player game. The first player, the Auctioneer, proposes new auction rules. The second player, the Misreporter, is trying to exploit these rules and find optimal ways to bid untruthfully. These two players interact and over time, the Misreporter becomes better at finding optimal bids and the Auctioneer becomes better at proposing auction mechanisms that increasingly get closer to being incentive compatible. We call this algorithm ALGNet and we show that it's as good or better than RegretNet, while being nearly hyper-parameter free.

\section{Summary of Contributions}

The contributions of this dissertation are summarized below:

\begin{itemize}
    \item In Chapter~\ref{chap:Zlearning}, we propose a novel approach to the Reinforcement Learning problem. Inspired by Statistical Physics, we construct a partition function for each state of the environment and derive the corresponding Bellman equation.
    Our approach has three main benefits. First, it results in simpler equations, especially if the environment is deterministic. Second, it is able to treat all rewards equally over time (no need for a discount factor). 
    Third, it learns policies that not only optimize for rewards  but also take entropy into account, favoring states from which many good outcomes are possible.
    
    Chapter~\ref{chap:Zlearning} is based on the following work:
    
    {\it \bibentry{rahme2019theoretical}}.
    
    \item In Chapter~\ref{chap:GeneralizationRL}, we study the
generalization problem in Reinforcement Learning from a Bayesian perspective by
modeling the uncertainty the agent has about the environment. We show how
incomplete and imperfect knowledge of the environment implicitly turns a fully
observed and Markovian environment (MDP) into a Partially Observed MDP
(POMDP), which we call the Epistemic POMDP. This novel point of view allows us to
derive a new RL algorithm, LEEP, with improved generalization properties.

Chapter~\ref{chap:GeneralizationRL} is based on the following work (* indicates equal contribution):

{\it \bibentry{rahme2021generalization}}.

\item In Chapter~\ref{chap:EquivariantNet}, we prove that symmetric auctions always admit optimal allocation and payment functions that are equivariant and then
proceed to build an equivariant neural network architecture that respects this symmetry.
We show that our algorithm, EquivariantNet, is not only more sample efficient than previous methods but can also generalize well to different auctions. 

Chapter~\ref{chap:EquivariantNet} is based on the following work:

{\it \bibentry{Rahme_Jelassi_Bruna_Weinberg_2021}}.

\item In Chapter~\ref{chap:ALGnet}, we construct a novel, stationary, and hyperparameter-free loss
function for the auction learning problem inspired by recent theoretical results from auction theory, and propose a novel formulation of the auction learning problem as a two player game (similar to GANs, \citep{goodfellow2014generative}). We show that the resulting algorithm
ALGNet is as good or better than RegretNet, while being nearly
hyper-parameter free. \\\\

Chapter~\ref{chap:ALGnet} is based on the following work:

{\it \bibentry{algnet}}.
\end{itemize}

\noindent The following two sections cover some background material on Reinforcement Learning and Auction Theory.
\clearpage
\section{Background in Reinforcement Learning}
\label{sec:RLintro}

In this section, we review the setup of the Reinforcement Learning problem as well as some of its basic concepts and approaches. A good reference on Reinforcement Learning can be found in \citet{sutton2018reinforcement}.
\subsection{The RL Problem}
\subsubsection{RL as a Markov Decision Problem}

\begin{figure}[H]
    \centering
    \includegraphics[width=0.8\linewidth]{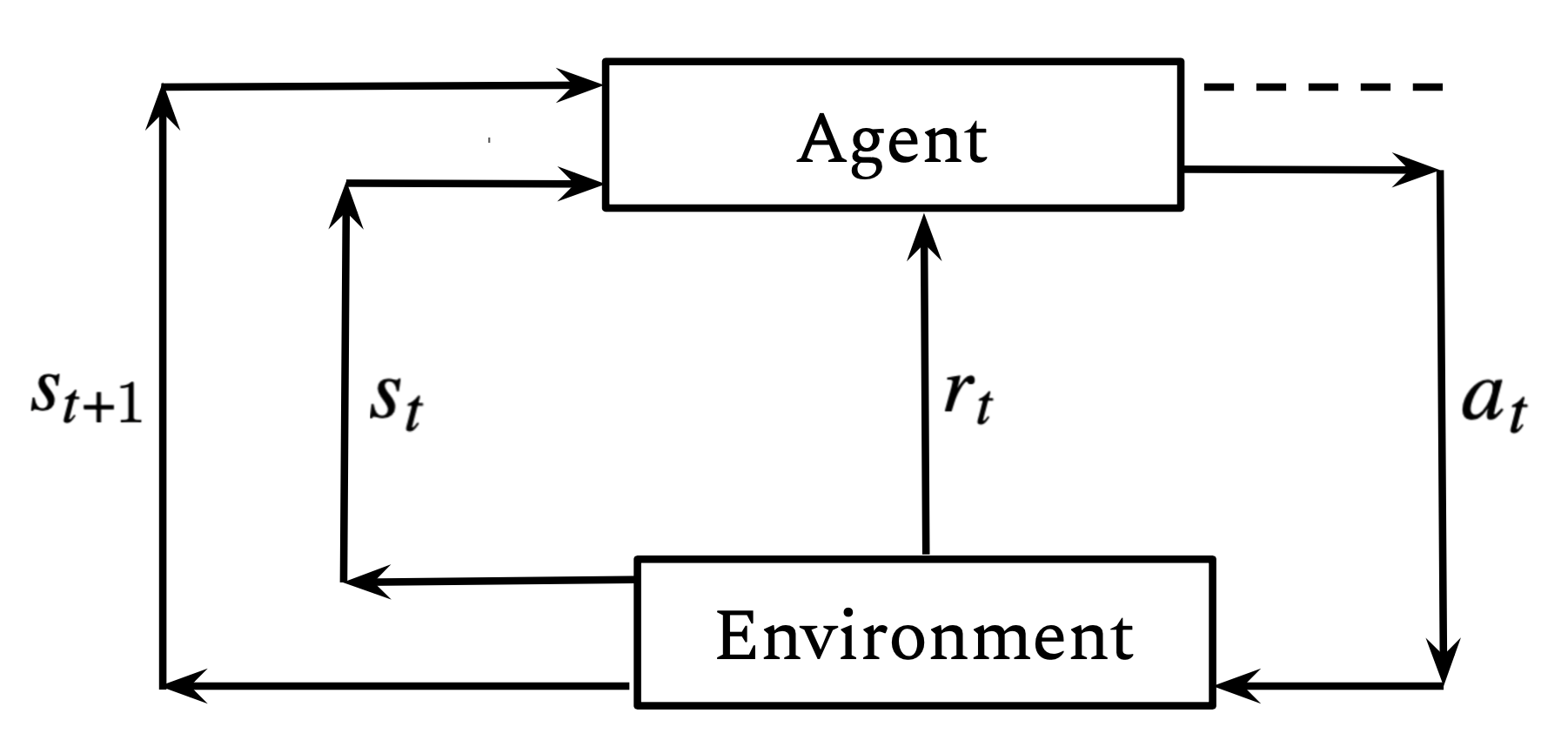}
    \caption{Representation of the agent–environment interaction in a Markov decision process.}

\end{figure}

Reinforcement Learning (RL) is an area of Machine Learning (ML) that studies how agents should behave in an environment in order to maximize their cumulative reward. The agent's sequential decision-making process is usually modeled as a Markov Decision Process (MDP). 

At every time step~$t$, the agent observes the state of the environment,~$s_t$, and then decides what action to take,~$a_t$. This action has two effects - first, it changes the state of the environment from~$s_t$ to~$s_{t+1}$ and second, it rewards the agent with a reward~$r_t$. The state~$s_{t+1}$ and the reward~$r_t$ are random variables - the same causes don't necessarily results in the same effects. In an MDP, the environment's dynamics are allowed to be stochastic but they have to be Markovian. This means that the distribution of states~$s_{t+1}$ that the agents lands in after taking action~{$a_t$} from state~$s_{t}$ only depends on~$s_t$ and~$a_t$, and not on past states~$\{s_{t'}\}_{{t'}<t}$ or past actions~$\{a_{t'}\}_{{t'}<t}$. This also holds for the reward $r_t$: its distribution is only a function of the initial state~$s_t$, the action taken~$a_t$, and the landing state~$s_{t+1}$.\\

\noindent More formally, an MDP is defined by the objects~$(\mcS,\mcA,\mcR,\mcP)$ where:
\begin{itemize}
\item $\mcS$ is the set of states the environment can be in,
\item $\mcA$ is the set of actions an agent can take, 
\item ${\mcP(s,a,s') = \mathbb{P}(s'\mid s,a)}$ is the probability of landing in state~$s'$ after taking action~$a$ while still in state~$s$,
\item $\mcR(s,a,s')$ is the reward resulting from the transition $s \xrightarrow[]{\text{a}} s'$. $\mcR(s,a,s')$ is a random variable. We usually assume that all rewards are bounded from above by~$\mcR_{\text{max}}$.
\end{itemize}

{\bf Deterministic MDPs:} An MDP is deterministic if~$\mcP(s,a,s')$ is~$0$ for all states~$s'$ except one, which will be conveniently denoted by~${s'=s+a}$. In this case we have~${\mcP(s,a,s') = \delta_{s+a}(s')}$ and we will concisely denote~$\mcR(s,a,s+a)$ by~$\mcR(s,a)$. $\mcR(s,a)$ is also assumed to be deterministic.

\subsubsection{Policies}

In RL, the policy $\pi$ describes how the agent acts in the environment.
$\pi(a \mid s)$ denotes the probability that the agent picks action~$a$ while in state~$s$.
If we denote the reward resulting from the $t$-th transition by $r_{t}$, so that ${r_t \coloneqq \mcR(s_t,a_t,s_{t+1})}$, we can express the cumulative reward $\mcR_{\text{total}}(\pi) $ of such a policy as: 
\begin{equation}
\label{eqn:rl_objective_into}    \mcR_{\text{total}}(\pi) = \mathbb{E}_{\substack{a_t \sim \pi(. \mid s_t) \\s_{t+1} \sim \mathbb{P}(. \mid s_t, a_t) }} \left[ \sum_{t=0}^{+\infty} \gamma^t r_t \right]\,,
\end{equation}
where the expectations are taken with respect to the realized sequences of states and actions, according to the policy and environment dynamics.
Here~${\gamma \in [0,1)}$ is called the \emph{discount factor} that can be interpreted as a preference for immediate rewards over future ones. The discount factor~$\gamma$ is also necessary for mathematical reasons: without this discount factor, many quantities in RL are not well defined.
For example, the infinite series determining $\mcR_{\text{total}}(\pi)$ could diverge.

\subsubsection{The RL Objective}

The goal of RL is to find an optimal policy $\pi^{*}$ that maximizes $\mcR_{\text{total}}(\pi)$:
\begin{align*}
\pi^{*} &= \argmax_{\pi} \,\, \mcR_{\text{total}}(\pi)\,.
\end{align*}

There are two main approaches to solving a RL problem: value function type approaches and policy gradient type approaches. In the following sections we will give a quick exposition of both of these approaches. These sections are not meant to be exhaustive or extensive by any measure, their goal is to give the reader some background that could help them contrast traditional approaches to RL with the novel approaches and methods that we propose in Chapters~\ref{chap:Zlearning}~and~\ref{chap:GeneralizationRL}.

\subsection{Value Function Approaches}
\subsubsection{Value functions associated with a policy}
The value function~$V$ associated with a policy~$\pi$ is a function of the state~$s$ that measures the expected cumulative reward an agent will get by following the policy~$\pi$ starting from state~$s$:
\begin{align*}
    V^{\pi}(s) = \mathbb{E}_{\substack{a_t \sim \pi(. \mid s_t) \\s_{t+1} \sim \mathbb{P}(. \mid s_t, a_t) }} \left[ \sum_{t=0}^{+\infty} \gamma^t r_t \mid s_0 =s \right]\,.
\end{align*}
The value functions at different states are connected by a recursion called the {\it Bellman equation}: 
\begin{align*}V^{\pi}(s) =  \mathbb{E}_{\substack{a \sim \pi(. \mid s) \\s' \sim \mathbb{P}(. \mid s, a) }} \left[ \mcR(s,a,s') + \gamma\, V^{\pi}(s')  \right]\,.
\end{align*}
In the Bellman equation above, the expectation is taken with respect to a single action and a single state transition.
Similar to $V^{\pi}$, we can define another type of value function~$Q^{\pi}$ (the {\it ``Q-function''}) which is a function of a state-action pair $(s,a)$, now measuring expected cumulative reward from following the policy $\pi$ after taking action~$a$ from state $s$:
\begin{align*}
Q^{\pi}(s,a) &= \mathbb{E}_{\substack{s_{t+1} \sim \mathbb{P}(. \mid s_t, a_t) \\ a_{t+1} \sim \pi(. \mid s_{t+1}) }} \left[ \sum_{t=0}^{+\infty} \gamma^t r_t \mid s_0 =s, a_0=a \right]\,.
\end{align*}
$Q^{\pi}$ also follows a Bellman equation given by:
$$Q^{\pi}(s,a) =  \mathbb{E}_{\substack{s' \sim \mathbb{P}(. \mid s, a) \\ a' \sim \pi(. \mid s') }} \left[ \mcR(s,a,s') + \gamma \,Q^{\pi}(s',a')  \right]. $$

\subsubsection{Optimal value functions}
When the policy $\pi$ is optimal, $\pi = \pi^{*}$, the Bellman equations for $V$ and $Q$ are referred to as the {\it optimal Bellman equations} and are given by:
\begin{equation}
\label{eqn:optimal_bellman_QV}
    \begin{aligned}
    V^{*}(s) &=  \max_{a\in \mcA} \,\, \mathbb{E}_{\substack{s' \sim \mathbb{P}(. \mid s, a) }} \left[ \mcR(s,a,s') + \gamma V^{*}(s')  \right]\,,\\
    Q^{*}(s,a) &= \mathbb{E}_{\substack{s' \sim \mathbb{P}(. \mid s, a)}} \left[ \mcR(s,a,s') + \gamma  \max_{a'\in \mcA} \,\, Q^{*}(s',a')  \right]\,.
    \end{aligned}
\end{equation}
These optimal Bellman equations are fixed point equations.
The mapping underlying these fixed point equations is called the {\it Bellman operator}.
One can show that when~${0\leq \gamma < 1}$, these Bellman operators are contractions of norm~$\gamma$.
As a result, when the Bellman operator is known, one can converge to $V^{*}$ and $Q^{*}$ by successive iterations of their Bellman operators starting from any initialization (via the Banach fixed-point theorem). The optimal policy $\pi^{*}$ can then be recovered from the optimal value function, as:
\begin{equation*}
    \pi^{*}(a\mid s) =  \begin{cases} 1 &\mbox{if } a = \argmax_{a'\in \mcA} \,\, Q^{*}(s,a') \,, \\ 0  &\mbox{otherwise }.\end{cases}
\end{equation*}

Finding the optimal value function through iteration of the Bellman operator is not always possible. In most settings, the dynamics of the MDP, $\mathcal{P}$, and the reward function, $\mathcal{R}$, are not fully known and as a result, it's not possible to solve the RL problem through an exact fixed point iteration scheme. Furthermore, even when the dynamics of the environment are known, it is not always possible to proceed through fixed point iteration and this is especially true for MDPs with large state and actions spaces. 

% Many approaches are possible to circumvent these limitations. In the following we will mention two of these methods, $Q$-Learning and its extension, Deep $Q$-Networks (DQN). 

\subsubsection{Exploration}
When the dynamics of the environment ($\mcR$ and $\mcP$) are unknown, the expectations in the Bellman equations cannot be computed exactly; they can only be estimated using samples collected through interactions with the environment.

The policy $\pi_{\text{exploration}}$ used by the agent to collect these samples and learn about the environment is called the {\it exploration policy}.
Depending on the problem, some exploration policies can be better than others.
Two of the most popular exploration policies are:
\begin{itemize}
    \item {\it $\epsilon$-greedy}: Pick $\argmax_{a'\in \mcA} \,\, Q(s,a')$ with probability $1-\epsilon$, and any action uniformly at random with probability $\epsilon$.
    \item {\it Boltzmann exploration of parameter $\beta$}: At state $s$, pick action $a$ proportionally to $\exp\left[\beta \,\, Q(s,a) \right]$.
\end{itemize}
Both of these exploration policies use current estimates of the $Q$-function.
In the beginning, when the agent only had a few interactions with the environment, the estimate for the optimal $Q$-function is very uncertain, and typical values for $\epsilon$ and $\beta$ are chosen to be $1$ and $0$ respectively. This corresponds to taking actions uniformly at random.
As the agent interacts more with the environment and collects more data, $\epsilon$ is decreased to~$0$ and $\beta$ is increased to a large number.
This reflects our confidence that the $Q$-function is becoming more accurate over time.

Some RL algorithms (e.g. $Q$-learning) only use the latest transition seen by the exploration policy to update the $Q$~values and as a result don't need to store past past transitions. Other algorithms (e.g. DQN)  don't limit themselves to the latest transition but also use previously seen transitions to update their current $Q$~values. When that is the case, all the interactions with the environment are recorded and stored in a dataset called a {\it Replay Buffer}. A typical entry in the replay buffer takes the form a tuple $(s_t, a_t, s_{t+1}, r_t, f_t)$, where $f_t$ is a Boolean entry that indicates whether the episode ended or is still ongoing. 

% In the following we're going to see how to use the data collected in the replay buffer to learn the optimal $Q$-function and policy. A RL algorithm that uses a replay buffer is qualified as {\it off policy} learning algorithms because it uses all previous interactions with the environment to improve the agent's current policy, including transitions collected with an older version of the agent's policy which is potentially a very different policy. We can contrast this with an {\it on policy} learning algorithm where the only data used to improve of policy $\pi_t$  comes from samples collected with the policy $\pi_t$ itself. 

\subsubsection{$Q$-learning}
The $Q$-learning algorithm is a popular value function based, RL learning algorithm first introduced in \citet{watkins1989learning}. $Q$-learning does not assume that the dynamics of the environment are known - it's a {\it model-free} algorithm. 

The algorithm goes as follows: Initially, all $Q$~values are initialized to arbitrary values (or randomly). Then, at each time step $t$, the agent looks at the state~$s_t$ of the environment, takes action~$a_t$, and observes the next state, $s_{t+1}$,  and the reward associated with the transition,~$r_t$.  The algorithm then updates the $Q$~value of the state-action pair $(s_t,a_t)$ according the following learning rule:

\begin{equation}
\label{eqn:qlearning_update}
    \underbrace{Q(s_t,a_t)}_{\text{updated value}} \longleftarrow \underbrace{Q(s_t,a_t)}_{\text{old value}} + \underbrace{\alpha_t}_{\text{learning rate}} \left(\underbrace{ r_t +\underbrace{\gamma}_{\text{discount factor}} \times \underbrace{\max_{a\in \mathcal{A}} Q(s_{t+1},a)}_{\text{optimal Q value at } s_{t+1}}}_{\text{target}}- \underbrace{Q(s_t,a_t)}_{\text{old value}} \right).
\end{equation}

The parameter $\alpha_t \in (0,1)$ in this equation is called the {\it learning rate} and can be time-dependent. Intuitively, this update rule is trying to close the gap between the left hand side and the right hand side of the optimal Bellman equations for the $Q$-function (Equation~\ref{eqn:optimal_bellman_QV}).  The $Q$-learning algorithm is  summarized in Algorithm~\ref{alg:qlearning}:

\begin{algorithm}[H] %b
  \caption{$Q$~learning}\label{alg:qlearning}
  \begin{algorithmic}[1]
\State Initialize $Q(s,a)$ for all $s \in \mathcal{S}$ and $a\in \mathcal{A}$ arbitrarily.  
\State If $s$ is a terminal state, set $Q(s,a) =0$ for all actions $a$.
    \For{ each episode}
        \While{ episode has not ended}
      \State Observe $s_t$ and choose action~$a_t$ according to an exploration policy.
      \State Take action $a_t$ and observe $s_{t+1}$ and $r_{t}$.
      \State Update $Q(s_t,a_t)$ according to the update rule~\ref{eqn:qlearning_update}.
        \EndWhile
    \EndFor
\State Return $Q$.
  \end{algorithmic}
\end{algorithm}

The $Q$-learning algorithm provably learns the optimal ~$Q$~function under some reasonable assumptions:
\begin{itemize}
    \item The learning rate $\alpha_t$ has to decrease to~$0$ but not too fast ($\sum \alpha_t$ should diverge and $\sum \alpha_t^2$ should converge).
    \item Each state action pair $(s,a)$ must be visited an infinite number of times by the exploration policy.
\end{itemize}
A more technical statement of these assumptions and a proof of the convergence result can be found in \citet{watkins1992q}.

$Q$-learning works well for MDPs with small state and action spaces but becomes intractable in larger ones. In the following section, we will see how to scale $Q$-learning to larger MDPs with the help of function approximations and deep learning. 

\subsubsection{Deep $Q$ Networks (DQN)}
In the previous section, the learning algorithm had to learn a table of $|\mathcal{S}| \times |\mathcal{A}|$ $Q$~values, one for each state-action pair~$(s,a)$. This is not possible for MDPs with a large state and action space, such as MDPs with continuous state spaces. This is where function approximations become useful. 

The idea is to parameterize the $Q$~function with a family of functions $\{Q_{\theta}\}_{\theta \in \mathbb{R}^d}$, typically a neural network, and then learn the optimal value of the parameter $\theta$ such that we have $Q_{\theta}(s,a) \approx Q^*(s,a)$.  Note that this problem is tractable because the dimensionality of the learning problem, $d$, is independent of the size of the MDP. The optimal value of the parameter $\theta$ is learned by minimizing the Bellman error with (a variant of) gradient descent:
\begin{equation}
\label{dqn_loss}
    \mathcal{L}(\theta) = \mathbb{E}_{s_t,a_t,s_{t+1}} [\,\, ( Q_{\theta}(s_t,a_t) -\underbrace{ [r_{t} +\gamma \times \max_{a' \in \mathcal{A}} Q_{\theta}(s_{t+1},a')}_{\text{target}}] \,\,\, )^2 \,\, ] \,.
\end{equation}

The expectation in $\mathcal{L}(\theta)$ is empirically estimated by sampling a batch of random transitions from the replay buffer. Optimizing $\mathcal{L}(\theta)$ is not as straightforward as it seems and many tricks are required to stabilize the learning algorithm. For instance, when computing the gradient of $\mathcal{L}(\theta)$, the target is treated as a constant and does not contribute to the overall gradient. In fact, the target is computed using a ``delayed" version of the parameter $\theta$. 

The details of the training procedure and the empirical tricks needed to stabilize the learning algorithm can be found in the original DQN paper \citep{mnih2013playing}. Many additional improvements to the DQN algorithm were discovered since its initial publication and some of the major ones are reported in \citet{hessel2018rainbow}.

\subsection{Policy Gradient Approaches}

Instead of learning an optimal value function from which an optimal policy can be inferred, policy gradient approaches tackle the RL problem by learning the optimal policy directly. In the following we will parameterize the policy space by a family of functions $\{\pi_{\theta}\}_{\theta \in \mathbb{R}^d}$.

\subsubsection{REINFORCE}
The RL objective (Equation~\ref{eqn:rl_objective_into}) can be re-written more explicitly as an expectation over trajectories $\tau = \{(s_t,a_t s_{t+1})\}_{0\leq t \leq T}$, as: 
$$J(\theta)= \mathbb{E}_{\tau \sim \rho_{\theta}} \left[ R(\tau) \right]\,,$$
where $R(\tau) = \sum_t \gamma^t r_t$ is the total reward encountered by the trajectory $\tau$, and:
$$\rho_{\theta}(\tau):=\underbrace{p_{0}\left(s_{0}\right)}_{\text{distribution of the initial state}} \prod_{t=0}^{T} \pi_{\theta}\left( a_{t} \mid s_{t} \right) \mathbb{P}\left[s_{t+1} \mid s_{t}, a_{t}\right]$$
is the probability of observing trajectory $\tau$ by following the policy $\pi_{\theta}$. Using the {\it log trick}, $\nabla_{\theta} \rho_{\theta}(\tau)=\rho_{\theta}(\tau) \nabla_{\theta} \log \rho_{\theta}(\tau)$,  we can write $\nabla_{\theta} J(\theta)$ as: 
$$\nabla_{\theta} J(\theta)=\mathbb{E}_{\tau \sim \rho_{\theta}}\left[\nabla_{\theta} \log \rho_{\theta}(\tau) R(\tau)\right]\,.$$
Since $\mathbb{P}\left[s_{t+1} \mid s_{t}, a_{t}\right]$ does not depend on $\theta$, we have:
$$\nabla_{\theta} \log \rho_{\theta}(\tau)=\sum_{t=0}^{T} \nabla_{\theta} \log \pi_{\theta}\left(s_{t}, a_{t}\right)\,,$$ 
and finally we find a very simple estimate  of $\nabla_{\theta} J(\theta)$:
\begin{equation}
\label{eqn:reinforce_0}
    \nabla_{\theta} J(\theta)=\mathbb{E}_{\tau \sim \rho_{\theta}}\left[\sum_{t=0}^{T} \nabla_{\theta} \log \pi_{\theta}\left(s_{t}, a_{t}\right) R(\tau)\right]\,.
\end{equation}
This result was first derived by \citet{williams1992simple}.
Optimizing $J(\theta)$ by following an empirical estimate of the gradient above~(Equation~\ref{eqn:reinforce_0}) is known at the REINFORCE algorithm.   
\begin{algorithm}[H] %b
  \caption{Vanilla REINFORCE}\label{alg:vanilla_reinforce}
  \begin{algorithmic}[1]
\State Set: $\alpha$: learning rate, $N$: number of iterations,  $B$: sample size.
\State Initialize $\pi_{\theta}$.
    \For{ i from 1 to N}
        \State Set $\nabla_{\theta} J(\theta) = 0$.
        \For{ b from 1 to B }
      \State Sample one trajectory $\tau$ and compute its total rewards $R(\tau)$. 
      \State Gradient accumulation: $\nabla_{\theta} J(\theta) \longleftarrow \nabla_{\theta} J(\theta) + \frac{1}{B} \sum_{t=0}^{T} \nabla_{\theta} \log \pi_{\theta}\left(s_{t}, a_{t}\right) R(\tau)$.
        \EndFor
        \State Update $\theta$:  $\theta \longleftarrow \theta + \alpha \nabla_{\theta} J(\theta)$.
    \EndFor
\State Return $\pi_{\theta}$.
  \end{algorithmic}
\end{algorithm}

\subsubsection{Beyond REINFORCE}

The gradient estimates in REINFORCE are usually very noisy as they suffer from high variance which can destabilize the learning process.  The algorithm can be made more reliable through the usage of variance reduction schemes. These usually involve introducing value function estimates inside the REINFORCE formula.
For instance, further analysis of Equation~\ref{eqn:reinforce_0} shows that it can be re-written as: 
$$\nabla_{\theta} J(\theta)=\mathbb{E}_{s \sim \rho_{\theta}, a \sim \pi_{\theta}}\left[\nabla_{\theta} \log \pi_{\theta}(s, a) Q^{\pi_{\theta}}(s, a)\right],$$
where $Q^{\pi_{\theta}}$ is the $Q$~function of policy $\pi_{\theta}$ and~$\rho_{\theta}$ is the state marginal distribution resulting from following the policy~$\pi_{\theta}$. This expression could be further re-written as:
\begin{equation}
\label{eqn:reinforce_advantage}
\nabla_{\theta} J(\theta)=\mathbb{E}_{s \sim \rho_{\theta}, a \sim \pi_{\theta}}\left[\nabla_{\theta} \log \pi_{\theta}(s, a) A^{\pi_{\theta}}(s, a)\right]\,,
\end{equation}
where $A^{\pi_{\theta}}(s, a):= Q^{\pi_{\theta}}(s,a) - V^{\pi_{\theta}}(s)$ is called the advantage function.  The advantage function measures how good an action is at a given state. A positive (negative) advantage indicates that the action is better (worse) than average. Using  Equation~\ref{eqn:reinforce_advantage} to estimate gradients leads to less noisy estimates and results in a more stable and improved policy learning algorithm. Further details on how to estimate~$A^{\pi_{\theta}}$ as well as a technical analysis of variance reduction schemes in policy gradient methods can be found in \citet{schulman2015high} and \citet{greensmith2004variance}.

REINFORCE minimizes the RL objective by taking a step in the direction of the gradient $\nabla_{\theta} J(\theta)$. While this can provably improve the policy in the limit of small step sizes, it is not necessarily the best direction to follow. Indeed, the gradient points towards the direction of steepest ascent when distances are measured using the Euclidean distance in the parameter space, $d(\theta_1, \theta_2) = ||\theta_1-\theta_2||_2$. Without any additional assumptions on the mapping $\theta \to \pi_\theta$, measuring distances in the parameter space might not be appropriate. Small differences in the parameter $\theta$ could lead to large differences in the policy space and as a result, in the agent's performance. Conversely, large changes in $\theta$ could correspond to infinitesimal policy changes and imperceptible changes of the agent's behavior. 

Consequently, it seems more appropriate to measure distances at the policy level instead of the underlying parameter's level. This results in a different gradient called the {\it natural gradient} and a different policy gradient algorithm called the {\it Natural Policy Gradient} (NPG) \citep{kakade2001natural}. Many policy gradient methods were then developed following up and improving on that line of work, including {\it Trust Region Policy Optimization} (TRPO) by \citet{trpo}, {\it Proximal Policy Optimization} (PPO) by \citet{Schulman2017ProximalPO},  and {\it Actor-Critic using Kronecker-factored Trust Region} (ACKTR) by \citet{wu2017scalable}.

\clearpage
\section{Background in Auction Theory}
\label{sec:IntroAuction}
In this section, we give a brief high level introduction to auction theory that gives more context for Chapters~\ref{chap:EquivariantNet}~and~\ref{chap:ALGnet}. A good reference on auction theory and mechanism design more generally can be found
in \citet{NisaRougTardVazi07} or \citet{roughgarden2016twenty}.

\subsection{A Simple Model For Auctions}

\begin{figure}[H]
    \centering
    \includegraphics[width=0.9\linewidth]{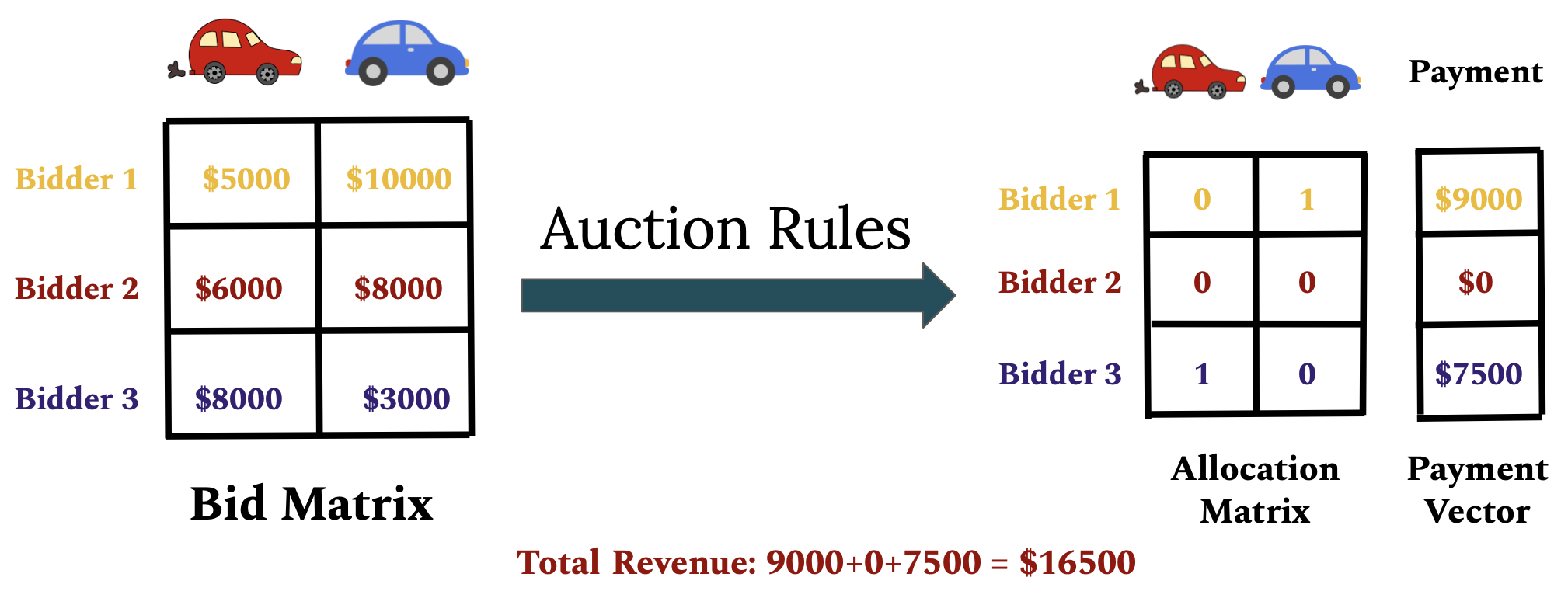}
    \caption{Illustration of the auction mechanics.}
\end{figure}

\subsubsection{Setting}

An auction consists of $n$ bidders and $m$ items. Let $N:=\{1, \cdots,n \}$ and $M:=\{1, \cdots,m \}$ denote the set of bidders and items respectively. 
At the start of the auction, each one of the bidders bids a sum of money on each one of the items. 
We will denote the bid of bidder $i$ on item $j$ by $b_{ij}$. 
These bids can be grouped into a matrix $B = \{b_{ij} \}_{i \in N, j \in M }$ called the bid matrix. 

\subsubsection{The allocation and payment functions}

An auction mechanism is characterized by two functions, the allocation function $g: \mathbb{R}^{n \times m} \to [0,1]^{n \times m}$ and the payment function $p: \mathbb{R}^{n \times m} \to \mathbb{R}^{n}$, both of which have the bid matrix $B$ as an input.  

The allocation function computes the allocation matrix $g(B)$ where  $g(B)_{ij}$ is the probability that bidder $i$ gets item $j$. Since it's possible for an item to not be allocated to any of the bidders, we have  $\forall j \in M, \,\, \sum_i g(B)_{ij} \leq 1$. It is sometimes convenient to denote the allocation function restricted to bidder~$i$, by by $g_i$, i.e.  $g_i(B) = [g(B)_{ij}]_{j\in M}$.

The payment function computes the payment vector $p(B)$ where $p(B)_{i}$ is the amount of money bidder $i$ has to pay to the Auctioneer. The revenue of the Auctioneer $P$ is the sum of the payments made by all the bidders,  $P=  \sum_{i=1}^n  p(B)_{i}$.

\subsubsection{Additive auctions}

A subset of items $S \subseteq M$ does not have the same value for each one of the bidders.  Each bidder~$i$ has his own valuation function, $v_i: 2^M \to \mathbb{R}$,  where $v_i(S)$ denotes how much bidder~$i$ values the subset of items $S$. In principle $v_i$ can be arbitrary, assigning arbitrary values to each one of the possible subsets. 

Depending on the context, it makes sense to consider simpler valuation functions that have more structure to them. For example, we could consider valuation functions in which the value of a basket of items is equal to the highest individual item value in that basket: $v_{i}(S)=\max _{j \in S} v_{i}(\{j\})$. This is called a {\it unit-demand valuation function}.  Other examples are value functions in which the value of a basket of items is equal to the sum of the values of its components: $v_{i}(S)= \sum_{j \in S} v_{i}(\{j\})$. Such valuations are called {\it additive}. Additive valuations are among the most studied valuation functions in auction theory and will be the main focus of this thesis.

An additive valuation function is fully specified by the individual values for each one of the items. In the following, we denote the value that bidder~$i$ gives for item~$j$ by $v_{ij}$. These values can be grouped into a matrix $V = \{v_{ij} \}_{i \in N, j \in M }$ called the {\it valuation matrix}. 
Note that in general,  the bid matrix can be different than the value matrix. In the following we will denote the $i$-th row of the value matrix $V$ by $\vec{v}_i$. 

\subsubsection{The utility of a bidder}

The utility of a bidder is the net amount of value obtained as a result of his participation in the auction. We can compute it as the difference between the total value of the items that the bidder got and the amount he had to pay to the auctioneer. The utility of bidder $i$, $u_i$, is given by : 
$$ u_i(v_i, B) = \underbrace{v_i(g(B))}_{\text{Value received by bidder i}} \,\,\, - \,\,\, \underbrace{p(B)_i}_{\text{Payment of bidder i}}\,.$$

Notice that the total value received by bidder~$i$ in this expression is computed using the valuation function of bidder~$i$. For additive auction, this expression can be re-written as:

\begin{equation}
    u_{i}(\vec{v}_i, B)=\sum_{j=1}^{m} g(B)_{i j} [\vec{v}_i]_j -p(B)_i = \sum_{j=1}^{m} g(B)_{i j} V_{i j}-p(B)_i\,.
\end{equation}

\subsection{Problem Statement}

While there exist infinite choices of allocation functions~$g$ and payment functions~$p$ that constitute a valid auction mechanism, in practice, we care about auctions that satisfy certain desirable properties. In the following we will focus on two desirable properties, {\it incentive compatibility} and {\it individual rationality}.\\

\noindent {\bf Notation}: Given a matrix $B \in \mathbb{R}^{n\times m}$ and $i \in \{1,\cdots,n\}$ we will denote the $i$-th row of the matrix $B$ by $\vec{b}_i$ and the $(n-1) \times m$ matrix that one gets from $B$ by removing $i$-th row by $B_{-i}$. Given a vector $\vec{b_i}'  \in \mathbb{R}^{ m}$ , we will denote the matrix that we get by replacing row $\vec(b)_i$ with  $\vec{b_i}'$ in $B$ by $(\vec{b}_i', B_{-i})$. The rows of $(\vec{b}_i', B_{-i})$ are $[\vec{b}_1,\cdots, \vec{b}_{i-1}, \vec{b_i}', \vec{b}_{i+1}, \cdots, \vec{b}_{n}]$.\\

\noindent If $B$ is a bid matrix, then $\vec{b}_i$ is the vector of bids of bidder~$i$, $B_{-i}$ is the bid matrix of all the bidders except bidder~$i$, and $(\vec{b_i}',B_{-i})$ is a bid matrix that we get if bidder~$i$ modifies his bid from $\vec{b_i}$ to $\vec{b_i}'$.

\subsubsection{Incentive Compatibility}

Strategic bidders seek to maximize their utility and may report bids that are different from their true valuations ($\vec{b}_i \neq \vec{v}_i$).  In hindsight, once all the bids are known, the optimal bid of bidder~$i$, $\vec{b}^*_i$, is given by: 

$$\vec{b}^*_i = \argmax_{\vec{b_i}' \in \mathbb{R}^{ m}} \,\,\, u_{i}\left(\vec{v}_{i},\left(\vec{b}_i', B_{-i}\right)\right) \,.$$

In general, we should expect that $\vec{b}^*_i$ is different from $\vec{v}_i$. In some auctions however, the utility of a bidder is always maximized when his bid is {\it truthful} regardless of the other bids and we have $\vec{b}^*_i = \vec{v}_i$. These auctions are called dominant strategy incentive compatible auctions (DSIC). The following provides a formal definition.

\begin{definition}
An auction $(g,p)$ is \textit{dominant strategy incentive compatible} (DSIC) if each bidder's utility is maximized by reporting truthfully no matter what the other bidders report. For every bidder $i,$ valuation $\vec{v}_i$, bid $\vec{b}_i\hspace{.02cm}'$ and bids $B_{-i}$, we have: $u_i(\vec{v}_i,(\vec{v}_i,B_{-i}))\geq u_i(\vec{v}_i,(\vec{b}_i\hspace{.02cm}',B_{-i})). $ 
\end{definition}

DSIC is a desirable property because it levels the playing field to all the bidders. It not only makes it easier for bidders to bid optimally (by bidding truthfully), but also makes it easier for the auctioneer to predict the outcome of an auction since the optimal bids are easily characterized. DSIC auctions are sometimes called {\it strategy-proof auctions}.

\subsubsection{Individual Rationality }
The payment function~$p$ in an auction can, in principle, charge a positive amount of money to a bidder who has not been allocated any of the items. It can also charge a bidder more than the total value of the items he got from the auction. In both of these cases, the utility of the bidder is negative, which means that the bidder is worse off after participating in the auction.  
An individually rational auction guarantees that such cases cannot happen to a bidder as long as he's bidding truthfully. A truthful bidder always has a non negative utility function regardless of what the other bidders decide to do. The following provides a formal definition.

\begin{definition}
An auction is \textit{individually rational} (IR) if for all $i$, $\vec{v}_i$ and $B_{-i}$ we have:
\begin{equation}\label{eq:IR_eq}
   u_i(\vec{v}_i,(\vec{v}_i,B_{-i}))\geq 0.  
\end{equation}
\end{definition}

Individually rational auctions are desirable because they encourage bidder participation in the auction. 

\subsubsection{The auction learning problem }
Each bidder knows how much each item is worth to them.  This information is not known to the other bidders nor to the auctioneer. This setting is referred to as a {\it private value auction}. However, a common assumption is that each bidder draws their value vector, $\vec{v}_i$, from some prior probability distribution, $D_i$, which is common knowledge.

The goal of the auctioneer is to design an incentive compatible, individually rational auction that maximizes his expected revenue given a prior on the bidder's value vectors $(\vec{v}_1, \cdots, \vec{v}_n) \sim (D_1, \cdots, D_n)$.

Formally, we can rewrite the problem as:
\begin{equation*}
\begin{aligned}
\hspace{-.51cm} \vspace{-1.81cm} \underset{(g,p)\in \mathcal{M} }{\text{min}}
  - \mathbb{E}_{V\sim D}\left[\sum_{i=1}^n p_i(V)\right] \quad \text{s.t.} \quad 
&  (g,p) \text{ is a DSIC auction} \,, \\
& (g,p) \text{ is a IR auction} \,. 
\end{aligned}
\end{equation*}

\subsection{Optimal Single Item Auction}
This section is intended to give the curious reader a peek into one of the most celebrated results in auction theory. While reading this section is not strictly necessary to understand the work presented in this thesis,  it could be interesting to contrast the analytical approach to finding the optimal auction presented in this section with the machine learning based approaches of Chapters~\ref{chap:EquivariantNet}~and~\ref{chap:ALGnet}.

Finding the optimal single item auction was fully resolved by Myerson in his seminal 1981 paper \citep{myerson1981optimal}. In this section, we include a high level derivation of Myerson's result that makes some simplifying assumptions. A complete and technical derivation can be found in Myerson's original paper \citep{myerson1981optimal}.

\subsubsection{The allocation function is monotonic.}
Our goal in this section is to prove that the optimal allocation function is monotonic non-decreasing with respect to the bids of each one of the bidders: 
$$\text{The function: }b_i \to [g((b_i, B_{-i}))]_i \text{   is monotonic non-decreasing}\,.$$
We remind the reader that since we're in a single item setting, $\vec{b}_i$ is a real number that we conveniently represent by $b_i$ . Intuitively this result makes sense. A bidder can expect to increase the probability of getting the item by increasing his bid, assuming that all the other bids remain constant. We will see that this result follows naturally from the DSIC property. \\

\noindent{\bf Notation: } In the following, we will fix a bidder~$i$ and the bids of all the other bidders~$B_{-i}$. This allows us to adopt the following convenient notations: ${g(b):= [g(b, B_{-i})]_i}$, ${p(b):= [p(b, B_{-i})]_i}$ and $u(v,b) := u_i(v,(b,B-i)) =  g(b)\times v -p(b)$. With these simplifying notations, our goal is to prove that the function $b \to g(b)$ is monotonic non decreasing.\\

 For a DSIC auction, we have $u(v,v) \geq u(b,v)$ for all $v$ and $b$. Through simple manipulations we can re-write this inequality as: 
\begin{align*}
    u(v,v) \geq u(b,v) & \Longleftrightarrow g(v)\times v -p(v) \geq g(b)\times v -p(b) \,, \\
    & \Longleftrightarrow (g(v)-g(b))\times v  \geq p(v) -p(b)\,.
\end{align*}

By permuting the roles of $b$ and $v$, we also get that $(g(v)-g(b))\times b  \leq p(v) -p(b)$. We conclude that: 

\begin{equation}
\label{eqn:double_inequality_first}
    (g(v)-g(b))\times v  \geq p(v) -p(b) \geq (g(v)-g(b))\times b \,.
\end{equation}

This inequality implies that $(g(v)-g(b))\times (v-b) \geq 0$ which proves that the function $g$ is monotonic non decreasing as claimed. 

\subsubsection{A relation between the allocation and payment functions}
In this section, we derive a relation between the allocation function $g$ and payment function $p$ for a truthful mechanism. For the sake of simplicity, we will make the assumption that the function $g$ is differentiable. While this assumption simplifies the proof, it is not a necessary one. Only the monotonicity of $g$ is required in the more general proof \citep{myerson1981optimal}. 

Taking $b<v$, we can rewrite equation~\ref{eqn:double_inequality_first} as: 
\begin{equation}
\label{eqn:double_inequality}
    \frac{(g(v)-g(b))}{v-b}\times v  \geq \frac{(p(v)-p(b))}{v-b} \geq \frac{(g(v)-g(b))}{v-b}\times b\,.
\end{equation}

In the limit of $b \to v$, we find that: 
\begin{equation}
\frac{d}{dz} p(z) = z \times \frac{d}{dz} g(z)\,.
\end{equation} 

For an individually rational auction, we have $p(0)=0$ and we get: 
\begin{equation}
p(b) = \int_{0}^b dz \,\,\, z \times \frac{d}{dz} g(z)\,.
\end{equation} 

Going back to our original (non simplified) notation, we can rewrite this equation as: 
\begin{equation}
\label{eqn:payment_function}
[p(b_i, B_{-i})]_i = \int_{0}^{b_i} db \,\,\, b \times \frac{d}{db} g(b,B_{-i})\,.
\end{equation} 

Equation~\ref{eqn:payment_function} shows that given an allocation function~$g$ there is at most one candidate payment function~$p$ such that the mechanism defined by $(g,p)$ is incentive compatible and individually rational. Conversely, one can prove that given a monotonic non decreasing allocation function~$g$, the mechanism defined by $(g,p)$ where $p$ is given by equation~\ref{eqn:payment_function} is DSIC and IR.

\subsubsection{Deriving the optimal allocation function}

Now that we have the characterization of the space of incentive compatible individually rational auctions, we move on to finding the revenue maximizing auction. 

Let's denote by $(D_1, \cdots, D_n)$ the probability distributions from which the bidders' value is sampled: $(v_1, \cdots, v_n) \sim (D_1, \cdots, D_n)$. Since we're in a single-item setting, the distribution $D_i$ is a probability distribution over the real numbers. To simplify the derivation, we will assume that all values are bounded by $v_{\text{max}}$ and that $D_i$ has a continuous probability density function which we'll denote as $f_i$. We will use $F_i$ to denote the corresponding cumulative distribution function. 

Conditioned on $B_{-i}$, the expected payment of bidder~$i$, $p_i$, is given by:
\begin{equation}
\begin{aligned}
p_i &= \mathbb{E}_{v_{i} \sim D_{i}}\left[p_i(v_i, B_{-i})\right] \,, \\
&=\int_{0}^{v_{\text{max}}} p_i(v_i, B_{-i}) f_{i}\left(v_{i}\right) d v_{i} \,, \\
&=\int_{0}^{v_{\text{max}}}\left[\int_{0}^{v_{i}} z \times g_{i}^{\prime}\left(z, B_{-i}\right) d z\right] f_{i}\left(v_{i}\right) d v_{i} \,,
\end{aligned}
\end{equation}
where in the last equality we used the characterization of the payment function in terms of the allocation function as found in  equation~\ref{eqn:payment_function}. This expression can be further simplified by first permuting the order of integration and then by integrating by parts: 

\begin{equation*}
\begin{aligned}
\int_{0}^{v_{\text{max}}}\left[\int_{0}^{v_{i}} z \times g_{i}^{\prime}\left(z, B_{-i}\right) d z\right] f_{i}\left(v_{i}\right) d v_{i}&=\int_{0}^{v_{\max }}\left[\int_{z}^{v_{\max }} f_{i}\left(v_{i}\right) d v_{i}\right] z \times g_{i}^{\prime}\left(z, B_{-i}\right) d z \,,\\
&=\int_{0}^{v_{\max }}\left(1-F_{i}(z)\right) \times z \times g_{i}^{\prime}\left(z, B_{-i}\right) d z \,, \\
&= - \int_{0}^{v_{\max }} g_{i}\left(z, B_{-i}\right) \times \left(1-F_{i}(z)-z f_{i}(z)\right) d z \,, \\
&= \int_{0}^{v_{\max }} \underbrace{\left(z -\frac{1-F_{i}(z)}{f_i(z)}\right)}_{:= \phi_i(z)} \times g_{i}\left(z, B_{-i}\right) \times f_{i}(z) d z \,.
\end{aligned}
\end{equation*}

The quantity $\phi_i(z)$ that appears under the integral is called the {\it virtual valuation} of bidder~$i$. Note that this quantity only depends on bidder~$i$, it does not depend on any of the other bidders, and it can be negative.

The total expected revenue for the auctioneer is then given by: 

\begin{equation}
\label{eqn:virtual_welfare}
    P = \mathbb{E}_{V \sim D} \left[ \sum_{i=1}^n p_i(V) \right] =  \mathbb{E}_{V \sim D} \left[ \sum_{i=1}^n \phi_i(v_i) \times g(V)_i\right]\,.
\end{equation}

From equation~\ref{eqn:virtual_welfare} we can see that the expected revenue is a linear combination of the virtual values. 

If all these virtual values are negative, then the optimal allocation is $g(V)_i=0$ for all bidders. Otherwise, optimality is reached by allocating the item to the bidder with the highest virtual value.  The optimal payment function can then be inferred using equation~\ref{eqn:payment_function}.

\subsubsection{An example}

To illustrate how that works in practice, let's consider the case where  $D_1= \cdots = D_n = \text{Uniform}([0,1])$. The virtual valuation function is then given by $\phi (b) = 2b -1 $.

If all the bids are smaller than $r = \frac{1}{2}$, then all the virtual bids are negative and the item is not allocated, so none of the bidders have to pay any amount of money to the auctioneer. $r$ is called the {\it reserve price}, which is the value under which the auctioneer is not willing to sell his item. 

If this is not the case, then there is at least one bidder that bid more than $r$. Without loss of generality we can assume that bidder~$1$ has the highest bid and bidder~$2$ has the second highest bid. In this case, bidder~$1$ gets the item. The allocation function of bidder~$1$ is given by: 
\begin{align*}
g_1(b_1) = \begin{cases} 0 & \mbox{if } b_1< \max(r, v_2) \,,   \\1 & \mbox{if } b_1 \geq \max(r, v_2) \,. \end{cases}
\end{align*}

Its derivative is given by $g_1'(b_1) = \delta_{\min(r, v_2)}(b_1)$ where $\delta$ is the Dirac function. By plugging this expression into equation~\ref{eqn:payment_function}, we find that bidder~$1$ has to pay $p_1 = \min(r, v_2)$ to the auctioneer. The optimal auction can be summed up by the following two cases:
\begin{align*}
\begin{cases} \mbox{If all the bids are below } r: \mbox{no one gets the item and no one pays. }   \\
\mbox{Else: the highest bidder gets the item and pays }  \max(r, \text{ second highest bid).} \end{cases}
\end{align*}
This is called a {\it second price auction with a reserve price}.

Generalizing these results to larger auctions is not straightforward. In fact, despite decades of research, there is no known analytical derivation that would enable us to systematically derive optimal mechanism for a general auction. In Chapter~\ref{chap:EquivariantNet}~and~\ref{chap:ALGnet} we will see a machine learning based approach to approximate optimal auctions.

\part{Reinforcement Learning}
\label{part:rl}
\chapter{ A Theoretical Connection Between {Statistical Physics and Reinforcement Learning}}
\label{chap:Zlearning}
\section{Abstract}
Sequential decision making in the presence of uncertainty and stochastic dynamics gives rise to distributions over state/action trajectories in reinforcement learning (RL) and optimal control problems.
This observation has led to a variety of connections between RL and inference in probabilistic graphical models (PGMs).
Here we explore a different dimension to this relationship, examining reinforcement learning using the tools and abstractions of statistical physics.
The central object in the statistical physics abstraction is the idea of a partition function~$\mcZ$, and here we construct a partition function from the ensemble of possible trajectories that an agent might take in a Markov decision process.
Although value functions and~$Q$-functions can be derived from this partition function and interpreted via average energies, the~$\mcZ$-function provides an object with its own Bellman equation that can form the basis of alternative dynamic programming approaches.
Moreover, when the MDP dynamics are deterministic, the Bellman equation for~$\mcZ$ is linear, allowing direct solutions that are unavailable for the nonlinear equations associated with traditional value functions. 
The policies learned via these~$\mcZ$-based Bellman updates are tightly linked to Boltzmann-like policy parameterizations. In addition to sampling actions proportionally to the exponential of the expected cumulative reward as Boltzmann policies would, these policies take {\it entropy} into account favoring states from which many outcomes are possible.
\section{Introduction}
One of the central challenges in the pursuit of machine intelligence is robust sequential decision making.
In a stochastic and uncertain environment, an agent must capture information about the distribution over ways they may act and move through the state space.
Indeed, the algorithmic process of planning and learning itself can lead to a well-defined distribution over state/action trajectories.
This observation has led to a variety of connections between reinforcement learning (RL) and inference in probabilistic graphical models (PGMs) \citep{Levine18}.
In some ways this connection is unsurprising: belief propagation (and its relatives such as the sum-product algorithm) is understood to be an example of dynamic programming \citep{koller2009probabilistic} and dynamic programming was developed to solve control problems \citep{bellman1966dynamic, bertsekas1995dynamic}.
Nevertheless, the exploration of the connection between control and inference has yielded fruitful insights into sequential decision making algorithms \citep{kalman1960new,attias2003planning,
ziebart2010modeling, kappen2011optimal,Levine18}.

In this chapter, we present another point of view on reinforcement learning as a distribution over trajectories, one in which we draw upon useful abstractions from statistical physics.
This view is in some ways a natural continuation of the agenda of connecting control to inference, as many insights in probabilistic graphical models have deep connections to, e.g., spin glass systems \citep{hopfield1982neural,yedidia2001generalized,zdeborova2016statistical}. 
More generally, physics has often been a source of inspiration for ideas in machine learning \citep{mackay2003information, mezard2009information}.
Boltzmann machines \citep{ackley85}, Hamiltonian Monte Carlo \citep{duane1987hybrid, neal2011mcmc,betancourt2017conceptual} and, more recently, tensor networks \citep{stoudenmire2016supervised} are a few examples.
In addition to direct inspiration, physics provides a compelling framework to reason about certain problems.
The terms \emph{momentum}, \emph{energy}, \emph{entropy}, and \emph{phase transition} are ubiquitous in machine learning.
However, abstractions from physics have generally not been so far helpful for understanding reinforcement learning models and algorithms.
That is not to say there is a lack of interaction; RL is being used in some experimental physics domains, but physics has not yet as directly informed RL as it has, e.g., graphical models \citep{MLPhysicalScience}.

Nevertheless, we should expect deep connections between reinforcement learning and physics: an RL agent is trying to find a policy that maximizes expected reward and many natural phenomena can be viewed through a minimization principle.
For example, in classical mechanics or electrodynamics, a mass or light will follow a path that minimizes a physical quantity called the \emph{action}, a property known as the \emph{principle of least action}. 
Similarly, in thermodynamics, a system with many degrees a freedom---such as a gas---will explore its configuration space in the search for a configuration that minimizes its free energy.
In reinforcement learning, rewards and value functions have a very similar flavor to energies, as they are extensive quantities and the agent is trying to find a path that maximizes them.
In RL, however, value functions are often treated as the central object of study.
This stands in contrast to statistical physics formulations of such problems in which the \emph{partition function} is the primary abstraction, from which all the relevant thermodynamic quantities---average energy, entropy, heat capacity---can be derived.
It is natural to ask, then, \emph{is there a theoretical framework for reinforcement learning that is centered on a partition function, in which value functions can be interpreted via average energies?}

In this chapter, we show how to construct a partition function for a reinforcement learning problem.
In a deterministic environment (Section~\ref{sec:deterministic}), the construction is elementary and very natural.
We explicitly identify the link between the underlying average energies associated with these partition functions and value functions of Boltzmann-like stochastic policies.
As in the inference-based view on RL, moving from deterministic to stochastic environments introduces complications.
In Section \ref{sec:stochastic}, we propose a construction for stochastic environments that results in realistic policies.
Finally, in Section \ref{sec:modelFree}, we show how the partition function approach leads to an alternative model-free reinforcement learning algorithm that does not explicitly represent value functions.

We model the agent's sequential decision-making task as a Markov decision process (MDP), as is typical.
The agent selects actions in order to maximize its cumulative expected reward until a final state is reached. 
The MDP is defined by the objects~$(\mcS,\mcA,\mcR,\mcP)$.
$\mcS$ and~$\mcA$ are the sets of states and actions, respectively.
${\mcP(s,a,s') = \mathbb{P}(s'\mid s,a)}$ is the probability of landing in state~$s'$ after taking action~$a$ from state~$s$.
$\mcR(s,a,s')$ is the reward resulting from this transition. 
We also make the following additional assumptions:

\begin{enumerate}
    \item $\mcS$ is finite,
    \item all rewards~$\mcR(s,a,s')$ are bounded from above by~$\mcR_{\text{max}}$ and are deterministic,
    \item the number of available actions is uniformly bounded over all states by~$d$.
\end{enumerate}
We also allow for terminal states to have rewards  even though there are no further actions and transitions. We denote these final-state rewards by~$\mcR(s_f)$. By shifting all rewards by~$\mcR_{\text{max}}$ we can assume without loss of generality that~${\mcR_{\text{max}}=0}$ making all transition  rewards~$\mcR(s,a,s')$ non positive. The final state rewards~$\mcR(s_f)$ are still allowed to be positive however.

\section{Partition Functions for Deterministic MDPs} \label{sec:deterministic}

Our starting point is to consider deterministic Markov decision processes.
Deterministic MDPs are those in which the transition probability distributions assign all their mass to one state.
Deterministic MDPs are a widely studied special case \citep{madani2002polynomial, wen2013efficient,  dekel2013better} and they are realistic for many practical control problems, such as robotic manipulation and locomotion, drone maneuver or machine-controlled scientific experimentation.
For the deterministic setting, we will use~${s+a}$ to denote the state that follows the taking of action~$a$ in state~$s$.
Similarly, we will denote the reward more concisely as~$\mcR(s,a)$.

\subsection{Construction of State-Dependent Partition Functions}
\label{sec:deterministicConstruction}
To construct a partition function, two ingredients are needed: a statistical ensemble, and an energy function~$E$ on that ensemble.
We will construct our ensembles from trajectories through the MDP; a trajectory~$\omega$ is a sequence of tuples~${\omega = (s_0, a_0, r_0), (s_1, a_1, r_1),\ldots,(s_T,a_T,r_T)}$ such that state~$s_{T+1}$ is a terminal state.
We use the notation~$s_t(\omega)$,~$a_t(\omega)$, and~$r_t(\omega)$ to indicate the state, action, and reward, respectively, of trajectory~$\omega$ at step~$t$.
Each state-dependent ensemble~$\Omega(s)$ is then the set of all trajectories that start at~$s$, i.e., for which~${s_0(\omega) = s}$.
We will use these ensembles to construct a partition function for each state~${s \in \mcS}$.
Taking~$|\omega|$ to be the length of the trajectory, we write the energy function as
\begin{align}
E(\omega) &= -\sum_{t=0}^{|\omega|-1} r_t(\omega) - R(s_{|\omega|}) = -\sum_{t=0}^{|\omega|} r_t(\omega)\,.
\end{align}
The form on the right takes a notational shortcut of defining~$r_{|\omega|}(\omega):=R(s_{T+1})$ for the reward of the terminal state.
Since the agent is trying to maximize their cumulative reward,~$E(\omega)$ is a reasonable measure of the agent's preference for a trajectory in the sense that lower energy solutions accumulate higher rewards.
Note in particular that the ground state configurations are the most rewarding trajectories for the agent.
With the ingredients~$\Omega(s)$ and~$E(\omega)$ defined, we get the following partition function
\begin{align}
\mcZ(s,\beta) &= \sum_{\omega \in \Omega(s)}e^{-\beta \,E(\omega)} \\
&= \sum_{\omega \in \Omega(s)} e^{\beta \sum_{t=0}^{|\omega|} r_t(\omega)}\,.
\end{align}
In this expression,~$\beta \geq 0~$ is a hyper-parameter that can be interpreted as the inverse of a temperature.
(This interpretation comes from statistical physics where~${\beta = \frac{1}{K_B T }}$, where~$K_B$ is the Boltzmann constant.)
This partition function does not distinguish between two trajectories having identical cumulative rewards but different lengths.
However, among equivalently rewarding trajectories, it seems natural to prefer shorter trajectories.
One way to encode this preference is to add an explicit penalty~${\mu\leq 0}$ on the length~$|\omega|$ of a trajectory, leading to a partition function
\begin{align}\label{eq:partition}
\mcZ(s,\beta) = \sum_{\omega \in \Omega(s)} e^{-\beta \, E(\omega) + \mu |\omega| }\,.
\end{align}
In statistical physics,~$\mu$ is called a \emph{chemical potential} and it measures the tendency of a system (such as a gas) to accept new particles.
It is sometimes inconvenient to reason about systems with a  fixed number of particles, adding a chemical potential offers a way to relax that constraint, allowing a system to have a varying number of particles while keeping the average fixed.

Note that since MDPs can allow for both infinitely long trajectories and infinite sets of finite trajectories,~$\Omega(s)$ can be infinite even in relatively simple settings.
In Appendix~\ref{proof:wellDefined}, we find that a sufficient condition for~$\mcZ(s,\beta)$ to be well defined is taking~${\mu < -\log{d}}$.
As written, the partition function in Eq.~\ref{eq:partition} is ambiguous for final states.
For clarity we define~${\mcZ(s_f,\beta) := e^{\beta \, R(s_f)}}$ for a terminal state~$s_f$.
We will refer to these as the boundary conditions.

Mathematically, the parameter~$\mu$ has a similar role as the one played by~$\gamma$, the discount rate commonly used in reinforcement learning problems.
They both make infinite series convergent in an infinite horizon setting, and ensure that the Bellman operators 
are contractions in their respective frameworks~(Appendices~\ref{proof:contraction}~and~\ref{proof:rhoContraction}).
However, when using~$\gamma$, the order in which the rewards are observed can have an impact on the learned policy which does not happen when~$\mu$ is used.
This could be a desirable property for some problems as it uncouples rewards from preferences for shorter paths.

\subsection{A Bellman Equation for~$\mcZ$}
\label{sec:deterministicBellman}

As we have defined an ensemble~$\Omega(s)$ for each state~${s\in\mcS}$, there is a partition function~$\mcZ(s,\beta)$ defined for each state.
These partition functions are all related through a Bellman-like recursion:
\begin{align}\label{eq:z-bellman}
\mcZ(s, \beta) &= \sum_{a } e^{\beta\, \mcR(s,a) +\mu} \;\mcZ(s+a, \beta)\,,
\end{align}
where, as before,~${s+a}$ indicates the state deterministically following from taking action~$a$ in state~$s$.
This Bellman equation can be easily derived by decomposing each trajectory~${\omega \in \Omega(s)}$ into two parts: the first transition resulting from taking initial action~$a$ and the remainder of the trajectory~$\omega'$ which is a member of~$\Omega(s+a)$.
The total energy and length can also be decomposed in the same way, so that:

\begin{align*}
\mcZ(s, \beta) &= \sum_{\omega \in \Omega(s)} e^{-\beta \,E(\omega) + \mu |\omega|}\\
&= \sum_{\omega \in \Omega(s)} e^{\beta\sum_{t=0}^{|\omega|}  r_{t}(\omega) + \mu |\omega|}\\
&=  \sum_{a\in\mcA} e^{\beta\,\mcR(s,a) +\mu}
\sum_{\omega' \in \Omega(s+a) }  e^{\beta \sum_{t=1}^{|\omega|}  r_{t}(\omega) + \mu( |\omega|-1)}\\
&= \sum_{a\in\mcA} e^{\beta\,\mcR(s,a) +\mu}
\sum_{\omega' \in \Omega(s+a) }e^{-\beta \,E(\omega') + \mu|\omega'|}\\
&=  \sum_{a } e^{\beta \,\mcR(s,a) +\mu} \;\mcZ(s+a, \beta)\,.
\end{align*}

\noindent Note in particular that this Bellman recursion is \textbf{linear} in~$\mcZ$.

\subsection{The Underlying Value Function and Policy}\label{sec:deterministicPolicy}
The partition function can be used to compute an average energy to shed light on the behavior of the system.
This average is computed under the Boltzmann (Gibbs) distribution induced by the energy on the ensemble of trajectories :
\begin{align}
\bbP(\omega \given \beta, \mu, s_0(\omega)=s) &= \frac{\indicator_{\Omega(s)}(\omega)}{\mcZ(s,\beta)}e^{-\beta\,E(\omega) + \mu|\omega|}\,.
\end{align}
In probabilistic machine learning, this is usually how one sees the partition function: as the normalizer for an energy-based learning model or an undirected graphical model (see, e.g., \citet{murray2004bayesian}). Under this probability distribution, high-reward trajectories are the most likely but sub-optimal ones could still be sampled. This approach is closely related to the {\it soft-optimality} approach to RL \citep{Levine18}.
This distribution over trajectories allows us to compute an average energy for state~$s$ either as an explicit expectation or as the partial derivative of the log partition function with respect to the inverse temperature:
\begin{align}
\langle E \rangle &= \sum_{\omega\in\Omega(s)}\frac{1}{\mcZ(s,\beta)}e^{-\beta\,E(\omega) + \mu|\omega|}
E(\omega) \notag \\
&= -\frac{\partial}{\partial \beta}\log \mcZ(s,\beta)\,.
\end{align}
The negative of the average energy is the value function:
$$V(s,\beta) := -\langle E \rangle = \frac{\partial}{\partial \beta}\log \mcZ(s,\beta).$$
This is an intuitive result: recall that the energy~$E(\omega)$ is low when the trajectory~$\omega$ accumulates greater rewards, so lower average energy indicates that the expected cumulative reward---the value---is greater.
Since the partition functions~$\{\mcZ(s,\beta)\}_{s \in S}$ are connected by a Bellman equation, we expect that the underlying value functions~$\{V(s,\beta)\}_{s \in S}$ would be connected in a similar way, and there is indeed a non-linear Bellman recursion:
\begin{align*}
V(s,\beta) &= \frac{\partial}{\partial\beta}\log\mcZ(s,\beta) \\
&= \frac{1}{\mcZ(s,\beta)}\frac{\partial}{\partial\beta}\mcZ(s,\beta) \\
&= \frac{1}{\mcZ(s,\beta)} \frac{\partial}{\partial \beta}  \sum_{a\in\mcA} e^{\beta \,\mcR(s,a) +\mu} \;\mcZ(s+a, \beta)\\
&= \frac{1}{\mcZ(s,\beta)} \sum_{a\in\mcA} e^{\beta \,\mcR(s,a) +\mu} \frac{\partial}{\partial\beta}\mcZ(s+a, \beta)+\mcR(s,a)e^{\beta \,\mcR(s,a) +\mu} \;\mcZ(s+a, \beta) \,.
%&=\frac{ \sum_{a } \left(r_{(s,a)}~e^{\beta r_{(s,a)} +\mu} ~\mcZ(s+a, \beta) + e^{\beta r_{(s,a)} -\mu} ~\frac{\partial}{\partial \beta}\mcZ(s+a, \beta)  \right)}{\mcZ(s,\beta)}\\
%&=\frac{ \sum_{a } \left(r_{(s,a)}~e^{\beta r_{(s,a)} -\mu} ~\mcZ(s+a, \beta) + e^{\beta r_{(s,a)} -\mu}~\mcZ(s+a, \beta) ~\frac{\frac{\partial}{\partial \beta}\mcZ(s+a, \beta)}{\mcZ(s+a, \beta)}  \right)}{\mcZ(s,\beta)}\\
%&=\frac{ \sum_{a } e^{\beta r_{(s,a)} +\mu} ~\mcZ(s+a, \beta)}{\mcZ(s,\beta)} \left[r_{(s,a)}+\frac{\frac{\partial}{\partial \beta}\mcZ(s+a, \beta)}{\mcZ(s+a, \beta)}\right]~\\
%&=\sum_{a} \frac{e^{\beta r_{(s,a)}+\mu} \mcZ(s+a, \beta )}{Z(s,\beta)}   [r_{(s,a)}  +   V(s+a,\beta )] \\
\end{align*}
The derivative rule for natural log gives us:

\begin{align*}
\frac{\partial}{\partial\beta}\mcZ(s,\beta)&=\mcZ(s,\beta)\,\,\frac{\partial}{\partial\beta}\log\mcZ(s,\beta)\\
&= \mcZ(s,\beta) \,\, V(s,\beta)
\end{align*} and as a result we have:
\begin{align}
V(s,\beta) &= \frac{1}{\mcZ(s,\beta)} \sum_{a\in\mcA} e^{\beta \,\mcR(s,a) +\mu}\mcZ(s+a,\beta) V(s+a,\beta) +\mcR(s,a)e^{\beta \,\mcR(s,a) +\mu} \;\mcZ(s+a, \beta)\notag\\
&= \frac{1}{\mcZ(s,\beta)} \sum_{a\in\mcA} e^{\beta \,\mcR(s,a) +\mu}\mcZ(s+a,\beta)\left[
V(s+a,\beta) + \mcR(s,a)
\right]\,.\label{eq:value-bellman}
\end{align}
Note that the quantities~$e^{\beta \,\mcR(s,a) +\mu}\mcZ(s+a,\beta)$ inside the summation of Equation~\ref{eq:value-bellman} are positive and sum to~$\mcZ(s,\beta)$ due to the Bellman recursion for~$\mcZ(s,\beta)$ from Equation~\ref{eq:z-bellman}.
Thus we can view this Bellman equation for~$V(s,\beta)$ as an expectation under a distribution on actions, i.e., a \emph{policy}:
\begin{align}
	&V(s,\beta) = \sum_{a\in\mcA} \pi(a \given s)\left[
V(s+a,\beta) + \mcR(s,a)
\right] \\
& \pi(a \given s) = \frac{\mcZ(s+a,\beta)}{\mcZ(s,\beta)}e^{\beta \,\mcR(s,a) +\mu}\,.
\end{align}
The policy~$\pi$ resembles a Boltzmann policy but strictly speaking it is not. A Boltzmann policy~$\pi_B$ selects actions proportionally to the exponential of their expected cumulative reward:
\begin{align*}
{\pi_{\text{B}}(a \mid s) \propto \exp\left(\beta \left[\mcR(s,a)+V(s+a) \right]\right)}.
\end{align*}
In particular,~$\pi_B$ does not take {\it entropy} into account: if two actions have the same expected optimal value, they will be picked with equal probability regardless of the possibility that one of them could achieve this optimality in a larger number of ways. In the partition function view,~$\pi$ does take entropy into account and to clarify this difference we will look at the two extreme cases~${\beta \to \{0,\infty\}}$.

 When~${\beta \to 0}$, where the temperature of the system is infinite, rewards become irrelevant and we find that:~${\pi(a \mid s) \propto \sum_{\omega \in \Omega(s+a) } e^{\mu |\omega|}}$.  This means that~$\pi$ is picking action~$a$ proportionally to the number of trajectories that begin with~${s+a}$.  Here the counting of trajectories happens in a weighted way: longer trajectories contribute less than shorter ones. This is different from a Boltzmann policy that would pick actions uniformly at random.
\medbreak
\begin{wrapfigure}{r}{0.45\textwidth}%
\vspace{-0.5cm}%
{\Large%
\begin{center}%
\begin{forest}%
sn edges%
[{\color{blue}~$S_0$}%
	[$S_1$%
		[{\color{red}~$S_4$}%
		]%
		[{\color{red}~$S_5$}%
		]%
	]%
	[$S_2$%
		[{\color{red}~$S_6$}%
		]%
	]%
	[$S_3$%
		[{\color{red}~$S_7$}%
		]%
	]%
]%
\end{forest}%
\end{center}%
}%

\captionof{figure}{Decision Tree MDP}\label{fig:tree}%

\end{wrapfigure}
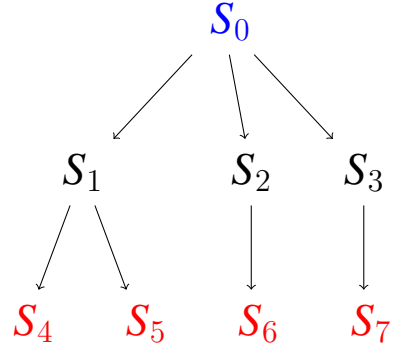%
 When~${\beta \to \infty}$, the low-temperature limit, we find in Section~\ref{proof:Boltzmann} that:
\begin{align*}
\pi(a \mid s) \propto  N_{\max}(s+a) \exp\left(\beta \left[\mcR(s,a)+V(s+a) \right]\right)
\end{align*}
where ${N_{\max}(s+a)}$ is a weighted count of the number of {\bf optimal} trajectories that begin at the state~${s+a}$. Boltzmann policies completely  ignore the~$N_{\max}$ entropic factor.\\

To illustrate this difference more clearly, we consider the  deterministic decision tree MDP shown in Figure~\ref{fig:tree} where~$S_0$ is the initial state and the leafs~$S_4$,~$S_5$,~$S_6$, and~$S_7$ are the final states. 
The arrows represent the actions available at each state.
There are no rewards and the boundary conditions are:~${\mcR(S_{4}) = \mcR(S_5) = \mcR(S_6) =1}$ and~${\mcR(S_7) = 0}$.
This gives us the boundary condition:

$${\mcZ(S_{4},\beta) = \mcZ(S_5,\beta) = \mcZ(S_6,\beta) =e^{\beta}} \text{ and }{\mcZ(S_7,\beta) = 1}.$$

Computing the $\mcZ$-functions at the intermediate states~$S_1, S_2$ and~$S_3$ we find:
$${\mcZ(S_1,\beta) = 2e^{\beta +\mu}},\,\,\,\, {\mcZ(S_2,\beta) = e^{\beta +\mu}}, \,\,\,\, {\mcZ(S_3,\beta) = e^{\mu}}.$$  
Finally we have~${\mcZ(S_0,\beta) = 3e^{\beta+2\mu}+e^{2\mu}}$.  
The underlying policy for picking the first action is given by:
\begin{align}
\pi_{\beta}(1\mid 0) =\frac{2e^{\beta +2\mu}}{3e^{\beta+2\mu}+e^{2\mu}} ~~~~~ \pi_{\beta}(2\mid 0) =\frac{e^{\beta +2\mu}}{3e^{\beta+2\mu}+e^{2\mu}} ~~~~~ \pi_{\beta}(3\mid 0) =\frac{e^{2\mu}}{3e^{\beta+2\mu}+e^{2\mu}}.
\end{align}
When~$\beta \to 0$,  we get:  
$$\pi_0(1\mid 0) = \frac{1}{2}, ~~ \pi_0(2 \mid 0) = \frac{1}{4}, ~~ \pi_0(3 \mid 0)  = \frac{1}{4}$$.
A Boltzmann policy would pick these three actions with equal probability.  The policy~$\pi$ is biased towards the heavier subtree.\\
When~$\beta \to \infty$ we get: 
$$\pi_{\infty}(1\mid 0) = \frac{2}{3}, ~~ \pi_{\infty}(2 \mid 0) = \frac{1}{3}, ~~ \pi_{\infty}(3 \mid 0)=0.$$ 
A Boltzmann policy would pick action~$1$ and~$2$ with a probability of~$\frac{1}{2}$. ~$\pi$ prefers states from which many possible optimal trajectories are possible.

\subsection{A Planning Algorithm}

When the dynamics of the environment are known, it is possible to to learn~$~\mcZ(s,\beta)$ by exploiting the Bellman equation (\ref{eq:z-bellman}). 
We denote by~${s \to s'}$ the property that there exists an action~$a$ that takes an agent from state~$s$ to state~$s'$. The reward associated with this transition will be denoted~$\mcR({s \to s'})$.  Let~${\mcZ(\beta) = [\mcZ(s,\beta)]_{s\in \mcS }}$ be the vector of all partition functions and~${C(\beta) \in \mathbb{R}^{|\mcS| \times |\mcS|}}$ be the matrix:
\begin{align}C(\beta)_{s,s'} =  \indicator_{s \to s'}e^{\beta \mcR({s \to s'})+\mu} + \indicator_{s = s' = \text{final state}}
\end{align}
$C(\beta)$ is a matrix representation of the Bellman operator in Equation~\ref{eq:z-bellman}. With these notations, the Bellman equations in (\ref{eq:z-bellman}) can be compactly written as:
\begin{equation}
    {\mcZ(\beta) =   C(\beta)\,  \mcZ(\beta)}
\end{equation}
highlighting the fact that~$\mcZ(\beta)$ is a fixed point of the map:
\begin{equation}
{\phi: X \to \,\, C(\beta) \,\,\, X}.
\end{equation}
In Appendix~\ref{proof:contraction}, we show that~$\phi$ is a contraction which makes it possible to learn~$\mcZ(\beta)$  by starting with an initial vector~$\mcZ_0$ having compatible boundary conditions and successively iterating the map~$\phi$:~~$\mcZ_{n+1} = C(\beta) \,\, \mcZ_{n}$. We could also interpret ~$\mcZ(\beta)$ as an eigenvector of~$C(\beta)$. In this context, this algorithm is simply doing a power method.

Interestingly, we can learn~$\mcZ(\beta)$ by solving the underdetermined linear system~${[I_{|\mcS|}- C(\beta)] \,~\mcZ(\beta) = 0_{|\mcS|}}$ with the right boundary conditions.
We show in Appendix~\ref{proof:Boltzmann} that the policies learned are related to Boltzmann policies which produce non linear Bellman equations at the value function level: 
 \begin{align}
 V(s, \beta) =  \sum_{a } \frac{e^{\beta \left(\mcR(s,a) + \gamma V(s+a, \beta) \right)}}{\mcW(s,\beta)}   [r_{(s,a)}  +   \gamma V(s+a,\beta )]
 \end{align}
where~$\gamma$ is the discount factor and~$\mcW(s,\beta) =\sum_{a } e^{\beta \left(\mcR(s,a) + \gamma V(s+a, \beta) \right)}$ is a normalization constant different from~$\mcZ(s,\beta)$.
By working with partition functions we transformed a non linear problem into a linear one. This remarkable result is reminiscent of linearly solvable MDPs \citep{todorov2007linearly}.\\\\
Once~$\mcZ$ is learned the agent's policy is given by:~${\,\mathbb{P}(a \mid s) \propto ~e^{\beta\mcR(s,a)}\mcZ(s+a,\beta)}$.

\section{Partition functions for Stochastic MDPs}
We now move to the more general MDP setting.  The dynamics of the environment can now be stochastic. However, as mentioned at the end of the introduction, we still assume that given an initial state~$s$, an action~$a$, and a landing state~$s'$, the reward~$\mcR(s,a,s')$ is deterministic.
\subsection{A First Attempt: Averaging the Bellman Equation}
\label{sec:unrealisticPolicy}
A first approach to incorporating the stochasticity of the environment is to average the right-hand side of the Bellman Equation~\ref{eq:z-bellman} and define~$\mcZ(s,\beta)$ as the solution of:
\begin{equation}
\label{eq:avgBellman}
\begin{aligned}
\mcZ(s,\beta) &= \sum_{a} \mathbb{E}_{s' \mid s,a} \left[e^{\beta \mcR(s,a,s') +\mu }~\mcZ(s',\beta)\right] \\
&= \sum_{a,s'} \mathbb{P}(s'\mid s,a) \, e^{\beta \mcR(s,a,s') +\mu }~\mcZ(s',\beta)\,.
\end{aligned}
\end{equation}
Interestingly, the solution of this equation can be constructed in the same spirit of Section~\ref{sec:deterministicConstruction} by summing a functional over the set of trajectories. If we define~$L(\omega)$ to be the log likelihood of a trajectory:~${L(\omega) =  \sum_{t=0}^{|\omega|-1} \log{\mathbb{P}(s_{t+1}\mid{s_t,a_t})}}$ then~$\mcZ(s,\beta)$ is defined by
\begin{align}
\label{eq:unrealisticZ}
\mcZ(s,\beta) = \sum_{\omega \in \Omega(s)} e^{-\beta E(\omega)+ \mu |\omega| + L(\omega) }\,,
\end{align}
satisfies the Bellman Equation~\ref{eq:avgBellman}. The proof can be found in Appendix~\ref{proof:avgBellman}. In Appendix~\ref{proof:unrelasticBellman} we derive the Bellman equation satisfied by the underlying value function~$V(s,\beta)$ and we find: 
\begin{align}
\label{eqn:unrealisticPolicy}
V(s,\beta)&= \sum_{a,s'} \frac{ e^{\beta  \mcR(s,a,s')+\mu} \,\,~\mcZ(s',\beta)}{\mcZ(s,\beta)}\times \mathbb{P}(s'\mid s,a) \times \left( \mcR(s,a,s') +V(s',\beta) \right) \,.
\end{align}
This Bellman equation does not correspond to a realistic policy; the policy depends on the landing state~$s'$ which is a random variable. The agent's policy and the environment's transitions cannot be decoupled. This is not surprising, from Equation~\ref{eq:unrealisticZ} we see that~$\mcZ$ puts rewards and transition probabilities on an equal footing. As a result an agent believes they can choose any available transition as long as they are willing to pay the price in log probability. This encourages risky behavior: the agent is encouraged to bet on highly unlikely but beneficial transitions. These observations were also noted in \citet{Levine18}. 

\subsection{A Variational Approach}\label{sec:stochastic}
Constructing a partition function for a stochastic MDP is not straightforward because there are two types of randomness: the first comes from the agent's policy and the second from stochasticity of the environment.  Mixing these two sources of randomness can lead to unrealistic policies as we saw in Section \ref{sec:unrealisticPolicy}. A more principled approach is needed.

We construct a new deterministic MDP~$(\tilde{\mcS},\tilde{\mcA},\tilde{\mcR},\tilde{\mcP})$ from~$(\mcS,\mcA,\mcR,\mcP)$. We take~$\tilde{\mcS}$ to be the space of probability distributions over~$\mcS$, similar to belief state representations for partially-observable MDPs \citep{astrom1965optimal,sondik1978optimal,kaelbling1998planning}.  We make the assumption that the actions~$\mcA$ are the same for all states and take~${\tilde{\mcA} = \mcA}$.
For~${\rho \in \tilde{\mcS}}$ and~${a \in \tilde{\mcA}}$ we  define~${\tilde{\mcP}(\rho,a) := {P_a}^T \rho}$ where~$P_a$ is the transition matrix corresponding to choosing action~$a$ in the original MDP.  We define~${\tilde{\mcR}(\rho,a) := \E_{s \sim \rho} \left[\E_{ s' \mid s,a} [\mcR(s,a,s')] \right]}$. 

$\mcS$ being finite, it has a finite number~$M$ of final states which we denote~$\{{f_i}\}_{i\in\{1,\cdots,M\}}$. The final states of~$\tilde{\mcS}$  are of the form~${\rho_f = \sum_{i=1}^M \alpha_i \delta_{{f_i}}}$ where~${0\leq \alpha_i \leq 1}$ verify~${\sum_{i=1}^M \alpha_i =1}$ and~$\delta_{f_i}$ is a Dirac delta function at state~$f_i$. The intrinsic value~$\rho_f$ of such a final state is then given by~${\mcR(\rho_f) = \sum_{i=1}^M \alpha_i \mcR({f_i})}$. This leads to the  boundary conditions: 
\begin{equation}
\label{eqn:rhoBoundary}
\begin{aligned}
\mcZ(\rho_f) &= \exp \left(\beta \sum_{i=1}^M \alpha_i \mcR({s_{f_i}})\right)\\
&= \prod_{i=1}^M \mcZ({f_i},\beta)^{\alpha_i} \,.
\end{aligned}
\end{equation}
This new MDP~$(\tilde{\mcS},\tilde{\mcA},\tilde{\mcR},\tilde{\mcP})$ is deterministic, and we can follow the same approach of Section~\ref{sec:deterministic} and construct a partition function ~$\mcZ(\rho,\beta)$ on~$\tilde{\mcS}$. ~$\mcZ(s,\beta)$ can be recovered by evaluating~$\mcZ(\delta_{s},\beta)$. From this construction we also get that~$\mcZ(\rho,\beta)$ satisfies the following Bellman equation:
\begin{align}
\label{eqn:bellmanRho}
\mcZ(\rho,\beta) = \sum_{a} e^{\beta \mcR(\rho,a) +\mu } ~\mcZ({P_{a}}^T \rho,\beta)\,.
\end{align}
Just as it is the case for deterministic MDPs, the Bellman operator associated with this equation is a contraction. This is proved in Appendix~\ref{proof:rhoContraction}. However~$\tilde{\mcS}$ is now infinite which makes solving Equation~\ref{eqn:bellmanRho} intractable.  We adopt a variational approach which consists in finding the best approximation of~$\mcZ(\rho,\beta)$ within a parametric family~$\{\mcZ_{\theta}\}_{\theta \in \Theta}$. We measure the fitness of a candidate through the following loss function:
$${\Delta(\theta) = \frac{1}{|\mcS|} \sum_{s \in \mcS} \left(\mcZ_{\theta}(\delta_{s},\beta) -\sum_{a} e^{\beta \mcR(\delta_{s},a) +\mu } ~\mcZ_{\theta}({P_{a}}^T \delta_{s},\beta) \right)^2 }$$.

For illustration purposes, and inspired by the form of the boundary conditions (Equation~\ref{eqn:rhoBoundary}), we  consider a simple parametric family given by  the partition functions of the form~${\mcZ_{\theta}(\rho) = \prod_{i=1}^{|\mcS|} {\theta_{i}}^{\rho_i}}$, where~${\theta \in \mathbb{R}^{|\mcS|}}$. The optimal~$\theta$ can be found using usual optimization techniques such as gradient descent. By evaluation of~$\mcZ_{\theta}$ at~${\rho = \delta_{S_i}}$ we see that we must have~${\theta_i = \mcZ(\delta_{S_i})= \mcZ(S_i)}$ and consequently we have~${\mcZ_{\theta}(\rho) = \prod_{i=1}^{|\mcS|} {\mcZ(S_i)}^{\rho_i}}$. The optimal solution satisfies the following Bellman equation:
\begin{align} 
\label{eqn:variationalBellman}
\mcZ(s,\beta) \approx \sum_{a}\prod_{s' \in \mcS}\left[ e^{\beta \mcR(s,a,s') +\mu } ~ \mcZ(s',\beta)\right]^{\mathbb{P}(s' \mid s,a)}
\end{align}
The underlying value function verifies:
$${V(s,\beta) \approx \sum_{a,s'} \pi(a \mid s)\, \mathbb{P}(s' \mid s,a) \left(\mcR(s,a,s') + V(s',\beta) \right)}~$$

\noindent where the policy~$\pi$ is given by~${\pi(a \mid s) \propto\prod_{s' \in \mcS} \left[e^{\beta  \mcR(s,a,s')+\mu} \,~\mcZ(s',\beta)\right]^{\mathbb{P}(s' \mid s,a)}}$. This approach leads to a realistic policy as its only dependency is on the current state, not a future one, unlike the policies arising from Equation~\ref{eqn:unrealisticPolicy}.

\section{The Model-Free Case}
\label{sec:modelFree}
\subsection{Construction of  State-Action-Dependent Partition Function}
In a model free setting, where the transition dynamics are unknown, state-only value functions such as~$V(s)$ are less useful than state-action value functions such as~$Q(s,a)$. Consequently, we will extend our construction to state-action partition functions~$\mcZ(s,a,\beta)$. For a deterministic environment, we extend the construction in Section \ref{sec:deterministic} and define~$\mcZ(s,a,\beta)$ by
\begin{align}
\mcZ(s,a,\beta) &=  \sum_{\omega \in \Omega(s,a)} e^{-\beta E(\omega)+\mu|\omega|} \\
&=\sum_{\omega \in \Omega(s,a)} e^{\beta \sum_{i=0}^{|\omega|}r_i+\mu|\omega|}
\end{align}
where~$\Omega(s,a)$ denotes the set of trajectories having~${(s_0,a_0)=(s,a)}$. Since ${\Omega(s) = \bigcup_{a\in \mcA} \Omega(s,a)}$, we have
${\mcZ(s,\beta) = \sum_{a} \mcZ(s,a,\beta)}$.  As a consequence of this construction,~$\mcZ(s,a,\beta)$ satisfies the following linear Bellman equation:
\begin{align}
\mcZ(s,a,\beta) = e^{\beta \mcR(s,a) +\mu}\sum_{a'} \mcZ(s+a,a',\beta)\,.
\end{align}
This Bellman equation can be easily derived by decomposing each trajectory~${\omega \in \Omega(s,a)}$ into two parts: the first transition resulting from taking initial action~$a$ and the remainder of the trajectory~$\omega'$ which is a member of~$\Omega(s+a,a')$ for some action~$a' \in \mcA$ .
The total energy and length can also be decomposed in the same way, so that:
\begin{align*}
\mcZ(s,a,\beta) &=  \sum_{\omega \in \Omega(s,a)} e^{-\beta E(\omega)+\mu|\omega|}\\
&=\sum_{\omega \in \Omega(s,a)} e^{\beta \sum_{i=0}^{|\omega|}r_i+\mu|\omega|}\\
&=  e^{\beta\,\mcR(s,a) +\mu}
\sum_{\omega \in \Omega(s,a) }  e^{\beta \sum_{t=1}^{|\omega|}  r_{t}(\omega) + \mu (|\omega|-1)}\\
&= e^{\beta\,\mcR(s,a) +\mu}
\sum_{\omega' \in \Omega(s+a) }e^{-\beta \,E(\omega') + \mu|\omega'|}\\
&= e^{\beta\,\mcR(s,a) +\mu}
\sum_{a' \in \mcA } \sum_{\omega' \in \Omega(s+a,a') } e^{-\beta \,E(\omega') + \mu|\omega'|}\\
&=   e^{\beta\,\mcR(s,a) +\mu}
\sum_{a' \in \mcA } \mcZ(s+a,a',\beta).
\end{align*}
 In the same spirit of Section \ref{sec:deterministicPolicy}, one can show that the average underlying value function ${Q(s,a,\beta)= \frac{\partial}{\partial \beta} \log{\mcZ(s,a,\beta)}}$ satisfies a Bellman equation: 
\begin{align}
&Q(s,a,\beta) = \mcR(s,a)+ \sum_{a'} \pi(a' \mid s+a)~Q(s+a,a',\beta) \\
&\pi(a \mid s ) = \frac{\mcZ(s,a,\beta)}{\sum_{a'} \mcZ(s,a',\beta)}
\end{align} 
$Q(s,a,\beta)$ can be then reinterpreted as the~$Q$-function of the policy~$\pi$. Similarly to the results of Section~\ref{sec:deterministicPolicy} and Appendix~\ref{proof:Boltzmann}, the policy~$\pi$ can be thought of a Boltzmann policy of parameter~$\beta$ that takes entropy into account. 
This construction can be extend to a stochastic environments by following the same approach used in Section~\ref{sec:stochastic}.\\

In the following we show how learning the state-action partition function~$\mcZ(s,a,\beta)$ leads to an alternative approach to model-free reinforcement learning  that does not explicitly represent value functions. 

\subsection{A Learning Algorithm}

In~$Q$-Learning, the update rule typically consists of a linear interpolation between the current value estimate and the one arising \emph{a posteriori}:\begin{align}
Q(s_t,a_t) \leftarrow (1-\alpha)Q(s_t,a_t)+\alpha\left(r_t+\gamma \max_{a_{t+1}}Q(s_{t+1},a_{t+1})\right)
\end{align}
 where~${\alpha\in [0,1]}$ is the learning rate and~$\gamma$ is the discount factor.  For~$\mcZ$-functions we will replace the linear interpolation with a geometric one. We take the update rule for~$\mcZ$-functions to be the following:
\begin{align}
\label{eqn:updateRule}
\mcZ(s_t,a_t,\beta) \leftarrow \mcZ(s_t,a_t,\beta)^{1-\alpha} \times \left(e^{\beta r_t +\mu}\sum_{a_{t+1}} \mcZ(s_{t+1},a_{t+1},\beta) \right)^{\alpha}\,.
\end{align}
To understand what this update rule is doing, it is insightful to look at how how the underlying ~$Q$-function, ~~${Q(s,a)=\frac{\partial}{\partial \beta} \log{\mcZ(s_t,a_t,\beta)}}$ is updated. We find:
\begin{align}
\label{eqn:QSARSA}
Q(s_t,a_t,\beta) \leftarrow (1-\alpha) Q(s_t,a_t,\beta) + \alpha \left(r_t+ \sum_{a_{t+1}}\frac{\mcZ(s_{t+1},a_{t+1},\beta)}{\sum_{a'} \mcZ(s_{t+1},a',\beta)}Q(s_{t+1},a_{t+1},\beta) \right)\,.
\end{align}
 We see that we recover a weighted version of the SARSA update rule. This update rule is referred to as \emph{expected} SARSA.  Expected SARSA is known to reduce the variance in the updates by exploiting knowledge about stochasticity in the behavior policy  and hence is considered an improvement over vanilla SARSA \citep{expectedSarsa}.
 
Since the underlying update rule is equivalent to the expected SARSA update rule, we can use any exploration strategy that works for expected SARSA. One exploration strategy could be~$\epsilon$-greedy which consists in taking action~$a = \mbox{argmax}_{a \in \mcA} \mcZ(s,a,\beta)$ with probability~$1-\epsilon$ and picking an action uniformly at random with probability~$\epsilon$.  Another possibility would be a Boltzmann-like exploration which consists in taking action~$a$ with probability~$\mathbb{P}(a \mid s) \propto \mcZ(s,a,\beta)$.

We would like to emphasize that even though the expected SARSA update is not novel, the learned policies through this updates rule are proper to the partition-function approach. In particular, the learned policies~$\pi(a \mid s) \propto \mcZ(s,a,\beta)$ are Boltzmann-like policies with some entropic preference properties as described in Section~\ref{sec:deterministicPolicy} and Appendix~\ref{proof:Boltzmann}.

\section{Conclusion}
In this chapter we discussed how planning and reinforcement learning problems can be approached through the tools and abstractions of statistical physics.
We started by constructing partition functions for each state of a deterministic MDP and then showed how to extend that definition to the more general stochastic MDP setting through a variational approach.
Interestingly, these partition functions have their own Bellman equation making it possible to solve planning and model-free RL problems without explicit reference to value functions. 
Nevertheless, conventional value functions can be derived from our partition function and interpreted via  average energies.
Computing the implied value functions can also shed some light on the policies arising from these algorithms. 
We found that the learned policies are closely related to Boltzmann policies with the additional interesting feature that they take {\it entropy} into consideration by favoring states from which many trajectories are possible.
Finally, we observed that working with partition functions is more natural in some settings.
In a deterministic environment for example, near-optimal Bellman equations become linear which is not the case in a value-function-centric approach.  

% \section{Acknowledgments}
% We would like to thank Alex Beatson, Weinan E, Karthik Narasimhan and Geoffrey Roeder for helpful discussions and feedback.
% This work was funded by a Princeton SEAS Innovation Grant and the Alfred P. Sloan Foundation.

\clearpage
\begin{subappendices}
\section{Deterministic MDPs}
\subsection{$\mcZ(s,\beta)$ is well defined}\label{proof:wellDefined}
\begin{prop}
$\mcZ(s,\beta) = \sum_{\omega \in \Omega(s)} e^{\beta \sum_{i=0}^{|\omega|} r_i + \mu |\omega| }$ is well defined for~$\mu < -\log d$.
\end{prop}

\begin{proof}
The MDP being finite,~$\mcS$ has a finite number of final state we can then find a constant~$K$ such that, for all final states~$s_f$ we have~$\mcR(s_f) \leq K$.
\begin{align*}
\mcZ(s,\beta) =& \sum_{\omega \in \Omega(s)} e^{\beta \sum_{i=0}^{|\omega|} r_i + \mu |\omega| } \\ 
=& \sum_{\omega \in \Omega(s)} e^{\beta \sum_{i=0}^{|\omega|-1} r_i + \beta \mcR(s_{|\omega|}) +  \mu |\omega| } \\ 
\leq&\,\,  e^{\beta K} \sum_{\omega \in \Omega(s)} e^{\beta \sum_{i=0}^{|\omega|-1} r_i  +  \mu |\omega| } \\
\leq& \,\,  e^{\beta K}  \sum_{\omega \in \Omega(s)} e^{  \mu |\omega| } \\
\leq&  \,\,  e^{\beta K} \sum_{n\in \mathbb{N}} e^{  \mu n} \sum_{\omega \in \Omega(s), \, \abs{\omega}=n} 1\\
\leq&  \,\,  e^{\beta K} \sum_{n\in \mathbb{N}} e^{  \mu n }d^n\\
=&  \,\,  e^{\beta K} \sum_{n\in \mathbb{N}} (e^{ \mu+ \log d })^n
\end{align*}
Where used the fact that all rewards~$\{r_i \}_{i \in \{0,...,|\omega|-1\}}$ are non positive and that the number of available actions at each state is bounded by~$d$.
When~$\mu < -\log {d}$, the sum ~~$\sum_{n \in \mathbb{N}} (e^{ \mu +\log d })^n$ becomes convergent and ~$\mcZ(s,\beta)$ is well defined. 
\end{proof}
\begin{remark}
$\mu < -\log {d}$ is a sufficient condition, but not a necessary one.~$\mcZ(s,\beta)$ could be well defined for all values of~$\mu$. This happens for instance when~$\Omega(s)$ is finite for all~$s$. 
\end{remark}

\subsection{The underlying policy is Boltzmann-like}
\label{proof:Boltzmann}

For high values of~$\beta$, the sum~$\sum_{\omega \in \Omega(s)} e^{\beta \sum_{i=0}^{|\omega|} r_i + \mu |\omega| }$ will become dominated by the contribution of few of its terms. As~$\beta \to +\infty$, the sum will be dominated by the contribution of the paths with the biggest reward. We have
\begin{align*}
\log{\mcZ(s,\beta)} \underset{\beta \to \infty}{\sim} \beta ~ \mbox{max} \left\{\sum_{i=0}^{|\omega|} r_i(\omega),\, \omega \in \Omega(s) \right\} 
\end{align*}
We see that~$ V(s, \beta)= \frac{\partial}{\partial \beta} \log \mcZ(s,\beta) \underset{\beta \to \infty}{\rightarrow} \mbox{max} \left\{\sum_{i=0}^{|\omega|} r_i(\omega),\, \omega \in \Omega(s) \right\} ~$. 

Since the MDP is finite and deterministic, it has a finite number of transitions and rewards. Consequently, the set~$\left\{\sum_{i=0}^{|\omega|} r_i(\omega),\, \omega \in \Omega(s) \right\}$ takes  discrete values, in particular, there is a finite gap~$\Delta$ between the maximum value and the second biggest value of this set. Let's denote by~$\Omega_{\text{max}}(s)$ the set of trajectories that achieve this maximum and by~$N_{\max}(s) = \underset{{\omega \in \Omega_{\text{max}}(s)}}{\sum} e^{\mu |\omega|}~$.

$N_{\max}(s)$ counts the number of trajectories~$\Omega_{\text{max}}(s)$ in a weighted way: longer trajectories contribute less than shorter ones. It is a measure of the size of~$\Omega_{\text{max}}(s)$ that takes into account our preference for shorter trajectories.   Putting everything together we get:
\begin{align*}
\left(\frac{\mcZ(s,\beta)}{e^{\beta V(s,\beta)}} - N_{\max}(s)\right)   \underset{\beta \to \infty}{\leq} e^{-\beta \Delta} \sum_{\omega \in \Omega(s)} e^{  \mu |\omega| } \underset{\beta \to \infty}{\rightarrow}  0 
\end{align*}
This shows that ~$\mcZ(s,\beta)\underset{\beta \to \infty}{\sim} N_{\max}(s)~e^{\beta V(s,\beta)}$, which results in the following policy for~$\beta>>1$:
\begin{align*}
\pi(a \mid s) \underset{\beta \to \infty}{\propto }  N_{\max}(s+a) e^{\beta\left(\mcR(s,a) + V(s+a, \beta)\right)}
\end{align*}
$\pi$ differs from a traditional Boltzmann policy in the following way:  if we have two actions~$a_1$ and~$a_2$ such that ~$\mcR(s,a_1) + V(s+a_1, \beta) =\mcR(s,a_2) + V(s+a_2, \beta)$ but there are twice more optimal trajectories spanning from~$s+a_1$ than there are from~$s+a_2$ then action~$a_1$ will be chosen twice as often as~$a_2$. This is to contrast with the usual Boltzmann policy that will pick~$a_1$ and~$a_2$ with equal probability. When~$N_{\max}(s)$ is the same for all~$s$, we recover a Boltzmann policy.  When~$\beta \to +\infty$ the policy converges to a an optimal policy and~$V$ converges to the optimal value function.

\subsection{$X \to C(\beta)X$~ is a contraction}\label{proof:contraction}

\begin{prop}
Let  ~$\mcX(\beta) = \left\{Z \in \mathbb{R}^{|\mcS|}_{+} \text {  such for all final states }~ s_f  ~\text{we have } Z_{s_f}=e^{\beta \mcR(s_f)} \right\}$  and let ${C(\beta)_{s,s'} =  \indicator_{s \to s'}e^{\beta \mcR(s \to s')+\mu} + \indicator_{s = s' = \text{final state}}}$.  The map defined by  
\begin{equation*}
         \psi: \begin{cases}
              \mcX(\beta)  &\to \mcX(\beta)   \\
               X    &\to C(\beta) \, X\\
           \end{cases}
\end{equation*}
is a contraction for the sup-norm: ~$||x||_{\infty} = \underset{i \in \{1,\cdots,|\mcS|\}}{\max} \abs{x_i}$.
\end{prop}
\begin{proof}
$\mcX(\beta)$ is the set of all possible partition functions with compatible boundary conditions. The matrix~$C(\beta)$ is more explicitly defined by:
\begin{align*}
C(\beta)_{s,s'} = \begin{cases}1 & \mbox{if s = s'  and state s is a final state.} \\0 & \mbox{if there is no one step transition from state s to state s'.} \\e^{\beta \mcR(s \to s')+\mu} & \mbox{if the transition from state s to state s' has reward }  \mcR(s \to s'). \end{cases}
\end{align*}
Because ~$C(\beta)_{s,s} =1$ when~$s$ is a final state, the map~$\psi$ is well defined (i.e.~$\mcX(\beta) \to \mcX(\beta)$).
Since the MDP is finite, it has a finite number of final state so there exists a constant~$K$ such that, for all final states~$s_f$ we have~$\mcR(s_f) \leq K$.\\\\
Let ~$X_1,~X_2 ~\in  \mcX(\beta)$ we have: 
\begin{align*}
\norm{ \psi(X_1)- \psi(X_2)}_{\infty} &=  \underset{i \in \{1,\cdots,|\mcS|\}}{\max} ~\abs{\left(C(\beta)X_1- C(\beta)X_2\right)_{i}}
\end{align*}
Without loss of generality we can assume that the MDP has~$m$ final states that are labeled~${|\mcS|-m+1,\cdots,|\mcS|}$.  Under this assumption we have:
\begin{align*}
\underset{i \in \{1,\cdots,|\mcS|\}}{\max} ~ \abs{(X_1- X_2)_{i}} = \underset{i \in \{1,\cdots,|\mcS|-m\}}{\max} ~\abs{(X_1- X_2)_{i}}
\end{align*}
This is because~$X_1$ and~$X_2$ have the same boundary conditions:
$${\forall \, s_f \in  \{|\mcS|-m+1,\cdots,|\mcS|\}, ~ (X_1)_{s_f} = (X_2)_{s_f}}$$. 
Since~$C(\beta)_{s_f,s_f}=1$ if~$s_f$ is the index a final state,~$C(\beta)X_1$ and~$C(\beta)X_2$ still have the same boundary conditions, we have:
$$\forall  s_f \in  \{|\mcS|-m+1,\cdots,|\mcS|\}, ~ [C(\beta) X_1]_{s_f} = [C(\beta)X_2]_{s_f}$$. 
This gives us:
\begin{align*}
\underset{s \in \{1,\cdots,|\mcS|\}}{\max} ~ \abs{\left(C(\beta)X_1- C(\beta)X_2\right)_{s}}  =\underset{s \in \{1,\cdots,|\mcS|-m\}}{\max} ~\abs{\left( C(\beta)X_1- C(\beta)X_2\right)_{s}} 
\end{align*}
For~$s \in \left\{1,\cdots,|\mcS| \right\}$, we have:
$ \abs{\left( C(\beta)X_1- C(\beta)X_2\right)_{s}}  = \abs{\sum_{s'=1}^{|\mcS|} [C(\beta)]_{s,s'} ~ (X_1-X_2)_{s'}}$.\\\\
Since there are at most~$d$ available actions at each state and the environment is deterministic, at most ~$d$ coefficients~$C(\beta)_{s,s'}$ in this sum are non zero. Because the rewards are non positive, the non zero ones can be bounded by~$e^{\mu}$.\\\\
 Putting all these pieces together we can write:
 
 $$\abs{\sum_{s'=1}^{|\mcS|} [C(\beta)]_{s,s'} ~ (X_1-X_2)_{s'}} \leq d \times e^{\mu}  \norm{X_1-X_2}_{\infty}$$.  
 
 Finally we get: 
 \begin{align*}
\norm{ C(\beta)X_1- C(\beta)X_2}_{\infty} \leq \underbrace{d \times e^{\mu}}_{< 1 ~\text{because } \mu < -\log d } \norm{ X_1-X_2}_{\infty}
\end{align*}
This proves that~$\psi$ is a contraction.
\end{proof}
\begin{remark}
We see here another mathematical similarity between the discount factor~$\gamma < 1$ usually used in RL and the chemical potential~$\mu <- \log d$. They both ensure that the Bellman operators are contractions.
\end{remark}

\clearpage
\section{Stochastic MDPs}

\subsection{Averaging the Bellman Equation and adding a likelihood cost are equivalent }\label{proof:avgBellman}

\begin{prop}
The partition function ~$\mcZ(s,\beta)$ defined by ~$\mcZ(s,\beta) := \sum_{\omega \in \Omega(s)} e^{-\beta E(\omega)+ \mu |\omega| + L(\omega) }~$ satisfies the following Bellman equation:
\begin{equation*}
\mcZ(s,\beta) = \sum_{a} \mathbb{E}_{s' \mid s,a} \left[e^{\beta \mcR(s,a,s') +\mu }~\mcZ(s',\beta)\right] 
\end{equation*}
\end{prop}

\begin{proof}

The proof follows the same path as the one in Section \ref{sec:deterministicBellman}. We decompose each trajectory~$\omega \in \Omega$ into two parts: the first transition resulting from taking a first action~$a$ and the rest of the trajectory~$\omega'$. The energy, the length and the likelihood of the trajectory can be decomposed in a similar way as the sum of the contribution of the first transition and the contribution of the rest of the trajectory. We get:
\begin{align*}
\mcZ(s,\beta) &= \sum_{\omega \in \Omega(s)} e^{-\beta E(\omega)+ \mu |\omega| +L(\omega) }\\
 &= \sum_{a, s'} e^{\beta \mcR(s,a,s')+ \mu +\log(\mathbb{P}(s'\mid{s,a})}\sum_{\omega' \in \Omega(s')}  e^{-\beta E(\omega')+ \mu |\omega'| +L(\omega') }\\
 &= \sum_{a, s'} e^{\beta \mcR(s,a,s')+ \mu +\log(\mathbb{P}(s'\mid{s,a})} ~ \mcZ(s',\beta) \\
 &= \sum_{a, s'} \mathbb{P}(s'\mid{s,a})~e^{\beta \mcR(s,a,s')+ \mu } ~ \mcZ(s',\beta) \\
 &=\sum_{a} \mathbb{E}_{s' \mid s,a} \left[e^{\beta \mcR(s,a,s') +\mu }~\mcZ(s',\beta)\right]
\end{align*}
This proves the equivalence.
\end{proof}

\subsection{Deriving the Unrealistic Bellman Equation}\label{proof:unrelasticBellman}
\begin{prop}
The value function~$V(s,\beta)= \frac{\partial}{\partial \beta} \log\mcZ(s,\beta)~$ where
\begin{equation*}
{\mcZ(s,\beta) = \sum_{\omega \in \Omega(s)} e^{-\beta E(\omega)+ \mu |\omega| + L(\omega) }} 
\end{equation*}
satisfies the following Bellman equation:
\begin{equation*}
V(s,\beta)  = \sum_{a,s'} \frac{e^{\beta \mcR(s,a,s') +\mu }~\mcZ(s',\beta)}{\mcZ(s,\beta)}~ \mathbb{P}(s' \mid s,a)  \left[  \mcR(s,a,s')+V(s',\beta) \right] 
\end{equation*}
\end{prop}
\begin{proof}
From Appendix~\ref{proof:avgBellman} we known that~$\mcZ(s,\beta)$ satisfies the Bellman equation: ${\mcZ(s,\beta) = \sum_{a} \mathbb{E}_{s' \mid s,a} \left[e^{\beta \mcR(s,a,s') +\mu }~\mcZ(s',\beta)\right]}$.
\begin{align*}
V(s,\beta) &= \frac{\partial}{\partial \beta} \log\mcZ(s,\beta) \\
&= \frac{\partial}{\partial \beta} \log \left(\sum_{a} \mathbb{E}_{s' \mid s,a} \left[e^{\beta \mcR(s,a,s') +\mu }~\mcZ(s',\beta)\right] \right)\\
&= \frac{1}{\mcZ(s,\beta)}~\frac{\partial}{\partial \beta} \left(\sum_{a} \mathbb{E}_{s' \mid s,a} \left[e^{\beta \mcR(s,a,s') +\mu }~\mcZ(s',\beta)\right] \right)\\
&= \frac{1}{\mcZ(s,\beta)}~ \sum_{a} \mathbb{E}_{s' \mid s,a} \left[ \frac{\partial}{\partial \beta} \left(e^{\beta \mcR(s,a,s') +\mu }~\mcZ(s',\beta)\right)\right] \\
&= \frac{1}{\mcZ(s,\beta)}~ \sum_{a} \mathbb{E}_{s' \mid s,a} \left[ \mcR(s,a,s')~ e^{\beta \mcR(s,a,s') +\mu }~\mcZ(s',\beta) +   e^{\beta \mcR(s,a,s') +\mu }~ \frac{\partial}{\partial \beta} \mcZ(s',\beta)\right] \\
&= \frac{1}{\mcZ(s,\beta)}~ \sum_{a} \mathbb{E}_{s' \mid s,a} \left[ \left( \mcR(s,a,s')+\frac{\frac{\partial}{\partial \beta} \mcZ(s',\beta)}{\mcZ(s',\beta)} \right)~ e^{\beta \mcR(s,a,s') +\mu }~\mcZ(s',\beta) \right] \\
&= \frac{1}{\mcZ(s,\beta)}~ \sum_{a} \mathbb{E}_{s' \mid s,a} \left[ \left( \mcR(s,a,s')+\frac{\partial}{\partial \beta} \log \mcZ(s',\beta)\right)~ e^{\beta \mcR(s,a,s') +\mu }~\mcZ(s',\beta) \right] \\
&= \frac{1}{\mcZ(s,\beta)}~ \sum_{a} \mathbb{E}_{s' \mid s,a} \left[ \left( \mcR(s,a,s')+V(s',\beta)\right)~ e^{\beta \mcR(s,a,s') +\mu }~\mcZ(s',\beta) \right] \\
&= \sum_{a,s'} \frac{e^{\beta \mcR(s,a,s') +\mu }~\mcZ(s',\beta)}{\mcZ(s,\beta)}~ \mathbb{P}(s' \mid s,a)  \left[  \mcR(s,a,s')+V(s',\beta) \right] 
\end{align*}
\end{proof}

\subsection{The Bellman operator of~$\mcZ(\rho,\beta)$ is a contraction } \label{proof:rhoContraction}

\begin{prop}
Let~${\mcD = \{\alpha \in \mathbb{R}^{|\mcS|} \text{ such that } ~\forall i \in {1,\cdots,|\mcS|}, ~ 0\leq \alpha_i \leq1 \text{ and } \sum_{i=1}^{,|\mcS|} \alpha_i =1 \}}$ and 
\begin{equation*}
\begin{aligned}
\mcX(\beta) = \scalebox{3}{\{}  X\in C^{0}\left(\mcD,\mathbb{R}\right) \text { s.t.}~ X(\rho_f) = &\exp\left[\beta \sum_{i=1}^M \alpha_i \mcR({{f_i}})\right] \\
&\text{ for mixtures of final states } \rho_f = \sum_{i=1}^M \alpha_i \delta_{{f_i}}  \scalebox{3}{\} }. 
\end{aligned}
\end{equation*}
The map defined by  \begin{equation*}
         \psi: \begin{cases}
              \mcX(\beta)  &\to \mcX(\beta)   \\
               X    &\to  \begin{cases}
              \mcD  &\to \mathbb{R}   \\  \rho &\to \sum_{a} e^{\beta \mcR(\rho,a) +\mu } ~X({P_{a}}^T \rho,\beta) \end{cases}\\
           \end{cases}
\end{equation*}
is a contraction for the sup-norm: ~$\norm{X}_{\infty} = \underset{\rho \in \mcD}{ \max} ~\abs{X(\rho)}$.
\end{prop}
\begin{proof}

$\mcD~$ is the standard ~$(|\mcS|-1)$-simplex in~$\mathbb{R}^{|\mcS|}$ and~$\mcX(\beta)$ be the set of continuous functions on~$\mcD~$ satisfying the right boundary conditions. The original MDP is finite, consequently it has a finite number~$M$ of final state and it is possible to find a constant~$K$ such that, for all final states~$s_f$ we have~$\mcR(s_f) \leq K$. \\\\
Let ~$X_1$,~$X_2 \in   \mcX(\beta)$.~$X_1$ and~$X_2$ have the same boundary conditions by construction. Not only that, ~$\psi(X_1)$ and~$\psi(X_2)$ have also the same boundary conditions since the map~$\psi$ doesn't alter boundary conditions. Consequently we have:
\begin{align*}
\norm{\psi(X_1)-\psi(X_2)}_{\infty} = \underset{\rho \in \mcD}{\text{max}}~ \abs{\psi(X_1)(\rho)-\psi(X_2)(\rho)} = \underset{\rho \in \mcD, ~\rho \text{ non final}}{\text{max}}~~ \abs{\psi(X_1)(\rho)-\psi(X_2)(\rho)}
\end{align*}
Finally we can write:
\begin{align*}
\norm{\psi(X_1)-\psi(X_2)}_{\infty} &= \underset{\rho \in \mcD, ~\rho \text{ non final}}{\text{max}}~~ \abs{\psi(X_1)(\rho)-\psi(X_2)(\rho)}  \\
&= \underset{\rho \in \mcD, ~\rho \text{ non final}}{\text{max}}~~ \abs{\sum_{a} e^{\beta \mcR(\rho,a) +\mu } ~\left[X_1({P_{a}}^T \rho,\beta)-X_2({P_{a}}^T \rho,\beta)\right]} \\
&\leq \underset{\rho \in \mcD, ~\rho \text{ non final}}{\text{max}}~~ \abs{\sum_{a} e^{\beta \mcR(\rho,a) +\mu }} \times \norm{X_1-X_2}_{\infty} \\
&\leq \underbrace{d \times e^{\mu}}_{<1 ~\text{because} ~\mu < - \log d}  \norm{X_1-X_2}_{\infty}
\end{align*}
Where we use the fact that all rewards are non positive and that the number of available actions is bounded by~$d$. This concludes the proof that the Bellman operator of~$\mcZ(\rho,\beta)$ is a contraction.\\\\
This proof is generalization of the proof presented in Appendix~\ref{proof:contraction} for MDPs with finite state spaces.
\end{proof}

\end{subappendices}

\chapter{Why Generalization in RL is Difficult:
Epistemic POMDPs and Implicit Partial
Observability}
\label{chap:GeneralizationRL}
\section*{Abstract}
Generalization is a central challenge for the deployment of reinforcement learning (RL) systems in the real world. In this chapter, we show that the sequential structure of the RL problem necessitates new approaches to generalization beyond the well-studied techniques used in supervised learning. While supervised learning methods can generalize effectively without explicitly accounting for epistemic uncertainty, we show that, perhaps surprisingly, this is not the case in RL. We show that generalization to unseen test conditions from a limited number of training conditions induces implicit partial observability, effectively turning even fully-observed MDPs into POMDPs. Informed by this observation, we recast the problem of generalization in RL as solving the induced partially observed Markov decision process, which we call the epistemic POMDP. We demonstrate the failure modes of algorithms that do not appropriately handle this partial observability, and suggest a simple ensemble-based technique for approximately solving the partially observed problem. Empirically, we demonstrate that our simple algorithm derived from the epistemic POMDP achieves significant gains in generalization over current methods on the Procgen benchmark suite.
\section{Introduction}
Generalization is a central challenge in machine learning. However, much of the research on reinforcement learning (RL) has been concerned with the problem of optimization: how to master a specific task through online or logged interaction.  Generalization to new test-time contexts has received comparatively less attention, although several works have observed empirically  \citep{Farebrother2018GeneralizationAR, Zhang2018ASO, Justesen2018IlluminatingGI, Song2020ObservationalOI} that generalization to new situations poses a significant challenge to RL policies learned from a fixed training set of situations. In standard supervised learning, it is known that in the absence of distribution shift and with appropriate inductive biases, optimizing for performance on the training set (i.e., empirical risk minimization) translates into good generalization performance. It is tempting to suppose that the generalization challenges in RL can be solved in the same manner as empirical risk minimization in supervised learning: when provided a training set of contexts,
learn the optimal policy within these contexts and then use that policy in new contexts at test-time.

Perhaps surprisingly, we show that such ``empirical risk minimization'' approaches can be sub-optimal for generalizing to new contexts in RL, even when these new contexts are drawn from the same distribution as the training contexts. As an anecdotal example of why this sub-optimality arises, imagine a robotic zookeeper for feeding otters that must be trained on some set of zoos. When placed in a new zoo, the robot must find and enter the otter enclosure. It can use one of two strategies: either peek through all the habitat windows looking for otters, which succeeds with 95\% probability in all zoos, or to follow an image of a hand-drawn map of the zoo that unambiguously identifies the otter enclosure, which will succeed as long as the agent is able to successfully parse the image. In every training zoo, the otters can be found more reliably using the image of the map, and so an agent trained to seek the optimal policy in the training zoos would learn a classifier to predict the identity of the otter enclosure from the map, and enter the predicted enclosure. This classification strategy is optimal on the training environments because the agent can learn to perfectly classify the training zoo maps, but it is \textit{sub-optimal} for generalization, because the learned classifier will never be able to perfectly classify every new zoo map at test-time.
Note that this task is \emph{not} partially observed, because the map provides full state information even for a memoryless policy. However, if the learned map classifier succeeds on anything less than 95\% of new zoos at test-time, the strategy of peeking through the windows, although always sub-optimal in the training environments, turns out to be a more reliable strategy for finding the otter habitat in a \textit{new} zoo, and results in higher expected returns at test-time. 

\begin{figure}
    \centering
    \includegraphics[width=0.9\linewidth]{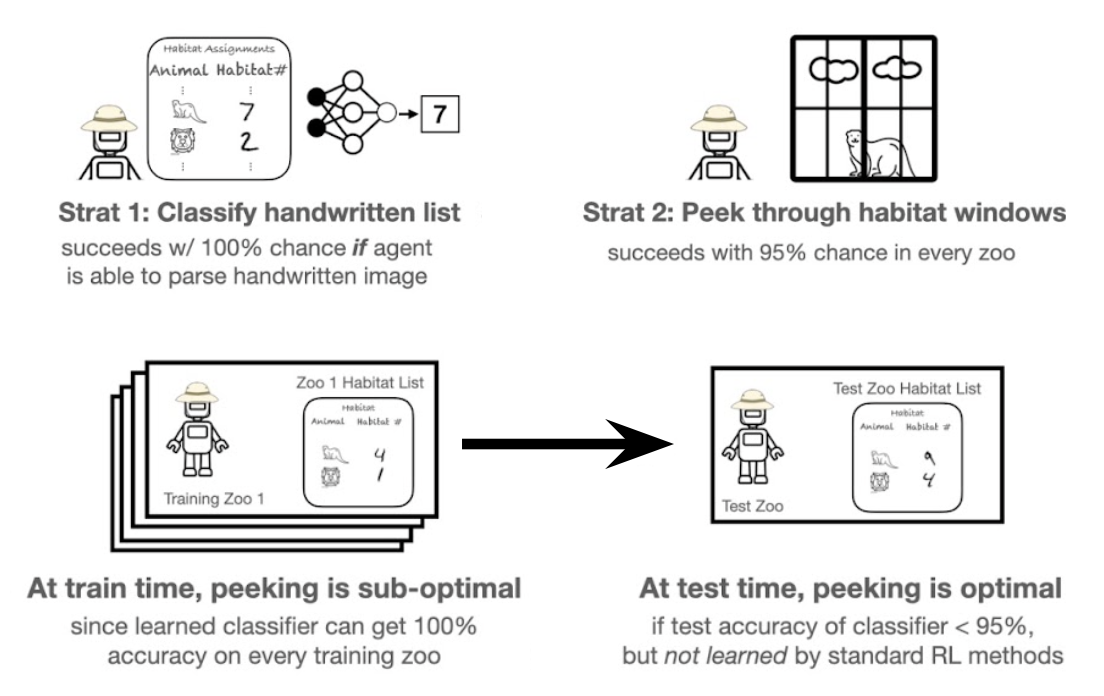}
    \caption{\footnotesize{\textbf{Visualization of the robotic zookeeper example.} Standard RL algorithms learn the classifier strategy, since it is optimal in every training zoo, but this strategy is sub-optimal for generalization because peeking generalizes better than the classifier at test-time.  This failure occurs due to the following disconnect: while the task is \textit{fully-observed} since the image uniquely specifies the location of the otter habitat, to an agent that has limited training data, the location is \textit{implicitly partially observed at test-time} because of the agent's epistemic uncertainty about the parameters of the image classifier.}}
    \label{fig:my_label}

\end{figure}

Although with enough training zoos, the zookeeper can learn a policy by solving the map classification problem, to generalize optimally when given a limited number of zoos requires a more intricate policy that is not learned by standard RL methods. 
How can we more generally describe the set of behaviors needed for a policy to generalize from a finite number of training contexts in the RL setting? We make the observation that, even in fully-observable domains, the agent's epistemic uncertainty renders the environment \textit{implicitly} partially observed at test-time. In the zookeeper example, although the hand-drawn map provides the exact location of the otter enclosure (and so the enclosure's location is technically fully observed), the agent cannot identify the true parameters of the map classifier from the small set of maps seen at training time, and so the location of the otters is implicitly obfuscated from the agent. 
We formalize this observation, and show that generalizing optimally at test-time corresponds to solving a partially-observed Markov decision process that we call an \textbf{epistemic POMDP}, induced by the agent's epistemic uncertainty about the test environment. 

That uncertainty about MDP parameters can be modeled as a POMDP is well-studied in Bayesian RL when training and testing on a single task in an online setting, primarily in the context of exploration \citep{Dearden1998BayesianQ, Duff2002OptimalLC, Strens2000ABF, Ghavamzadeh2015BayesianRL}. However, as we will discuss, this POMDP interpretation has significant consequences for the generalization problem in RL, where an agent cannot collect more data online, and must instead learn a policy from a fixed set of training contexts that generalizes to new contexts at test-time. We show that standard RL methods that do not explicitly account for this implicit partial observability can be arbitrarily sub-optimal for test-time generalization in theory and in practice. The epistemic POMDP underscores the difficulty of the generalization problem in RL, as compared to supervised learning, and provides an avenue for understanding how we should approach generalization under the sequential nature and non-uniform reward structure of the RL setting.  
Maximizing expected return in an approximation of the epistemic POMDP emerges as a principled approach to learning policies that generalize well, and we propose LEEP, an algorithm that uses an ensemble of policies to approximately learn the Bayes-optimal policy for maximizing test-time performance. 

The primary contribution of this chapter is to use Bayesian RL techniques to re-frame generalization in RL as the problem of solving a partially observed Markov decision process, which we call the \textit{epistemic POMDP}. The epistemic POMDP highlights the difficulty of generalizing well in RL, as compared to supervised learning. We demonstrate the practical failure modes of standard RL methods, which do not reason about this partial observability, and show that maximizing test-time performance may require algorithms to explicitly consider the agent's epistemic uncertainty during training. Our work highlights the importance of not only finding ways to help neural networks in RL generalize better, but also on learning policies that degrade gracefully when the underlying neural network eventually does fail to generalize. Empirically, we demonstrate that LEEP, which maximizes return in an approximation to the epistemic POMDP, achieves significant gains in test-time performance over standard RL methods on several ProcGen benchmark tasks.

\section{Related Work}

Many empirical studies have demonstrated the tendency of RL algorithms to overfit significantly to their training environments \citep{Farebrother2018GeneralizationAR, Zhang2018ASO, Justesen2018IlluminatingGI, Song2020ObservationalOI}, and the more general increased difficulty of learning policies that generalize in RL as compared to seemingly similar supervised learning problems \citep{Zhang2018NaturalEB, Zhang2018ADO, Whiteson2011ProtectingAE, Liu2020RegularizationMI}. These empirical observations have led to a newfound interest in algorithms for generalization in RL, and the development of benchmark RL environments that focus on generalization to new contexts from a limited set of
training contexts sharing a similar structure (state and action spaces) but possibly different dynamics and rewards~\citep{Nichol2018GottaLF, Cobbe2019QuantifyingGI, Kuttler2020TheNL, Cobbe2020LeveragingPG, Stone2021TheDC}.

\textbf{Generalization in RL.} Approaches for improving generalization in RL have fallen into two main categories: improving the ability of  function approximators to generalize better with inductive biases, and incentivizing behaviors that are easier to generalize to unseen contexts. To improve the representations learned in RL, prior work has considered imitating environment dynamics \citep{Jaderberg2017ReinforcementLW, Stooke2020DecouplingRL},  seeking bisimulation relations \citep{Zhang2020LearningIR, Agarwal2021ContrastiveBS}, and more generally,  addressing representational challenges in the RL optimization process \citep{Igl2019GeneralizationIR, Jiang2020PrioritizedLR}. In image-based domains,  inductive biases imposed via neural network design have also been proposed to improve robustness to certain factors of variation in the state~\citep{Lee2020NetworkRA,Kostrikov2020ImageAI, Raileanu2020AutomaticDA}. The challenges with generalization in RL that we will describe in this chapter stem from the deficiencies of MDP objectives, and cannot be fully solved by choice of representations or functional inductive biases.
In the latter category, one approach is domain randomization, varying environment parameters such as coefficients of friction or textures, to obtain behaviors that are effective across many candidate parameter settings~\citep{Sadeghi2017CAD2RLRS, Tobin2017DomainRF, Rajeswaran2017EPOptLR, Sim2Real2018, kang2019generalization}.
Domain randomization sits within a class of methods that seek robust policies by injecting noise into the agent-environment loop, whether in the state \citep{Stulp2011LearningTG}, the action (e.g., via max-entropy RL) \citep{Cobbe2019QuantifyingGI}, or intermediary layers of a neural network policy  (e.g., through information bottlenecks) \citep{Igl2019GeneralizationIR, Lu2020DynamicsGV}. In doing so, these methods effectively introduce partial observability into the problem; while not necessarily equivalent to that of the epistemic POMDP, it may indicate why these methods generalize well empirically.

\textbf{Bayesian RL:} Our work recasts generalization in RL within the Bayesian RL framework, the problem of acting optimally %in a partially-observed MDP induced by 
under a belief distribution over MDPs (see Ghavamzadeh et al.~\citep{Ghavamzadeh2015BayesianRL} for a survey). Bayesian uncertainty has been studied in many sub-fields of RL \citep{Ramachandran2007BayesianIR, Lazaric2010BayesianMR, Jeon2018ABA, Zintgraf2020VariBADAV}, the most prominent being for exploration and learning efficiently in the online RL setting. Bayes-optimal behavior in RL is often reduced to acting optimally in a POMDP, or equivalently, a belief-state MDP \citep{Duff2002OptimalLC}, of which our epistemic POMDP is a specific instantiation. Learning the Bayes-optimal policy exactly is intractable in all but the simplest problems \citep{Weber1992OnTG, Poupart2006AnAS}, and many works in Bayesian RL have studied relaxations that remain asymptotically optimal for learning, for example with value of perfect information \citep{Dearden1998BayesianQ, Dearden1999ModelBB} or Thompson sampling \citep{Strens2000ABF, Osband2013MoreER, Russo2014LearningTO}.
Our main contribution is to revisit these classic ideas in the context of generalization for RL. We find that the POMDP interpretation of Bayesian RL \citep{Dearden1998BayesianQ, Duff2002OptimalLC, ross2007bayes} provides new insights on inadequacies of current algorithms used in practice, and explains why generalization in RL can be more challenging than in supervised learning. Being Bayesian in the generalization setting also requires new tools and algorithms beyond those classically studied in Bayesian RL, since test-time generalization is measured using regret over a \textit{single} evaluation episode, instead of throughout an online training process. As a result, algorithms and policies that minimize short-term regret (i.e., are more exploitative) are preferred over traditional algorithms like Thompson sampling that explore thoroughly to ensure asymptotic optimality at the cost of short-term regret.

\section{Problem Setup}

We consider the problem of learning RL policies given a set of training contexts that generalize well to new unseen contexts. This problem can be formalized in a Markov decision process (MDP) where the agent does not have full access to the MDP at training time, but only particular initial states or conditions. Before we describe what this means, we must describe the MDP~$\gM$, which is given by a tuple
$(\gS, \gA, r, T, \rho, \gamma)$, with state space $\gS$, action space $\gA$, Markovian transition function~$T(s_{t+1} |s_t, a_t)$, bounded reward function $r(s_t,a_t)$, and initial state distribution $\rho(s_0)$. A policy $\pi$ induces a discounted state distribution $d^{\pi}(s) = (1-\gamma) \mathbb{E}_{\pi}[\sum_{t \geq 0} \gamma^t 1(s_t = s)]$, and achieves return ~${J_\gM(\pi) = \E_\pi[\sum_{t \geq 0} \gamma^tr(s_t, a_t)]}$ in the MDP.

Classical results establish that a deterministic Markovian (memoryless) policy~$\pi^*$ maximizes this objective amongst all  history-dependent policies.

We focus on generalization in contextual MDPs where the agent is only trained on a training set of contexts, and seeks to generalize well to new contexts. A contextual MDP is an MDP in which the state can be decomposed as~${s_t = (c, s_t')}$, a context vector~${c \in \gC}$ that remains constant throughout an episode, and a sub-state~${s' \in \gS'}$ that may vary:~${\gS \coloneqq \gC \times \gS'}$. Each context vector corresponds to a different situation that the agent might be in, each with slightly different dynamics and rewards, but some shared structure across which an agent can generalize.  During training, the agent is allowed to interact only within a sampled subset of contexts~${\ctrain \subset \gC}$. The generalization performance of the agent is measured by the return of the agent's policy in the full contextual MDP $J(\pi)$, corresponding to expected performance when placed in potentially new contexts. While our examples and experiments will be in contextual MDPs, our theoretical results also apply to other RL generalization settings where the full MDP cannot be inferred unambiguously from the data available during training, for example in offline reinforcement learning \citep{levine2020offline}.

\section{Warmup: A Sequential Classification RL Problem}
\label{sec:classification_as_rl}

\begin{figure}
    \centering
    \includegraphics[width=\linewidth]{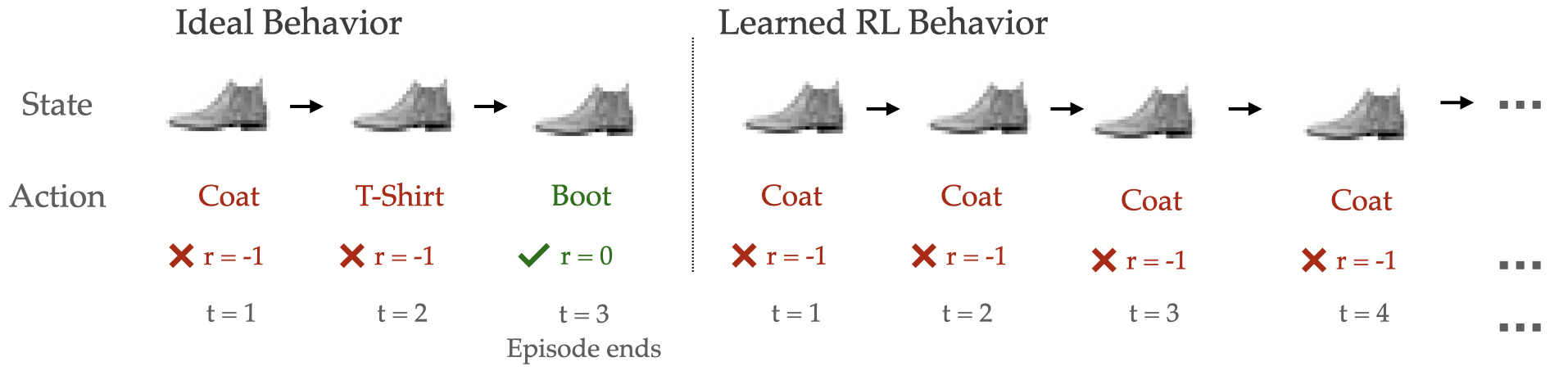}

    \caption{\footnotesize{\textbf{ Sequential Classification RL Problem.} In this task, an agent must keep guessing the label for an image until it gets it correct. To avoid low test return, policies should change actions if the label guessed was incorrect, but  standard RL methods fail to do so, instead guessing the same incorrect label repeatedly.}
    }
    \label{fig:classification}
\end{figure}

\begin{figure}
    \centering
    \includegraphics[width=0.65\linewidth]{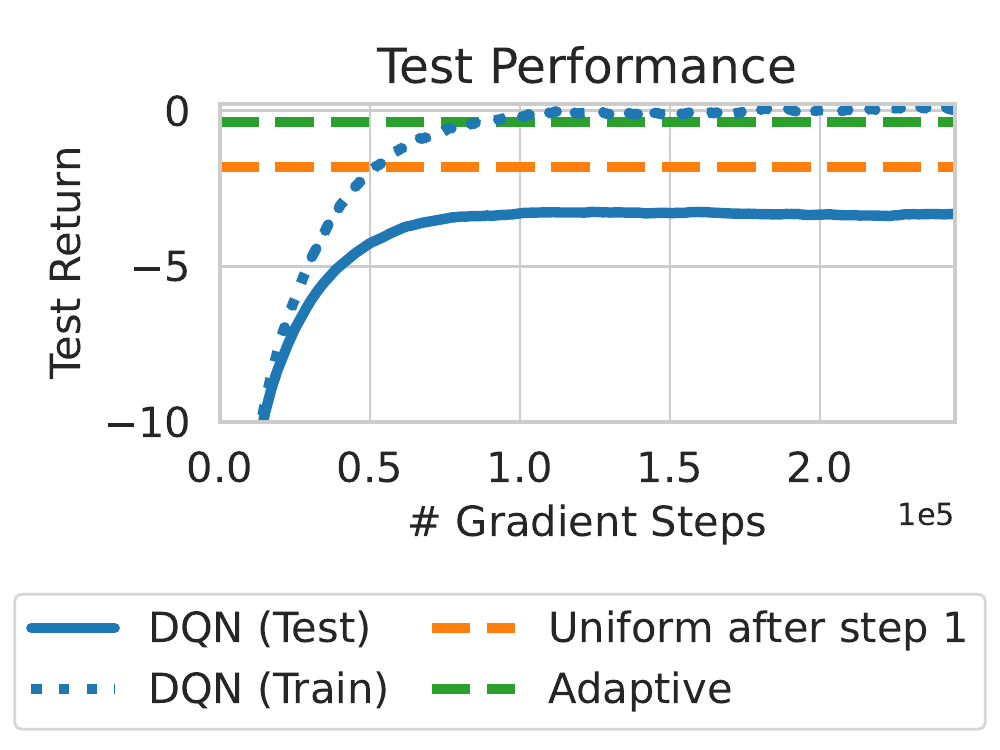}
    \includegraphics[width=0.65\linewidth]{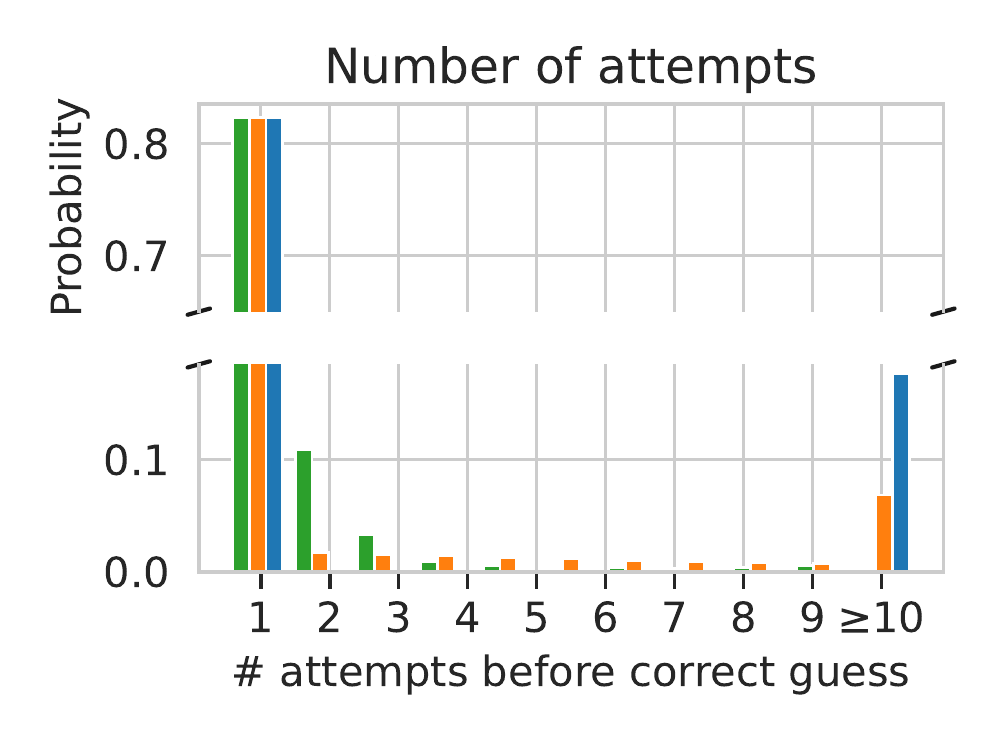}

    \caption{\footnotesize{\textbf{DQN on RL FashionMNIST.} DQN achieves lower test performance than simple variants that leverage the  structure of the RL problem.}}
\label{fig:classification_results}

\end{figure}

We begin our study of generalization in RL with an example problem that is set up to be close to a supervised learning task where generalization is relatively well understood: image classification on the FashionMNIST dataset~\citep{FMNIST}. In this environment (visualized in Figure \ref{fig:classification}), an image from the dataset is sampled (the context) at the beginning of an episode and held fixed; the agent must identify the label of the image to complete the episode. If the agent guesses correctly, it receives a reward of $0$ and the episode ends; if incorrect, it receives a reward of $-1$ and the episode continues, so it must attempt another guess for the \emph{same} image at the next time step. This RL problem is near identical to supervised classification, the core distinction being that an agent may interact with the same image over several timesteps in an episode instead of only one attempt as in supervised learning. Note that since episodes may last longer than a single timestep, this problem is not a contextual bandit. 

The optimal policy in both the one-step and sequential RL version of the problem deterministically outputs the correct label for the image, because the image fully determines the label (in other words, it is a fully observed MDP). However, this optimal strategy generally cannot be learned from a finite training set, since some generalization error is unavoidable. With a fixed training set, the strategy for generalizing in classification remains the same: deterministically choose the label the agent is most confident about. However, the RL setting introduces two new factors: the agent gets multiple tries at classifying the same image, and it knows if an attempted label is incorrect. To generalize best to new test images, an RL policy must leverage this additional structure, for example by trying many possible labels, or by changing actions if the previous guess was incorrect.

A standard RL algorithm, which estimates the optimal policy on the empirical MDP defined by the dataset of training images does not learn to leverage these factors, and instead learns behavior highly sub-optimal for generalization. We obtained a policy by running DQN~\citep{DQN} (experimental details in Appendix~\ref{appendix:FashionMnistImplementation}), whose policy deterministically chooses the same label for the image at every timestep. Determinism is not specific to DQN, and is inevitable in any RL method that models the problem as an MDP because the optimal policy in the MDP is always deterministic and Markovian. The learned deterministic policy either guesses the correct label immediately, or guesses incorrectly and proceeds to make the same incorrect guess on every subsequent time-step. We compare performance in Figure \ref{fig:classification_results} with a version of the agent that starts to guess randomly if incorrect on the first timestep, and a different agent that acts by process of elimination: first choosing the action it is most confident about, if incorrect, then the second, and so forth. Although all three versions have the same training performance, the learned RL policy generalizes more poorly than these alternative variants that exploit the sequential nature of the problem. In Section~\ref{sec:understandingOptimality}, we will see that this process-of-elimination is, in some sense, the optimal way to generalize for this task. This experiment reveals a tension: learning policies for generalization that rely on an MDP model fail, even though the underlying environment \textit{is} an MDP. This failure holds in any MDP model with limited data, whether the empirical MDP or more sophisticated MDPs that use uncertainty estimates in their construction.

\section{Modeling Generalization in RL as an Epistemic POMDP}

To better understand test-time generalization in RL, we study the problem under a Bayesian perspective. We show that training on limited training contexts leads to an implicit partial observability at test-time that we describe using a formalism called the epistemic POMDP. 

\subsection{The Epistemic POMDP}
\label{sec:epistemic_pomdp}

In the Bayesian framework, when learning given a limited amount of evidence $\gD$ from an MDP $\gM$, we can use a prior $\gP(\gM)$ to construct a posterior belief distribution $\gP(\gM | \gD)$ over the identity of the MDP. For learning in a contextual MDP, $\gD$ corresponds to the environment dynamics and reward in training contexts $\ctrain$ that the agent can interact with, and the posterior belief distribution $\gP(\gM | \gD)$ models the agent's uncertainty about the behavior of the environment in contexts that it has not seen before (e.g. uncertainty about the label for a test-set image in the example from Section \ref{sec:classification_as_rl}).  

Since the agent only has partial access to the MDP $\gM$ during training, the agent does not know which MDP from the posterior distribution is the true environment, and must act at test-time under this uncertainty. Following a reduction common in Bayesian RL \citep{Duff2002OptimalLC, Ghavamzadeh2015BayesianRL}, we model this test-time uncertainty using a partially observed MDP that we will call the \textbf{epistemic POMDP}. The epistemic POMDP is structured as follows: each new episode in the POMDP begins by sampling a single MDP~${\gM \sim \gP(\gM | \gD)}$ from the posterior, and then the agent interacts with~$\gM$ until the episode ends in this MDP. The agent does not observe \textit{which} MDP was sampled, and since the MDP remains fixed for the duration of the episode, this induces implicit partial observability. 

Effectively, each episode in the epistemic POMDP corresponds to acting in one of the possible environments that is consistent with the evidence that the agent is allowed access to at training time.
\newline \newline The epistemic POMDP is formally defined as the tuple  ${\gM^{\po} = (\gS^\po, \gO^\po, \gA, T^\po, r^\po, \rho^\po, \gamma)}$.  A state in this POMDP~${s_t^\po = (\gM, s_t)}$ contains the identity of the current MDP being acted in~$\gM$, and the current state in this MDP~$s_t$; we write the state space as~${\gS^\po = \mathbf{M} \times \gS}$, where~$\mathbf{M}$ is the space of MDPs with support under the prior.

The agent only observes~${o_t^\po = s_t}$,  the state in the MDP (${\gO^\po = \gS}$), but \textbf{not} the identity of the MDP, $\gM$. The initial state distribution is defined by the posterior distribution: $\rho^\po((\gM, s_0)) = \gP(\gM | \gD)\rho_\gM(s_0)$,
and the transition and reward functions in the POMDP reflect the dynamics in the current MDP:
\begin{equation}
    T^\po((\gM', s') \mid (\gM, s), a) = \delta(\gM' = \gM)T_{\gM}(s' | s,a)~~~~~~r^\po((\gM, s), a) = r_{\gM}(s,a)\,.
\end{equation}
\textbf{Example }(Sequential Image Classification)\textbf{.} We begin by explicitly describing the induced epistemic POMDP for the task from Section \ref{sec:classification_as_rl}. The agent's uncertainty concerns how images are mapped to labels, and each MDP $\mathcal{M}$ in the posterior distribution corresponds to a different potential labelling function $Y_{\mathcal{M}}: x \mapsto y$ that is consistent with the training dataset. Each episode in the epistemic POMDP, a different MDP $\mathcal{M}$ and corresponding labeller $Y_{\mathcal{M}}$ is sampled from the posterior distribution, alongside an image $x \sim p(x)$. The agent must guess the label assigned by this labelling function $y \coloneqq Y_{\mathcal{M}}(x)$, but is only provided the image $x$ and \textbf{not} the identity of the labeller $Y_{\mathcal{M}}$. We emphasize that the context remains \textit{fully observed} in the epistemic POMDP (the image $x$ is provided to the agent); what is partially observed is how the environment dynamics will behave for the context (what label the image corresponds to). 

What makes the epistemic POMDP a useful tool for understanding generalization in RL is that performance in the epistemic POMDP $\gM^\po$ corresponds exactly to the expected return of the agent at test-time when the prior is well-specified.

\begin{restatable}{proposition}{epistemicpomdpdefn}
\label{prop:epistemic_pomdp_defn}
If the true MDP $\gM$ is sampled from $\gP(\gM)$, and evidence $\gD$ from $\gM$ is provided to an algorithm during training, then the expected test-time return of $\pi$ is equal to its performance in the epistemic POMDP $\gM^\po$.
\begin{equation}
\label{eqn:pomdp_objective}
    J_{\gM^\po}(\pi) 
    = \E_{\gM \sim \gP(\gM)}[J_{\gM}(\pi) \mid \gD].
\end{equation}
In particular, the optimal policy in $\gM^\po$ is Bayes-optimal for generalization to the unknown MDP $\gM$: it receives the highest expected test-time return amongst all possible policies.
\end{restatable}

The epistemic POMDP is based on well-understood concepts in Bayesian reinforcement learning, and Bayesian modeling more generally. However, in contrast to prior works on Bayesian RL, we are specifically concerned with settings where there is a training-test split, and performance is measured by a single test episode. While using Bayesian RL to accelerate exploration or minimize regret has been well-explored~\citep{Ghavamzadeh2015BayesianRL}, we rather use the Bayesian lens specifically to understand generalization -- a perspective that is distinct from prior work on Bayesian RL. Towards this goal, the equivalence between test-time return and expected return in the epistemic POMDP allows us to use performance in the POMDP as a proxy for understanding how well current RL methods can generalize.

\subsection{Understanding Optimality in the Epistemic POMDP}
\label{sec:understandingOptimality}

We now study the structure of the epistemic POMDP, and use it to characterize properties of Bayes-optimal test-time behavior and the sub-optimality of alternative policy learning approaches. The majority of our results follow from well-known results about POMDPs, so we present them here informally, with formal statements and proofs in Appendix \ref{appendix:Sec5}. 

\textbf{Example ctd.} ~~Acting optimally in the epistemic POMDP for the sequential image classification task requires maximizing return over the distribution of labels that is induced by the posterior distribution $p(y|x, \gD) = \E_{\gM \sim \gP(\gM|\gD)}[1(Y_\gM(x) = y)]$. A deterministic policy (as is learned by standard RL algorithms) is a high-risk strategy in the POMDP; it receives exceedingly low return if the labeller outputs a different label than the one predicted. The Bayes-optimal generalization strategy corresponds to a process of elimination: first choose the most likely label $a = \argmax p(y|x, \gD)$; if this is incorrect, eliminate it and choose the next-most likely, repeating until the correct label is finally chosen. Amongst memoryless policies, the optimal behavior is stochastic, sampling actions according to the distribution $\pi^*(a|x) \propto \sqrt{p(y|x, \gD)}$ (derivation in Appendix~\ref{appendix:FashionMnistTheory}).

The characteristics of the optimal policy in the epistemic POMDP for the image classification RL problem match well-established results that optimal POMDP policies are generally memory-based \citep{monahan1982state}, and amongst memoryless policies, the optimal policy may be stochastic \citep{singhPOMDP, Montfar2015GeometryAD}. Because of the equivalence between the epistemic POMDP and test-time behavior, these maxims are also true for Bayes-optimal behavior when maximizing test-time performance.

\begin{principle}
\label{rmk:optimal_policy}
The Bayes-optimal policy for maximizing test-time performance is in general non-Markovian. When restricted to Markovian policies, the Bayes-optimal policy is in general stochastic.
\end{principle}
The reason that Bayes-optimal generalization often requires memory is that the experience collected thus far in the episode contains information about the identity of the MDP being acted in (which is hidden from the agent observation), and to maximize expected return, the agent must adapt its subsequent behavior to incorporate this new information. The fact that acting optimally at test-time formally requires adaptivity (or stochasticity for memoryless policies) highlights the difficulty of generalizing well in RL, and provides a new perspective for understanding the success various empirical studies have found in  improving generalization performance using recurrent networks \citep{openAIcube, simToRealPeng} and stochastic regularization penalties \citep{Stulp2011LearningTG, Cobbe2019QuantifyingGI, Igl2019GeneralizationIR, Lu2020DynamicsGV}.

It is useful to understand to what degree the partial observability plays a role in determining Bayes-optimal behavior. When the partial observability is insignificant, the epistemic POMDP objective can coincide with a surrogate MDP approximation, and Bayes-optimal solutions can be attained with standard fully-observed RL algorithms. For example, if there is a policy that is simultaneously optimal in \textit{every} MDP from the posterior, then an agent need not worry about the (hidden) identity of the MDP, and just follow this policy. Perhaps surprisingly, this kind of condition is difficult to relax: we show in Proposition \ref{prop:deterministic-dominated-random} that even if a policy is optimal in many (but not all) of the MDPs from the posterior, this seemingly ``optimal'' policy can generalize poorly at test-time.

Moreover, under partial observability, optimal policies for the MDPs in the posterior may differ substantially from Bayes-optimal behavior: in Proposition \ref{prop:suboptimal-actions-everywhere}, we show that the Bayes-optimal policy may take actions that are sub-optimal in \textit{every} environment in the posterior. These results indicate the brittleness of learning policies based on optimizing return in an MDP model when the agent has not yet fully resolved the uncertainty about the true MDP parameters.
\begin{principle}[Failure of MDP-Optimal Policies, Propositions
\ref{prop:deterministic-dominated-random}, \ref{prop:suboptimal-actions-everywhere}]
\label{rmk:failure_optimal_policy}
The expected test-time return of policies that are learned by maximizing reward in any MDP from the posterior, as standard RL methods do, may be arbitrarily low compared to that of Bayes-optimal behavior. 
\end{principle}

As Bayes-optimal memoryless policies are stochastic, one may wonder if simple strategies for inducing stochasticity, such as adding $\epsilon$-greedy noise or entropy regularization, can alleviate the sub-optimality that arose with deterministic policies in the previous paragraph. In some cases, this may be true; one particularly interesting result is that in certain goal-reaching problems, entropy-regularized RL can be interpreted as optimizing an epistemic POMDP objective with a specific form of posterior distribution over reward functions (Proposition~\ref{appendix:MaxEntRL}) \citep{eysenbachMaxEnt}. For the more general setting, we show in Proposition \ref{prop:stochastic-dominated-random} that entropy regularization and other general-purpose techniques can similarly catastrophically fail in epistemic POMDPs.

\begin{principle}[Failure of Generic Stochasticity, Proposition \ref{prop:stochastic-dominated-random}]
\label{rmk:failure_stochastic}
The expected test-time return of policies learned with stochastic regularization techniques like maximum-entropy RL that are agnostic of the posterior $\gP(\gM|\gD)$ may be arbitrarily low compared to that of Bayes-optimal behavior.
\end{principle}

This failure happens because the degree of stochasticity used by the Bayes-optimal policy reflects the agent's epistemic uncertainty about the environment; since standard regularizations are agnostic to this uncertainty, the learned behaviors often do not reflect the appropriate level of stochasticity needed. A maze-solving agent acting Bayes-optimally, for example, may choose to act deterministically in mazes like those it has seen at training, and on others where it is less confident, rely on random exploration to exit the maze, inimitable behavior by regularization techniques agnostic to this uncertainty.

Our analysis of the epistemic POMDP highlights the difficulty of generalizing well in RL, in the complexity of Bayes-optimal policies (Remark~\ref{rmk:optimal_policy}) and the deficiencies of our standard MDP-based RL algorithms (Remark~\ref{rmk:failure_optimal_policy}~and Remark~\ref{rmk:failure_stochastic}). While MDP-based algorithms can serve as a useful starting point for acquiring generalizable skills, learning policies that perform well in new test-time scenarios may require more complex algorithms that attend to the epistemic POMDP structure that is implicitly induced by the agent's epistemic uncertainty.

\section{Learning Policies that Generalize Well Using the Epistemic POMDP}

When the epistemic POMDP $\gM^\po$ can be exactly obtained, we can learn RL policies that generalize well to the true (unknown) MDP $\gM$ by learning an optimal policy in the POMDP. In this oracle setting, any POMDP-solving method will suffice, and design choices like policy function classes (e.g. recurrent vs Markovian policies) or agent representations (e.g. belief state vs PSRs) made based on the requirements of the specific domain. However, in practice, the epistemic POMDP can be challenging to approximate due to the difficulties of learning coherent MDP models and maintaining a posterior over such MDP models in high-dimensional domains. 

In light of these challenges, we now focus on practical methods for learning generalizable policies when the exact posterior distribution (and therefore true epistemic POMDP) cannot be recovered exactly. We derive an algorithm for learning the optimal policy in the epistemic POMDP induced by an approximate posterior distribution $\hat{\gP}(\gM| \gD)$ with finite support. We use this to motivate LEEP, a simple ensemble-based algorithm for learning policies in the contextual MDP setting.

\subsection{Policy Optimization in an Empirical Epistemic POMDP}

Towards a tractable algorithm, we assume that instead of the true posterior $\gP(\gM | \gD)$, we only have access to an empirical posterior distribution $\hat{\gP}(\gM | \gD)$ defined by $n$ MDP samples from the posterior distribution $\{\gM_i\}_{i\in [n]}$. This empirical posterior distribution induces an empirical epistemic POMDP $\hat{\gM}^\po$; our ambition is to learn the optimal policy in this POMDP. Rather than directly learning this optimal policy as a generic POMDP solver might, we recognize that $\hat{\gM}^\po$ corresponds to a collection of $n$ MDPs \footnote{Note that when the true environment is a contextual MDP, the sampled MDP $\gM_i$ does not correspond to a single context within a contextual MDP --- each MDP $\gM_i$ is an \emph{entire} contextual MDP with many contexts.} and decompose the optimization problem to mimic this structure. We will learn $n$ policies $\pi_1, \cdots, \pi_n$, each policy $\pi_i$ in one of the MDPs $\gM_i$ from the empirical posterior, and combine these policies together to recover a single policy $\pi$ for the POMDP. Reducing the POMDP policy learning problem into a set of MDP policy learning problems can allow us to leverage the many recent advances in deep RL for scalably solving MDPs. The following theorem links the expected return of a policy $\pi$ in the empirical epistemic POMDP $\hat{\gM^\po}$, in terms of the performance of the policies $\pi_i$ on their respective MDPs $\gM_i$.

\begin{restatable}{proposition}{theoremLowerBound}
\label{thm:lower_bound}
Let $\pi,\pi_1,\cdots \pi_n$ be memoryless, and define $r_{\max} = \max_{i, s, a} |r_{\gM_i}(s,a)|$.  The expected return of $\pi$ in $\hat{\gM}^\po$
is bounded below as:
\begin{equation}
\label{eqn:lower_bound}
         J_{\hat{\gM}^\po}(\pi) \geq \frac{1}{n} \sum_{i=1}^n J_{\gM_i}(\pi_i) - \frac{\sqrt{2}r_{\max}}{(1-\gamma)^2n} \sum_{i=1}^n \E_{s \sim d_{\gM_i}^{\pi_i}}\left[\sqrt{ D_{KL}\left(\pi_i(\cdot | s)\,\, || \,\, \pi(\cdot | s)\right)}\right],
\end{equation}
\end{restatable}

This proposition indicates that if the policies in the collection $\{\pi_i\}_{i \in [n]}$ all achieve high return in their respective MDPs (first term) and are imitable by a single policy $\pi$ (second term), then $\pi$ is guaranteed to achieve high return in the epistemic POMDP. In contrast, if the policies cannot be closely imitated by a single policy, this collection of policies may not be useful for learning in the epistemic POMDP using the lower bound. This means that it may not sufficient to naively optimize each policy $\pi_i$ on its MDP $\gM_i$ without any consideration to the other policies or MDPs, since the learned policies are likely to be different and difficult to jointly imitate. To be useful for the lower bound, each policy $\pi_i$ should balance between maximizing performance on its MDP and minimizing its deviation from the other policies in the set. The following proposition shows that if the policies are trained jointly to ensure this balance, it in fact recovers the optimal policy in the empirical epistemic POMDP.

\begin{restatable}{proposition}{propositionLinkFunction}
\label{prop:link-version}
Let $f: \{\pi_i\}_{i\in[n]} \mapsto \pi$ be a function that maps $n$ policies to a single policy satisfying $f(\pi,\cdots,\pi) = \pi $ for every policy $\pi$, and let $\alpha$ be a hyperparameter satisfying $\alpha \geq \frac{\sqrt{2} r_{\text{max}}}{(1-\gamma)^2n}$. Then letting $\pi_1^*, \dots \pi_n^*$ be the optimal solution to the following optimization problem:

\begin{equation}
\label{eq:algo_objective}
    \{\pi^{*}_i\}_{i\in[n]} = \argmax_{\pi_1, \cdots, \pi_n} \frac{1}{n} \sum_{i=1}^n J_{\gM_i}(\pi_i)  - \alpha \sum_{i=1}^n \E_{s \sim d_{\gM_i}^{\pi_i}}\left[\sqrt{ D_{KL}\left(\pi_i(\cdot | s)\,\, || \,\, f(\{\pi_i\})(\cdot | s)\right)}\right],
\end{equation}
the policy $\pi^{*} \coloneqq f(\{\pi_i^*\}_{i\in[n]})$ is optimal for the empirical epistemic POMDP $\hat{\gM}^\po$.
\end{restatable}

\subsection{A Practical Algorithm for Contextual MDPs: LEEP}

Proposition \ref{prop:link-version} provides a foundation for a practical algorithm for learning policies when provided training contexts $\ctrain$ from an unknown contextual MDP. In order to use the proposition in a practical algorithm, we must discuss two problems: how posterior samples $\gM_i \sim \gP(\gM | \gD)$ can be approximated, and how the function $f$ that combines policies should be chosen.

\textbf{Approximating the posterior distribution: } Rather than directly maintaining a posterior over transition dynamics and reward models, which is especially difficult with image-based observations, we can approximate samples from the posterior via a bootstrap sampling technique \citep{bootstrappedDQN}. To sample a candidate MDP $\gM_i$, we sample with replacement from
the training contexts $\ctrain$ to get a new set of contexts $\ctrain^i$, and define $\gM_i$ to be the empirical MDP on this subset of training contexts. Rolling out trials from the posterior sample $\gM_i$ then corresponds to selecting a context at random from $\ctrain^i$, and then rolling out that context. Crucially, note that $\gM_i$ still corresponds to a \emph{distribution} over contexts, not a single context, since our goal is to sample from the posterior entire contextual MDPs.
\begin{algorithm}[b] 
  \caption{Linked Ensembles for the Epistemic POMDP (LEEP)}\label{algobox}
  \begin{algorithmic}[1]
    \State Receive training contexts $\ctrain$, number of ensemble members $n$
    \State Bootstrap sample training contexts to create $\ctrain^1, \dots \ctrain^n$, where $\ctrain^i \subset \ctrain$.
\State Initialize $n$ policies: $\pi_1, \dots \pi_n$
    \For{iteration $k = 1, 2, 3, \dots $}
        \For{policy $i = 1, \dots,  n$}
      \State Collect environment samples in training contexts $\ctrain^i$ using policy $\pi_i$
      \State Take gradient steps wrt $\pi_i$ on these samples with augmented RL loss:
    
      $$\pi_i \gets \pi_i - \eta \nabla_i (\gL^{RL}(\pi_i) + \alpha \E_{s \sim \pi_i, \ctrain^i}[ D_{KL}(\pi_i(a|s) \| \max_j \pi_j(a|s))])$$

        \EndFor
    \EndFor
    \State Return $\pi = \max_i \pi$: $\pi(a|s) = \frac{\max_i \pi_i(a|s)}{\sum_{a'} \max_i \pi_i(a'|s)}$.

  \end{algorithmic}
\end{algorithm}

\textbf{Choosing a link function:}
The link function $f$ in Proposition \ref{prop:link-version} that combines the set of policies together effectively serves as an inductive bias: since we are optimizing in an approximation to the true epistemic POMDP and policy optimization is not exact in practice, different choices can yield combined policies with different characteristics. Since optimal behavior in the epistemic POMDP must consider all actions, even those that are potentially sub-optimal in all MDPs in the posterior (as discussed in Section~\ref{sec:understandingOptimality}), we use an ``optimistic'' link function that does not dismiss any action that is considered by at least one of the policies, specifically ${f(\{\pi_i\}_{i\in[n]}) = (\max_i \pi_i)(a|s) \coloneqq  \frac{\max \pi_i(a|s)}{\sum_{a'} \max \pi_i(a'|s)}}$. 

\textbf{Algorithm:} We learn a set of $n$ policies $\{\pi_{i}\}_{i \in [n]}$, using a policy gradient algorithm to implement the update step. To update the parameters for $\pi_{i}$, we take gradient steps via the surrogate loss used for the policy gradient, augmented by a disagreement penalty between the policy and the combined policy $f(\{\pi_{i}\}_{i \in [n]})$ with a penalty parameter $\alpha > 0$, as in Equation \ref{eq:ppo_surrogate}:
\begin{equation}
\label{eq:ppo_surrogate}
    \gL(\pi_i) = \gL^{RL}(\pi_i) + \alpha \E_{s \sim \pi_i, \gM_i}[ D_{KL}(\pi_i(a|s) \| \max_j \pi_j(a|s))].
\end{equation}
Combining these elements together leads to our method, LEEP, which we summarize in Algorithm \ref{algobox}. In our implementation, we use PPO for $\gL^{RL}(\pi_i)$ \citep{Schulman2017ProximalPO}. In summary, LEEP bootstrap samples the training contexts to create overlapping sets of training contexts $\ctrain^1, \dots \ctrain^n$. Every iteration, each policy $\pi_i$ generates rollouts in training contexts chosen uniformly from its corresponding $\ctrain^i$, and is then updated according to Equation \ref{eq:ppo_surrogate}, which both maximizes the expected reward and minimizes the disagreement penalty between each $\pi_i$ and the combined policy $\pi = \max_j \pi_j$.

While this algorithm is structurally similar to algorithms for multi-task learning that train a separate policy for each context or group of contexts with a disagreement penalty~\citep{Distral, ghosh2018divideandconquer}, the motivation and the interpretation of these approaches are completely distinct. In multi-task learning, the goal is to solve a given set of tasks, and these methods promote transfer via a loss that encourages the solutions to the tasks to be in agreement. In our setting, while we also receive a set of tasks (contexts), the goal is not to maximize performance on the training tasks, but rather to learn a policy that maximizes performance on unseen test tasks. The method also has a subtle but important distinction: each of our policies $\pi_i$ acts on a sample from the contextual MDP posterior (which captures epistemic uncertainty), \emph{not} a single training context~\citep{Distral} or element from a disjoint partitioning~\citep{ghosh2018divideandconquer} (which does not). This distinction is crucial, since our generalization performance requires our aim is not to make it easier to solve the training contexts, but the opposite: prevent the algorithm from overfitting to the individual training contexts. Correspondingly, our experiments confirm that such multi-task learning approaches do not provide the same generalization benefits as our approach.

\section{Experiments}

The primary ambition of our empirical study is to test the hypothesis that policies that are learned through (approximations of) the epistemic POMDP do in fact attain better test-time performance than those learned by standard RL algorithms. We do so on the Procgen benchmark~\citep{Cobbe2020LeveragingPG}, a challenging suite of diverse tasks with image-based observations testing generalization to unseen contexts. 
\begin{enumerate}[topsep=0mm,itemsep=0.1mm]
    \item Does LEEP derived from the epistemic POMDP lead to improved test-time performance over standard RL methods?
    \item Can LEEP prevent overfitting when provided a limited number of training contexts?
    \item How do different algorithmic components of LEEP affect test-time performance?
   
\end{enumerate}

The Procgen benchmark is a set of procedurally generated games, each with different generalization challenges. In each game, during training, the algorithm can interact with 200 training levels, before it is asked to generalize to the full distribution of levels. The agent receives a $64 \times 64 \times 3$ image observation, and must output one of 15 possible actions. We instantiate our method using an ensemble of $n=4$ policies, a penalty parameter of $\alpha=1$, and PPO \citep{Schulman2017ProximalPO} to train the individual policies (full implementation details in Appendix~\ref{appendix:LEEP}). 
%%SL.5.22: saying we test on "easy" makes it sound kind of lame, perhaps it would be good somewhere here to reference prior work that reports generalization actually being hard in this setting

\begin{figure}
    \centering
\includegraphics[width=0.9\linewidth]{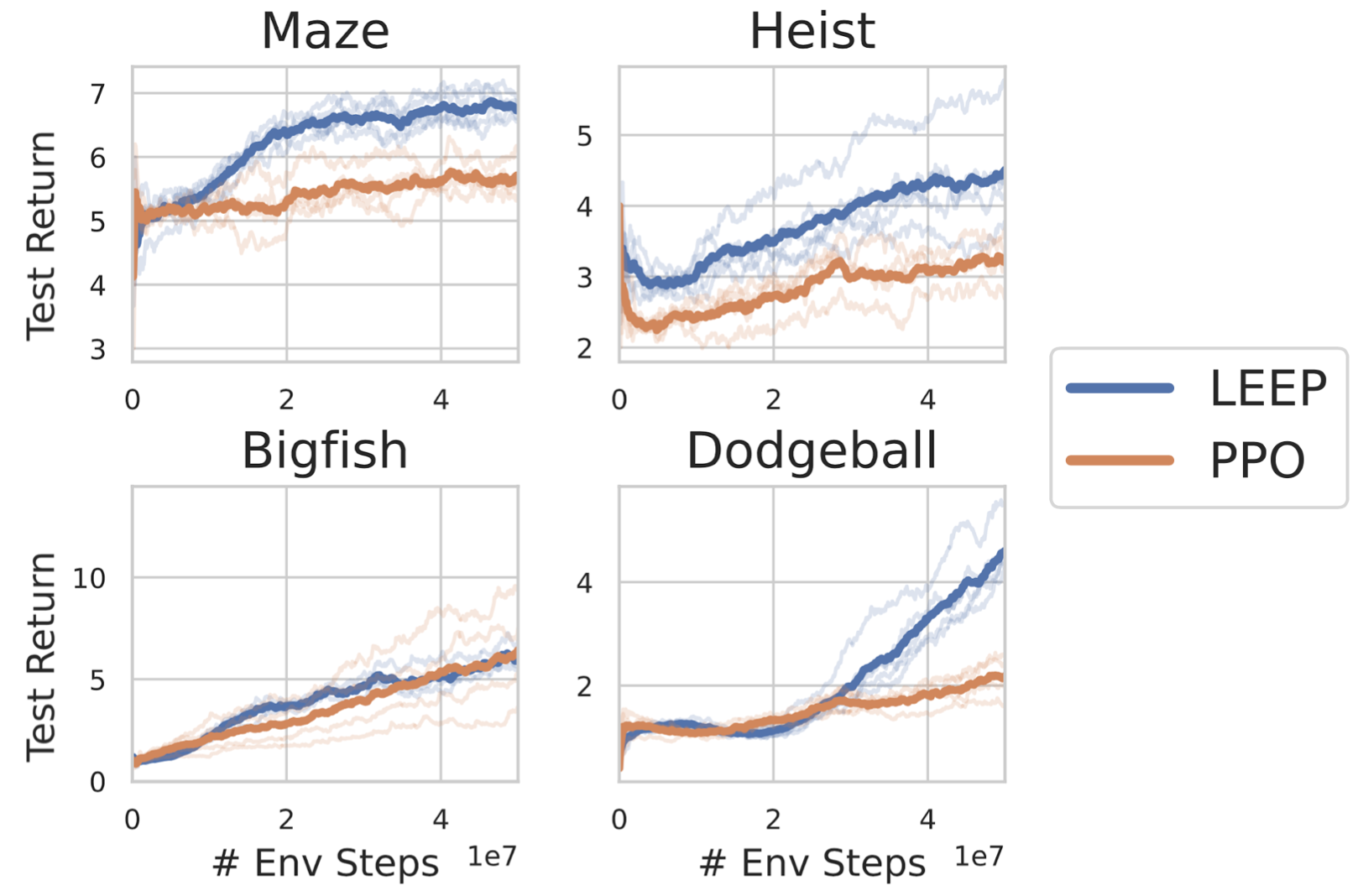}

    \caption{\footnotesize Test set return for LEEP and PPO throughout training in four Procgen environments (averaged across 5 random seeds). LEEP achieves higher test returns than PPO on three tasks (Maze, Heist and Dodgeball) and matches test return on Bigfish while having less variance across seeds.   }

    \label{fig:ppo-comparaison-2}
\end{figure}

We evaluate our method on four games in which prior work has found a large gap between training and test performance, and which we therefore conclude pose a significant generalization challenge~\citep{Cobbe2020LeveragingPG, Jiang2020PrioritizedLR, Raileanu2020AutomaticDA}: Maze, Heist, BigFish, and Dodgeball. In Figure~\ref{fig:ppo-comparaison-2}, we compare the test-time performance of the policies learned using our method to those learned by a PPO agent with entropy regularization. In three of these environments (Maze, Heist, and Dodgeball), our method outperforms PPO by a significant margin, and in all cases, we find that the generalization gap between training and test performance is lower for our method than PPO (Appendix~\ref{appendix:morePlots}). 
\begin{figure}

    \centering
    \includegraphics[width=0.85\linewidth]{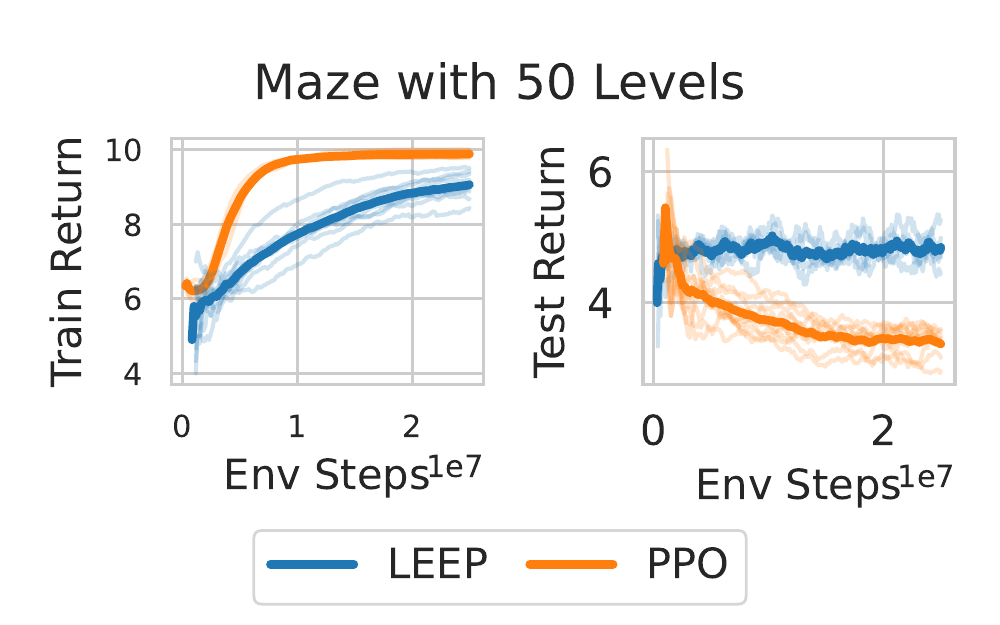}\\
    \vspace{5em}
    \includegraphics[width=0.65\linewidth]{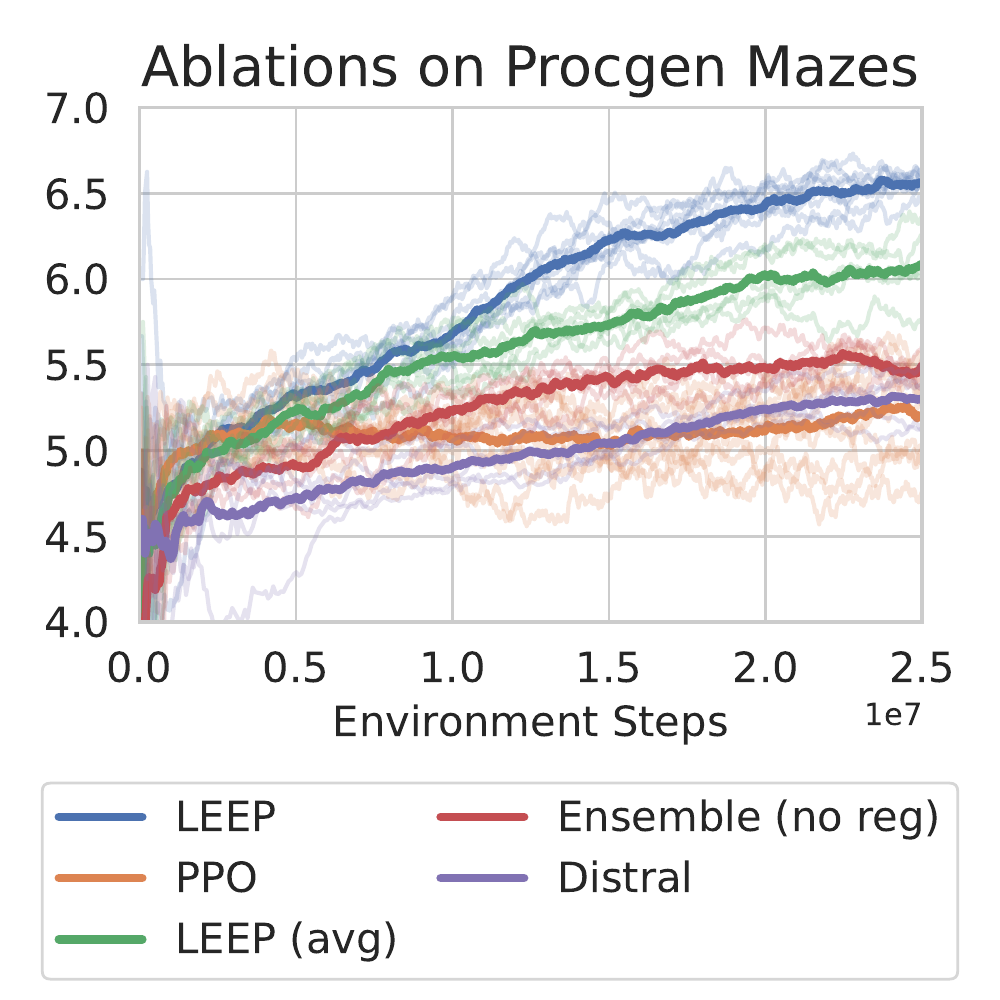}

    \caption{\footnotesize{\textbf{(top)} Performance of LEEP and PPO with only 50 training levels on Maze. \textbf{(bottom)}  Ablations of LEEP in Maze.}}
    \label{fig:ablations}

\end{figure}To understand how LEEP behaves with fewer training contexts, we ran on the Maze task with only 50 levels (Figure~\ref{fig:ablations} (top)); the test return of the PPO policy decreases through training, leading to final performance worse than the starting random policy, but our method avoids this degradation. 

We perform an ablation study on the Maze and Heist environments (Maze in Figure 4, Heist in Appendix~\ref{appendix:morePlots}) to rule out potential confounding causes for the improved generalization that our method displays on the Procgen benchmark tasks. 

First, to see if the performance benefit derives solely from the use of ensembles, we compare LEEP to a Bayesian model averaging strategy that trains an ensemble of policies without regularization (``Ensemble (no reg)''), and uses a mixture of these policies. This strategy does improve performance over the PPO policy, but does not match LEEP, indicating the usefulness of the regularization. Second, we compared to a version of LEEP that combines the ensemble policies together using the average $\frac{1}{n}\sum_{i=1}^n \pi_i(a|s)$ (``LEEP (avg)''). This link function achieves worse test-time performance than the optimistic version, which indicates that the inductive bias conferred by the $\max_i \pi_i$ link function is a useful component of the algorithm. We also compare to Distral, a multi-task learning method with different motivations but similar structure to LEEP: this method helps accelerate learning on the provided training contexts (figures in Appendix~\ref{appendix:morePlots}), but does not improve generalization performance as LEEP does. We additionally ablated the two key hyperparameters in LEEP, the number of ensemble members $n$ and the penalty coefficient $\alpha$  (Table in Appendix \ref{appendix:hyperparam_sweep}).

\section{Discussion}

It has often been observed experimentally that generalization in RL poses a significant challenge, but it has so far remained an open question as to whether the RL setting itself presents additional generalization challenges beyond those seen in supervised learning. In this chapter, we answer this question in the affirmative, and show that, in contrast to supervised learning, generalization in RL results in a new type of problem that cannot be solved with standard MDP solution methods, due to partial observability induced by epistemic uncertainty. We call the resulting partially observed setting the epistemic POMDP, where uncertainty about the true underlying MDP results in a challenging partially observed problem. We present a practical approximate method that optimizes a bound for performance in an approximation of the epistemic POMDP, and show empirically that this approach, which we call LEEP, attains significant improvements in generalization over other RL methods that do not properly incorporate the agent's epistemic uncertainty into policy optimization. A limitation of this approach is that it optimizes a crude approximation to the epistemic POMDP with a small number of posterior samples, and may be challenging to scale to better approximations to the true objective. Developing algorithms that better model the epistemic POMDP and optimize policies within is an exciting avenue for future work, and we hope that this direction will lead to further improvements in generalization in RL.

\begin{subappendices}
\section{FashionMNIST Classification}
\label{appendix:FashionMnist}
\subsection{Implementation details}
\label{appendix:FashionMnistImplementation}

\textbf{Environment:} The RL image classification environment consists of a dataset of labelled images. At the beginning of each episode, a new image and its corresponding label are chosen from the dataset, and held fixed for the entire episode. Each time-step, the agent must pick an action corresponding to one of the labels. If the picked label is correct, the agent gets a reward of $r=0$, and the episode ends, and if the picked label is incorrect, then the agent gets a reward of $r=-1$, and the episode continues to the next time-step (where it must guess another label for the \textit{same} image). The total return for a trajectory corresponds to the number of incorrect guesses the agent makes for the image. We enforce a time-limit of $20$ timesteps in the environment to prevent infinite-length trajectories of incorrect guessing.

We train the agent on a dataset of $10000$ FashionMNIST images subsampled from the training set, and test on the FashionMNIST test dataset. Note that this task is very similar, but not exactly equivalent to maximizing predictive accuracy for supervised classification: if the episode ended regardless of whether or not the agent was correct, then it would correspond exactly to classification. 

\textbf{Algorithm:} We train a DQN agent on the training environment using the min-Q update rule from TD3 \citep{Fujimoto2018AddressingFA}. The Q-function architecture is a convolutional neural network (CNN) with the architecture from Kostrikov et al \citep{pytorchrl}. To ensure that the agent does not suffer from poor exploration during training, the replay buffer is pre-populated with one copy of every possible transition in the training environment (that is, where every action is taken for every image in the training dataset). The variant labelled ``Uniform after step 1'' in Figure \ref{fig:classification_results} follows the DQN policy for the first time-step, and if this was incorrect, then at all subsequent time-steps, takes a random action uniformly amongst the 10 labels. For the variant labelled ``Adaptive'', we train a classifier $p_\theta(y|x)$ on the training dataset of images with the same architecture as the DQN agent. The adaptive agent follows a process-of-elimination strategy; formally, the action taken by the adaptive agent at time-step $t$ is given by $\argmax_{a \notin \{a_1, \dots, a_{t-1}\}} p_\theta(y=a|x)$.

\subsection{Derivation of Bayes-optimal policies}
\label{appendix:FashionMnistTheory}

In the epistemic POMDP for the RL image classification problem, each episode, an image $x \in \gX$ is sampled randomly from the dataset, and a label $y \in \gY$ sampled randomly for this image from the distribution $p(y|x, \gD)$. This label is \textit{held fixed} for the entire episode. For notation, let $Y = \{1, \dots, d\}$, so that a label distribution $p(y|x)$ can be written as a vector in the probability simplex on $\R^d$. We emphasize two settings: $\gamma=0$ (the supervised learning setting), and $\gamma=1$ (an RL setting), where the expected return of an agent is the average number of incorrect guesses made. 

\subsubsection{Memory-based policy}

Since the optimal memory-based policy in a POMDP is deterministic \citep{monahan1982state}, we restrict ourselves to analyzing the performance of deterministic memory-based policies. In the following we will narrow the search space even further.

Since the episode ends after the agent correctly classifies an image and the reward structure incentives the agent to solve the task as quickly as possible, an agent acting optimally will never repeat the same action twice. Indeed, the agent will not have the opportunity to repeat the right action twice because the episode would have ended after the first time it tried it. Furthermore, trying a wrong action twice is also not optimal as in incurs addition negative reward. Therefore, we can limit our search space to policies that try each action once. These policies differ by the ordering in which they try each one of these $d$ labels.

At the beginning of every episode, a image $x$ is sampled uniformly at random among all training images and its true label $y$ (during that episode) is sampled from $p(y|x, \gD)$.  Let $\pi$ be policy that tries each of the $d$ actions exactly once in its first $d$ trials. Let $T^{\pi}_{y}$ denotes the time when policy $\pi$ tries action $y$. Note that $(T^{\pi}_{y})_{y\in \gY}$ is a permutation. When the label chosen is $y$, the cumulative reward of $\pi$ for that episode is given by $r = \frac{\gamma^{ T_y^\pi}-1}{1-\gamma}$ and the expected cumulative reward (across episodes)  is given by: 

\begin{equation}
    \begin{aligned}
    J(\pi) :=& \sum_{y \in \gY} p(y|x, \gD) \frac{\gamma^{ T_y^\pi}-1}{1-\gamma}\\
    =& \frac{1}{1-\gamma}\left(\sum_{y \in \gY} p(y|x, \gD) \gamma^{ T_y^\pi} - 1\right)
    \end{aligned}
\end{equation}

From that expression, we see that in order to maximize its expected cumulative reward, a policy $\pi$ has to maximize $\sum_{y \in \gY} p(y|x, \gD) \gamma^{ T_y^\pi}$ which can be interpreted as the dot product of the vector $[p(y|x, \gD)]_{y \in \gY}$ and $[\gamma^{ T_y^\pi}]_{y \in \gY}$. By the rearrangement inequality, we know that this dot product is maximized when the components of the vectors are arranged in the same ordering. 

If we denote by $y_{(1)}, \dots y_{(d)}$ be the labels sorted in order of probability under the belief distribution: $p(y_{(1)} | x, \gD) \geq p(y_{(2)} | x, \gD) \geq \dots \geq p(y_{(d)} | x, \gD)$. Since $0<\gamma<1$ the rearrangement inequality implies that 

 the expected return is maximized when $T_{y_{(t)}}^\pi = t$. This corresponds to a policy that tries the labels sequentially from the most likely to the least likely.

\subsubsection{Memoryless policy policy} 
\label{app:memorylessPolicy}
Consider a memoryless policy that takes actions according to the distribution $\pi(\cdot|x)$ for the image $x$. When the true label is $y$ for the image $x$, the number of incorrect guesses is distributed as $\text{Geom}(p=\pi(y|x))$. 

When the agent guesses correctly the label $y$ at the $t-$th guess then the cumulative reward is given by $r= - \frac{1-\gamma^t}{1-\gamma}$.  This happens with probability $(1-\pi(y\mid x))^t \times \pi(y \mid x)$. The expected return for policy $\pi$ evaluated on image $x$ is then given by: 

\begin{equation}
\begin{aligned}   
J(\pi|x) &= - \sum_{y\in\gY} \sum_{t=0}^\infty (1-\pi(y|x))^t \pi(y|x) \frac{1 - \gamma^t}{1-\gamma} p(y|x, \gD) \\
&=\sum_{y\in\gY} p(y|x, \gD) \frac{\pi(y \mid x)-1}{1 - \gamma(1-\pi(y|x))}
\end{aligned}
\end{equation}

When $\gamma=0$ (supervised learning problem), $ J(\pi) = \sum_{y\in\gY} p(y|x, \gD) \left(\pi(y \mid x)-1\right)$ is a linear function of $\pi$ and as expected, the optimal policy is to deterministically choose the label with the highest probability: $\pi^{*}(y \mid x) = 1\left[y = \argmax_{y\in\gY} p(y|x, \gD)\right]$.

When $\gamma>0$, the optimal policy is the solution to a constrained optimization problem that can be solved with Lagrange multipliers. When $\gamma=1$, the optimal policy can be written explicitly as: 
\begin{equation}
    \pi^{*} (y\mid x) = \frac{1}{\lambda} \sqrt{p(y|x, \gD)}
\end{equation}
where $\lambda$ is a normalization constant.
\clearpage

\section{Theoretical Results}
\epistemicpomdpdefn*
\begin{proof}
This proposition follows directly from the definition of the epistemic POMDP. If the MDP $\gM$ is sampled from $\gP(\gM)$ and $\gD$ is witnessed, then the posterior distribution over MDPs is given by $\gP(\gM | \gD)$, and the expected test-time return of $\pi$ given the evidence is 
\[\E_{\gM \sim \gP(\gM)}[J_\gM(\pi) | \gD] \coloneqq \E_{\gM \sim \gP(\gM | \gD)}[J_{\gM}(\pi)].\]

In the epistemic POMDP, where an episode corresponds to randomly sampling an MDP from $\gP(\gM | \gD)$, and a single episode being evaluated in this MDP, the expected return can be expressed identically:

\begin{align}
J_{\gM^\po}(\pi) &\coloneqq \E_{\gM \sim \gP(\gM | \gD)}[\E_{\pi, \gM}[\sum_{i=0}^\infty \gamma^t r(s_t, a_t)]] \\
&= \E_{\gM \sim \gP(\gM | \gD)}[J_{\gM}(\pi)].
\end{align}

\end{proof}
\label{appendix:Sec5}
\subsection{Optimal MDP Policies can be Arbitrarily Suboptimal}

\begin{restatable}{proposition}{dominatedRandom}
\label{prop:deterministic-dominated-random}
Let $\eps > 0$. There exists posterior distributions $\gP(\gM | \gD)$ where a deterministic Markov policy $\pi$ is optimal with probability at least $1-\epsilon$,

\begin{equation}
P_{\gM \sim \gP(\gM|\gD)}\left(\pi \in \argmax_{\pi'} J_\gM(\pi')\right) \geq 1-\eps,\end{equation}

but is outperformed by a uniformly random policy in the epistemic POMDP: $J_{\gM^\po}(\pi) < J_{\gM^\po}(\pi_{\text{unif}})$.
\end{restatable}
\begin{proof}
Consider two deterministic MDPs, $\gM_A$, and $\gM_B$ that both have two states and two actions: ``stay'' and ''switch''.
In both MDPs, the reward for the ``stay'' action is always zero.
In~$\gM_A$ the reward for ``switch'' is always 1, while in~$\gM_B$ the reward for ``switch'' is $-c$ for~$c>0$.
The probability of being in~$\gM_B$ is~$\epsilon$ while the probability of being in~$\gM_A$ is~$1-\epsilon$.
Clearly, the policy ``always switch'' is optimal in~$\gM_A$ and so is~$\epsilon$-optimal under the distribution on MDPs.
The expected discounted reward of the ``always switch'' policy is:
\begin{align}
    J(\pi_{\text{always switch}}) &= (1-\epsilon)\frac{1}{1-\gamma} - \epsilon\frac{c}{1-\gamma}\\
    &= \frac{1}{1-\gamma}(1-\epsilon - c\epsilon)\\
    &= \frac{1}{1-\gamma}(1-(c+1)\epsilon)\,.
\end{align}
On the other hand, we can consider a policy which selects actions uniformly at random.
In this case, the expected cumulative reward is 

\begin{align}
    J(\pi_{\text{random}}) &= (1-\epsilon)\frac{1}{2}\frac{1}{1-\gamma} - \epsilon \frac{c}{2}\frac{1}{1-\gamma}\\
    &= \frac{1}{2}\frac{1}{1-\gamma}(1 - \epsilon - c\epsilon)\\
    &= \frac{1}{2}\frac{1}{1-\gamma}(1 - (c+1)\epsilon)\\
    &= \frac{1}{2}J(\pi_{\text{always switch}})\,.
\end{align}
Thus for any~$\epsilon$ we can find a~${c > \frac{1}{\epsilon}-1}$ such that both policies have negative expected rewards and we prefer the random policy for being half as negative.
\end{proof}

\label{appendix:prop51}
\subsection{Bayes-optimal Policies May Take Suboptimal Actions Everywhere}

We formalize the remark that optimal policies for the MDPs in the posterior distribution may be poor guides for determining what the Bayes-optimal behavior is in the epistemic POMDP. The following proposition shows that there are epistemic POMDPs where the support of actions taken by the MDP-optimal policies is disjoint from the actions taken by the Bayes-optimal policy, so no method can ``combine'' the optimal policies from each MDP in the posterior to create Bayes-optimal behavior. 
\begin{proposition}
\label{prop:suboptimal-actions-everywhere}
There exist posterior distributions $\gP(\gM | \gD)$ where the support of the Bayes-optimal memoryless policy $\pi^{*\po}(a|s)$ is disjoint with that of the optimal policies in each MDP in the posterior. Formally, writing $\supp(\pi(a|s)) = \{a \in \gA: \pi(a|s) > 0\}$, then $\forall \gM$ with $\gP(\gM | \gD) > 0$ and $\forall s$:

\[\supp(\pi^{*\po}(a|s)) \cap \supp(\pi_\gM^*(a|s))  = \emptyset\]
\end{proposition}

\begin{proof}
The proof is a simple modification of the construction in Proposition~\ref{prop:deterministic-dominated-random}. Consider two deterministic MDPs, $\gM_A$, and $\gM_B$ with equal support under the posterior, where both have two states and three actions: ``stay'', ''switch 1'', and ``switch 2''. In both MDPs, the reward for the ``stay'' action is always zero. In~$\gM_A$ the reward for ``switch'' is always 1, while in~$\gM_B$ the reward for ``switch'' is $-2$. The reward structure for ``switch 2'' is flipped: in $\gM_A$, the reward for ``switch 2'' is $-2$, and in $\gM_B$, the reward is $1$. Then, the policy ``always switch'' is optimal in $\gM_A$, and the policy ``always switch 2'' is optimal in $\gM_B$. However, any memoryless policy that takes either of these actions receives negative reward in the epistemic POMDP, and is dominated by the Bayes-optimal memoryless policy ``always stay'', which achieves 0 reward.
\end{proof}

\label{appendix:suboptimalActions}
\subsection{MaxEnt RL is Optimal for a Choice of Prior}

We describe a special case of the construction of Eysenbach and Levine \citep{eysenbachMaxEnt}, which shows that maximum-entropy RL in a bandit problem recovers the Bayes-optimal POMDP policy in an epistemic POMDP similar to that described in the RL image classification task.

Consider the family of MDPs $\{\gM_k\}_{k \in [n]}$ each with one state and $n$ actions, where taking action $k$ in MDP $\gM_k$ yields zero reward and the episode ends, and taking any other action yields reward $-1$ and the episode continues. Effectively, $\gM_k$ corresponds to a first-exit problem with ``goal action'' $k$. Note that this MDP structure is exactly what we have for the RL image classification task for a single image. Also consider the surrogate bandit MDP $\hat{\gM}$, also with one state and $n$ actions, but in which taking action $k$ yields reward $r_k$ with immediate episode termination. The following proposition shows that running max-ent RL in $\hat{\gM}$ recovers the optimal memoryless policy in a particular epistemic POMDP supported on $\{\gM_k\}_{k \in [n]}$.

\begin{proposition}
Let $\pi^* = \arg\max_{\pi \in \Pi} J_{\hat{\gM}}(\pi) + \gH(\pi)$ be the max-ent solution in the surrogate bandit MDP $\hat{\gM}$. Define the distribution $\gP(\gM|\gD)$ on $\{\gM_k\}_{k \in [n]}$ as $\gP(\gM_k|\gD) = \frac{\exp(2r_k)}{\sum_{j} \exp(2r_j)}$. Then, $\pi$ is the optimal memoryless policy in the epistemic POMDP $\gM^\po$ defined by $\gP(\gM | \gD)$. 
\end{proposition}
\begin{proof}
See Eysenbach and Levine \citep[Lemma 4.1]{eysenbachMaxEnt}. The optimal policy $\pi^*$ is given by $\pi^*(a=k) = \frac{\exp(r_k)}{\sum_{j} \exp(r_j)}$. We know from Appendix~\ref{app:memorylessPolicy} that this policy is optimal for epistemic POMDP $\gM^{\po}$ when $\gamma=1$.
\end{proof}
If allowing time-varying reward functions, this construction can be extended beyond ``goal-action taking'' epistemic POMDPs to the more general ``goal-state reaching'' setting in an MDP, where the agent seeks to reach a specific goal state, but the identity of the goal state hidden from the agent \citep[Lemma 4.2]{eysenbachMaxEnt}. 

\label{appendix:MaxEntRL}
\subsection{Failure of MaxEnt RL and Uncertainty-Agnostic Regularizations}
We formalize the remark made in the main text that while the Bayes-optimal memoryless policy is stochastic, methods that promote stochasticity in an uncertainty-agnostic manner can fail catastrophically. We begin by explaining the significance of this result: it is well-known that stochastic policies can be arbitrarily sub-optimal in a single MDP, and can be outperformed by deterministic policies. The result we describe is more subtle than this: there are epistemic POMDPs where any attempt at being stochastic in an uncertainty-agnostic manner is sub-optimal, and \textit{ also} any attempt at acting completely deterministically is also sub-optimal. Rather, the characteristic of Bayes-optimal behavior is to be stochastic in \textit{some} states (where it has high uncertainty), and not stochastic in others, and a useful stochastic regularization method must modulate the level of stochasticity to calibrate with regions where it has high epistemic uncertainty.

\begin{proposition}
\label{prop:stochastic-dominated-random}
Let $\alpha > 0, c > 0$. There exist posterior distributions $\gP(\gM | \gD)$, where the Bayes-optimal memoryless policy $\pi^{*\po}$ is stochastic. However, every memoryless policy $\pi_s$ that is ``everywhere-stochastic'', in that $\forall s \in \gS: \gH(\pi_s(a| s)) > \alpha$,
can have performance arbitrarily close to the uniformly random policy:
\[ \frac{J(\pi_s) - J(\pi_\text{unif})}{J(\pi^{*\po}) - J(\pi_\text{unif})} < c \]
\end{proposition}

\begin{figure}
    \centering
    \includegraphics[width=0.75\linewidth]{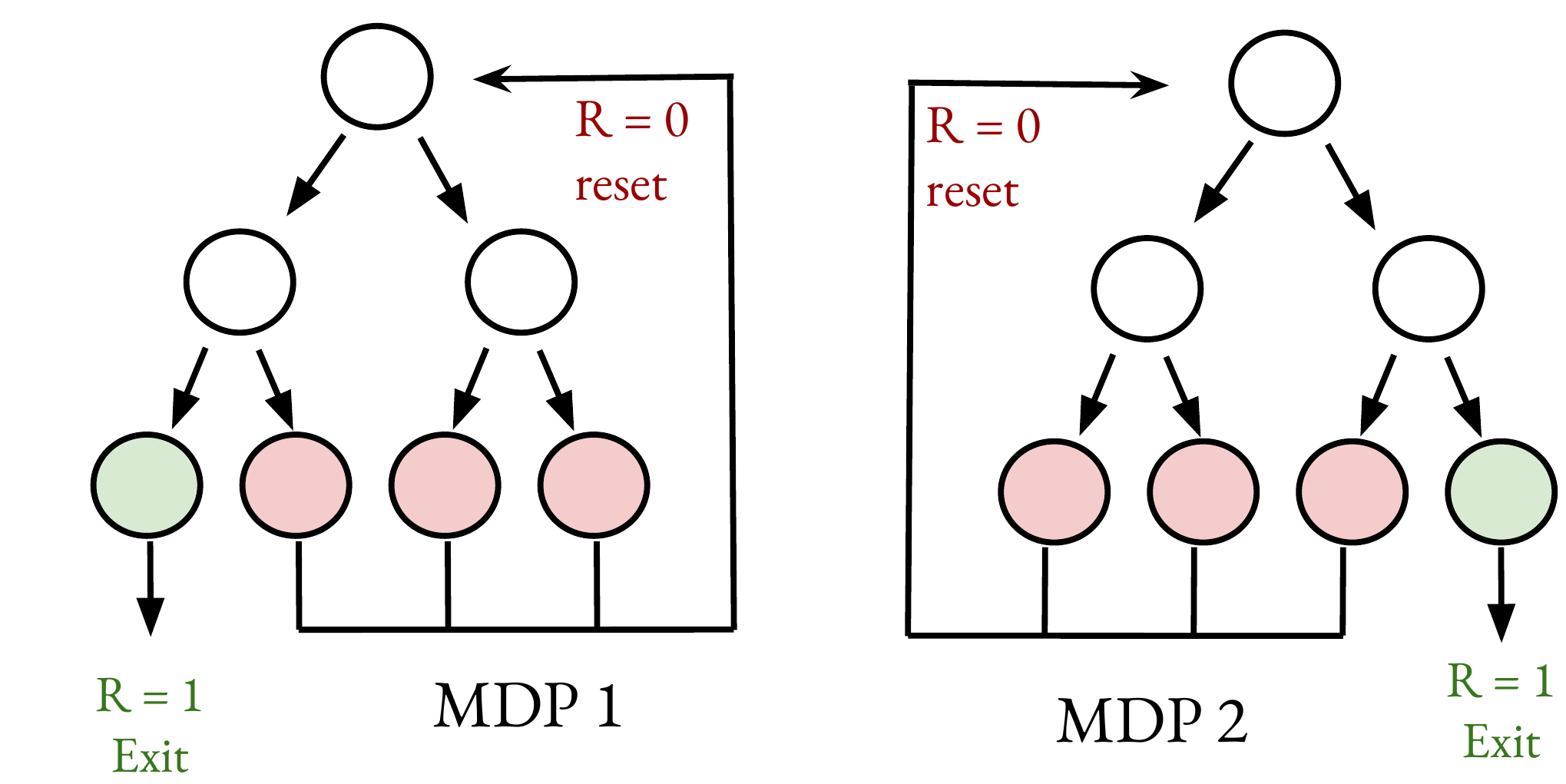}
    \caption{Visual description of Binary Tree MDPs described in proof of Proposition \ref{prop:stochastic-dominated-random} with depth $n=3$.}
    \label{fig:stochastic_illustration}
\end{figure}
\begin{proof}

Consider two binary tree MDP with $n$ levels, $\gM_1$ and $\gM_2$. A binary tree MDP, visualized in Figure \ref{fig:stochastic_illustration}, has $n$ levels, where level $k$ has $2^k$ states. On any level $k < n$, the agent can take a ``left'' action or a ``right'' action, which transitions to the corresponding state in the next level. On the final level, if the state corresponds to the terminal state (in green), then the agent receives a reward of $1$, and the episode exits, and otherwise a reward of $0$, and the agent returns to the top of the binary tree. The two binary tree MDPs $\gM_1$ and $\gM_2$ are identical except for the final terminal state: in $\gM_1$, the terminal state is the left-most state in the final level, and in $\gM_2$, the terminal state is the right-most state. Reaching the goal in $\gM_1$ corresponds to taking the ``left'' action repeatedly, and reaching the goal in $\gM_2$ corresponds to taking the ``right'' action repeatedly. We consider the posterior distribution that places equal mass on $\gM_1$ and $\gM_2$, $\gP(\gM_1|\gD)=\gP(\gM_2|\gD)=\frac{1}{2}$. A policy that reaches the correct terminal state with probability $p$ (otherwise reset) will visit the initial state a $\text{Geom}(p)$ number of times, and writing $\bar{\gamma} \coloneqq \gamma^n$, will achieve return $\frac{\bar{\gamma}p}{1 - \bar{\gamma} + p\bar{\gamma}} = \frac{1}{1 + \frac{1}{p} \frac{1-\bar{\gamma}}{\bar{\gamma}}}$.

\textit{Uniform policy:} A uniform policy randomly chooses between ``left'' and ``right'' at all states, and will reach all states in the final level equally often, so the probability it reaches the correct goal state is  $\frac{1}{2^n}$. Therefore, the expected return is $J(\pi_{\text{unif}}) = \frac{1}{1 + 2^n\frac{1-\bar{\gamma}}{\bar{\gamma}}}$.

\textit{Bayes-optimal memoryless policy:} The Bayes-optimal memoryless policy $\pi^{*\po}$ chooses randomly between ``left'' and ``right'' at the top level; on every subsequent level, if the agent is in the left half of the tree, the agent deterministically picks ``left'' and on the right half of the tree, the agent deterministically picks ``right''. Effectively, this policy either visits the left-most state or the right-most state in the final level. The Bayes-optimal memoryless policy returns to the top of the tree a $\text{Geom}(p=\frac{1}{2})$ number of times, and the expected return is given by $J(\pi^{*\po}) = \frac{1}{1 + 2\frac{1-\bar{\gamma}}{\bar{\gamma}}}$. 

\textit{Everywhere-stochastic policy:} Unlike the Bayes-optimal policy, which is deterministic in all levels underneath the first, an everywhere-stochastic policy will sometimes take random actions at these lower levels, and therefore can reach states at the final level that are neither the left-most or right-most states (and therefore always bad). We note that if $\gH(\pi(a|s)) > \alpha$, then there is some $\beta > 0$ such that $\max_a \pi(a|s) < 1 - \beta$. For an $\alpha$-everywhere stochastic policy, the probability of taking at least one incorrect action increases as the depth of the binary tree grows, getting to the correct goal at most probability $\frac{1}{2}(1-\beta)^{n-1}$. The maximal expected return is therefore $J(\pi_s) \leq \frac{1}{1 + 2(\frac{1}{1-\beta})^{n-1} \frac{1-\bar{\gamma}}{\bar{\gamma}}}$

\[J(\pi^{*\po}) = \frac{1}{1 + 2\frac{1-\bar{\gamma}}{\bar{\gamma}}} ~~~~~~ J(\pi_s) = \frac{1}{1 + 2(\frac{1}{1-\beta})^{n-1} \frac{1-\bar{\gamma}}{\bar{\gamma}}} ~~~~~~J(\pi_{\text{unif}}) = \frac{1}{1 + 2^n\frac{1-\bar{\gamma}}{\bar{\gamma}}}\]

As $n\to \infty$, $J(\pi^{*\po}), J(\pi_s)$ and $J(\pi_{\text{unif}})$ will converge to zero. Using asymptotic analysis we can determine their speed of convergence and find that: 

\[J(\pi^{*\po}) \sim \frac{\bar{\gamma}}{2} ~~~~~~ J(\pi_s) \sim \frac{\bar{\gamma}}{ 2(\frac{1}{1-\beta})^{n-1}} ~~~~~~J(\pi_{\text{unif}}) \sim \frac{\bar{\gamma}}{ 2^n}\]

Using these asymptotics, we find that 

\[ \frac{J(\pi_s) - J(\pi_\text{unif})}{J(\pi^{*\po}) - J(\pi_\text{unif})} \sim \frac{1}{ (\frac{1}{1-\beta})^{n-1}} = (1-\beta)^{n-1}, \]

which shows that this ratio can be made arbitrarily small as we increase $n$. \qedsymbol \\

\textit{An aside: deterministic policies. } While this proposition only discusses the failure mode of stochastic policies, \textit{all} deterministic memoryless policies in this environment also fail. A deterministic policy $\pi_d$ in this environment continually loops through one path in the binary tree repeatedly, and therefore will only ever reach one goal state, unlike the Bayes-optimal policy which visits both possible goal states. The best deterministic policy then either constantly takes the ``left'' action (which is optimal for $\gM_1$), or constantly takes the ``right'' action (which is optimal for $\gM_2$). Any other deterministic policy reaches a final state that is neither the left-most nor the right-most state, and will always get $0$ reward. The expected return of the optimal deterministic policy is $J(\pi_d) = \frac{\bar{\gamma}}{2}$, receiving $\bar{\gamma}$ reward in one of the MDPs, and $0$ reward in the other. When the discount factor $\gamma$ is close to $1$, the maximal expected return of a deterministic policy is approximately $\frac{1}{2}$, while the expected return of the Bayes-optimal policy is approximately $1$, indicating a sub-optimality gap.
\end{proof} 
\clearpage

\subsection{Proof of Proposition~\ref{thm:lower_bound}}
\theoremLowerBound*
\begin{proof}
Before we begin, we recall some basic tools from analysis of MDPs. For a memoryless policy $\pi$, the state-action value function  $Q^\pi(s,a)$ is given by:
\begin{equation}
{Q^\pi(s, a) = \E_\pi[\sum_{t \geq 0} \gamma^t r(s_t, a_t) | s_0 = s, a_0=a]}.
\end{equation}
The advantage function $A^\pi(s, a)$ is defined as:
\begin{equation}
{A^\pi(s, a) = Q^\pi(s, a) - \E_{a \sim \pi(\cdot|s)}[Q^\pi(s, a)]}.
\end{equation}
The performance difference lemma \citep{kakadeLangford2002} relates the expected return of two policies $\pi$ and $\pi'$ in an MDP $\gM$ via their advantage functions as 

\begin{equation}
J_\gM(\pi') = J_{\gM}(\pi) + \frac{1}{1-\gamma}\E_{s \sim d_\gM^{\pi'}}[\E_{a \sim \pi'}[A_\gM^\pi(s ,a)]].
\end{equation}

We now begin the derivation of our lower bound:
\begin{equation}
\begin{aligned}
   J_{\hat{\gM}^{\po}}(\pi) &=\frac{1}{n} \sum_{i=1}^n J_{\gM_i}(\pi)\\
   &=\frac{1}{n} \sum_{i=1}^n  J_{\gM_i}(\pi_i) + \frac{1}{n} \sum_{i=1}^n  \left[J_{\gM_i}(\pi) -J_{\gM_i}(\pi_i) \right] \\
    &= \frac{1}{n} \sum_{i=1}^n  J_{\gM_i}(\pi_i) - \frac{1}{n(1-\gamma)} \sum_{i=1}^n  \mathbb{E}_{s\sim d_{\gM_i}^{\pi_i}}\left[ \mathbb{E}_{a\sim \pi_i}\left[A_{\gM_i}^{\pi}(s,a) \right]\right] \\
    &= \frac{1}{n} \sum_{i=1}^n  J_{\gM_i}(\pi_i) - \frac{1}{n(1-\gamma)} \sum_{i=1}^n  \mathbb{E}_{s\sim d_{\gM_i}^{\pi_i}}\left[ \mathbb{E}_{a\sim \pi_i}\left[A_{\gM_i}^{\pi}(s,a) \right]-\mathbb{E}_{a\sim \pi}\left[A_{\gM_i}^{\pi}(s,a) \right]\right] \\
\end{aligned}
\end{equation}
In the last equality we used the fact that $\mathbb{E}_{a\sim \pi}\left[A^{\pi}(s,a) \right] =0$. From there we proceed to derive a lower bound:
\begin{equation}
\begin{aligned}
    J_{\hat{\gM^\po}}(\pi) &= \frac{1}{n} \sum_{i=1}^n  J_{\gM_i}(\pi_i) - \frac{1}{n(1-\gamma)} \sum_{i=1}^n  \mathbb{E}_{s\sim d_{\gM_i}^{\pi_i}}\left[ \mathbb{E}_{a\sim \pi_i}\left[A_{\gM_i}^{\pi}(s,a) \right]-\mathbb{E}_{a\sim \pi}\left[A_{\gM_i}^{\pi}(s,a) \right]\right] \\
    &\geq \frac{1}{n} \sum_{i=1}^n  J_{\gM_i}(\pi_i) - \frac{2 r_{max}}{n(1-\gamma)^2} \sum_{i=1}^n \mathbb{E}_{s\sim d_{\gM_i}^{\pi_i}}\left[ D_{TV}\left(\pi_i(\cdot \mid s);\pi(\cdot \mid s)\right) \right] \\
    &\geq \frac{1}{n} \sum_{i=1}^n  J_{\gM_i}(\pi_i) - \frac{\sqrt{2} r_{max}}{(1-\gamma)^2n} \sum_{i=1}^n \mathbb{E}_{s\sim d_{\gM_i}^{\pi_i}}\left[ \sqrt{D_{KL}\left(\pi_i(\cdot \mid s)\,\, || \,\, \pi(\cdot \mid s)\right)} \right] \\
\end{aligned}
\end{equation}
where the first inequality is since $|A_{\gM_i}^\pi(s,a)| \leq \frac{r_{\max}}{1-\gamma}$ and the second from Pinsker's inequality. Our intention in this derivation is not to obtain the tightest lower bound possible, but rather to illustrate how bounding the advantage can lead to a simple lower bound on the expected return in the POMDP. The inequality can be made tighter using other bounds on $|A_{\gM_i}^\pi(s,a)|$, for example using $A_{\max} = \max_{i, s, a} |A_{\gM_i}^\pi(s,a)|$, or potentially a bound on the advantage that varies across state.
\end{proof}
\subsection{Proof of Proposition~\ref{prop:link-version}}
\propositionLinkFunction*

\begin{proof}
By Proposition~\ref{thm:lower_bound} we have that $\forall \alpha \geq \frac{\sqrt{2}r_{\text{max}}}{(1-\gamma)^2n}$:
\begin{equation}
J_{\hat{\gM^\po}}(f(\{\pi_i^*\})) \geq \frac{1}{n} \sum_{i=1}^n J_{\gM_i}(\pi_i^*) - \alpha \sum_{i=1}^n \E_{s \sim d_{\gM_i}^{\pi_i^*}}\left[\sqrt{ D_{KL}\left(\pi_i^*(\cdot | s)\,\, || \,\, f(\{\pi_i^*\})(\cdot | s)\right)}\right].
\end{equation}
Now, write $\pi'^* \in \argmax_\pi J_{\hat{\gM^\po}}(\pi)$ to be an optimal policy in the empirical epistemic POMDP, and consider the collection of policies $\{\pi'^*, \pi'^*, \dots, \pi'^*\}$.  Since $\{\pi_i^*\}$ is the optimal solution to Equation~\ref{eq:algo_objective}, we have
\begin{equation}
\begin{aligned}
J_{\hat{\gM^\po}}(f(\{\pi_i^*\})) &\geq \frac{1}{n} \sum_{i=1}^n J_{\gM_i}(\pi'^*) - \alpha \sum_{i=1}^n \E_{s \sim d_{\gM_i}^{\pi'^*}}\left[\sqrt{ D_{KL}\left(\pi'^*(\cdot | s)\,\, || \,\, f(\{\pi'^*\})(\cdot | s)\right)}\right]\\
&= \frac{1}{n} \sum_{i=1}^n J_{\gM_i}(\pi'^*)\\
&= J_{\hat{\gM^\po}}(\pi'^*),
\end{aligned}
\end{equation}
where the second line here uses the fact that $f(\pi'^*, \dots, \pi'^*) = \pi'^*$. Therefore $\pi^* \coloneqq f(\{\pi_i^*\})$ is optimal for the empirical epistemic POMDP.
\end{proof}

\clearpage
\section{Procgen Implementation and Experimental Setup}
\label{appendix:LEEP}

We follow the training and testing scheme defined by Cobbe et al. \citep{Cobbe2020LeveragingPG} for the Procgen benchmarks: the agent trains on a fixed set of levels, and is tested on the full distribution of levels. Due to our limited computational budget, we train on the so-called ``easy'' difficulty mode using the recommended $200$ training levels. Nonetheless, many prior work has found a significant generalization gap between test and train performance even in this easy setting, indicating it a useful benchmark for generalization \citep{Cobbe2020LeveragingPG, Raileanu2020AutomaticDA, Jiang2020PrioritizedLR}. We implemented LEEP on top of an existing open-source codebase released by Jiang et al. \citep{Jiang2020PrioritizedLR}. Full code is provided in the supplementary for reference.

LEEP maintains $n=4$ policies $\{\pi_i\}_{i \in [n]}$, each parameterized by the ResNet architecture prescribed by \citet{Cobbe2020LeveragingPG}. In LEEP, each policy is optimized to maximize the entropy-regularized PPO surrogate objective alongside a one-step KL divergence penalty between itself and the linked policy $\max_i \pi_i$; gradients are not taken through the linked policy. 
\begin{equation*}
\begin{aligned}
\E_{\pi_i}[\min(r_t(\pi)A^\pi(s,a), \text{clip}(r_t(\pi), 1-\eps, 1+\eps) \,\, & A^\pi(s,a) + \beta \gH(\pi_i(a|s)) \\
& - \textcolor{purple}{\alpha D_{KL}(\pi_i(a|s) \| \max_j \pi_j(a|s))}]
\end{aligned}
\end{equation*}
Note that this update in Equation 6 is not exactly solving the optimization problem dictated by Equation 5, since it leverages a one-step estimator for the gradient of the KL penalty in the PG objective, a heuristic known to lead to better optimization in PPO and other deep policy gradient methods. If the proper estimator for the KL penalty is substituted in, then the Bayes-optimal policy in the empirical epistemic POMDP is an optimal solution for Equation 6.

The penalty hyperparameter $\alpha$ was obtained by performing a hyperparameter search on the Maze task for all the comparison methods (including LEEP) amongst $\alpha \in [0.01, 0.1, 1.0, 10.0]$. Since LEEP trains $4$ policies using the same environment budget as a single PPO policy, we change the number of environment steps per PPO iteration from $16384$ to $4096$, so that the PPO baseline and each policy in our method takes the same number of PPO updates. All other PPO hyperparameters are taken directly from \citet{Jiang2020PrioritizedLR}. 

In our implementation, we parallelize training of the policies across GPUs, using one GPU for each policy. We found it infeasible to run more ensemble members due to GPU memory constraints without significant slowdown in wall-clock time. Running LEEP on one Procgen environment for 50 million steps requires approximately 5 hrs in our setup on a machine with four Tesla T4 GPUs. 

\section{Procgen Results}
\subsection{Main Experimental Results}
\label{appendix:morePlots}
\begin{figure}[H]
    \centering
    \includegraphics[width=\linewidth]{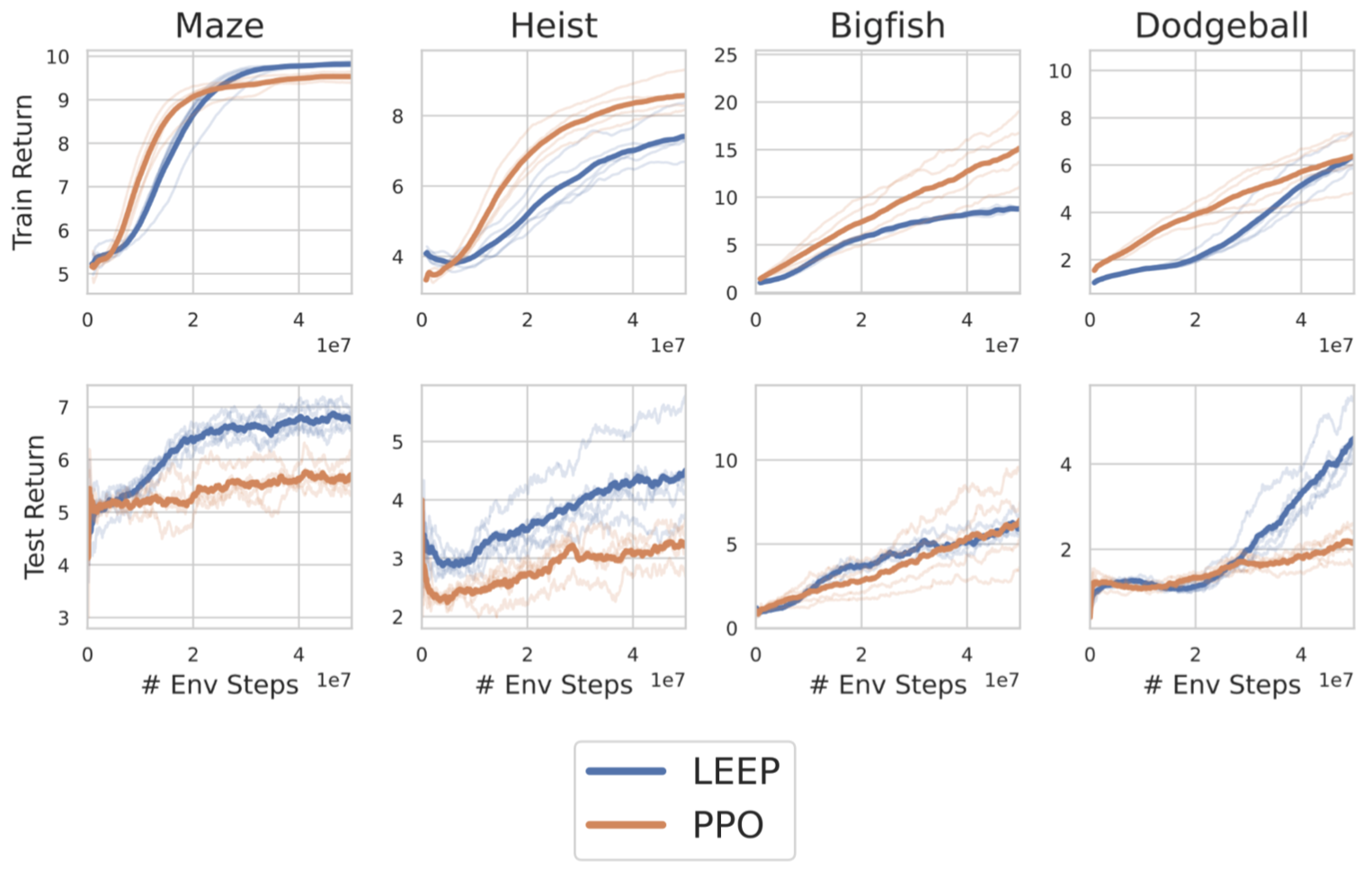}
    \caption{Training (top) and test (bottom) returns for LEEP and PPO on four Procgen environments. Results averaged across 5 random seeds. LEEP achieves equal or higher training return compared to PPO, while having a lower generalization gap between test and training returns.}
    \label{fig:appendix_all_procgen}
\end{figure}

\begin{figure}[H]
    \centering
    \includegraphics[width=\linewidth]{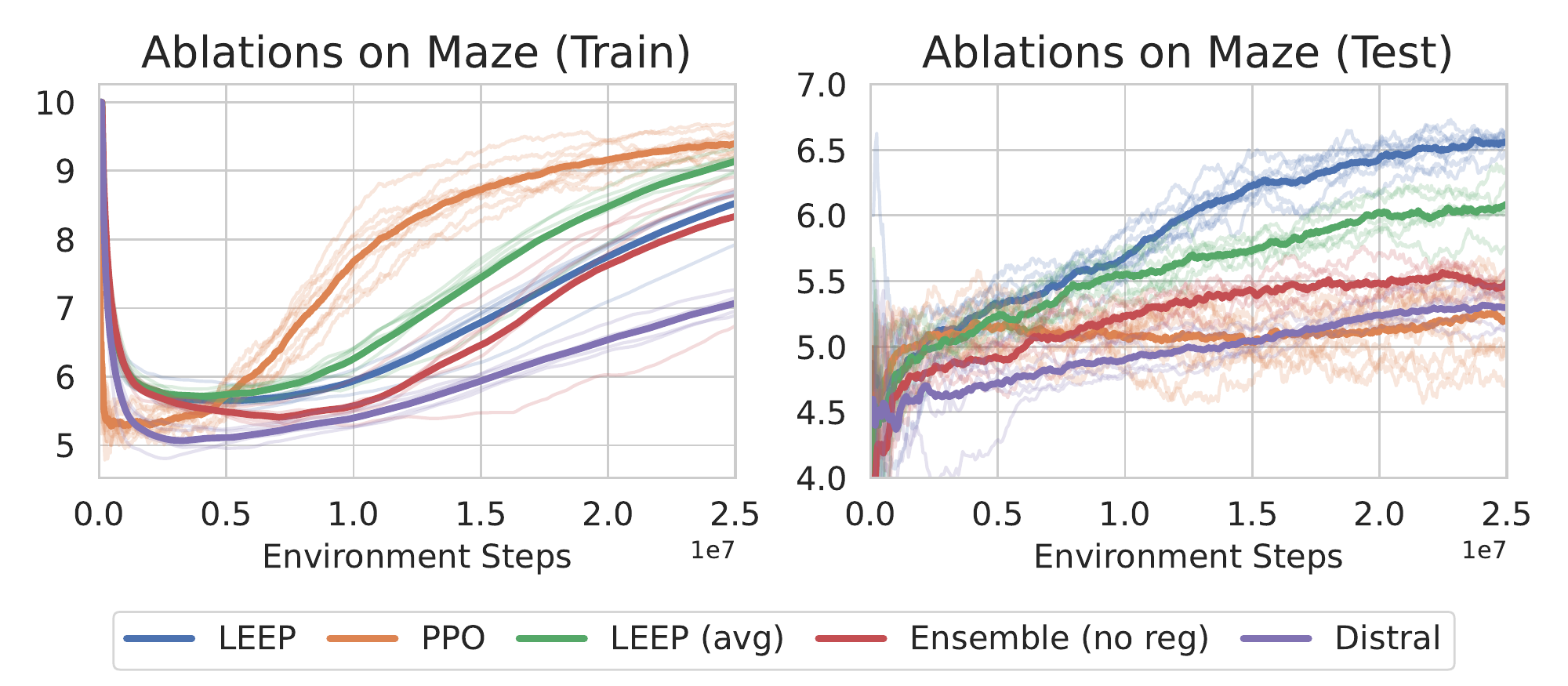}
\includegraphics[width=\linewidth]{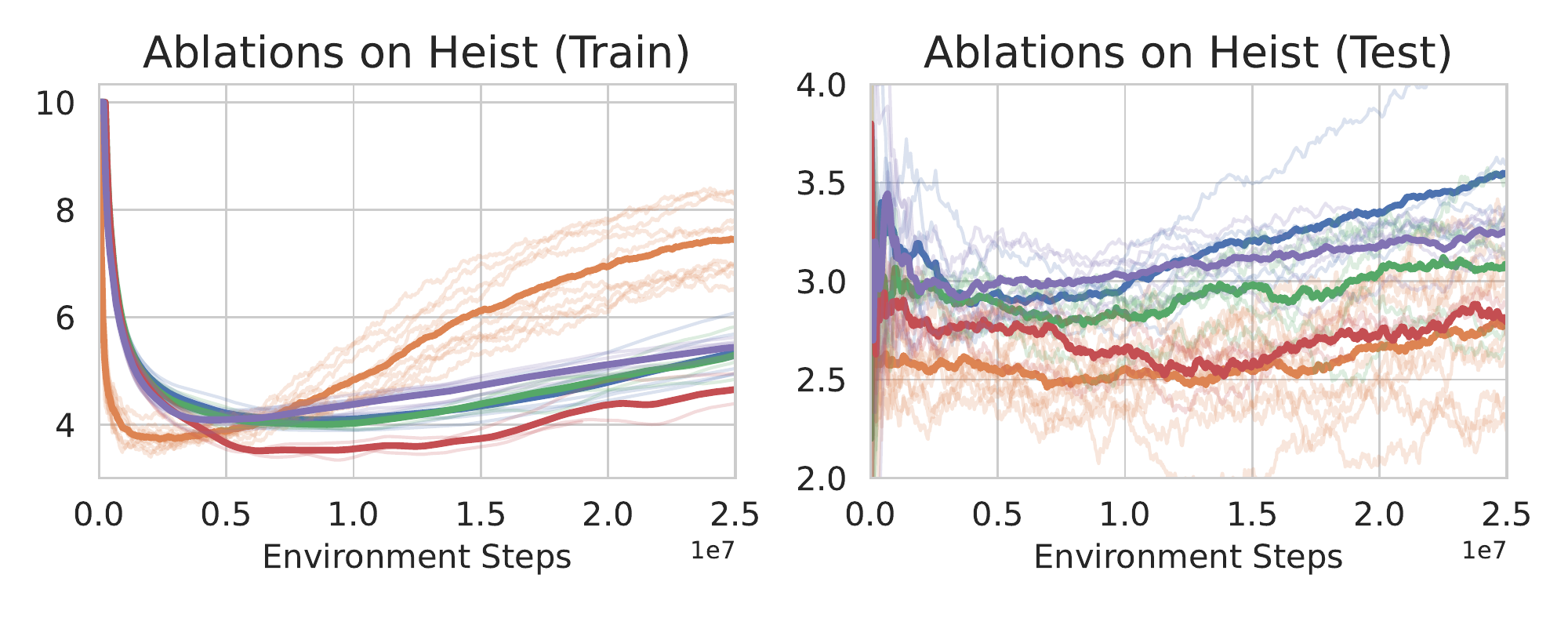}
    \caption{Training and test returns for various ablations and comparisons of LEEP.}
    \label{fig:appendix_procgen_ablations}
\end{figure}

\begin{figure}[H]
    \centering
    \includegraphics[width=\linewidth]{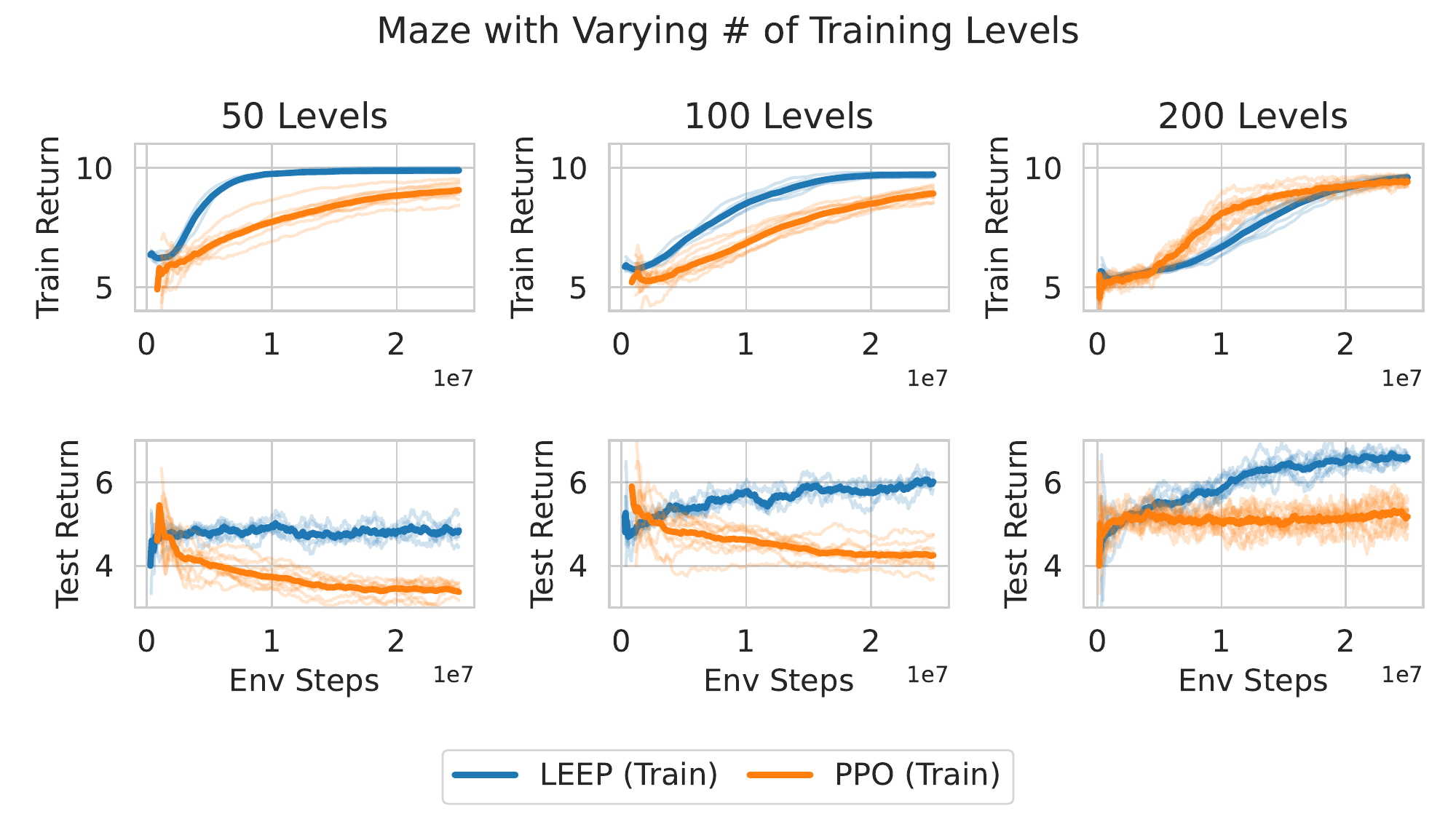}
    \caption{Performance of LEEP and PPO as the number of training levels provided varies. While the learned performance of the PPO policy is worse than a \textit{random policy} with less training levels, LEEP avoids this overfitting and in general, demonstrates a smaller train-test performance gap than PPO. }
    \label{fig:appendix_maze_varying_levels}
\end{figure}

\subsection{Ablations of LEEP Hyperparameters}
\label{appendix:hyperparam_sweep}
\textbf{Number of ensemble members (n): } We ran an ablation study on the Procgen Maze task to understand how the number of ensemble members affects the performance of LEEP. We found that for an equal number of gradient steps per ensemble member, LEEP does equally well with $n=4$ and $8$ ensemble members, but poorly with only 1 or 2 ensemble members (see Figure attached). These results indicate that at least on the Maze task, using n=4 ensemble members is an appropriate balance between approximating the true epistemic POMDP with higher fidelity and minimizing the sample complexity incurred by needing to train more ensemble members with on-policy RL methods.

\begin{table}[H]
\begin{tabular}{l|llll}
\textbf{\# Ensemble members (n)} & \textbf{1}   & \textbf{2}  & \textbf{4}   & \textbf{8}  \\ \hline \Tstrut
Maze         & 5.11 $\pm$ 0.24 & 5.85 $\pm$ 0.4 & 6.53 $\pm$ 0.12 & 6.91 $\pm$ 0.1
\end{tabular}
\end{table}

\textbf{Penalty coefficient ($\alpha$):} We performed a coarse hyperparameter sweep on the four Procgen domains, testing values $\alpha \in 10^{\{-2, -1, 0, 1, 2\}}$. The results in the table below indicated that performance is roughly consistent for $\alpha = \{0.1, 1, 10\}$, so while performance does depend on this hyperparameter, it is not overly sensitive, and values around $1$ are likely to be a good default initialization.

\begin{table}[H]
\begin{tabular}{l|llllll}
\textbf{Penalty parameter($\alpha$)} & \textbf{0} & \textbf{0.01} & \textbf{0.1} & \textbf{1}   & \textbf{10}  & \textbf{100} \\ \hline \Tstrut
Maze                                  & 5.78       & 5.725         & 5.94 $\pm$ 0.22 & 6.53 $\pm$ 0.12 & 6.54 $\pm$ 0.15 & 5.7          \\
Heist                                 & 3.3        & 3.4           & 3.2 $\pm$ 0.6   & 3.73 $\pm$ 0.45 & 3.65 $\pm$ 0.5  & 3.15         \\
Bigfish                               & 1.57       & 2.35          & 2.85 $\pm$ 0.64 & 4.16 $\pm$ 0.42 & 3.30 $\pm$ 0.38 & 1.21         \\
Dodgeball                             & 0.65       & 0.94          & 0.78 $\pm$ 0.2  & 1.69 $\pm$ 0.18 & 1.42 $\pm$ 0.4  & 1.64        
\end{tabular}
\end{table}

\subsection{LEEP and \textit{implicit} partial observability}
\label{appendix:recurrent_context}

One common confusion that may arise is that LEEP seeks to overcome partial observability of the contexts, as is done for dynamics generalization in POMDPs (e.g. \citep[][]{Lee2020ContextawareDM}). This is not the case. Works on dynamics generalization in POMDPs assume that contexts in the true underlying environment are partially observable (e.g. friction coefficients unobserved by a robot without the proper sensors), and the aim to infer this context using memory. In the epistemic POMDP, \textit{the context is not partially observable}; rather, what is partially observable is how the system dynamics will behave for any provided context, capturing the agent’s epistemic uncertainty that stems from the limited training contexts.  

We conducted a didactic experiment on Procgen to empirically support the claim that the partial observability modelled by dynamics generalization methods \citet{Lee2020ContextawareDM} does not replace explicit handling of epistemic uncertainty provided by our method (since this is a different problem). We train a recurrent context encoder that takes in the trajectory seen so far and predicts the identity of the training level. The last hidden layer of this encoder is taken as a “context vector” and fed in as input into a policy alongside the original state, creating an adaptive recurrent policy since this context vector can change through a trajectory. We tested this model on our four Procgen tasks, and made two observations. First, the learned policy, despite being recurrent, does not achieve higher test-time performance than PPO. This is not surprising, because the task is fully observed at training-time. Second, the learned context encoder is able to predict the identity of the training level with~$> 99\%$ accuracy; that is, the contexts are fully observed and so mechanisms that try to predict the context are unlikely to provide benefit. 

The issue is that recurrence and adaptation by themselves are not sufficient to ensure high generalization performance; rather they must be combined with the appropriate model of partial observability that captures the agent’s epistemic uncertainty (for LEEP, by statistical bootstrapping on the set of training contexts) to achieve good generalization.

\begin{table}[H]
\begin{tabular}{l|l|l|l|l}
\textbf{Test Return after 25M steps}        &\textbf{ Maze}         & \textbf{Heist}        & \textbf{Bigfish}      & \textbf{Dodgeball}    \\ \hline \Tstrut
PPO                                & 5.11 $\pm$ 0.24 & 2.84 $\pm$ 0.46 & 3.89 $\pm$ 1.64 & 1.68 $\pm$ 0.33 \\ 
PPO + Recurrent Context Encoder & 5.25 $\pm$ 0.5  & 2.83 $\pm$ 1.04 & 2.74 $\pm$ 1.1  & 1.57 $\pm$ 0.3  \\ 
LEEP                               & 6.53 $\pm$ 0.12 & 3.73 $\pm$ 0.45 & 4.16 $\pm$ 0.42 & 1.69 $\pm$ 0.18 \\ 
\end{tabular}
\end{table}

\end{subappendices}
\clearpage

\part{Auction Design}
\label{part:auctions}
\chapter{A Permutation-Equivariant Neural Network Architecture For
Auction Design}
\label{chap:EquivariantNet}
Designing an incentive compatible auction that maximizes expected revenue is a central problem in Auction Design. Theoretical approaches to the problem have hit some limits in the past decades and analytical solutions are known for only a few simple settings. Computational approaches to the problem through the use of LPs  have their own set of limitations. Building on the success of deep learning, a new approach was recently proposed by \cite{dutting2017optimal} in which the auction is modeled by a feed-forward neural network and the design problem is framed as a learning problem. The neural architectures used in that work are general purpose and do not take advantage of any of the symmetries the problem could present, such as permutation equivariance. In this chapter, we consider auction design problems that have permutation-equivariant symmetry and construct a neural architecture that is capable of perfectly recovering the permutation-equivariant optimal mechanism, which we show is not possible with the previous architecture. We demonstrate that permutation-equivariant architectures are not only capable of recovering previous results, they also have better generalization properties.
  
\section{Introduction}

%https://www.cs.cornell.edu/home/kleinber/networks-book/networks-book-ch09.pdf

Designing truthful auctions is one of the core problems that arise in economics. Concrete examples of auctions include sales of treasury bills, art sales by Christie’s or Google Ads. Following seminal work of Vickrey~\citep{vickrey1961counterspeculation} and Myerson~\citep{myerson1981optimal}, auctions are typically studied in the \emph{independent private valuations} model: each bidder has a valuation function over items, and their payoff depends only on the items they receive. Moreover, the auctioneer knows aggregate information about the population that each bidder comes from, modeled as a distribution over valuation functions, but does not know precisely each bidder's valuation. Auction design is challenging since the valuations are private and bidders need to be encouraged to report their valuations truthfully. The auctioneer aims at designing an incentive compatible auction that maximizes revenue.   

While auction design has existed as a sub-field of economic theory for several decades, complete characterizations of the optimal auction only exist for a few settings. Myerson resolved the optimal auction design problem when there is a single item for sale \citep{myerson1981optimal}.  However, the problem is not completely understood even in the extremely simple setting with just a single bidder and two items. While there have been some partial characterizations~\citep{manelli2006bundling, manelli2010bayesian,Pavlov11, WangT14, daskalakis2017strong}, and algorithmic solutions with provable guarantees~\citep{Alaei11, AlaeiFHHM12, AlaeiFHH13, CaiDW12a, CaiDW12b}, neither the analytic nor algorithmic approach currently appears tractable for seemingly small instances.

%While Myerson's Nobel prize-winning work provides a clean characterization of the single-item optimum \citep{myerson1981optimal}, optimal \emph{multi-item} auctions provably suffer from numerous formal measures of intractability (including computational intractability, high description complexity, and non-monotonicity, and others)~\citep{DaskalakisDT14, ChenDPSY14, ChenDOPSY15, ChenMPY18, HartR15, Thanassoulis04}.

 Another line of work to confront this theoretical hurdle consists in building automated methods to find the optimal auction. Early works \citep{conitzer2002complexity,conitzer2004self} framed the problem as a linear program. However, this approach suffers from severe scalability issues as the number of constraints and variables is exponential in the number of bidders and items \citep{guo2010computationally}. Later, \citet{sandholm2015automated} designed algorithms to find the optimal auction. While scalable, they are however limited to specific classes of auctions known to be incentive compatible. 
 
 A more recent research direction consists in building deep learning architectures that design auctions from samples of bidder valuations. \citet{dutting2017optimal} proposed RegretNet, a feed-forward architecture to find near-optimal results in several known multi-item settings and obtain new mechanisms in unknown cases. This architecture however is not data efficient and can require a large number of valuation samples to learn an optimal auction in some cases. This inefficiency is not specific to RegretNet but is characteristic of neural network architectures that do not incorporate any inductive bias.
%  While this seem reasonable in some settings, it amounts to overfitting when the valuations belong to a small interval. Lastly, RegretNet is a feedforward architecture and its search space may be too huge in some settings. This may imply longer training phase and potentially not finding the optimum.
 
 %\cite{tacchetti2019neural} proposed a new architecture that restricts the learned mechanism to the class of VCG auctions \citep{vickrey1961counterspeculation}. While these works apply to more general settings compared to early work, there are still drawbacks. First, in many cases, the optimal mechanism is known to be not VCG-based, so the architecture in \cite{tacchetti2019neural} simply will not find it; counter-examples can be found in  \cite{manelli2006bundling,DaskalakisDT17,Thanassoulis04,Pavlov11}.
 %
 % Matt
 In this chapter, we build a deep learning architecture for \textit{multi-bidder symmetric auctions}. These are auctions which are invariant to relabeling the items or bidders. More specifically, such auctions are \emph{anonymous} (in that they can be executed without any information about the bidders, or labeling them) and \emph{item-symmetric} (in that it only matters what bids are made for an item, and not its a priori label). 
 
 It is now well-known that when bidders come from the same population that the optimal auction itself is anonymous. Similarly, if items are \emph{a priori indistinguishable} (e.g. different colors of the same car --- individuals certainly value a red vs.~blue car differently, but there is nothing objectively more/less valuable about a red vs.~blue car), the optimal auction is itself item-symmetric. In such settings, our approach will approach the true optimum 
%  (but with better generalization, and 
in a way which retains this structure (see Contributions below). Even without these conditions, the optimal auction is often symmetric anyway: for example, ``bundling together'' (the auction which allows bidders to pay a fixed price for all items, or receive nothing) is item-symmetric, and is often optimal even when the items are a priori distinguishable. 
 
Beyond their frequent optimality, such auctions are desirable objects of study \emph{even when they are suboptimal}. For example, seminal work of Hartline and Roughgarden which pioneered the study of ``simple vs.~optimal auctions'' analyzes the approximation guarantees achievable by anonymous auctions~\cite{HartlineR09}, and exciting recent work continues to improve these guarantees~\cite{AlaeiHNPY15,JinLTX19,JinLQTX19}. Similarly, \citet{daskalakis2012symmetries} develop algorithms for item-symmetric instances, and exciting recent work show how to leverage item-symmetric to achieve near-optimal auctions in completely general settings~\citep{kothari2019approximation}. To summarize: symmetric auctions are known to be optimal in many settings of interest (even those which are not themselves symmetric). Even in settings where they are not optimal, they are known to yield near-optimal auctions. And even when they are only approximately optimal, seminal work has identified them as important objects of study owing to their simplicity. In modern discussion of auctions, they are also desirable due to fairness considerations.

While applying existing feed-forward architectures as RegretNet to symmetric auctions is possible, we show in~Section~\ref{sec:NN_arch_equiv} that RegretNet struggles to find symmetric auctions, \emph{even when the optimum is symmetric}. To be clear, the architecture's performance is indeed quite close to optimal, but the resulting auction is not ``close to symmetric''. This chapter proposes an architecture that outputs a symmetric auction symmetry by design.

 \subsection*{Contributions}
 \label{sec:contributions}

This chapter identifies three drawbacks from using the RegretNet architecture when learning with symmetric auctions. First, RegretNet is incapable of finding symmetric auctions when the optimal mechanism is known to be symmetric. Second, RegretNet is sample inefficient, which is not surprising since the architecture does not incorporate any inductive bias. Third, RegretNet is incapable of generalizing to settings with a different number of bidders of objects. In fact, by construction, the solution found by RegretNet can only be evaluated on settings with exactly the same number of bidders and objects of the setting it was trained on.

We address these limitations by proposing a new architecture EquivariantNet, that outputs symmetric auctions. EquivariantNet is an adaption of the deep sets architecture \citep{hartford2018deep} to symmetric auctions.  This architecture is parameter-efficient and is able to  recover some of the optimal results in the symmetric auctions literature.  Our approach outlines three important benefits: 
\begin{itemize}
    \item[--] \textit{Symmetry}: our architecture outputs a symmetric auction by design. It is immune to permutation-sensitivity as defined in Section~\ref{sec:ffpe} which is related to fairness.
    \item[--] \textit{Sample generalization}: Because we use domain knowledge, 
    our architecture converges to the optimum with fewer valuation samples.
    \item[--] \textit{Out-of-setting generalization}: Our architecture does not require hard-coding the number of bidders or items during training --- training our architecture on instances with $n$ bidders and $m$ items produces a well-defined auction even for instances with $n'$ bidders and $m'$ items. Somewhat surprisingly, we show in~\ref{sec:num_exp} some examples where our architecture trained on $1$ bidder with $5$ items generalizes well even to $1$ bidder and $m$ items, for any $m \in \{2,10\}$.
\end{itemize}
 We highlight that the novelty of this work is not to show that a new architecture is a viable alternative to RegretNet.  Instead we are solving three fundamental limitations we identified for the RegretNet architecture. These three problems are not easy to solve in principle, it is surprising that a change of architecture solves all of them in the context of symmetric auctions. We would also like to emphasize that both RegretNet and EquivariantNet are capable of learning auction with near optimal revenue and negligible regret. It is not possible to significantly outperform RegretNet on these aspects. The way we improve over RegretNet is by having better sample efficiency, out-of-setting generalization and by ensuring that our solutions are exactly equivariant. 

% rather on adapting an existing method to return symmetric auctions and providing new understanding on these auctions. Our architecture can be seen as a tool for researchers to confirm or refute hypotheses and for this reason, we run our experiments on synthetic data.

The chapter decomposes as follows. Section~\ref{sec:setting_equiv} introduces the standard notions of auction design. Section~\ref{sec:NN_arch_equiv} presents our permutation-equivariant architecture to encode symmetric auctions. Finally, Section~\ref{sec:num_exp} presents numerical evidence for the effectiveness of our approach.

\subsection*{Related work} 

\iffalse

\paragraph{Symmetries in auctions.} Following the Wilson Doctrine~\cite{Wilson85}, auctioneers aim to design auctions which are as simple as possible, while sacrificing as little in optimality as possible. One way in which an auction can be simple is if it is \emph{anonymous} (invariant under permutations of bidders), and a study of such auctions, even in settings in which they are provably suboptimal, dates back at least to seminal work of Hartline and Roughgarden~\citep{HartlineR09, AlaeiHNPY15, JinLQ19, JinLTX19, JinLQTX19}. Another way an auction can be simple is if it is simple to describe the options to buyers, perhaps because the auction is invariant under permutations of items (e.g. ``pay $p_k$ to choose any set of $k$ items'' vs. ``pay $p_S$ to receive exactly set $S$ of items''). This property has been exploited to design faster algorithms in the case when item distributions are themselves invariant under permutations of items, but also in general~\citep{daskalakis2012symmetries, kothari2019approximation}. We exploit these latter characterizations to build our neural network architecture. 

\fi

\paragraph{Auction design and machine learning.} Machine learning and computational learning theory have been used in several ways to design auctions from samples of bidder valuations. Some works have focused sample complexity results for designing optimal revenue-maximizing auctions. This has been established in single-parameter settings \citep{DhangwatnotaiRY15, cole2014sample,morgenstern2015pseudo, medina2014learning,huang2018making, DevanurHP16, HartlineT19,RoughgardenS16, GonczarowskiN17, GuoHZ19},  multi-item auctions \citep{dughmi2014sampling, GonczarowskiW18}, combinatorial auctions \citep{balcan2016sample,morgenstern2016learning,syrgkanis2017sample} and allocation mechanisms \citep{narasimhan2016general}. Machine learning has also been used to optimize different aspects of mechanisms \citep{lahaie2011kernel,dutting2015payment}. All these aforementioned differ from ours as we resort to deep learning for finding optimal auctions. 

\paragraph{Auction design and deep learning.} While \cite{dutting2017optimal} is the first paper to design auctions through deep learning, several other paper followed-up this work. \cite{feng2018deep} extended it to budget constrained bidders, \cite{golowich2018deep} to the facility location problem. \cite{tacchetti2019neural} built architectures based on the Vickrey-Clarke-Groves auctions. Recently, \cite{shen2019automated} and \cite{dutting2017optimal} proposed architectures that \textit{exactly} satisfy incentive compatibility but are specific to \textit{single-bidder} settings. In this work, we aim at \textit{multi-bidder} settings and build permutation-equivariant networks that return nearly incentive compatibility symmetric auctions. 
% Lastly, we would like to mention that a concurrent work \citep{dutting2017optimal} similarly imposed symmetries in the architecture. However, they did not provide any details on their approach nor exhibited their numerical performance.

\section{Symmetries and learning problem in auction design}\label{sec:setting_equiv}

We review the framework of auction design and the problem of finding truthful mechanisms. We then present symmetric auctions and similarly to \cite{dutting2017optimal}, frame auction design as a learning problem.

\subsection{Auction design and symmetries}

\paragraph{Auction design.} We consider the setting of additive auctions with $n$ bidders with $N=\{1,\dots,n\}$ and $m$ items with $M=\{1,\dots,m\}.$ Each bidder $i$ is has value $v_{ij}$ for item $j$, and values the set $S$ of items at $\sum_{j \in S} v_{ij}$. Such valuations are called \emph{additive}, and are perhaps the most well-studied valuations in multi-item auction design~\citep{HartN12,HartN13, LiY13, BabaioffILW14, DaskalakisDT14, HartR15, CaiDW16,daskalakis2017strong, BeyhaghiW19}.

The designer does not know the full valuation profile $V = (v_{ij})_{i\in N, j\in M}$, but just a distribution from which they are drawn. Specifically, the valuation vector of bidder $i$ for each of the $m$ items $\vec{v}_i=(v_{i1}, \dots, v_{im})$ is drawn from a distribution $D_i$ over $\mathbb{R}^m$  (and then, $V$ is drawn from $D:= \times_i D_i$). The designer asks the bidders to report their valuations (potentially untruthfully), then decides on an allocation of items to the bidders and charges a payment to them. 

%In what follows, we abuse notation $V\in D$ to mean that $V\in \mbox{Supp}(D)$, where $\mbox{Supp}(D)$ denotes the support of the distribution $D$.
 
\begin{definition}
 An auction is a pair $(g,p)$ consisting of a randomized allocation rule $g=(g_1,\dots,g_n)$ where $g_i\colon \mathbb{R}^{n\times m}\rightarrow [0,1]^{m}$ such that for all $V$, and all $j$, $\sum_i (g_i(V))_j\leq 1$ and payment rules $p=(p_1,\dots,p_n)$ where $p_i\colon \mathbb{R}^{n\times m}\rightarrow \mathbb{R}_{\geq 0}$ . 
 \end{definition}

 Given reported bids $B=(b_{ij})_{i\in N, j\in M}$, the auction computes an allocation probability $g(B)$ and payments $p(B)$. %$g(B)_{i,j}$ 
 $[g_i(B)]_j$
 is the probability that bidder $i$ received object $j$ and $p_i(B)$ is the price bidder $i$ has to pay to the mechanism.
 In what follows, $\mathcal{M}$ denotes the class of all possible auctions.
  
 \begin{definition}
 The utility of bidder $i$ is defined by $u_i(\vec{v}_i,B)= \sum_{j=1}^m [g_i(B)]_j v_{ij} -p_i(B).$ 
\end{definition}

% are strategic and
Bidders seek to maximize their utility and may report bids that are different from their valuations. Let $V_{-i}$ be the valuation profile without element $\vec{v}_i$, similarly for $B_{-i}$ and $D_{-i}=\times_{j\neq i}D_j$. We aim at auctions that invite bidders to bid their true valuations through the notion of incentive compatibility.
%denote the possible valuation profiles of bidders other than bidder $i.$

\begin{definition}\label{def:DSIC}
An auction $(g,p)$ is \textit{dominant strategy incentive compatible} (DSIC) if each bidder's utility is maximized by reporting truthfully no matter what the other bidders report. For every bidder $i,$ valuation $\vec{v}_i \in D_i$, bid $\vec{b}_i\hspace{.02cm}'\in D_i$ and bids $B_{-i}\in D_{-i}$, \; $u_i(\vec{v}_i,(\vec{v}_i,B_{-i}))\geq u_i(\vec{v}_i,(\vec{b}_i\hspace{.02cm}',B_{-i})).$ 
\end{definition}
\noindent Additionally, we aim at auctions where each bidder receives a non-negative utility.
\begin{definition}\label{def:IR_equiv}
An auction is \textit{individually rational} (IR) if for all $i\in N, \; \vec{v}_i\in D_i$ and $B_{-i}\in D_{-i},$
\begin{equation}\label{eq:IR_eq_equiv}\tag{IR}
   u_i(\vec{v}_i,(\vec{v}_i,B_{-i}))\geq 0.  
\end{equation}
\end{definition}

In a DSIC auction, the bidders have the incentive to truthfully report their valuations and therefore, the revenue on valuation profile $V$ is defined as $\sum_{i=1}^n p_i(V).$ Optimal auction design aims at finding a DSIC auction that maximizes the expected revenue $rev:=\mathbb{E}_{V\sim D}[\sum_{i=1}^n p_i(V)]$. 
  
\paragraph{Linear program.} We frame the problem of optimal auction design as an optimization problem where we seek an auction that minimizes the negated expected revenue among all IR and DSIC auctions. Since there is no known characterization of DSIC mechanisms in the multi-bidder setting, we resort to the relaxed notion of \textit{ex-post regret}. It measures the extent to which an auction violates DSIC, for each bidder.

\begin{definition}
The ex-post regret for a bidder $i$ is the maximum increase in his utility when  considering all his possible bids and fixing the bids of others. For a valuation profile $V$, the ex-post regret for a bidder $i$ is $rgt_{i}(V)=\max_{\vec{v}_i\hspace{.02cm}'\in \mathbb{R}^m} u_i(\vec{v}_i;(\vec{v}_i\hspace{.02cm}',V_{-i}))-u_i(\vec{v}_i;(\vec{v}_i,V_{-i})).$ In particular, DSIC is equivalent to 
\begin{equation}\label{eq:regretDSIC_equiv}\tag{IC} 
    rgt_{i}(V)=0, \;  \forall i \in N.
\end{equation}
\end{definition}

Therefore, by setting~\eqref{eq:regretDSIC_equiv} and \eqref{eq:IR_eq_equiv} as constraints, finding an optimal auction is equivalent to the following linear program
\begin{equation*}\tag{LP}\label{eq:exact_prob_equiv}
\begin{aligned}
\hspace{-.41cm}\underset{(g,p)\in \mathcal{M} }{\text{min}}
  - \mathbb{E}_{V\sim D}\left[\sum_{i=1}^n p_i(V)\right] \quad \text{s.t.} \quad 
&   rgt_{i}(V)=0, \hspace{1.55cm} \forall i \in N,\;  \forall V \in D, \\
& u_i(\vec{v}_i,(\vec{v}_i,B_{-i}))\geq 0, \quad  \forall i\in N,\; \vec{v}_i\in D_i, B_{-i}\in D_{-i}.
\end{aligned}
\end{equation*}

\iffalse

\begin{equation*}\tag{LP}\label{eq:exact_prob}
\begin{aligned}
& \underset{(g,p)\in \mathcal{M} }{\text{min}}
& &  - \mathbb{E}_{V\sim D}\left[\sum_{i=1}^n p_i(V)\right]\\
& \text{s.t.}
& &  rgt_{i}(V)=0, \hspace{1.65cm} \forall i \in N,\;  \forall V \in D, \\
&&&  u_i(\vec{v}_i,(\vec{v}_i,B_{-i}))\geq 0, \quad  \forall i\in N,\; \vec{v}_i\in D_i, B_{-i}\in D_{-i}.
\end{aligned}
\end{equation*}

\fi

\paragraph{Symmetric auctions. } Equation~\ref{eq:exact_prob_equiv} is intractable due to the exponential number of constraints. However, in the setting of \textit{symmetric} auctions, it is possible to reduce the search space of the problem as shown in \autoref{prop:equiv_sol}. We first define  the notions of bidder- and item-symmetries. 
\begin{definition}
The valuation distribution $D$ is bidder-symmetric if for any permutation of the bidders ${\varphi_b\colon N\rightarrow N},$ the permuted distribution ${D_{\varphi_b}:=D_{\varphi_b(1)}\times \dots \times D_{\varphi_b(n)}}$ satisfies: $D_{\varphi_b}=D$.
\end{definition}
Bidder-symmetry intuitively means that the bidders are a priori indistinguishable (although individual bidders will be different). This holds for instance in auctions where the identity of the bidders is anonymous, or if $D_i = D_j$ for all $i,j$ (bidders are i.i.d.).
% For example, perhaps all the designer knows is that each bidder is a consumer in the US. Then each consumer is drawn from the same distribution, although the particular consumers drawn are indeed distinct.

\begin{definition}
Bidder $i$'s valuation distribution $D_i$ is item-symmetric if for any items $x_1,\dots,x_m$ and any permutation $\varphi_o\colon M\rightarrow M,$ $D_i(x_{\varphi_o(1)},\dots,x_{\varphi_o(m)} )=D_i(x_1,\dots,x_m )$. 
\end{definition}

Intuitively, item-symmetry means that the items are also indistinguishable but not identical. It holds when the distributions over the items are i.i.d. but this is not a necessary condition. Indeed, the distribution $\{(a,b,c)\in \mathcal{U}(0,1)^{\otimes 3}: a+b+c=1\}$ is not i.i.d. but is item-symmetric.

\begin{definition}\label{ass:bid_it_sym}
An auction is symmetric if its valuation distributions are bidder- and item-symmetric.
\end{definition}

We now define the notion of permutation-equivariance that is important in symmetric auctions.

\begin{definition}
The functions $g$ and $p$ are permutation-equivariant if for any two permutation matrices $\Pi_{n}\in \{0,1\}^{n\times n}$ and $\Pi_{m}\in \{0,1\}^{m\times m}$, and any valuation matrix $V$,  we have $ g(\Pi_{n}\,V\,\Pi_{m})=\Pi_{n}\, g(V)\, \Pi_{m}$ and $p(\Pi_{n}\,V\,\Pi_{m})=\Pi_{n} \, p(V)$.
\end{definition}

\begin{theorem}\label{prop:equiv_sol}
When the auction is symmetric, there exists an optimal solution to \eqref{eq:exact_prob} that is permutation-equivariant.
\end{theorem}

\autoref{prop:equiv_sol} is originally proved in \cite{daskalakis2012symmetries} and its proof is reminded in \autoref{app:thm} for completeness. It encourages to reduce the search space in \eqref{eq:exact_prob_equiv} by only optimizing over permutation-equivariant allocations and payments. We implement this idea in Section~\ref{sec:NN_arch_equiv} where we build equivariant neural network architectures. Before, we frame auction design as a learning problem. 

\subsection{Auction design as a learning problem}\label{par:auction}
Similarly to \cite{dutting2017optimal}, we formulate auction design as a learning problem. We learn a parametric set of auctions $(g^w,p^w)$ where $w\in \mathbb{R}^d$  parameters and $d\in\mathbb{N}$. Directly solving \eqref{eq:exact_prob_equiv} is challenging in practice. Indeed, the auctioneer must have access to the bidder valuations which are unavailable to her. Since she has access to the valuation distribution, we relax~\eqref{eq:exact_prob_equiv} and replace the IC constraint for all $V\in D$ by the expected constraint $\mathbb{E}_{V\sim D}[rgt_i(V)]=0$ for all $i\in N.$. In practice, the expectation terms are computed by sampling $L$ bidder valuation profiles drawn i.i.d. from $D$. The empirical ex-post regret for bidder $i$ is
\begin{align}
      \widehat{rgt}_i(w)  = \frac{1}{L}\sum_{\ell=1}^L \max_{\vec{v}_i\hspace{.02cm}'\in \mathbb{R}^m} u_i^w(\vec{v}_i^{(\ell)};(\vec{v}_i\hspace{.02cm}',V_{-i}^{(\ell)}))-u_i(\vec{v}_i^{(\ell)};(\vec{v}_i^{(\ell)},V_{-i}^{(\ell)})), \label{eq:emp_reg}\tag{$\hat{R}$}
\end{align}
 where $u_i^w(\vec{v}_i,B):=\sum_{j=1}^m [g_i^w(B)]_j v_{ij} -p_i^w(B)$ is the utility of bidder $i$ under the parametric set of auctions  $(g^w,p^w)$. Therefore, the learning formulation of \eqref{eq:exact_prob_equiv} is
 \begin{equation}\label{eq:empirical_prob_equiv}\tag{$\widehat{\mathrm{LP}}$}
\begin{aligned}
\underset{w\in \mathbb{R}^d}{\text{min}}
-\frac{1}{L}\sum_{\ell=1}^L \sum_{i=1}^n p_i^w(V^{(\ell)}) \quad \text{s.t.}\quad \widehat{rgt}_i(w)=0,\; \forall i \in N. 
\end{aligned}
\end{equation}

\cite{dutting2017optimal} justify the validity of this reduction from \eqref{eq:exact_prob_equiv} to \eqref{eq:empirical_prob_equiv} by showing that the gap between the expected regret and the empirical regret is small as the number of samples increases. Additionally to being DSIC, the auction must satisfy IR. The learning problem \eqref{eq:empirical_prob_equiv} does not ensure this but we will show how to include this requirement in the architecture in \autoref{sec:NN_arch_equiv}.

\section{A Permutation-equivariant neural network architecture}
\label{sec:NN_arch_equiv}

We first show that feed-forward architectures as RegretNet \citep{dutting2017optimal} may struggle to find a symmetric solution in auctions where the optimal solution is known to be symmetric.
We then describe our neural network architecture, EquivariantNet that learns symmetric auctions. EquivariantNet is build using exchangeable matrix layers \citep{hartford2018deep}. 

\subsection{Feed-forward nets and permutation-equivariance}\label{sec:ffpe}

In the following experiments we use the RegretNet architecture  with the exact same training procedure and parameters as found in \cite{dutting2017optimal} .

\paragraph{Permutation-sensitivity.}Given $L$ bidders valuation samples ${\{B^{(1)},\dots,B^{(L)}\} \in \mathbb{R}^{n\times m}}$,  we generate for each bid matrix $B^{(\ell)} $ all its possible permutations 
${B_{\Pi_n,\Pi_m}^{(\ell)}:=\Pi_{n}B^{(\ell)}\Pi_{m}},$ where ${\Pi_{n}\in \{0,1\}^{n\times n}}$ and ${\Pi_{m}\in \{0,1\}^{m\times m}}$ are permutation matrices.  We then compute the revenue for each one of these bid matrices and obtain a revenue matrix ${R\in \mathbb{R}^{n!m!\times L}}$. Finally, we compute $h_R \in \mathbb{R}^{L}$ where $[h_R]_j = \max_{i\in [n!m!]} R_{ij}-\min_{i\in [n!m!]}R_{ij}$. The distribution given by the entries of $h_R$ is a measure of how close the auction is to permutation-equivariance.
A symmetric mechanism satisfies $h_R = (0,\dots,0)^{\top}$. 
Our numerical investigation considers the following auction settings:
 \begin{itemize}
     \item[--] (I) One bidder and two items, the item values are drawn from $\mathcal{U}[0,1].$ Optimal revenue: 0.55 \cite{manelli2006bundling}.
    \item[--](II) Four bidders and five items, the item values are drawn from $\mathcal{U}[0,1].$ 
 \end{itemize}

\begin{figure}[h]
\begin{minipage}{.5\textwidth}
 \centering
  \includegraphics[width=1.05\linewidth]{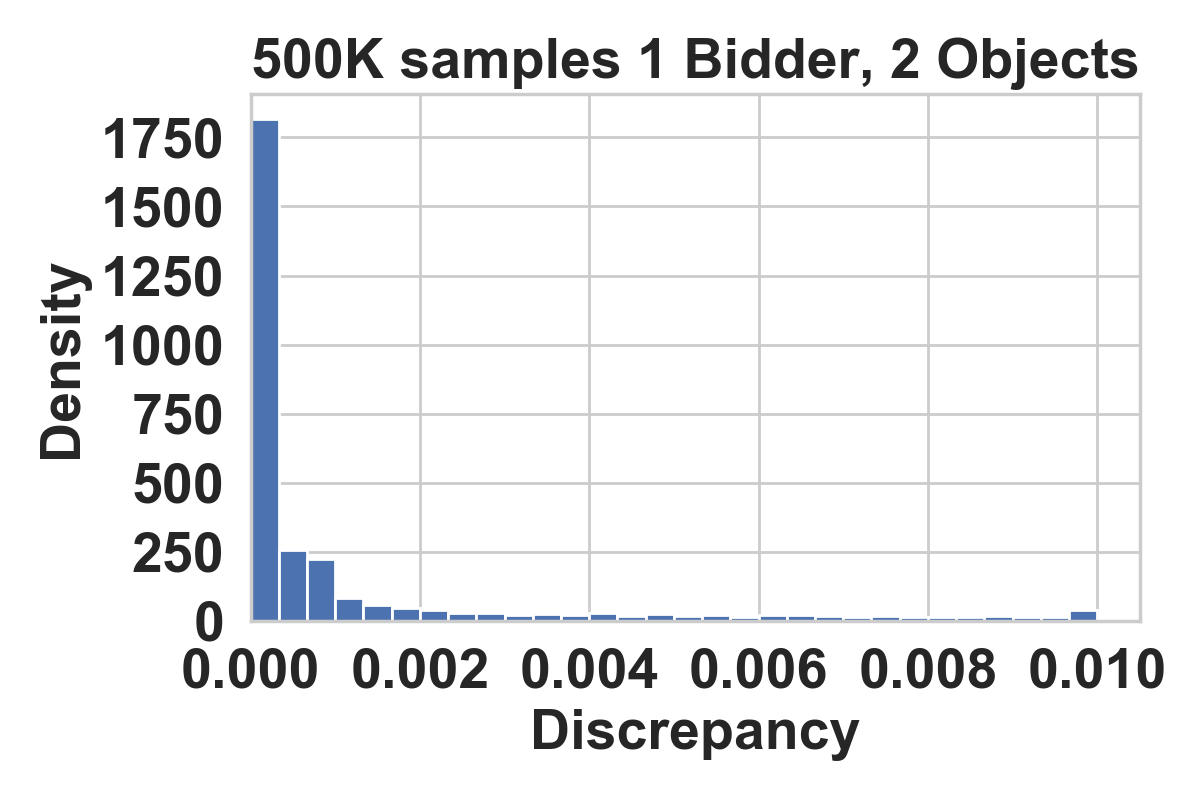}%.64
\vspace*{-.3cm}
\captionsetup{labelformat=empty}
  \caption{(a)}%\hspace{-2.9cm}
\end{minipage}%
\begin{minipage}{.5\textwidth}
 \centering
%\hspace{-2.9cm}
  \includegraphics[width=1.05\linewidth]{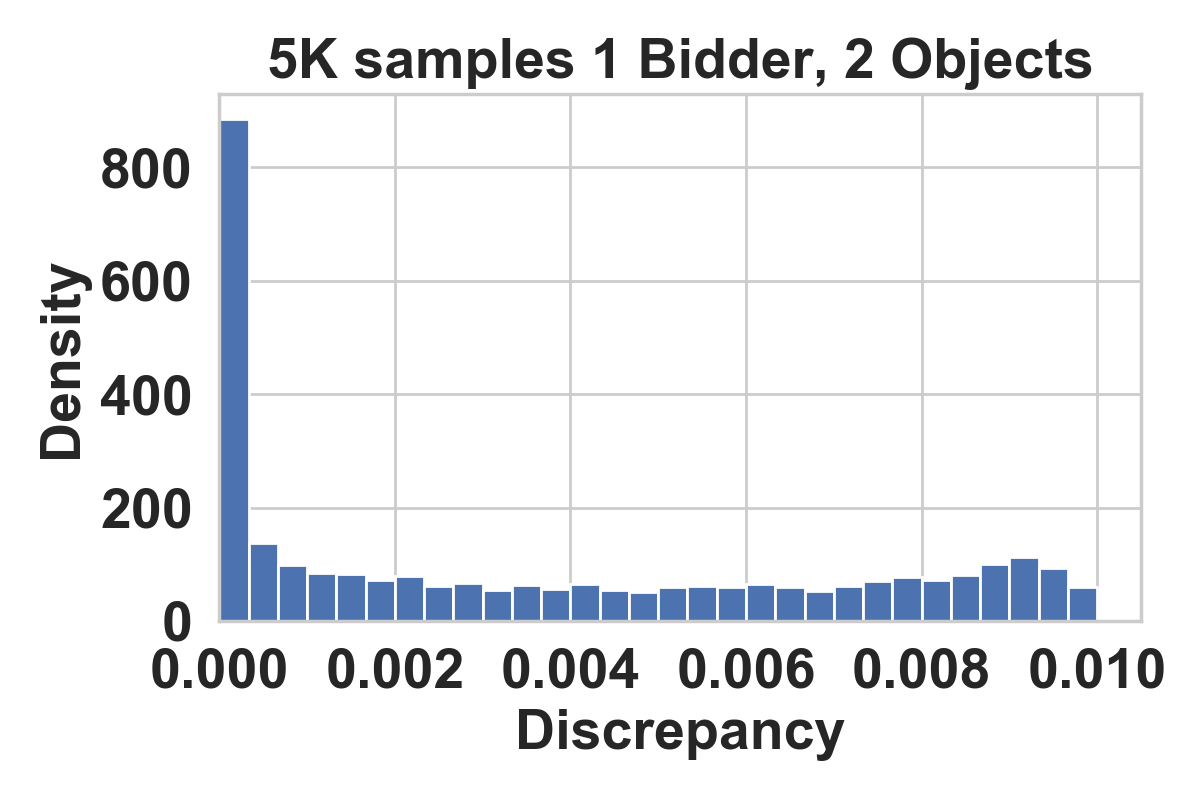}%.64
  \vspace*{-.3cm}
  \captionsetup{labelformat=empty}
\caption{(b)}%\hspace{-7.8cm}
\end{minipage}
\begin{minipage}{.5\textwidth}
%\hspace{-.4cm}
 \centering
  \includegraphics[width=1.05\linewidth]{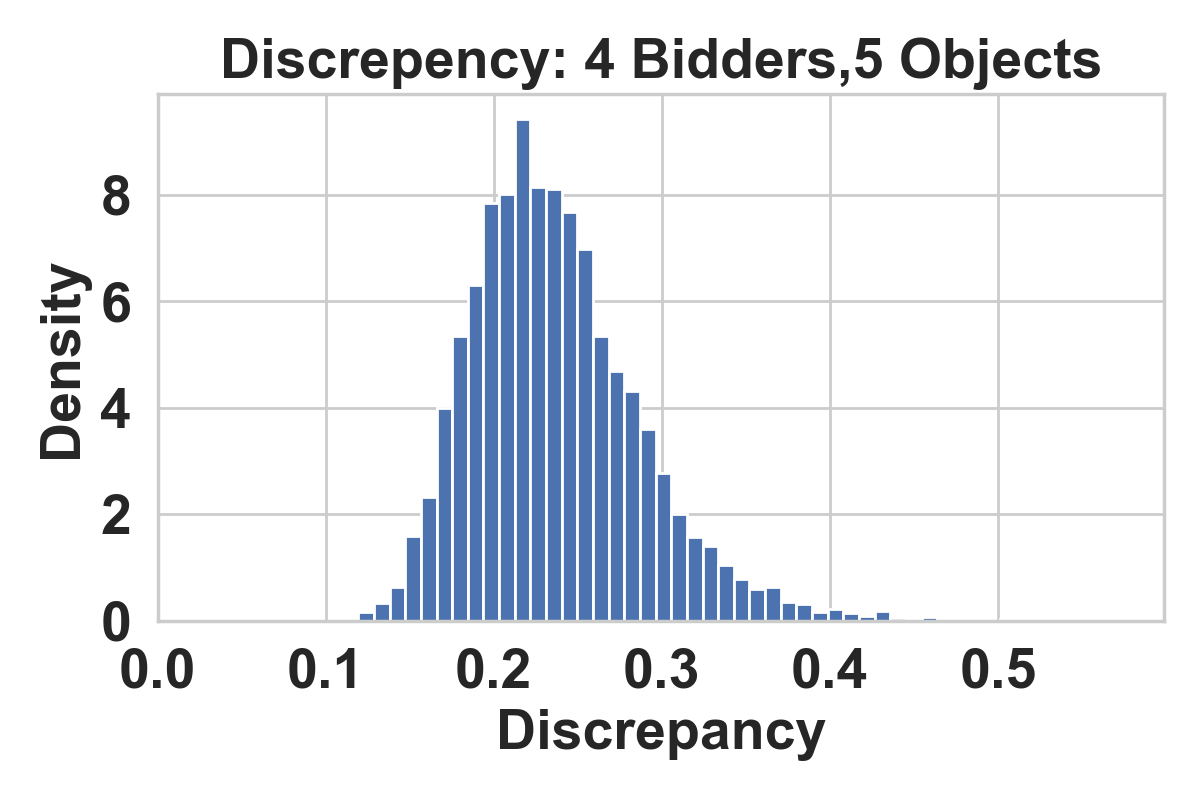}
\vspace*{-.6cm}
\captionsetup{labelformat=empty}
  \caption{(c)}
\end{minipage}%
\begin{minipage}{.5\textwidth}
%\hspace{-2.9cm}
 \centering
  \includegraphics[width=1.05\linewidth]{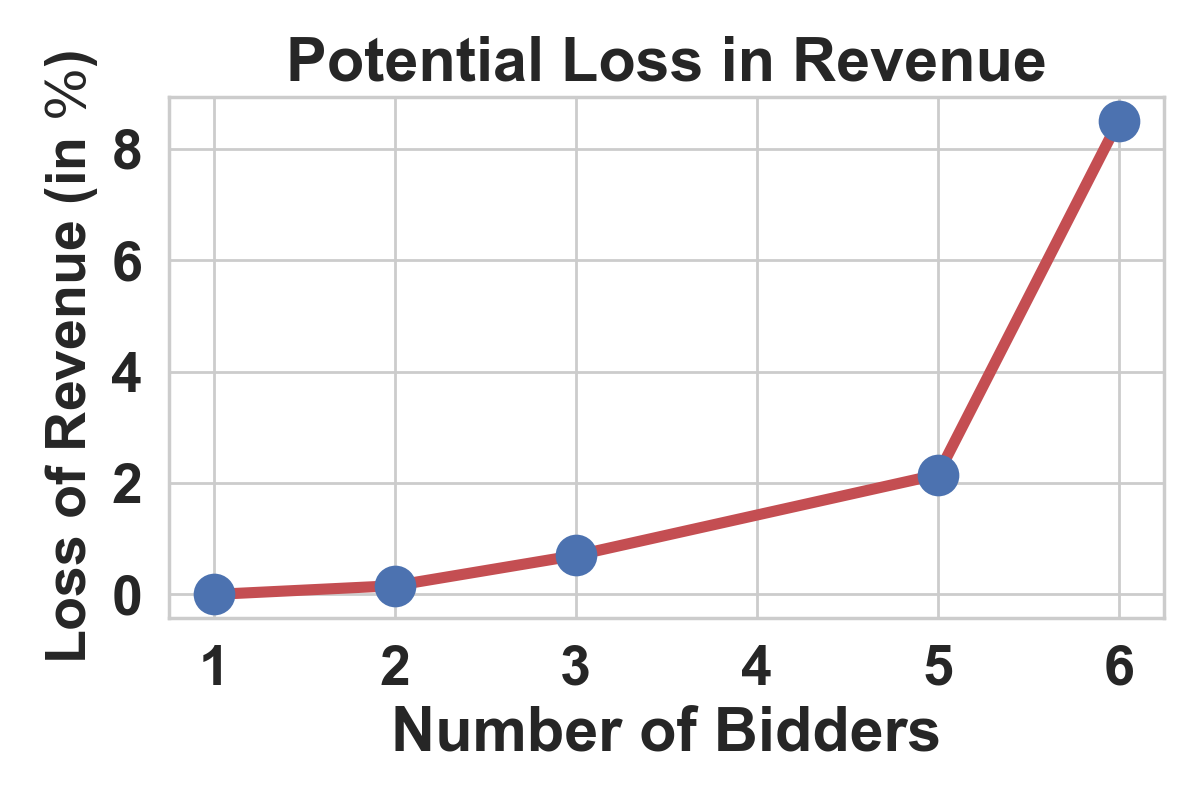}
  \vspace*{-.6cm}
  \captionsetup{labelformat=empty}
\caption{(d)}%\hspace{-7.8cm}
\end{minipage}
\setcounter{figure}{0}    
\caption{(a)-(b): Distribution $h_R$ when varying the number of training samples (a) 500 000 (b) 5000 samples. (c): Histogram of the distribution $h_R$ for setting (II). (d): Maximum revenue loss when varying the number of bidders for setting (III$_n$)}.\label{fig:PE_FF}
\end{figure}

\autoref{fig:PE_FF}~(a)-(b) presents the distribution of $h_R$ of the optimal auction learned for setting (I) when varying the number of samples $L.$ When $L$ is large, the distribution is almost concentrated at zero and therefore the network is almost able to recover the permutation-equivariant solution. When $L$ is small, $h_R$ is less concentrated around zero and therefore, the solution obtained is non permutation-equivariant. \\

As the problem's dimensions increase, this lack of permutation-invariance becomes more dramatic. \autoref{fig:PE_FF}~(c) shows $h_R$ for the optimal auction mechanism learned for setting (II) when trained with $5\cdot 10^5$ samples. Contrary to (I), almost no entry of $h_R$ is located around zero, they are concentrated around between $0.1$ and $0.4$ i.e.\ between 3.8\% and 15\% of the estimated optimal revenue. 

\paragraph{Exploitability.} Finally, to highlight how important equivariant solutions are, we analyze the worst-revenue loss that the auctioneer can incur when the bidders act adversarially. Indeed, since different permutations can result in different revenues for the auction, cooperative bidders could pick among the $n!$ possible permutations of their labels the one that minimized the revenue of the mechanism and present themselves in that order. Instead of getting a revenue of $R_{opt} = \mathbb{E}_{V\sim D}\left[\sum_{i=1}^n p_i(V)\right]$, the auctioneer would get a revenue of $R_{adv} = \mathbb{E}_{V\sim D}\left[\min_{\Pi_n}\{\sum_{i=1}^n p_i(\Pi_n V)\}\right]$. The percentage of revenue loss is given by $ l = 100 \times \frac{R_{opt}-R_{adv}}{R_{opt}}$. We compute $l$ in 
in the following family of settings: 
\begin{itemize}
    \item[--] (III$_n$) $n$ additive bidders and ten item where the item values are drawn from $\mathcal{U}[0,1].$
\end{itemize}
In \autoref{fig:PE_FF}~(d) we plot $l(n)$ the loss in revenue as a function of $n$. As the number of bidders increases, the loss becomes more substantial getting over the 8\% with only 6 bidders.

While it is unlikely that all the bidders will collide and exploit the bidding mechanism in real life, these investigations of permutation sensitivity and exploitability give us a sense of how far the solutions found by RegretNet are from being bidder-symmetric. The underlying real problem with non bidder-symmetric solution has to do with fairness. RegretNet finds mechanisms that do not treat all bidders equally. Their row number in the bid matrix matters, two bidders with the same bids will not get the same treatment.  If the mechanism is equivariant however, all bidders will be treated equally by design, there are no biases or special treatments.  Aiming for symmetric auctions is important and
to this end, we design a permutation-equivariant architecture.

\begin{figure*}[t]
  \centering
  \includegraphics[width=\linewidth]{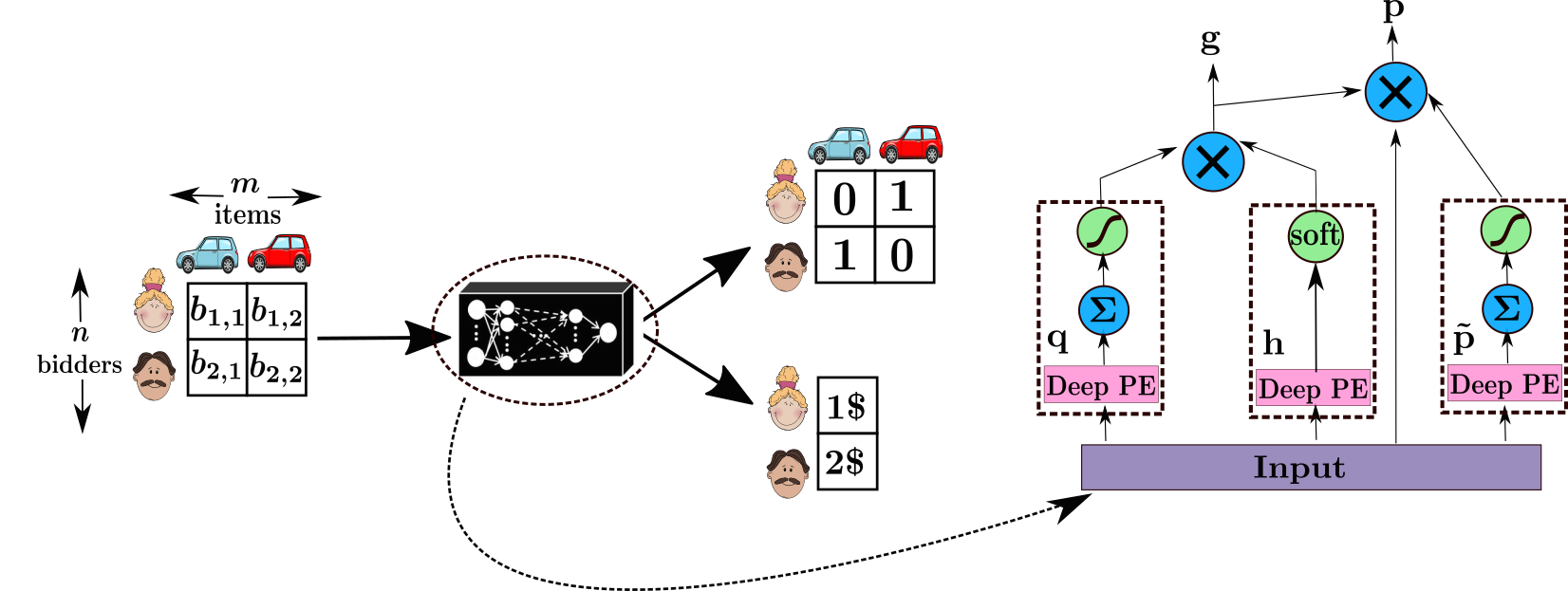}
  \caption{\textit{Left:} Auction design setting. \textit{Right:} EquivariantNet: Deep permutation-equivariant architecture for auction design. Deep PE denotes the deep permutation-equivariant architecture described in \autoref{sec:PE_net}, $\sum$ the sum over rows/columns operations, $\times$ the multiplication operations, soft stands for soft-max and the curve for sigmoid. The network outputs an allocation $g$ and a payment~$p.$
  }\label{fig:arch}
\end{figure*} 

\subsection{Architecture for symmetric auctions (EquivariantNet)}\label{sec:PE_net}
 Our input is a bid matrix $B = (b_{i,j})\in \mathbb{R}^{n\times m}$ drawn from a bidder-symmetric and item-symmetric distribution. We aim at learning a randomized allocation neural network $g^w\colon \mathbb{R}^{n\times m}\rightarrow [0,1]^{n\times m}$ and a payment network $p^w \colon \mathbb{R}^{n\times m}\rightarrow \mathbb{R}_{\geq 0}^n$. The symmetries of the distribution from which $B$ is drawn and \autoref{prop:equiv_sol} motivates us to model $g^w$ and $p^w$ as permutation-equivariant functions. To this end, we use \textit{exchangeable matrix layers} \citep{hartford2018deep} and their definition is reminded in Section~\ref{app:permeq_net}. We now describe the three modules of the allocation and payment networks \autoref{fig:arch}.\\
 
The first network outputs a vector $q^w(B) \in [0,1]^m$ such that entry $q_j^w(B)$ is the probability that item $j$ is allocated to any of the $n$ bidders. The architecture consists of three modules. The first one is a deep permutation-equivariant network with tanh activation functions. The output of that module is a matrix $Q\in \mathbb{R}^{n\times m}$. The second module transforms $Q$ into a vector $\mathbb{R}^{m}$ by taking the average over the rows of $Q$. We finally apply the sigmoid function to the result to ensure that $q^w(B) \in [0,1]^m$. This architecture ensures that $q^w(B)$ is invariant with respect to bidder permutations and equivariant with respect to items permutations.\\

The second network outputs a matrix $h(B) \in [0,1]^{n\times m}$ where $h_{ij}^w$ is the probability that item $j$ is allocated to bidder $i$ conditioned on item $j$ being allocated. The architecture consists of a deep permutation-equivariant network with tanh activation functions followed by softmax activation function so that $\sum_{i=1}^n h_{ij}^w(B) =  1$. This architecture ensures that $q^w$ equivariant with respect to object and bidder permutations.\\

By combining the outputs of $q^w$ and $h^w,$ we compute the allocation function $g^w\colon \mathbb{R}^{n\times m}\rightarrow [0,1]^{n\times m}$ where $g_{ij}(B)$ is the probability that the allocated item $j$ is given to bidder $i$. Indeed, using conditional probabilities, we have $g_{ij}^w(B)=q_j^w(B) h_{ij}^w(B).$ Note that $g^w$ is a permutation-equivariant function.\\

The third network outputs a vector $p(B) \in \mathbb{R}_{\geq 0}^n$ where $\tilde{p}_i^w $ is the fraction of bidder's $i$ utility that she has to pay to the mechanism. Given the allocation function $g^w$, bidder $i$ has to pay an amount  $p_i = \tilde{p}_i(B)\sum_{j=1}^m g_{ij}^w(B)B_{ij}$. 
Individual rationality is ensured by having $\tilde{p}_i\in [0,1]$.
The architecture of $\tilde{p}^w$ is almost similar to the one of $q^w$. Instead of averaging over the rows of the matrix output by the permutation-equivariant architecture, we average over the columns.

\subsection{Optimization and training}

The optimization and training procedure of EquivariantNet is similar to \cite{dutting2017optimal}. For this reason, we briefly mention the outline of this procedure and remind the details in Section~\ref{app:opt_train}. We apply the augmented Lagrangian method to  \eqref{eq:emp_reg}. The Lagrangian with a quadratic penalty is:

\begin{align*}
    \mathcal{L}_{\rho}(w;\lambda)&=-\frac{1}{L}\sum_{\ell=1}^L \sum_{i\in N} p_i^w(V^{(\ell)})+\sum_{i\in N}\lambda_i\widehat{rgt}_i(w)+\frac{\rho}{2}\sum_{i \in N}\left(\widehat{rgt}_i(w)\right)^2, 
\end{align*}
where $\lambda \in\mathbb{R}^n$ is a vector of Lagrange multipliers and $\rho>0$ is a fixed parameter controlling the weight of the quadratic penalty. The solver alternates between the updates on model parameters and Lagrange multipliers:  $w^{new}\in\mathrm{argmax}_w \mathcal{L}_{\rho}(w^{old},\lambda^{old})$ and   $\lambda_i^{new}=\lambda_i^{old}+\rho\cdot \widehat{rgt}_i(w^{new}), $ $\forall i \in N.$

\section{Experimental Results}\label{sec:num_exp}

We start by showing the effectiveness of our architecture in symmetric and asymmetric auctions. We then highlight its sample-efficiency for training and its ability to extrapolate to other settings. More details about the setup and training can be found in Section~\ref{app:opt_train} and Section~\ref{app:setup}.

\paragraph{Evaluation.} In addition to the revenue of the learned auction on a test set, we also evaluate the corresponding empirical average regret over bidders $\widehat{rgt} = \frac{1}{n} \sum_{i=1}^n \widehat{rgt}_i $. We evaluate these terms by running gradient ascent on $v_i'$ with a step-size of $0.001$ for $\{300,500\}$ iterations (we test $\{100,300\}$ different random initial $v_i'$ and report the one achieves the largest regret).

\paragraph{Known optimal solution.} We first consider instances of single bidder multi-item auctions where the optimal mechanism is known to be symmetric. While independent private value auction as (I) fall in this category, the following item-asymmetric auction has surprisingly an optimal symmetric solution.
\begin{itemize}
    \item (IV) One bidder and two items where the item values are independently drawn according to the probability densities $f_1(x)=5/(1+x)^6$ and $f_2(y)=6/(1+y)^7.$ Optimal solution in \cite{daskalakis2017strong}. 
 \end{itemize}

\begin{figure}[ht]
\begin{minipage}{.45\textwidth}
{\hspace{-.3cm}\resizebox{1\columnwidth}{!}{%{
\begin{tabular}{ cccc } 
\toprule
 Dist. & $rev$ & $rgt$ & OPT  \\
\midrule
  (I) & 0.551 & 0.00013  &0.550      \\ 
 (IV) & 0.173 & 0.00003 & 0.1706    \\
 (V) &  0.873 & 0.001 & 0.860 \\ 
\bottomrule
\end{tabular}}}
\vspace*{.7cm}
\captionsetup{labelformat=empty}
\caption*{(a)}
\end{minipage} \hspace{3em}
\begin{minipage}{.45\textwidth}
{\hspace{-.3cm}\resizebox{1.05\columnwidth}{!}{%{
\begin{tabular}{ ccccc } 
\toprule
\multicolumn{1}{c}{ }
& \multicolumn{2}{c}{ EquivariantNet}& \multicolumn{2}{c}{ RegretNet} \\
\cmidrule(lr){2-3}
\cmidrule(lr){4-5}
$\lambda_2$ & $rev$ & $rgt$ & $rev_F$ & $rgt_F$ \\
\cmidrule(lr){1-1}
\cmidrule(lr){2-2}
\cmidrule(lr){3-3}
\cmidrule(lr){4-4}
\cmidrule(lr){5-5}
 0.01 &  0.37 & 0.0006 & 0.39 & 0.0003\\
% 0.01 &  0.3705 & 0.00060 & 0.3867 & 0.00031\\
 0.1 & 0.41 &0.0004 &0.41 &0.0007 \\
% 0.1 & 0.4136 &0.00036 &0.4144 &0.00068 \\
 1 & 0.86 &0.0005 &0.84  & 0.0012 \\
% 1 & 0.8562 &0.00054 &0.8381  & 0.00124 \\
 10 & 3.98 &0.0081 & 3.96 &0.0056\\
% 10 & 3.988 &0.00812 & 3.963 &0.00567\\
\bottomrule
\end{tabular}}}
\captionsetup{labelformat=empty}
\caption*{(b)}
%\caption{\hspace*{-1.2cm}(b)}
\end{minipage}%
\caption{(a): Test revenue and regret found by EquivariantNet for settings (I), (IV) and (V). For setting (V)  OPT is the optimal revenue from VVCA and AMA$_{\mathrm{bsym}}$ families of auctions \citep{sandholm2015automated}. For settings (I) and (IV), OPT is the theoretical optimal revenue. (b): Test revenue/regret for setting (VI) when varying $\lambda_2$ ($\lambda_1=1$). $rev_F$ and $rgt_F$ are computed with RegretNet.}\label{fig:tables}

\end{figure}

\begin{figure}[!ht]
\centering
\begin{minipage}{\textwidth}
\centering
  \includegraphics[width=0.60\linewidth]{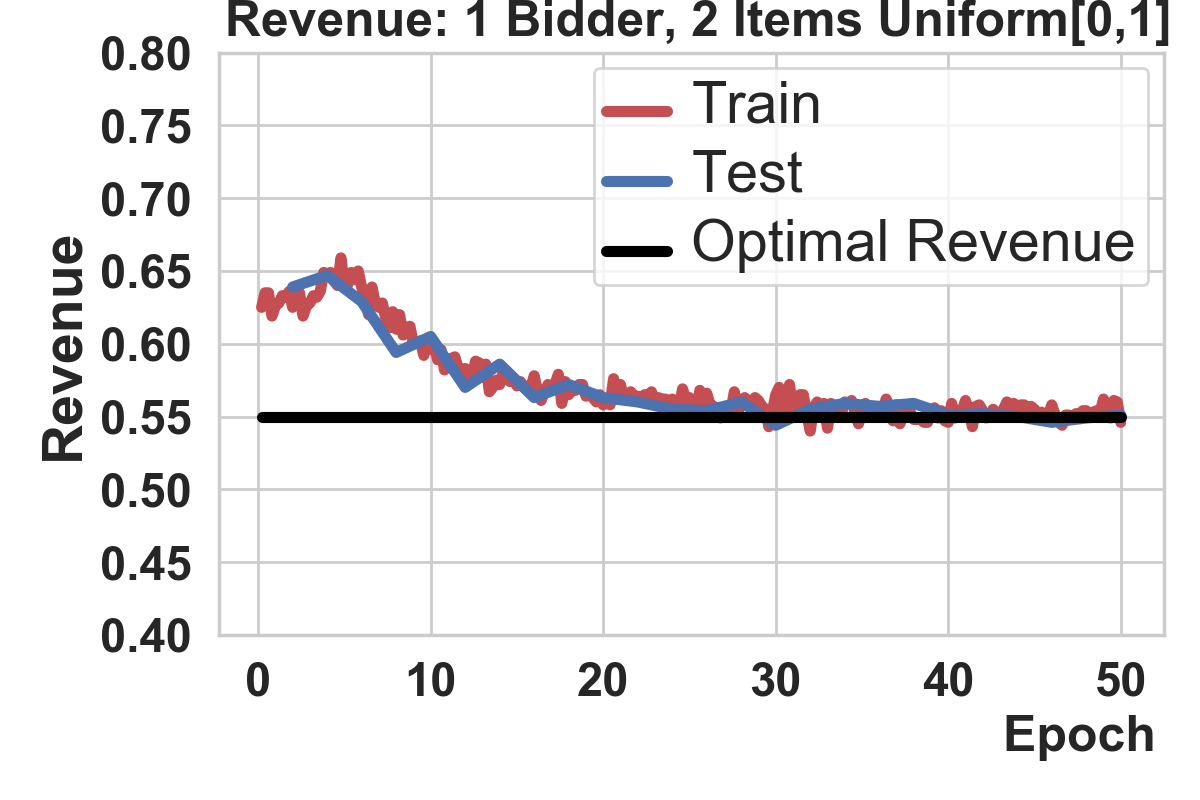}
  \captionsetup{labelformat=empty}
  \caption*{(a) Train/test revenue as a function of epochs for setting (I) for EquivariantNet. The revenue converges to the theoretical optimum (0.55).}
  
\bigskip

  \includegraphics[width=0.60\linewidth]{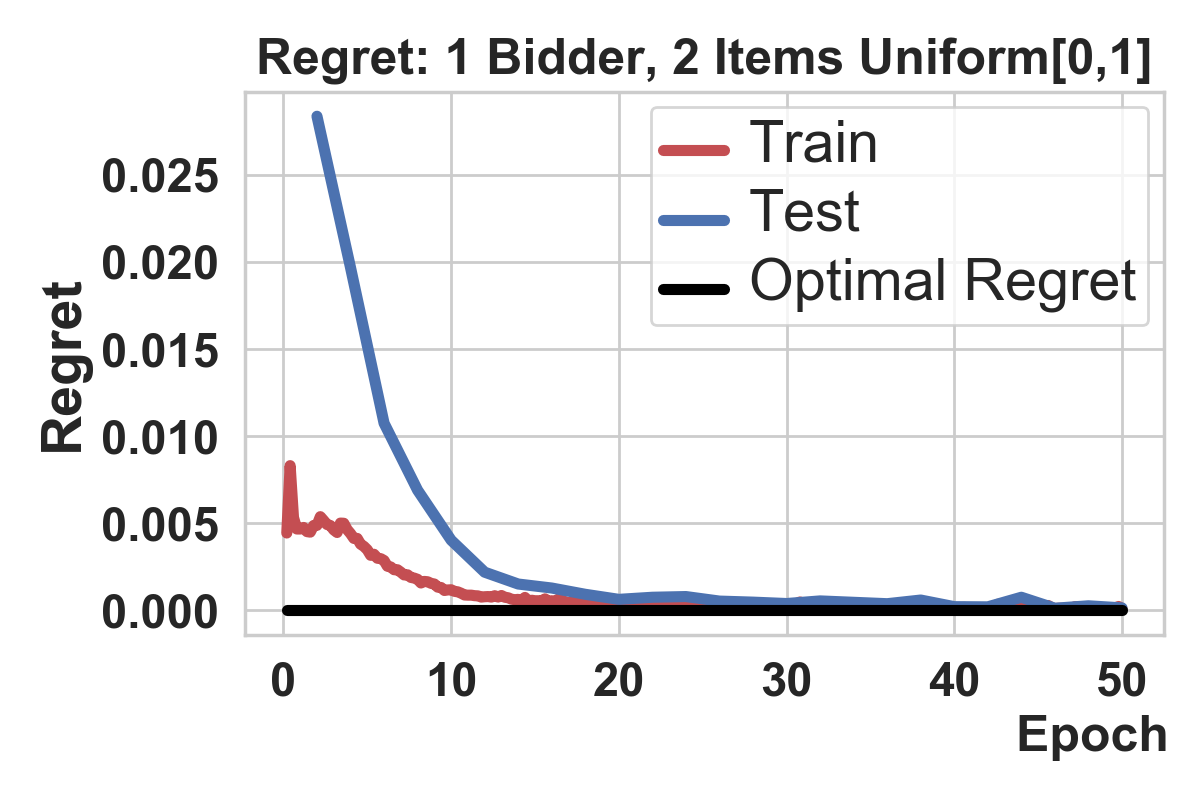}
  \captionsetup{labelformat=empty}
\caption*{(b) Train/test regret as a function of epochs for setting (I) for EquivariantNet. The regret converges to 0.}
\end{minipage}
\bigskip
\caption{ EquivariantNet learn the optimal auction for the setting (I). }
\label{fig:symm_auc}
\end{figure}

The two first lines in \autoref{fig:tables}(a) report the revenue and regret of the mechanism learned by our model. The revenue  is very close to the optimal one, and the regret is  negligible. Remark that the learned auction may achieve a revenue slightly above the optimal incentive compatible auction. This is possible because although small, the regret is non-zero. \autoref{fig:symm_auc}(a)-(b) presents a plot of revenue and regret as a function of training epochs for the setting (I).

\paragraph{Unknown optimal solution.}
Our architecture is also able to recover a permutation-equivariant solution in settings for which the optimum is not known analytically such as:

\begin{itemize}
    \item[--] (V) Two additive bidders and two items where bidders draw their value for each item from $\mathcal{U}[0,1].$  
\end{itemize}

 We compare our solution to the optimal auctions from the VVCA and AMA$_{\mathrm{bsym}}$ families of incentive compatible auctions from \citep{sandholm2015automated}. The last line of \autoref{fig:tables}(a) summarizes our results.
 
 \medskip

\paragraph{Non-symmetric optimal solution.} Our architecture returns satisfactory results in asymmetric auctions. (VI) is a setting where there may not be permutation-equivariant solutions. 

\begin{itemize}
    \item[--] (VI)  Two bidders and two items where the item values are independently drawn according to the probability densities $f_1(x)=\lambda_1^{-1}e^{-\lambda_1 x}$ and $f_2(y)=\lambda_2^{-1}e^{-\lambda_2 y},$ where $\lambda_1, \lambda_2>0.$ 
 \end{itemize}
 
\autoref{fig:tables}(b) shows the revenue and regret of the final auctions learned for setting~(VI). When $\lambda_1=\lambda_2,$ the auction is symmetric and so, the revenue of the learned auction is very close to the optimal revenue, with negligibly small regret. However, as we increase the gap between $\lambda_1$ and $\lambda_2,$ the asymmetry becomes dominant and the optimal auction does not satisfy permutation-equivariance. We remark that our architecture does output a solution with near-optimal revenue and small regret.

\paragraph{Sample-efficiency.} Our permutation-equivariant architecture exhibits solid generalization properties when compared to the feed-forward architecture RegretNet. When enough data is available at training, both architectures generalize well to unseen data and the gap between the training and test losses goes to zero. However, when fewer training samples are available, our equivariant architecture generalizes while RegretNet struggles to. This may be explained by the inductive bias in our architecture.\\

We demonstrate this for auction (V) with a training set of $20$ samples and  plot the training and test losses as a function of time (measures in epochs) for both architectures in \autoref{fig:general_exp}(a).

% \begin{figure}[h]
% \begin{minipage}{.33\textwidth}
% \centering
%   \includegraphics[width=\linewidth]{./images/20_Generalization.png}
% \vspace*{-.4cm}
% \captionsetup{labelformat=empty}
%   \caption{(a)}
% \end{minipage}
% \begin{minipage}{.33\textwidth}
% \centering
%   \includegraphics[width=\linewidth]{./images/1xiGeneralization.png}
% \captionsetup{labelformat=empty}
% \vspace*{-.4cm}
%   \caption{(b)}
% \end{minipage}%
% \begin{minipage}{.33\textwidth}
% \centering
%   \includegraphics[width=\linewidth]{./images/2xiGeneralization.png}
%   \vspace*{-.4cm}
%   \captionsetup{labelformat=empty}
% \caption{(c)}
% \end{minipage}
% \setcounter{figure}{4}    
% \caption{(a): Train and test losses (V) with 20 training samples.  (b): Generalization revenue versus RegretNet for experiment $(\alpha).$ (c): Generalization revenue versus RegretNet for experiment $(\beta).$ }\label{fig:general_exp}
% \end{figure}

\begin{figure}[H]
\centering
\begin{minipage}{.7\textwidth}
  \includegraphics[width=\linewidth]{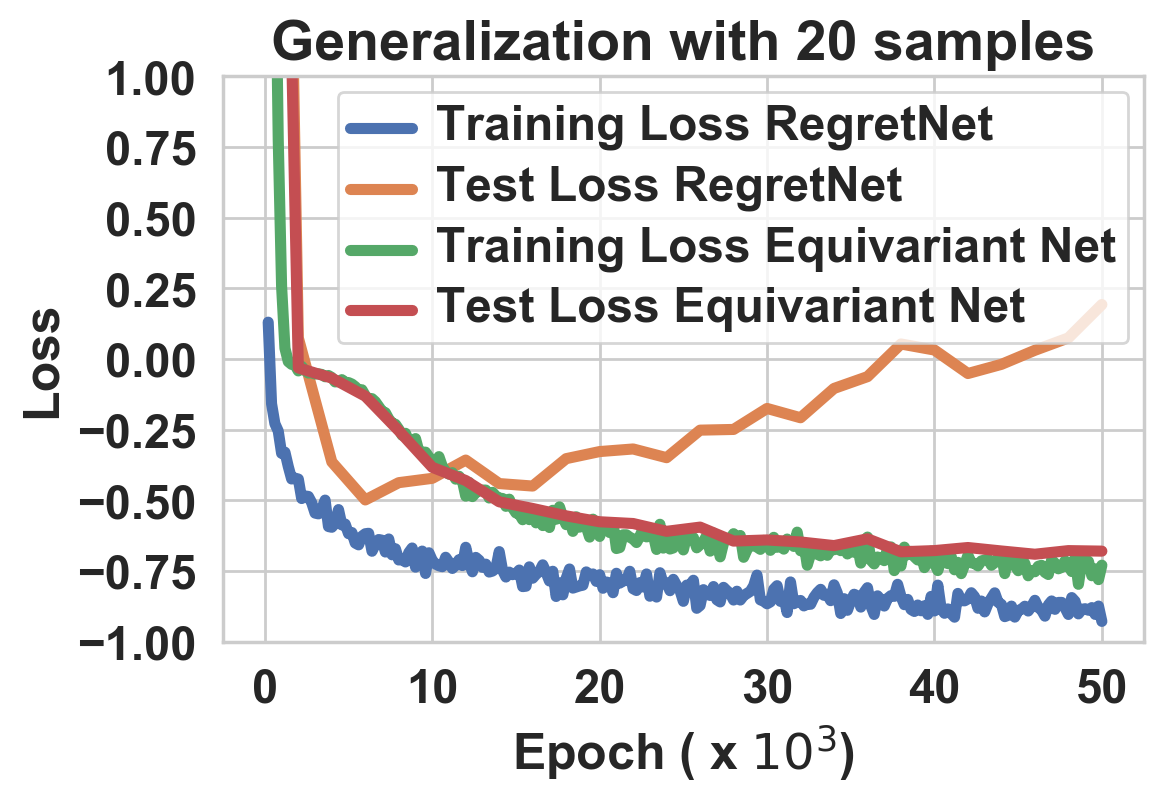}
\vspace*{-.4cm}
\captionsetup{labelformat=empty}
\end{minipage}
\caption{Train and test losses (the Lagrangian) for setting (V) with 20 training samples. RegretNet and EquivariantNet both achieve small losses on the training set, only EquivariantNet is generalizes to the testing set.   }\label{fig:general_exp}
\end{figure}

\paragraph{Out-of-setting generalization.} The number of parameters in our permutation equivariant architecture does not depend on the size of the input. Given an architecture that was trained on samples of size $(n,m)$, it is possible to evaluate it on samples of any size $(n',m')$ (More details in Section~\ref{app:permeq_net}). This  evaluation is not well defined for feed-forward architectures where the dimension of the weights depends on the input size. We use this advantage to check whether models trained in a fixed setting perform well in totally different ones. 
 
\begin{itemize}
\item[--]($\alpha$) Train an equivariant architecture on 1 bidder, 5 items and test it on 1 bidder, $n$ items for $n = 2 \cdots 10$. All the items values are sampled independently from $\mathcal{U}[0,1].$ 

\item[--]($\beta$) Train an equivariant architecture on 2 bidders, 3 objects and test it on 2 bidders, $n$ objects for $n= 2 \cdots 6$.  All the items values are sampled independently from $\mathcal{U}[0,1].$ 
\end{itemize}
\autoref{fig:general_exp_out}(a)-(b) reports the test revenue that we get for different values of $n$ in ($\alpha)$ and ($\beta$) and compares it to the empirical optimal revenue. Our baseline for that is RegretNet. Surprisingly, our model does generalize well. It is worth mentioning that knowing how to solve a larger problem such as $1 \times 5$ does not automatically result in a capacity to solve a smaller one such as $1\times 2$; the generalization does happen on both ends. Our approach looks promising regarding out of setting generalization. It generalizes well when the number of objects varies and the number of bidders remain constants. However, generalization to settings where the number of bidders varies is more difficult due to the complex interactions between bidders. We do not observe good generalization with our current method. 

\begin{figure}[h]
\begin{minipage}{.5\textwidth}
\centering
  \includegraphics[width=\linewidth]{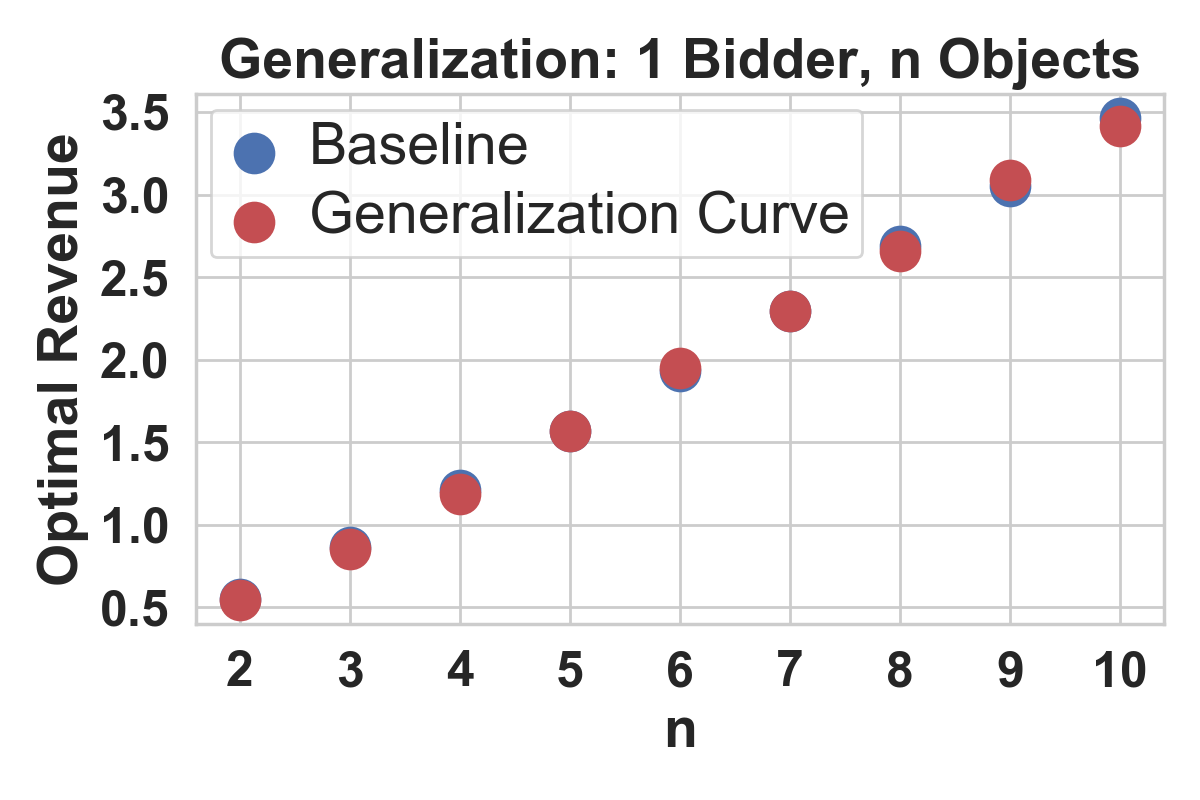}
\captionsetup{labelformat=empty}
\vspace*{-.4cm}
  \caption*{(a)}
\end{minipage}%
\begin{minipage}{.5\textwidth}
\centering
  \includegraphics[width=\linewidth]{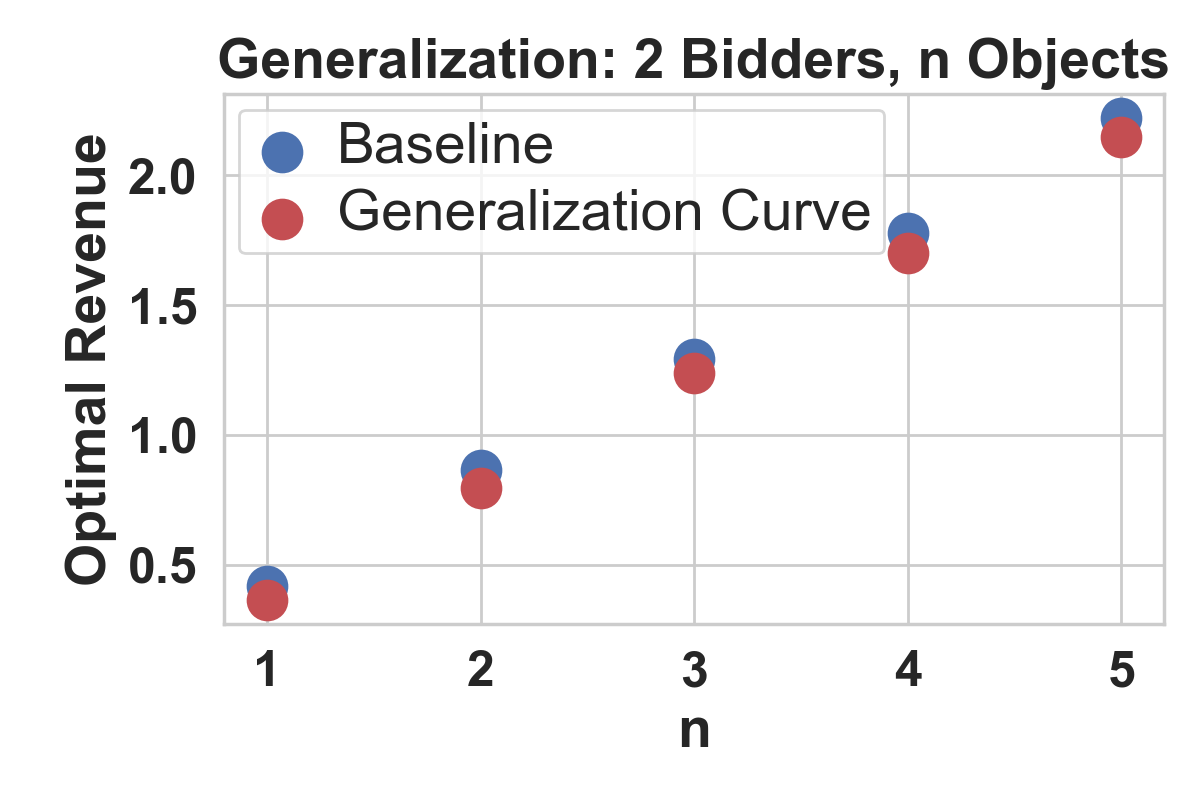}
  \vspace*{-.4cm}
  \captionsetup{labelformat=empty}
\caption*{(b)}
\end{minipage}
\caption{Generalization revenue of EquivariantNet in experiment $(\alpha)$ and $(\beta).$ Each baseline point is computed using a RegretNet architecture trained from scratch.  }\label{fig:general_exp_out}
\end{figure}

\section*{Conclusion}

We explored the effect of adding domain knowledge in neural network architectures for auction design. We built a permutation-equivariant architecture to design symmetric auctions and highlighted its multiple advantages. It recovers several known optimal results and provides competitive results in asymmetric auctions. Compared to fully connected architectures, it is more sample efficient and is able to generalize to settings it was not trained on. In a nutshell, this chapter insists on the importance of bringing domain-knowledge to the deep learning approaches for auctions.\\
 
Our architecture presents some limitations. It assumes that all the bidders and items are permutation-equivariant. However, in some real-world auctions, the item/bidder-symmetry only holds for a group of bidders/items. More advanced architectures such as Equivariant Graph Networks \citep{maron2018invariant} may solve this issue. Another limitation is that we only consider additive valuations. An interesting direction would be to extend our approach to other settings as unit-demand or combinatorial auctions.

% \paragraph{Acknowledgements.}
% We would like to thank Ryan P. Adams, Devon Graham and Andrea Tacchetti
% for helpful discussions. The work of Jad Rahme was funded by a Princeton SEAS Innovation Grant. Samy Jelassi and Joan Bruna's work are partially supported by the Alfred P. Sloan Foundation, NSF RI-1816753, NSF CAREER CIF 1845360, NSF CHS-1901091, Samsung Electronics, and the Institute for Advanced Study. The work of S. Matthew Weinberg was supported by NSF CCF-1717899.

\begin{subappendices}

\section{Permutation-equivariant network}\label{app:permeq_net}

In this section, we remind the  \textit{exchangeable matrix layers} introduced by \citet{hartford2018deep}. These layers are are a generalization of the deep sets architecture by  \citet{zaheer2017deep}. We briefly describe this architecture here and invite the reader to look at the original paper for details. \\

The architecture consists in several layers and each of them is constituted of multiple channels. Each layer is specified by the number of input channels $K$ channels and the number of outputs channels $O$. The input of such a layer is a tensor $B$ of size $(K,n,m)$ and the output is another tensor $Y$ of size  $(O,n,m)$.  The first element of these tensor is the channel number. In the following we will denote by $B^{(k)}_{i,j}$ the element of $B$ of index $(k,i,j)$ and similarly for  $Y^{(o)}_{i,j}$.\\

In addition to the $K$ and $O$, an exchangeable layer is defined by a set of five weights $w_1 , w_2, w_3, w_4 \in \mathbb{R}^{K \times O}$ and $w_5 \in \mathbb{R}$. Given these weights, the element $(i,j)$ of the $o$-th output channel $Y_{i,j}^{(o)}$ is given by:
\begin{equation}\label{eq:mult_channel}
\small
\begin{aligned}
\small
 Y_{i,j}^{(o)}=\sigma \Bigg ( &\sum_{k=1}^K w_1^{(k,o)}B_{i,j}^{(k)}+\frac{w_2^{(k,o)}}{n}\sum_{i'}B_{i',j}^{(k)} \\
 &+\frac{w_3^{(k,o)}}{m}\sum_{j'}B_{i,j'}^{(k)}       +\frac{w_4^{(k,o)}}{nm}\sum_{i',j'}B_{i',j'}^{(k)}+w_5^{(o)}
 \Bigg)
 \end{aligned}
\end{equation}
This layer preserved permutation-equivariance. This was first proven in \citet{hartford2018deep}.  Additionally, the number of parameters of each layer only depends on $K$ and $O$ is does not depend on the dimension of the input (i.e. $m$ and $n$). In particular, we can apply this exchangeable layer to any tensor of size $(K,n',m')$ for any value of $n'$ and $m'$ and the resulting output will be a tensor of size $(O,n',m')$.\\

We can compose these exchangeable layers as long as the number of channels of the output of one layer is equal to the number of input channels required by the following layer. By composing many such layer of this form we get a deep exchangeable neural network. This deep network preserved permutation-equivariance since this property is preserved by every layer. In addition, this network can be evaluated on an input of any dimension $n$ and $m$. We use this property of the network to test our mechanisms on settings with different number of bidders and objects. Without this property out of setting generalization not possible.
\clearpage
\section{Proof of \autoref{prop:equiv_sol}}\label{app:thm}
Notation: For a matrix $B \in \mathbb{R}^{nm}$ we will denote the $i$th line by $B_i \in \mathbb{R}^{m} $ or $[B]_i \in \mathbb{R}^{m}$.Let $D$ denote an equivariant distribution on $\mathbb{R}^{nm}$ .  Let $g\colon\mathbb{R}^{nm} \rightarrow \mathbb{R}^{nm}$ and $p\colon\mathbb{R}^{nm}\rightarrow \mathbb{R}^{n} $  be solutions to the following problem: 
$$p = \mbox{argmax}\,\, \mathbb{E}_{B \sim D} \left[ \sum_{i=1}^n p_i(B) \right]$$
subject to:
$$\langle\, [g(B)]_i \, , \,B_i \, \rangle \geq p_i(B),$$
and 
$$\langle  \,[g(B_i,B_{-i})]_i \,, \,B_i \,\rangle -  p_i(B_i,B_{-i}) \geq \langle \,[g(B'_i,B_{-i})]_i \,,\,B_i \,\rangle -  p_i(B'_i,B_{-i}), \,\, \,\, \forall B'_i \in \mathbb{R}^m .$$

Let $\Pi_n$ and $\Pi_m$ be two permutation matrices of sizes $n$ and $m$. In particular $\Pi_n$ and $\Pi_m$ are orthogonal matrices and in the following we use that $\Pi_n^{-1} =\Pi_n^{T} $ and $\Pi_m^{-1} =\Pi_m^{T} $. Let's define: 

\begin{equation*}
\begin{aligned}
g^{\Pi_{n},\Pi_{m}}(B) &= \Pi_{n}^{-1}\, g(\Pi_{n}\,B\, \Pi_{m}) \, \Pi_{m}^{-1} \\
p^{\Pi_{n},\Pi_{m}}(B) &= \Pi_{n}^{-1} \, p(\Pi_{n}\,B\, \Pi_{m}).
\end{aligned}
\end{equation*}

Let's prove that if $(g,p)$ is a solution to the problem then so is  $(g^{\Pi_{n},\Pi_{m}},p^{\Pi_{n},\Pi_{m}})$. First we show that $(g^{\Pi_{n},\Pi_{m}},p^{\Pi_{n},\Pi_{m}})$ still satisfy the previous constraints.
\begin{equation*}
\begin{aligned}
\langle [g^{\Pi_{n},\Pi_{m}}(B)]_i \, , \, B_i \, \rangle & = \langle [\Pi_{n}^{-1}g(\Pi_{n}\,B\,\Pi_{m})\Pi_{m}^{-1}]_i \, , \, B_i \, \rangle \\
& = \langle [\Pi_{n}^{-1}g(\Pi_{n}\,B\,\Pi_{m})]_i \Pi_{m}^{-1} \, , \, B_i \, \rangle \\
& = \langle [\Pi_{n}^{-1}g(\Pi_{n}\,B\,\Pi_{m})]_i \, , \, B_i \Pi_{m} \, \rangle \\
& = \langle [\Pi_{n}^{-1}g(\Pi_{n}\,B\,\Pi_{m})]_i \, , \, [B \Pi_{m}]_i \, \rangle \\
& = \langle [\Pi_{n}^{-1}g(\Pi_{n}\,B\,\Pi_{m})]_i \, , \, [B \Pi_{m}]_i \, \rangle.
\end{aligned}
\end{equation*}
Let's denote by $\phi$ the permutation on the indices corresponding to the $\Pi_{n}$ permutation. then we have: 

\begin{equation*}
\begin{aligned}
[\Pi_{n}^{-1}g(\Pi_{n}\,B\,\Pi_{m})]_i  &= [g(\Pi_{n}\,B\,\Pi_{m})]_{\phi^{-1}(i)}\\
[B \Pi_{m}]_i &= [\Pi_{n}B \Pi_{m}]_{\phi^{-1}(i)}.
\end{aligned}
\end{equation*}
This gives us that: 

\begin{equation*}
\begin{aligned}
\langle [g^{\Pi_{n},\Pi_{m}}(B)]_i \, , \, B_i \, \rangle & = \langle [\Pi_{n}^{-1}g(\Pi_{n}\,B\,\Pi_{m})]_i \, , \, [B \Pi_{m}]_i \, \rangle \\
& = \langle [g(\Pi_{n}\,B\,\Pi_{m})]_{\phi^{-1}(i)} \, , \, [\Pi_{n}B \Pi_{m}]_{\phi^{-1}(i)} \, \rangle \\ 
& \geq  [p(\Pi_{n}B \Pi_{m})]_{\phi^{-1}(i)}\\
& = [\Pi_{n}^{-1}p(\Pi_{n}B \Pi_{m})]_{i} \\
&= [p^{\Pi_{n},\Pi_{m}}(B)]_{i}.
\end{aligned}
\end{equation*}
This shows that $(g^{\Pi_{n},\Pi_{m}},p^{\Pi_{n},\Pi_{m}})$ satisfies the first constraint. We now move to the second constraint. \\

Let's write $\tilde{B} = (B'_{i},B_{-i})$. As a reminder, this is the matrix $B$ where the $i$th line has been replaced with $B'_i$. We need to show that: 

$$\langle  \,[g^{\Pi_{n},\Pi_{m}}(B)]_i \,, \,B_i \,\rangle -  p_i^{\Pi_{n},\Pi_{m}}(B) \geq \langle \,[g^{\Pi_{n},\Pi_{m}}(\tilde{B})]_i \,,\,B_i \,\rangle -  p^{\Pi_{n},\Pi_{m}}_i(\tilde{B}).$$

Using the previous computations we find that: 
$$\langle  \,[g^{\Pi_{n},\Pi_{m}}(B)]_i \,, \,B_i \,\rangle -  p_i^{\Pi_{n},\Pi_{m}}(B)  = \langle [g(\Pi_{n}\,B\,\Pi_{m})]_{\phi^{-1}(i)} \, , \, [\Pi_{n}B \Pi_{m}]_{\phi^{-1}(i)} \, \rangle - [p(\Pi_{n}B \Pi_{m})]_{\phi^{-1}(i)}, $$
where $\phi$ is the permutation associated with $\Pi_n$. Since $g$ and $p$ satisfy the second constraint we have:  
\begin{equation*}
\begin{aligned}    
\langle  \,[g^{\Pi_{n},\Pi_{m}}(B)]_i \,, \,B_i \,\rangle -  p_i^{\Pi_{n},\Pi_{m}}(B) & = \langle [g(\Pi_{n}\,B\,\Pi_{m})]_{\phi^{-1}(i)} \, , \, [\Pi_{n}B \Pi_{m}]_{\phi^{-1}(i)} \, \rangle - [p(\Pi_{n}B \Pi_{m})]_{\phi^{-1}(i)}\\
 &\geq  \langle [g(\Pi_{n}\,\tilde{B}\,\Pi_{m})]_{\phi^{-1}(i)} \, , \, [\Pi_{n}\tilde{B} \Pi_{m}]_{\phi^{-1}(i)} \, \rangle - [p(\Pi_{n}\tilde{B} \Pi_{m})]_{\phi^{-1}(i)}
 \\
 &= \langle  \,[g^{\Pi_{n},\Pi_{m}}(\tilde{B})]_i \,, \,B_i \,\rangle -  p_i^{\Pi_{n},\Pi_{m}}(\tilde{B}).
\end{aligned}  
\end{equation*}

This concludes the proof that $(g^{\Pi_{n},\Pi_{m}},p^{\Pi_{n},\Pi_{m}})$ satisfy the constraints. Now we have to show that $p^{\Pi_{n},\Pi_{m}}$ is optimal.

\begin{equation*}
\begin{aligned}
\mathbb{E}_{B \sim D} \left[ \sum_{i=1}^n p^{\Pi_{n},\Pi_{m}}(B) \right] &=  \mathbb{E}_{B \sim D} \left[ \langle \,  p^{\Pi_{n},\Pi_{m}}(B) \, , \bf{1} \, \rangle \right]\\
&=  \mathbb{E}_{B \sim D} \left[ \langle \,   \Pi_{n}^{-1} \, p(\Pi_{n}\,B\, \Pi_{m}) \, , \bf{1} \, \rangle \right] \\
&=  \mathbb{E}_{B \sim D} \left[ \langle \,   p(\Pi_{n}\,B\, \Pi_{m}) \, , \bf{1} \, \rangle \right]\\
&=  \mathbb{E}_{B \sim D} \left[ \langle \,   p(B) \, , \bf{1} \, \rangle \right]\\
&= \mathbb{E}_{B \sim D} \left[ \sum_{i=1}^n p_i(B) \right],
\end{aligned}
\end{equation*}
where we used that $\Pi_n^{-1} =\Pi_n^{T} $,  $ \Pi_n \bf{1} = \bf{1} $  and that $\Pi_{n}\,B\, \Pi_{m} \sim D $ since $D$ is an equivariant distribution. This shows that if $p$ is optimal then $p^{\Pi_{n},\Pi_{m}}$ is also optimal since they have the same expectation. We conclude that  $(g^{\Pi_{n},\Pi_{m}},p^{\Pi_{n},\Pi_{m}})$ is an optimal solution.  Let's define
\begin{equation*}
\begin{aligned}
\tilde{g}(B) &= \mathbb{E}_{\Pi_{n},\Pi_{m}} \left[ g^{\Pi_{n},\Pi_{m}}(B)\right] \\
\tilde{p}(B) &= \mathbb{E}_{\Pi_{n},\Pi_{m}} \left[ p^{\Pi_{n},\Pi_{m}}(B)\right].
\end{aligned}
\end{equation*}

Here, in the expectation, $\Pi_{n}$ and $\Pi_{m}$ are drawn uniformly at random. Since the problem and constraints are convex, $(\tilde{g},\tilde{p})$ is also an optimal solution to the problem as a convex combination of optimal solutions. Let's prove that $\tilde{g}$ and $\tilde{p}$ are equivariant functions.

\begin{equation*}
\begin{aligned}
\tilde{g}(\Pi_{n}\,B\,\Pi_{m}) &= \mathbb{E}_{\Pi'_{n},\Pi'_{m}} \left[ g^{\Pi'_{n},\Pi'_{m}}(\Pi_{n}\,B\,\Pi_{m})\right] \\ 
&= \mathbb{E}_{\Pi'_{n},\Pi'_{m}} \left[ {\Pi'_{n}}^{-1}g(\Pi'_{m}\,\Pi_{n}\,B\,\Pi_{m}\,\Pi'_{m}){\Pi'_{m}}^{-1}\right] \\ 
&= {\Pi_{n}}^{-1} \,\, \mathbb{E}_{\Pi'_{n},\Pi'_{m}} \left[ {(\Pi'_{n}\Pi_{n})}^{-1}g(\Pi'_{n}\,\Pi_{n}\,B\,\Pi_{m}\,\Pi'_{m}){(\Pi_{m}\Pi'_{m})}^{-1}\right]\,\, {\Pi_{m}}^{-1}. \\ 
\end{aligned}
\end{equation*}
If ${\Pi'_{n}}$ and ${\Pi'_{m}}$ are uniform among permutation then so is $\Pi'_{n}\Pi_{n}$ and $\Pi'_{m}\Pi_{m}$. So through a change of variable we find that: 

\begin{equation*}
\begin{aligned}
\tilde{g}(\Pi_{n}\,B\,\Pi_{m}) 
&= {\Pi_{n}}^{-1} \,\, \mathbb{E}_{\Pi'_{n},\Pi'_{m}} \left[ {\Pi'_{n}}^{-1}g(\Pi'_{n}\,B\,\Pi'_{m}){\Pi'_{m}}^{-1}\right]\,\, {\Pi_{m}}^{-1} \\ 
&= {\Pi_{n}}^{-1} \,\, \tilde{g}(B) \,\, {\Pi_{m}}^{-1}.
\end{aligned}
\end{equation*}
This shows that $\tilde{g}$ is equivariant. The proof that $\tilde{p}$ is equivariant is similar. 

\begin{equation*}
\begin{aligned}
\tilde{p}(\Pi_{n}\,B\,\Pi_{m}) &= \mathbb{E}_{\Pi'_{n},\Pi'_{m}} \left[ p^{\Pi'_{n},\Pi'_{m}}(\Pi_{n}\,B\,\Pi_{m})\right] \\ 
&= \mathbb{E}_{\Pi'_{n},\Pi'_{m}} \left[ {\Pi'_{n}}^{-1}p(\Pi'_{m}\,\Pi_{n}\,B\,\Pi_{m}\,\Pi'_{m})\right] \\ 
&= {\Pi_{n}}^{-1} \,\, \mathbb{E}_{\Pi'_{n},\Pi'_{m}} \left[ {(\Pi'_{n}\Pi_{n})}^{-1}p(\Pi'_{n}\,\Pi_{n}\,B\,\Pi_{m}\,\Pi'_{m})\right]. \\ 
\end{aligned}
\end{equation*}
By doing a change of variable as before we find: 
\begin{equation*}
\begin{aligned}
\tilde{p}(\Pi_{n}\,B\,\Pi_{m}) 
&= {\Pi_{n}}^{-1} \,\, \mathbb{E}_{\Pi'_{n},\Pi'_{m}} \left[ {\Pi'_{n}}^{-1}p(\Pi'_{n}\,B\,\Pi'_{m})\right] \\ 
&= {\Pi_{n}}^{-1} \,\, \tilde{p}(B),
\end{aligned}
\end{equation*}
$(\tilde{g},\tilde{p})$ is an equivariant optimal solution, this concludes the proof.

\clearpage
\section{Optimization and training procedures}\label{app:opt_train}

Our training algorithm is the same as the one found in \cite{dutting2017optimal}. We made that choice to better illustrate the intrinsic advantages of our permutation equivariant architecture.  We include implementation details here for completeness and additional details can be found in the original paper.\\

\begin{algorithm}[h]
\caption{Training Algorithm} 
\label{alg:langevin_dynamics}
\begin{algorithmic}[1]
%\State Initialization:
\State \textbf{Input}: Minibatches $\mathcal{S}_1,\dots,\mathcal{S}_T$ of size $B$
\State \textbf{Parameters}: $\gamma >0, \, \eta >0, \, c>0, \, R\in \mathbb{N}, \,  T\in \mathbb{N}, \,T_{\rho}\in \mathbb{N}, \,T_{\lambda}\in \mathbb{N}.$
\State \textbf{Initialize Parameters}: $ \rho^0 \in \mathbb{R}, \, w^0\in \mathbb{R}^d, \, \lambda^0\in\mathbb{R}^n,  $
\State \textbf{Initialize Misreports:} ${v_i'}^{(\ell)}\in V_i, \,\, \forall \ell \in [B], \, i\in N.$ 
%\State Main routine:
\For {$t=0,\dots,T$}
        \State Receive minibatch $\mathcal{S}_t = \{V^{(1)},\dots,V^{(B)}\}.$
        \For{$r=0,\dots,R$}
        \State
        \vspace{-1.1cm}
        
        \begin{align*}
            \hspace{.6cm}&\forall \ell \in [B], \, i\in n:\\
            &{v_i'}^{(\ell)} \leftarrow {v_i'}^{(\ell)}+\gamma \nabla_{v_i'}u_i^{w_t} ({v_i}^{(\ell)};({v_i'}^{(\ell)},V_{-i}^{(\ell)}))
        \end{align*}
        \EndFor
        
        \vspace{-.4cm}
        \State Get Lagrangian gradient using \eqref{eq:grad_Lag} and update $w^t$:\\
         \hspace{1cm} $w^{t+1}\leftarrow w^t-\eta \nabla_w \mathcal{L}_{\rho^t}(w^t). $

        \State Update $\rho$ once in $T_{\rho}$ iterations:
        \If{$t$ is a multiple of $T_{\rho}$}
        \State $\rho^{t+1} \leftarrow \rho^{t} + c $ 
        \Else
        \State $\rho^{t+1}\leftarrow \rho^t$
        \EndIf
        \State Update Lagrange multipliers once in $T_{\lambda}$ iterations:
        \If{$t$ is a multiple of $T_{\lambda}$}
        \State $\lambda_i^{t+1}\leftarrow \lambda_i^t + \rho^t \, \widehat{rgt}_i(w^{t}),\forall i\in N$
        \Else
        \State $\lambda^{t+1}\leftarrow \lambda^t$
        \EndIf
\EndFor
%\EndProcedure
\end{algorithmic}
\end{algorithm}

We generate a training dataset of valuation profiles $\mathcal{S}$ that we then divide into mini-batches of size $B$. Typical sizes for $\mathcal{S}$ are $\{5000,50000,500000 \}$ and typical batch sizes are $\{50,500,50000\}$.  We train our networks over for several epochs (typically $\{50,80 \}$) and we apply a random shuffling of the training data for each new epoch. We denote the minibatch received at iteration $t$ by $\mathcal{S}_t=\{V^{(1)},\dots,V^{(B)}\}.$ The update on model parameters involves an unconstrained optimization of $\mathcal{L}_{\rho}$ over $w$ and is performed using a gradient-based optimizer. Let $\widehat{rgt}_i(w)$ be the empirical regret in \eqref{eq:emp_reg} computed on mini-batch $\mathcal{S}_t.$ The gradient of $\mathcal{L}_{\rho}$ with respect to $w$ is given by: 
\vspace*{-.27cm}
\begin{equation}\label{eq:grad_Lag}
\begin{split}
    \nabla_w \mathcal{L}_{\rho}(w)&=-\frac{1}{B}\sum_{\ell=1}^B\sum_{i \in N} \nabla_w p_i^w(V^{(\ell)})\\
    &\hspace{-.5cm}+\sum_{i \in N}\sum_{\ell=1}^B\lambda_i^t g_{\ell,i}+\rho_t \sum_{i \in N}\sum_{\ell=1}^B\widehat{rgt}_i(w)g_{\ell,i}, 
\end{split}
\end{equation}
where 
\begin{align*}
    g_{\ell,i}=\nabla_w\left[\max_{v_i'\in V_i} u_i^w(v_i^{(\ell)};(v_i',V_{-i}^{(\ell)}))-u_i^w(v_i^{(\ell)};(v_i^{(\ell)},V_{-i}^{(\ell)}))\right].
\end{align*}
The terms $\widehat{rgt}_i$ and $g_{\ell,i}$ requires us to compute the maximum over misreports for each bidder $i$ and valuation profile $\ell$. To compute this maximum we optimize the function $v_i' \to u_i^w(v_i^{(\ell)};(v_i',V_{-i}^{(\ell)}))$ using another gradient based optimizer. \\

For each $i$ and valuation profile $\ell$, we maintain a misreports valuation ${v_i'}^{(\ell)}$.  For every update on the model parameters $w^t,$  we perform $R$ gradient updates to compute the optimal misreports: ${v_i'}^{(\ell)}={v_i'}^{(\ell)}+\gamma \nabla_{{v_i'}^{(\ell)}}u_i^w(v_i^{(\ell)};({v_i'}^{(\ell)},V_{-i}^{(\ell)})),$ for some $\gamma >0.$  In our experiments, we use the Adam optimizer \citep{kingma2014adam} for updates on model $w$ and ${v_i'}^{(\ell)}. $ Typical values are $R = 25$ and $\gamma = 0.001$ for the training phase.  During testing, we use a larger number of step sizes $R_{test}$ to compute these optimal misreports and we try bigger number initialization, $N_{init}$, that are drawn from the same distribution of the valuations. Typical values are $R_{test} = \{200,300\}$ and $N_{init} = \{100,300\}$.
When the valuations are constrained to an interval (for instance $[0,1]$), this optimization inner loop becomes constrained and we make sure that the values we get for  $v_i'$ are realistic by projecting them to their domain after each gradient step. \\

The parameters $\lambda^t$ and $\rho_t$ in the Lagrangian are not constant but they are updated over time.  $\rho_t$ is initialized at a value $\rho_0$ is incremented every $T_{\rho}$ iterations, $\rho_{t+1} \leftarrow \rho_{t} + c $. Typical values are $\rho_{0} = \{0.25,1\}$,  $c = \{0.25,1, 5\}$ and $T_{\rho} = \{2, 5\}$ epochs. $\lambda_t$ is initialized at a value $\lambda_0$ is updates every $T_{\lambda}$ iterations according to $\lambda_i^{t+1} \leftarrow \lambda_i^t + \rho_t \, \widehat{rgt}_i(w^{t}),\forall i\in N$.  Typical values are $\lambda^{0}_i = \{0.25,1,5\}$ and $T_{\lambda} = \{2\}$ iterations.

\section{Setup} \label{app:setup}
 
We implemented our experiments using PyTorch. A typical deep exchangeable network consists of 3 hidden layers of 25 channels each. Depending on the experiment, we generated a dataset of $\{5000,50000,500000\}$ valuation profiles and chose mini batches of sizes $\{50,500,5000\}$ for training. 
The optimization of the augmented Lagrangian was typically run for $\{50, 80\}$ epochs. The value of $\rho$ in the augmented Lagrangian was set to $1.0$ and incremented every $2$ epochs. An update on $w^t$ was performed for every mini-batch using the Adam optimizer with a learning rate of $0.001$. For each update $w^t$, we ran $R=25$ misreport update steps with a learning rate of $0.001.$ An update on $\lambda^t$ was performed once every $100$ minibatches.

\end{subappendices}

\chapter{Auction Learning as a Two Player Game}
\label{chap:ALGnet}

Designing an incentive compatible auction that maximizes expected revenue is a central problem in Auction Design. While theoretical approaches to the problem have hit some limits, a recent research direction initiated by \cite{dutting2017optimal} consists in building neural network architectures to find optimal auctions. We propose two conceptual deviations from their approach which result in enhanced performance. First, we use recent results in theoretical auction design to introduce a \emph{time-independent} Lagrangian. This not only circumvents the need for an expensive hyper-parameter search (as in prior work), but also provides a single metric to compare the performance of two auctions (absent from prior work). Second, the optimization procedure in previous work uses an inner maximization loop to compute optimal misreports.  We amortize this process through the introduction of an additional neural network. We demonstrate the effectiveness of our approach by learning competitive or strictly improved auctions compared to prior work. Both results together further imply a novel formulation of Auction Design as a two-player game with stationary utility functions.

\section{Introduction}

Efficiently designing truthful auctions is a core problem in Mathematical Economics. Concrete examples include the sponsored search auctions conducted by companies as Google or auctions run on platforms as eBay. Following seminal work of Vickrey~\citep{vickrey1961counterspeculation} and Myerson~\citep{myerson1981optimal}, auctions are typically studied in the \emph{independent private valuations} model: each bidder has a valuation function over items, and their payoff depends only on the items they receive. Moreover, the auctioneer knows aggregate information about the population that each bidder comes from, modeled as a distribution over valuation functions, but does not know precisely each bidder's valuation (outside of any information in this Bayesian prior). A major difficulty in designing auctions is that valuations are private and bidders need to be incentivized to report their valuations truthfully.  The goal of the auctioneer is to design an incentive compatible auction which maximizes expected revenue.

Auction Design has existed as a rigorous mathematical field for several decades and yet, complete characterizations of the optimal auction only exist for a few settings. While Myerson's Nobel prize-winning work provides a clean characterization of the single-item optimum \citep{myerson1981optimal}, optimal \emph{multi-item} auctions provably suffer from numerous formal measures of intractability (including computational intractability, high description complexity, non-monotonicity, and others)~\citep{daskalakis2014complexity, ChenDPSY14, ChenDOPSY15, ChenMPY18, HartR15, Thanassoulis04}.

 An orthogonal line of work instead develops deep learning architectures to find the optimal auction.  \cite{dutting2017optimal} initiated this direction by proposing RegretNet, a feed-forward architecture.  They frame the auction design problem as a constrained learning problem and lift the constraints into the objective via the augmented Lagrangian method. Training RegretNet involves optimizing this Lagrangian-penalized objective, while simultaneously updating network parameters and the Lagrangian multipliers themselves. This architecture produces impressive results: recovering near-optimal auctions in several known multi-item settings, and discovering new mechanisms when a theoretical optimum is unknown. 

  Yet, this approach presents several limitations. On the conceptual front, our main insight is a connection to an exciting line of recent works~\citep{HartlineL10, HartlineKM11,BeiH11,daskalakis2012symmetries, rubinstein2018simple, DughmiHKN17, CaiOVZ19} 
 on $\varepsilon$-truthful-to-truthful reductions.\footnote{By $\epsilon$-truthful, we mean the expected total regret $R$ is bounded by  $\epsilon$. See \autoref{prop:main} for a definition of $R$. } On the technical front, we identify three areas for improvement. First, their architecture is difficult to train in practice as the objective is non-stationary. Specifically, the Lagrangian multipliers are time-dependent and they increase following a pre-defined schedule, which requires careful hyperparameter tuning (see~\autoref{subsec:dutt} for experiments illustrating this). Leveraging the aforementioned works in Auction Theory, we propose a \emph{stationary} Lagrangian objective. Second, all prior work inevitably finds auctions which are not \emph{precisely} incentive compatible, and does not provide a metric to compare, say, an auction with revenue $1.01$ which is $0.002$-truthful, or one with revenue $1$ which is $0.001$-truthful. We argue that our stationary Lagrangian objective serves as a good metric (and that the second auction of our short example is ``better'' for our metric). Finally, their training procedure requires an inner-loop optimization (essentially, this inner loop is the bidders trying to maximize utility in the current auction), which is itself computationally expensive. We use amortized optimization to make this process more efficient.

 \section*{Contributions}

This chapter leverages recent work in Auction Theory to formulate the learning of revenue-optimal auctions as a two-player game. We develop a new algorithm ALGnet (Auction Learning Game network) that produces competitive or better results compared to \cite{dutting2017optimal}'s RegretNet. In addition to the conceptual contributions, our approach yields the following improvements (as RegretNet is already learning near-optimal auctions, our improvement over RegretNet is not due to significantly higher optimal revenues).

\begin{itemize}
    \item[--] \textit{Easier hyper-parameter tuning}: By constructing a time-independent loss function, we circumvent the need to search for an adequate parameter scheduling.  Our formulation also involves less hyperparameters, which makes it more robust.

    \item[--] \textit{A metric to compare auctions}: We propose a metric to compare the quality of two auctions which are not incentive compatible.

    \item[--] \textit{More efficient training}: We replace the inner-loop optimization of prior work with a neural network, which makes training more efficient. 

    \item[--] \textit{Online auctions}: Since the learning formulation is time-invariant, ALGnet is able to quickly adapt in auctions where the bidders' valuation distributions varies over time. Such setting appears for instance in the online posted pricing problem studied in \citet{bubeck2017online}. 
\end{itemize}

 Furthermore, these technical contributions together now imply a novel formulation of auction learning as a two-player game (not zero-sum) between an auctioneer and a misreporter. The auctioneer is trying to design an incentive compatible auction that maximizes revenue while the misreporter is trying to identify breaches in the truthfulness of these auctions. 

The chapter decomposes as follows. Section~\ref{sec:setting} introduces the standard notions of auction design. Section~\ref{sec:dutt} presents our game formulation for auction learning. Section~\ref{sec:NN_arch} provides a description of ALGnet and its training procedure. Finally, Section~\ref{sec:exps} presents numerical evidence for the effectiveness of our approach.

\section*{Related work}

\paragraph{Auction design and machine learning.} Machine learning and computational learning theory have been used in several ways to design auctions from samples of bidder valuations. Machine learning has been used to analyze the sample complexity of designing optimal revenue-maximizing auctions. This includes the framework of single-parameter settings \citep{morgenstern2015pseudo,huang2018making, HartlineT19,RoughgardenS16, GonczarowskiN17, GuoHZ19},  multi-item auctions \citep{dughmi2014sampling, GonczarowskiW18}, combinatorial auctions \citep{balcan2016sample,morgenstern2016learning,syrgkanis2017sample} 
and allocation mechanisms \citep{narasimhan2016general}. Other works have leveraged machine learning to optimize different aspects of mechanisms \citep{lahaie2011kernel,dutting2015payment}. Our approach is different as we build a deep learning architecture for auction design. 

\paragraph{Auction design and deep learning.} While \cite{dutting2017optimal} is the first paper to design auctions through deep learning, several other paper followed-up this work. \cite{feng2018deep} extended it to budget constrained bidders, \cite{golowich2018deep} to the facility location problem. \cite{tacchetti2019neural} built architectures based on the Vickrey-
Clarke-Groves mechanism. \cite{rahme2020permeq} used permutation-equivariant networks to design symmetric auctions.  \cite{shen2019automated} and \cite{dutting2017optimal} proposed architectures that \textit{exactly} satisfy incentive compatibility but are specific to \textit{single-bidder} settings. While all the previously mentioned papers consider a non-stationary objective function, we formulate a time-invariant objective that is easier to train and that makes comparisons between mechanisms possible.

\section{Auction design as a time-varying learning problem}\label{sec:setting}

We first review the framework of auction design and the problem of finding truthful mechanisms. We then recall the learning problem proposed by \cite{dutting2017optimal} to find optimal auctions. 

\subsection{Auction design and linear program}

\paragraph{Auction design.} We consider an auction with $n$ bidders and $m$ items. We will denote by $N=\{1,\dots,n\}$ and $M=\{1,\dots,m\}$ the set of bidders and items. Each bidder $i$ values item $j$ at a valuation denoted $v_{ij}$. We will focus on \emph{additive} auctions. These are auctions where the value of a set $S$ of items is equal to the sum of the values of the elements in that set at $\sum_{j \in S} v_{ij}$. Additive auctions are perhaps the most well-studied setting in multi-item auction design~\citep{HartN12, LiY13,daskalakis2014complexity, CaiDW16,daskalakis2017strong}.

The auctioneer does not know the exact valuation profile $V = (v_{ij})_{i\in N, j\in M}$ of the bidders in advance but he does know the distribution from which they are drawn: the valuation vector of bidder $i$, $\vec{v}_i=(v_{i1}, \dots, v_{im})$ is drawn from a distribution $D_i$ over $\mathbb{R}^m$. We will further assume that all bidders are independent and that $D_1 = \cdots = D_n $. As a result V is drawn from $D:= \otimes_{i=1}^n D_i = D_1^{\otimes^n}$.

\begin{definition}
 An auction is defined by a randomized allocation rule $g=(g_1,\dots,g_n)$ and a payment rule $p=(p_1,\dots,p_n)$ where $g_i\colon \mathbb{R}^{n\times m}\rightarrow [0,1]^{m}$ and $p_i\colon \mathbb{R}^{n\times m}\rightarrow \mathbb{R}_{\geq 0}$. Additionally for all items $j$ and valuation profiles $V$, the $g_i$ must satisfy $\sum_i [g_i(V)]_j\leq 1$.
 \end{definition}
 
Given a bid matrix $B=(b_{ij})_{i\in N, j\in M}$, $[g_i(B)]_j$
 is the probability that bidder $i$ receives object $j$ and $p_i(B)$ is the price bidder $i$ has to pay to the auction. The condition $\sum_i [g_i(V)]_j\leq 1$ allows the possibility for an item to be not allocated. 
 
 \begin{definition}
 The utility of bidder $i$ is defined by $u_i(\vec{v}_i,B)= \sum_{j=1}^m [g_i(B)]_j v_{ij} -p_i(B).$ 
\end{definition}

Bidders seek to maximize their utility and may report bids that are different from their true valuations.
In the following, we will denote by $B_{-i}$ the $(n-1)\times m$ bid matrix without bidder $i$, and by $(\vec{b}_i', B_{-i})$ the $n\times m$ bid matrix that inserts $\vec{b}_i'$ into row $i$ of $B_{-i}$ (for example: $B:= (\vec{b}_i, B_{-i})$.
We aim at auctions that incentivize bidders to bid their true valuations.

\begin{definition}\label{def:DSIC}
An auction $(g,p)$ is \textit{dominant strategy incentive compatible} (DSIC) if each bidder's utility is maximized by reporting truthfully no matter what the other bidders report. For every bidder $i,$ valuation $\vec{v}_i \in D_i$, bid $\vec{b}_i\hspace{.02cm}'\in D_i$ and bids $B_{-i}\in D_{-i}$,  $u_i(\vec{v}_i,(\vec{v}_i,B_{-i}))\geq u_i(\vec{v}_i,(\vec{b}_i\hspace{.02cm}',B_{-i})). $ 
\end{definition}
 
\begin{definition}\label{def:IR}
An auction is \textit{individually rational} (IR) if for all $i\in N, \; \vec{v}_i\in D_i$ and $B_{-i}\in D_{-i},$
\begin{equation}\label{eq:IR_eq}\tag{IR}
   u_i(\vec{v}_i,(\vec{v}_i,B_{-i}))\geq 0.  
\end{equation}
\end{definition}

In a DSIC auction, the bidders have the incentive to truthfully report their valuations and therefore, the revenue on valuation profile $V$ is $ \sum_{i=1}^n p_i(V).$ Optimal auction design aims at finding a DSIC and IR auction that maximizes the expected revenue $rev:= \mathbb{E}_{V\sim D}[\sum_{i=1}^n p_i(V)]$.  Since there is no known characterization of DSIC mechanisms in the multi-item setting, we resort to the relaxed notion of \textit{ex-post regret}. It measures the extent to which an auction violates DSIC.

\begin{definition}
The ex-post regret for a bidder $i$ is the maximum increase in his utility when considering all his possible bids and fixing the bids of others. For a valuation profile $V$, it is given by $r_{i}(V)=\max_{\vec{b}_i\hspace{.02cm}'\in \mathbb{R}^m} u_i(\vec{v}_i,(\vec{b}_i\hspace{.02cm}',V_{-i})) - u_i(\vec{v}_i,(\vec{v}_i,V_{-i}))$.  In particular, DSIC is equivalent to 
\begin{equation}\label{eq:regretDSIC}\tag{IC} 
    r_{i}(V)=0, \;  \forall i \in N,  \forall \, V \in D.
\end{equation}

The bid $\vec{b}_i'$ that achieves $r_i(V)$ is called the optimal misreport of bidder $i$ for valuation profile $V$.
\end{definition}

Therefore, finding an optimal auction is equivalent to the following linear program:

\begin{equation*}\tag{LP}\label{eq:exact_prob}
\begin{aligned}
\hspace{-.51cm} \vspace{-1.81cm} \underset{(g,p)\in \mathcal{M} }{\text{min}}
  - \mathbb{E}_{V\sim D}\left[\sum_{i=1}^n p_i(V)\right] \quad \text{s.t.} \quad 
&   r_{i}(V)=0, \hspace{1.55cm} \forall ~i \in N,\;  \forall ~V \in D, \\[-3.5\jot]
& u_i(\vec{v}_i,(\vec{v}_i,B_{-i}))\geq 0, \,\, \forall i\in N,\; \vec{v}_i\in D_i,B_{-i}\in D_{-i}.
\end{aligned}
\end{equation*}

\subsection{Auction design as a learning problem}\label{par:auction}
As the space of auctions $\mathcal{M}$ may be large, we will set a parametric model. In what follows, we consider the class of  auctions $(g^w,p^w)$ encoded by a neural network of parameter $w\in \mathbb{R}^d$.  The corresponding utility and regret function will be denoted by $u_i^w$ and $r_i^w$.

Following \cite{dutting2017optimal}, the formulation  ~\eqref{eq:exact_prob} is relaxed: the IC constraint for all $V\in D$ is replaced by the expected constraint $\mathbb{E}_{V\sim D}[r_i^w(V)]=0$ for all $i\in N.$ The justification for this relaxation can be found in \cite{dutting2017optimal}. By replacing expectations with empirical averages, the learning problem becomes:

 \begin{equation}\label{eq:empirical_prob}\tag{$\widehat{\mathrm{LP}}$}
\begin{aligned}
\underset{w\in \mathbb{R}^d}{\text{min}}
-\frac{1}{L}\sum_{\ell=1}^L \sum_{i=1}^n p_i^w(V^{(\ell)}) \quad \text{s.t.}\quad\widehat{r}_i^w:=\frac{1}{L}\sum_{\ell=1}^L r_i^w(V^{(\ell)})=0,\; \forall i \in N. 
\end{aligned}
\end{equation}
 
The learning problem \eqref{eq:empirical_prob} does not ensure \eqref{eq:IR_eq}. However, this constraint is usually built into the parametrization (architecture) of the model: by design, the only auction mechanism considered satisfy \eqref{eq:IR_eq}.  

Implementation details can be found in  \cite{dutting2017optimal,rahme2020permeq} or in Sec~\ref{sec:NN_arch}.

\section{Auction learning as a two-player game}\label{sec:dutt}

We first present the optimization and the training procedures for \eqref{eq:empirical_prob} proposed by \cite{dutting2017optimal}. We then demonstrate with numerical evidence that this approach presents two limitations: hyperparameter sensitivity and lack of interpretability. Using the concept of $\varepsilon$-truthful to truthful reductions, we construct a new loss function that circumvents these two aspects. Lastly, we resort to amortized optimization and re-frame the auction learning problem as a two-player game.

 \subsection{The augmented Lagrangian method and its shortcomings}\label{subsec:dutt}

\paragraph{Optimization and training.} We briefly review the training procedure proposed by \cite{dutting2017optimal} to learn optimal auctions. The authors apply the augmented Lagrangian method to solve the constrained problem  \eqref{eq:empirical_prob} and consider the loss:
\begin{equation*}
\begin{aligned}
    \mathcal{L}(w;\lambda; \rho)=-\frac{1}{L}\sum_{\ell=1}^L \sum_{i\in N} p_i^w(V^{(\ell)})+\sum_{i\in N}\lambda_i r_i^w(V^{(\ell)})+\frac{\rho}{2}\sum_{i \in N}\left(r_i^w(V^{(\ell)})\right)^2, 
\end{aligned}
\end{equation*}
where $\lambda \in\mathbb{R}^n$ is a vector of Lagrange multipliers and $\rho>0$ is a parameter controlling the weight of the quadratic penalty. More details about the training procedure can be found in Appendix A.

\paragraph{Scheduling consistency problem.} The parameters $\lambda$ and $\rho$ are time-varying. Indeed, their value changes according to a pre-defined scheduling of the following form: 1) Initialize $\lambda$ and $\rho$ with respectively   $\lambda^0$ and $\rho^0$,
2)  Update $\rho$ every $T_{\rho}$ iterations : $\rho^{t+1} \leftarrow \rho^{t} + c ,$  where $c$ is a pre-defined constant,
3) Update $\lambda$ every $T_{\lambda}$ iterations according to $\lambda_i^t \leftarrow \lambda_i^t + \rho^t \, \widehat{r}_i^{w^{t}} $.

Therefore, this scheduling requires to set up five hyper parameters $(\lambda^0,\rho^0,c,T_{\lambda},T_{\rho})$. Some of the experiments found \cite{dutting2017optimal} were about learning an optimal mechanism for an $n $-bidder $m$-item auction ($n \times m$) where the valuations are iid $\mathcal{U}[0,1]$. Different scheduling parameters were used for different values of $n$ and $m$.  We report the values of the hyper parameters used for the $1 \times 2$, $3 \times 10$ and $5 \times 10$ settings in \autoref{tab:scheduling}(a). A natural question is whether the choice of parameters heavily affects the performance. We proceed to a numerical investigation of this questions by trying different schedulings (columns) for different settings (rows) and report our the results in \autoref{tab:scheduling}(b).

\begin{table}[H]
\caption{(a): Scheduling parameters values set in \cite{dutting2017optimal} to reach optimal auctions in $n\times m$ settings with $n$ bidders, $m$ objects and i.i.d. valuations sampled from $\mathcal{U}[0,1].$ (b): Revenue $\mathit{rev}:=\mathbb{E}_{V\sim D}[\sum_{i=1}^n p_i(V)]$ and average regret per bidder $\mathit{reg}:= \nicefrac{1}{n} \, \mathbb{E}_{V \in D}\left[\sum_{i=1}^n r_i(V)\right] $ for $n\times m$ settings when using the different parameters values set reported in (a). }\label{tab:scheduling}  
\begin{minipage}{0.3\textwidth}
{\hspace{-.0cm}\resizebox{0.95\columnwidth}{!}{ 
\begin{tabular}{lccc}
\toprule
\multirow{1}{*}{ } & 
\multicolumn{1}{c}{ $1 \times 2$} & \multicolumn{1}{c}{ $3 \times 10$} & \multicolumn{1}{c}{ $5 \times 10$}
\\
\midrule
$\lambda^0$ & 5 & 5 & 1 \\
% \midrule
$\rho^0 $ & 1 & 1 & 0.25 \\
% \midrule
c & 50 & 1 & 0.25 \\
% \midrule
$T_{\lambda}$ & $10^2$ & $10^2$ & $10^2$ \\
% \midrule
$T_{\rho}$ & $10^4$ & $10^4$ & $10^5$ \\
\bottomrule
\end{tabular}}}
            \caption*{(a) }%Scheduling Parameters.
\end{minipage}
%\hfillx
\begin{minipage}{0.7\textwidth}

{\resizebox{0.95\columnwidth}{!}{ \begin{tabular}{lccccccccc}
\toprule
% \multirow{2}{*}{\diagbox[width=7.2em]{Setting}{~~Schedule}} 
\multicolumn{1}{c}{ }
& \multicolumn{6}{c}{ Scheduling} \\
\cmidrule(lr){2-7}
\multicolumn{1}{c}{}
& \multicolumn{2}{c}{ $1 \times 2$} & \multicolumn{2}{c}{ $3 \times 10$} & \multicolumn{2}{c}{ $5 \times 10$}
\\
% \cline{2-7}
    % \midrule
 \cmidrule(r){2-3}
  \cmidrule(r){4-5}
  \cmidrule(r){6-7}
Setting & $\mathit{rev}$ & $\mathit{rgt}$ & $\mathit{rev}$ & $\mathit{rgt}$ & $\mathit{rev}$ & $\mathit{rgt}$
\\
\cmidrule(lr){1-1}
\cmidrule(lr){2-2}
\cmidrule(lr){3-3}
\cmidrule(lr){4-4}
\cmidrule(lr){5-5}
\cmidrule(lr){6-6}
\cmidrule(lr){7-7}

% \hline
 $1 \times 2$
& 0.552 & 
 0.0001 & 0.573 & 0.0012 & 0.332 & 0.0179
\\
% \hline

 $3 \times 10$  
& 4.825 & 0.0007 & 5.527 & 0.0017 & 5.880 & 0.0047 
\\
% \hline

 $5 \times 10$ 
& 4.768 & 0.0006 & 5.424 & 0.0033 & 6.749 & 0.0047 
\\
\bottomrule

\end{tabular}}}

\caption*{(b)}%Revenue and Regret for different Schedules
        \end{minipage}
\end{table}

The auction returned by the network dramatically varies with the choice of scheduling parameters. When applying the  parameters of $1\times 2$ to $5\times 10$, we obtain a revenue that is lower by 30\%! The performance of the learning algorithm strongly depends on the specific values of the hyperparameters. Finding an adequate scheduling requires an extensive and time consuming hyperparameter search.

\paragraph{Lack of interpretability.} 

How should one compare two mechanisms with different expected revenue and regret? Is a mechanism $M_1$ with revenue $P_1 = 1.01$ and an average total regret $R_1 = 0.02$ better than a mechanism $M_2$ with $P_2 = 1.0$ and $R_2 = 0.01$?  The approach in \cite{dutting2017optimal} cannot answer this question. To see that, notice that when $\lambda_1 = \cdots = \lambda_n =\lambda$  we can rewrite ${\mathcal{L}(w;\lambda; \rho) = - P + \lambda R + \frac{\rho}{2} R^2}$. Which mechanism is better depends on the values of $\lambda$ and $\rho$. For example if $\rho=1$ and $\lambda =0.1$ we find that $M_1$ is better, but if $\rho=1$ and $\lambda =10$ then $M_2$ is better. Since the values of $\lambda$ and $\rho$ change with time, the Lagrangian approach in \cite{dutting2017optimal} cannot provide metric to compare two mechanisms.

\subsection{A time-independent and interpretable loss function for auction learning}\label{sec:new_loss}

Our first contribution consists in introducing a new loss function for auction learning that addresses the two first limitations of \cite{dutting2017optimal} mentioned in Section~ \ref{subsec:dutt}. We first motivate this loss in the one bidder case and then extend it to auctions with many bidders.

\subsubsection{Mechanisms with one bidder}
\label{subsection:onebidder}

\begin{proposition}
\label{prop:main}[\cite{BalcanBHY05}, attributed to Nisan]
Let $\mathcal{M}$ be an additive auction with $1$~bidder and $m$~items. Let $P$ and $R$  denote the expected revenue and regret, $P  = \mathbb{E}_{V \in D}\left[ p(V) \right]$ and $R = \mathbb{E}_{V \in D}\left[ r(V)\right]$.
There exists a mechanism $\mathcal{M}^{*}$ with expected revenue $P^{*} = (\sqrt{P }-\sqrt{R })^2$ and zero regret $R^{*} = 0$.
\end{proposition}

A proof of this proposition can be found in Appendix C. Comparing two mechanisms is straightforward when both of them have zero-regret: the best one achieves the highest revenue. \autoref{prop:main} allows a natural and simple extension of this criteria for non zero-regret mechanism with one bidder: { we will say that $M_1$ is better than $M_2$ if and only if $M_1^{*}$ is better than $M_2^{*}$:}

\begin{equation}
    M_1 \geq M_2 \iff P^{*}(M_1) \geq P^{*}(M_2) \iff \sqrt{P_1}-\sqrt{R_1}  \geq \sqrt{P_2}-\sqrt{R_2}
\end{equation}
{ Using our metric, we find that a one bidder mechanism with revenue of $1.00$ and regret of $0.01$ is ''better'' than one with revenue $1.01$ and regret $0.02$.}

\subsubsection{Mechanisms with multiple bidders}

Let $M_1$ and $M_2$ be two mechanisms with $n$ bidders and $m$ objects. Let $P_i$ and $R_i$  denote their total expected revenue and regret, $P_i  = \mathbb{E}_{V \in D}\left[ \sum_{j=1}^n p_j(V) \right]$ and $R_i = \mathbb{E}_{V \in D}\left[\sum_{j=1}^n r_j(V)\right]$. {We can extend  our metric}  derived in Section~\ref{subsection:onebidder} to the multiple bidder by the following: 
\begin{equation}
  M_1 \text{ is ''better'' than } M_2 \iff   M_1 \geq M_2 \iff  \sqrt{P_1}-\sqrt{R_1}  \geq \sqrt{P_2}-\sqrt{R_2}
\end{equation}
When $n=1$ we recover the criteria from Section~\ref{subsection:onebidder} that is backed by \autoref{prop:main}. When $n > 1$, it is considered a major open problem whether the extension of \autoref{prop:main} still holds. Note that a multi-bidder variant of \autoref{prop:main} \emph{does} hold under a different solution concept termed ``Bayesian Incentive Compatible''~\citep{rubinstein2018simple,CaiOVZ19}, supporting the conjecture that \autoref{prop:main} indeed extends.\footnote{An auction is \emph{Bayesian Incentive Compatible} if every bidder maximizes their expected utility by truthful reporting \emph{in expectation over the other bidders' truthful bids}. Compare this to Dominant Strategy Incentive Compatible (our work), where every bidder maximizes their expected utility by truthful reporting \emph{for all realizations of the other bidders' bids.}} Independently of whether or not \autoref{prop:main} holds, this reasoning implies a candidate loss function for the multi-bidder setting which we can evaluate empirically.

This way of comparing mechanisms motivates the use of loss function: $\mathcal{L}(P,R) = - (\sqrt{P}-\sqrt{R})$ instead of the Lagrangian from Section~\ref{sec:dutt}, and indeed this loss function works well in practice. We empirically find the loss function $\mathcal{L}_{m}(P,R) = -(\sqrt{P}-\sqrt{R}) + R$ further accelerates training, as it further (slightly) biases towards mechanisms with low regret. Both of these loss function are time-independent and hyperparameter-free.

\subsection{Amortized misreport optimization}\label{sec:amort}

To compute the regret $r_i^w(V)$ one has to solve the optimization problem:  
$$\max_{\vec{v}_i\hspace{.02cm}'\in \mathbb{R}^m} u_i^w(\vec{v}_i,(\vec{v}_i\hspace{.02cm}',V_{-i})) - u_i^w(\vec{v}_i,(\vec{v}_i,V_{-i}))$$. 
In \cite{dutting2017optimal}, this optimization problem is solved with an inner optimization loop for each valuation profile. In other words, computing the regret of each valuation profile is solved separately and independently, from scratch. 

If two valuation profiles are very close to each other, one should expect that the resulting optimization problems to have close results. We leverage this to improve training efficiency.

We propose to amortize this inner loop optimization. Instead of solving all these optimization problems independently, we will instead learn one neural network $M^{\phi}$ that tries to predict the solution of all of them. $M^{\phi}$ takes as entry a valuation profile and maps it to the optimal misreport:
\begin{equation}
M^{\phi}:
\begin{cases}
\mathbb{R}^{n \times m} & \to \mathbb{R}^{n \times m}\\ 
V =\left[\vec{v_i}\right]_{i \in N} & \to \left[ \mbox{argmax}_{\vec{v}' \in D} u_i(\vec{v_i},(\vec{v}',V_{-i}))\right]_{i \in N}
\end{cases}
\end{equation}
The loss $\mathcal{L}_{r}$ that $M^{\phi}$ is trying to minimize follows naturally from that definition and is then given by: $
\mathcal{L}_{r}(\phi,w) = - \mathbb{E}_{V \in D} \left[\sum_{i=1}^n u^w_i(\vec{v_i},([M^{\phi}(V)]_i,V_{-i})) \right].$

\subsection{Auction learning as a two-player game}

  In this section, we combine the ideas from Sections~\ref{sec:new_loss} and \ref{sec:amort} to obtain a new formulation for the auction learning problem as a two-player game between an Auctioneer with parameter $w$ and a Misreporter with parameter $\varphi$. The optimal parameters for the auction learning problem $(w^{*},\varphi^{*})$ are a Nash Equilibrium for this game.
  
  The Auctioneer is trying to design a truthful \eqref{eq:regretDSIC} and rational \eqref{eq:IR_eq} auction that maximizes revenue. The Misreporter is trying to maximize the bidders' utility, for the current auction selected by Auctioneer, $w$. This is achieved by minimizing the loss function $\mathcal{L}_{r}(\phi,w)$ wrt to $\phi$ (as discussed in Sec~\ref{sec:amort}). The Auctioneer in turn maximizes expected revenue, for the current misreports as chosen by Misreporter. This is achieved by minimizing $\mathcal{L}_{m}(w,\varphi) = -(\sqrt{P^{w}}+\sqrt{R^{w,\varphi}}) + R^{w,\varphi}$  with respect to $w$ (as discussed in Sec~\ref{sec:new_loss}). Here, $R^{w,\varphi}$ is an estimate of the total regret that auctioneer computes for the current Misreporter $\varphi$, 
  $R^{w,\varphi} = \frac{1}{L}\sum_{\ell=1}^L \sum_{i\in N}\left(u_i^w(\vec{v}_i,([M^{\phi}(V)]_i ,V_{-i}))- u_i^w(\vec{v}_i,(\vec{v}_i,V_{-i}))\right).$ This game formulation can be summarized in Figure~\ref{fig:gamefigure}.

\begin{figure}

\begin{minipage}[t]{0.90\textwidth}

  \centering\raisebox{\dimexpr \topskip-\height}{%
  \includegraphics[width=\textwidth]{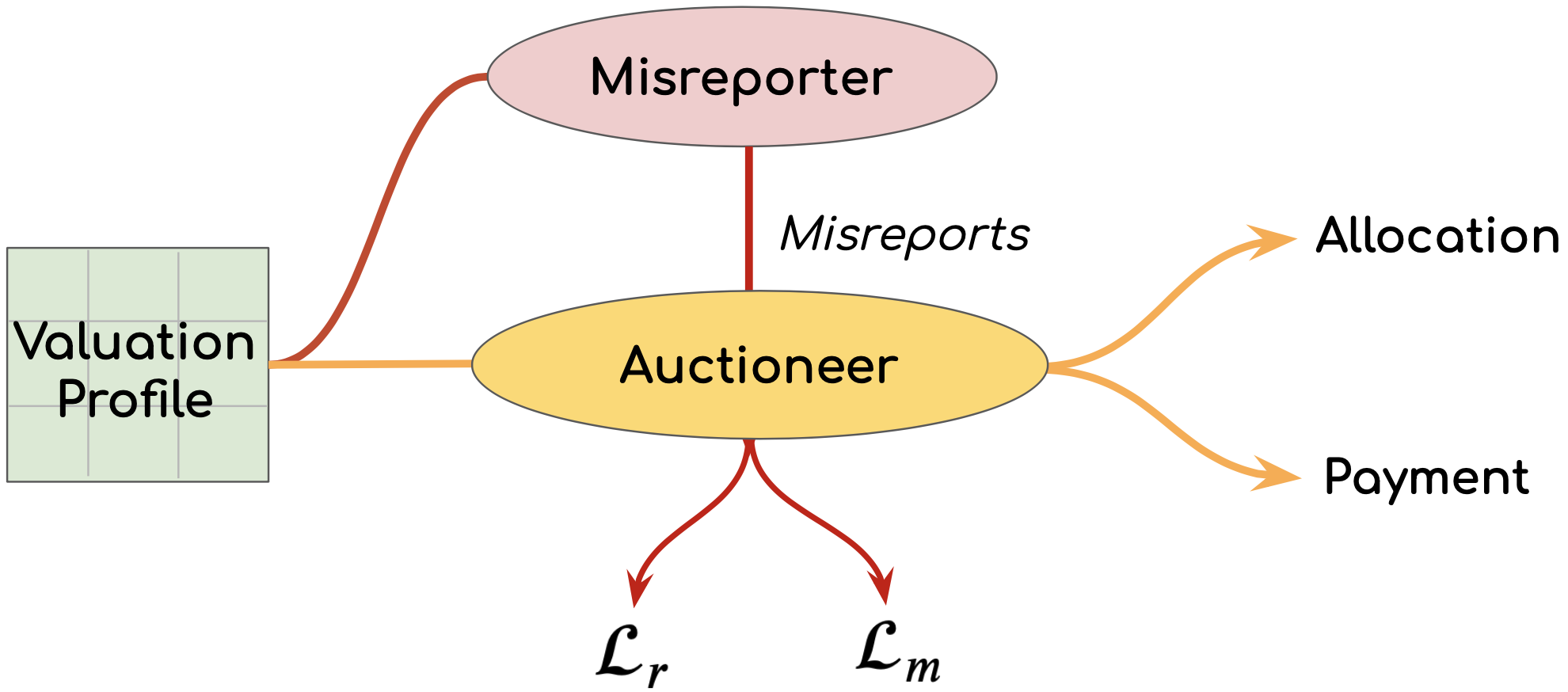}}
  \label{fig1}
\end{minipage}\hfill  \\

\begin{minipage}[t]{0.90\textwidth}
\begin{align}
\label{eq:game}\tag{G} 
\begin{aligned}
  \text{Misreporter: } \begin{cases}
    \textbf{loss: } \mathcal{L}_r(\varphi,w) \\
    \textbf{parameter: }\varphi
    \end{cases}  
    %\min_{\varphi} \mathcal{L}_r(R(w,\varphi))
    \quad \text{Auctioneer: } 
    %\min_{w} \mathcal{L}_m(P(w),R(w,\varphi))
    \begin{cases}
    \textbf{loss: } \mathcal{L}_m(w,\varphi) \\
    \textbf{parameter: } w 
    \end{cases}
\end{aligned}
\end{align}
\end{minipage}
\caption{Diagrammatic representation of the two-play auction learning game.}\label{fig:gamefigure}

\end{figure}

\bigskip
\begin{remark}
 The game formulation \eqref{eq:game} reminds us of Generative Adversarial Networks \citep{goodfellow2014generative}. Contrary to GANs, it is not a zero-sum game. 
\end{remark}

\section{Architecture and training procedure}\label{sec:NN_arch}

We describe ALGnet, a feed-forward architecture solving for the game formulation \eqref{eq:game} and then provide a training procedure. ALGnet consists in two modules that are the auctioneer's module and the misreporter's module. These components take as input a bid matrix $B = (b_{i,j})\in \mathbb{R}^{n\times m}$ and are trained jointly. Their outputs are used to compute the regret and revenue of the auction. 

\paragraph{Notation.} We use  $ \text{MLP}(d_{\text{in}},n_l,h,d_{\text{out}})$ to refer to  a fully-connected neural network  with input dimension $d_{\text{in}}$, output dimension $d_{\text{out}}$ and $n_l$ hidden layers of width $h$ and $\mathrm{tanh}$ activation function. $\mathrm{sig}$ denotes the sigmoid activation function. Given a matrix $B=[\vec{b}_1,\dots,\vec{b}_n]^{\top}\in \mathbb{R}^{n\times m},$ we define for a fixed $i\in N$, the matrix $B_{(i)}:=[\vec{b}_i,\vec{b}_1,\dots,\vec{b}_{i-1},\vec{b}_{i+1},\dots,\vec{b}_n].$

\subsection{The Auctioneer's module}

 It is composed of an allocation network that encodes a  randomized  allocation $g^w\colon \mathbb{R}^{nm} \to [0,1]^{nm}$ and a  payment network that encodes a payment rule $p^w\colon \mathbb{R}^{nm} \to \mathbb{R}^{n}$.

\paragraph{Allocation network.} It computes the allocation probability of item $j$ to bidder~$i$, $[g^w(B)]_{ij}$, as $[g^w(B)]_{ij} = [f_1(B)]_j\cdot [f_2(B)]_{ij}$ where $f_1\colon \mathbb{R}^{n\times m} \to [0,1]^{m}$ and $f_2\colon\mathbb{R}^{n\times m} \to [0,1]^{m \times n }$ are functions 
computed by two feed-forward neural networks.\\
\hspace*{.2cm} -- $[f_1(B)]_j$ is the probability that object $j \in M$ is allocated and is given by ${[f_1(B)]_j = \mathrm{sig}\left(\text{MLP}(nm,n_{a},h_a,n)\right)}$. \\
\hspace*{.3cm} --\; $[f_2(B)]_{ij}$ is the probability that item $j\in M$ is allocated to bidder $i\in N$ conditioned on object $j$ being allocated. A first MLP computes  $l_j:=\text{MLP}(nm,n_{a},h_a,m)(B_{(j)})$ for all $j\in M$. The network then concatenates all these vectors $l_j$ into a matrix $L\in \mathbb{R}^{n\times m}$. A softmax activation function is finally applied to $L$ to ensure feasibility i.e. for all $j\in M, \sum_{i\in N}L_{ij} = 1.$

\paragraph{Payment network.}

It computes the payment $[p^w(B)]_i$ for bidder $i$ as $[p^w(B)]_i = \tilde{p}_i \sum_{j=1}^m B_{ij}[g^w(B)]_{ij},$ where $\tilde{p}\colon \mathbb{R}^{n\times m}\rightarrow [0,1]^n.$ $\tilde{p}_i$ is the fraction of bidder’s $i$ utility that she has to pay to the mechanism. We  compute $ \tilde{p}_i = \mathrm{sig}\left(\text{MLP}(nm,n_{p},h_p,1)\right)(B_{(i)}).$  Finally, 
notice that by construction $[p^w(B)]_i \leq \sum_{j=1}^m B_{ij}g^w(B)_{ij}$ which ensures that \eqref{eq:IR_eq} is respected.

\subsection{The Misreporter's module}
The module consists in an $\text{MLP}(nm,n_{M},h_M,m)$ followed by a projection layer $\mathrm{Proj}$ that ensure that the output of the network is in the domain $D$ of the valuation.  For example when the valuations are restricted to $[0,1]$, we can take $\mathrm{Proj=sig},$ if they are non negative number,we can take $\mathrm{Proj=SoftPlus}$. The optimal misreport for bidder $i$ is then given by $\text{Proj} \circ \text{MLP}(nm,n_{M},h_M,m)(B_{(i)}) \in \mathbb{R}^m$. Stacking these vectors gives us the misreport matrix $M^{\phi}(B)$.

\subsection{Training procedure and optimization}
We optimize the game \eqref{eq:game} over the space of  neural networks parameters $(w,\varphi).$ The algorithm is easy to implement (Algorithm~\autoref{alg:gameAlgorithm}).
 At each time $t$, we sample a batch of valuation profiles of size $B$.  The algorithm performs $\tau$ updates for the Misreporter's network (line 9) and one update on the Auctioneer's network (line 10). Moreover, we often reinitialize the Misreporter's network every $T_{init}$ steps in the early phases of the training ($t\leq T_{limit}$). This step is not necessary but we found empirically that it speeds up training. 

\begin{algorithm}[H]
\caption{ALGnet training} 
\label{alg:gameAlgorithm}
\begin{algorithmic}[1]
%\State Initialization:
\State \textbf{Input}: number of agents, number of objects.
\State \textbf{Parameter}: $\gamma >0 ;\; B, T, T_{init},  T_{limit}, \tau \in \mathbb{N}.$
\State \textbf{Initialize} misreport's and auctioneer's nets.
%\State Main routine:
\For {$t=1,\dots,T$}
    \If{$t \equiv 0\; \mathrm{mod}\;  T_{init}$ and $t < T_{Limit}$}:\\
      \hspace{1cm}  Reinitialize Misreport Network
        \EndIf
\State Sample valuation batch $S$ of size $B$.
\For{$s=1,\dots,\tau$}
    \State \hspace{-.2cm}\mbox{ $\phi^{s+1}\leftarrow \phi^{s}-\gamma \nabla_\phi \mathcal{L}_{r}(\phi^{s},w^t)(S). $}
    \EndFor
\State \mbox{\hspace{-.1cm} $w^{t+1}\leftarrow w^t-\gamma \nabla_w \mathcal{L}_{m}(w^t,\phi)(S). $}
\EndFor
\end{algorithmic}
\end{algorithm}

\section{Experimental Results}\label{sec:exps}

We show that ALGnet  can  recover  near-optimal  auctions for settings where  the  optimal  solution  is known and that it  can find new auctions for settings where analytical solution are not known. Since RegretNet is already capable of discovering near optimal auctions, one cannot expect ALGnet to achieve significantly higher optimal revenue than RegretNet. The results obtained are competitive or better than the ones obtained in \cite{dutting2017optimal} while requiring much less hyperparameters (Section~\ref{sec:dutt}). 

We also evaluate ALGnet in online auctions and compare it to RegretNet.

For each experiment, we compute the total revenue $\mathit{rev}:=\mathbb{E}_{V\sim D}[\sum_{i\in N} p_i^{w}(V)]$ and average regret $\mathit{rgt}:=\nicefrac{1}{n}\,\mathbb{E}_{V\sim D}[\sum_{i\in N} r_i^{w}(V)]$ on a test set of $10,000$ valuation profiles. We run each experiment 5 times with different random seeds and report the average and standard deviation of these runs. In our comparisons we make sure that ALGnet and RegretNet have similar sizes for  fairness (Appendix~D).

\subsection{Auctions with known and unknown optima}

\paragraph{Known settings. } We show that ALGnet is capable of recovering near optimal auction in different well-studied auctions that have an analytical solution. These are one bidder and two items auctions where the valuations of the two items $v_1$ and $v_2$ are independent. We consider the following settings:
\begin{itemize}
    \item (A): $v_1$ and $v_2$ are  i.i.d. from  $\mathcal{U}[0,1]$
    \item (B): $v_1 \sim \mathcal{U}[4,16]$ and $v_2\sim \mathcal{U}[4,7]$
    \item  (C): $v_1$ has density $f_1(x)=5/(1+x)^6$ and $v_2$ has density $f_2(y)=6/(1+y)^7.$ 
\end{itemize}

(A) is the celebrated Manelli-Vincent auction \citep{manelli2006bundling}; (B) is a non-i.i.d. auction and (C) is a non-i.i.d. heavy-tail auction and both of them are studied in \citet{daskalakis2017strong}. We compare our results to the theoretical optimal auction (\autoref{fig:table_small_auct}). (\citet{dutting2017optimal} does not evaluate RegretNet on settings (B) \& (C)).   During the training process, $\mathit{reg}$ decreases to 0 while  $\mathit{rev}$ and $P^{*}$ converge to the optimal revenue.  For (A), we also plot $\mathit{rev}$, $\mathit{rgt}$ and $P^{*}$ as function of the number of epochs and we compare it to RegretNet (\autoref{fig:comparaisonRegretNet}).

 Contrary to ALGnet, we observe that RegretNet overestimates the revenue in the early stages of training at the expense of a higher regret. As a consequence, ALGnet learns the optimal auction faster than RegretNet while being schedule-free and requiring less hyperparameters.

\begin{table}[H]
\centering
\caption{Revenue \& regret of ALGnet for settings (A)-(C). }
\label{fig:table_small_auct}
    \begin{tabular}{lccccccc}
     \toprule
    \multirow{2}{*}{} & 
    \multicolumn{2}{c}{Optimal} & \multicolumn{2}{c}{ALGnet (Ours)}
    \\
    \cmidrule(lr){2-3}
    \cmidrule(lr){4-5}
    & $\mathit{rev}$ & $\mathit{rgt}$ & $\mathit{rev}$ & $\mathit{rgt ~~(\times 10^{-3})}$
    \\
    \cmidrule(lr){2-2}
    \cmidrule(lr){3-3}
    \cmidrule(lr){4-4}
    \cmidrule(lr){5-5}
    (A)
    & $0.550$ & 
    $0$ & 0.555 $(\pm 0.0019)$ & $0.55~(\pm 0.14)$ 
    \\
    % \hline
    (B)
    & 9.781 & 0  & 9.737 $(\pm 0.0443)$ & $0.75~(\pm 0.17)$  
    \\
    % \hline
    
    (C)
    & 0.1706 & $ 0 $& 0.1712 $(\pm 0.0012)$ & $ 0.14~(\pm 0.07)$
    \\ 
    \hline
    \end{tabular}
\end{table}

\begin{figure}[h]

\begin{subfigure}{0.5\textwidth}
  \centering
   \includegraphics[width=1.05\linewidth]{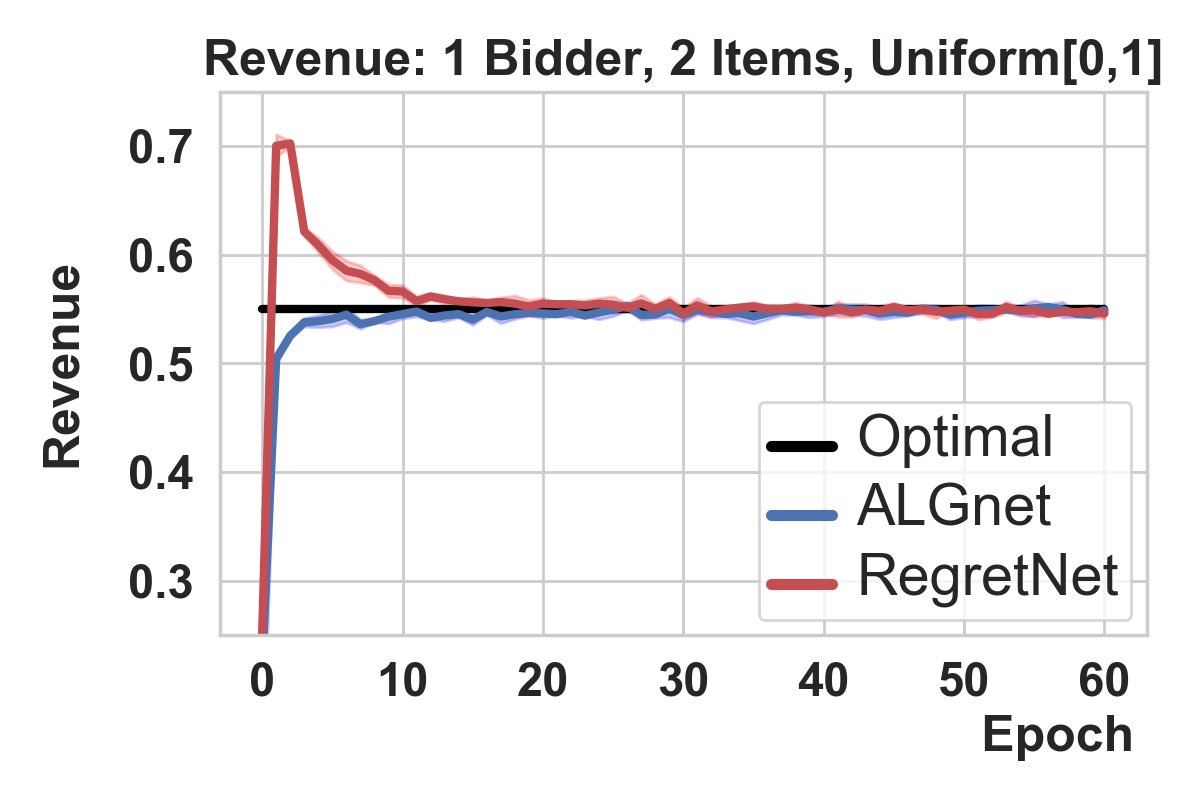}
  \caption{}
  \label{fig:1}
\end{subfigure}
\begin{subfigure}{0.5\textwidth}
  \centering
  \includegraphics[width=1.05\linewidth]{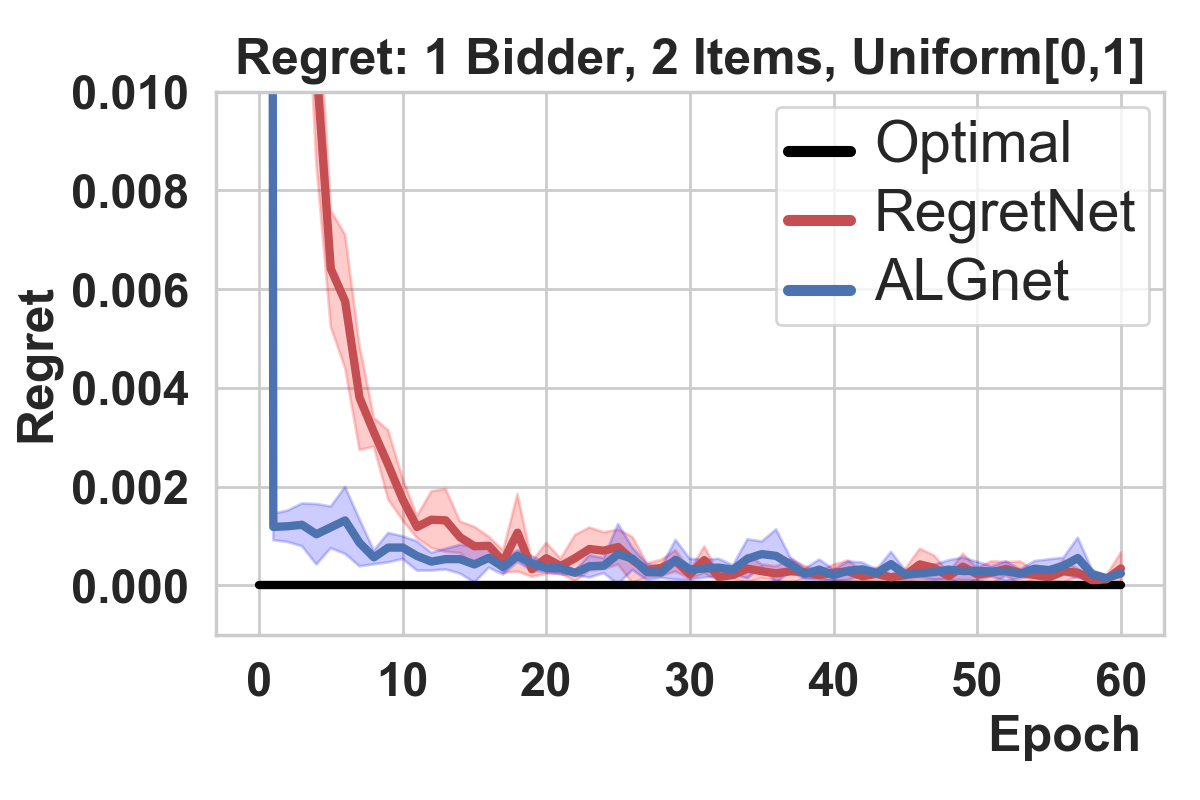}
  \caption{}
  \label{fig:2}
\end{subfigure}
\begin{subfigure}{0.5\textwidth}
  \centering
  \includegraphics[width=1.05\linewidth]{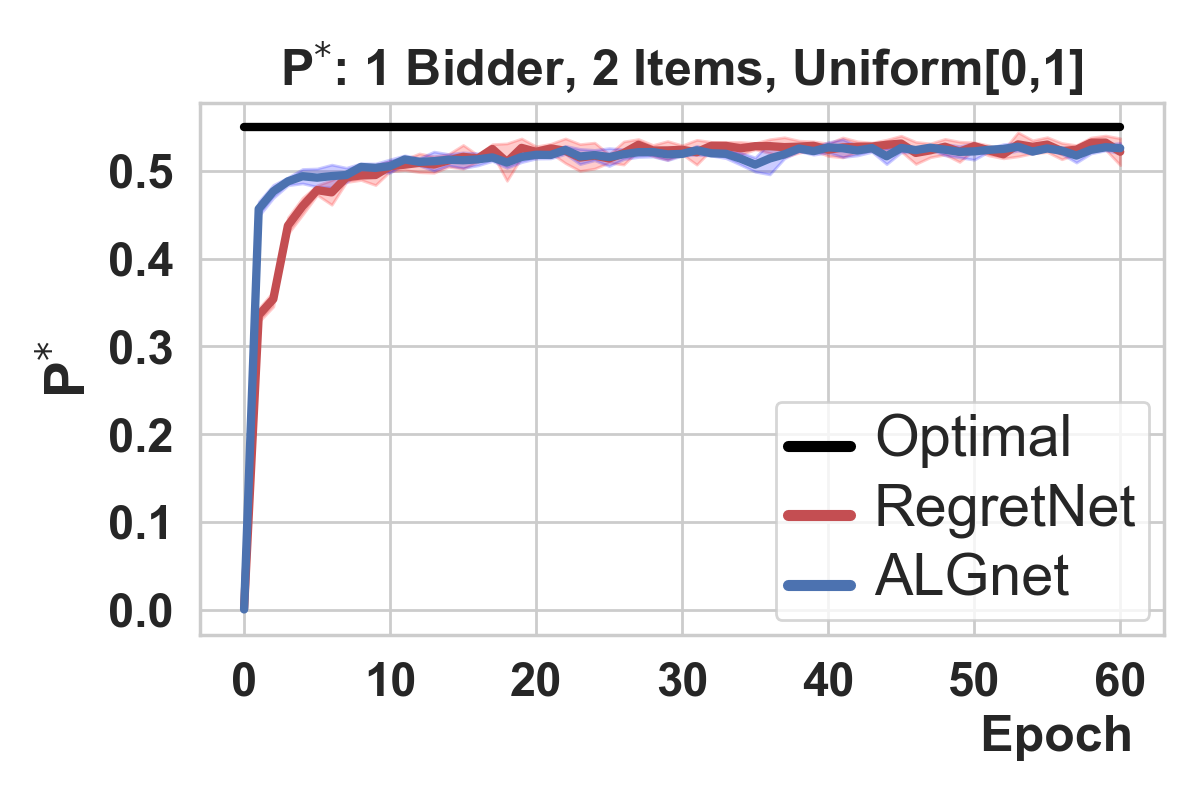} 
  \caption{}
  \label{fig:3}
\end{subfigure}
\begin{subfigure}{0.5\textwidth}
  \centering
  \includegraphics[width=1.05\linewidth]{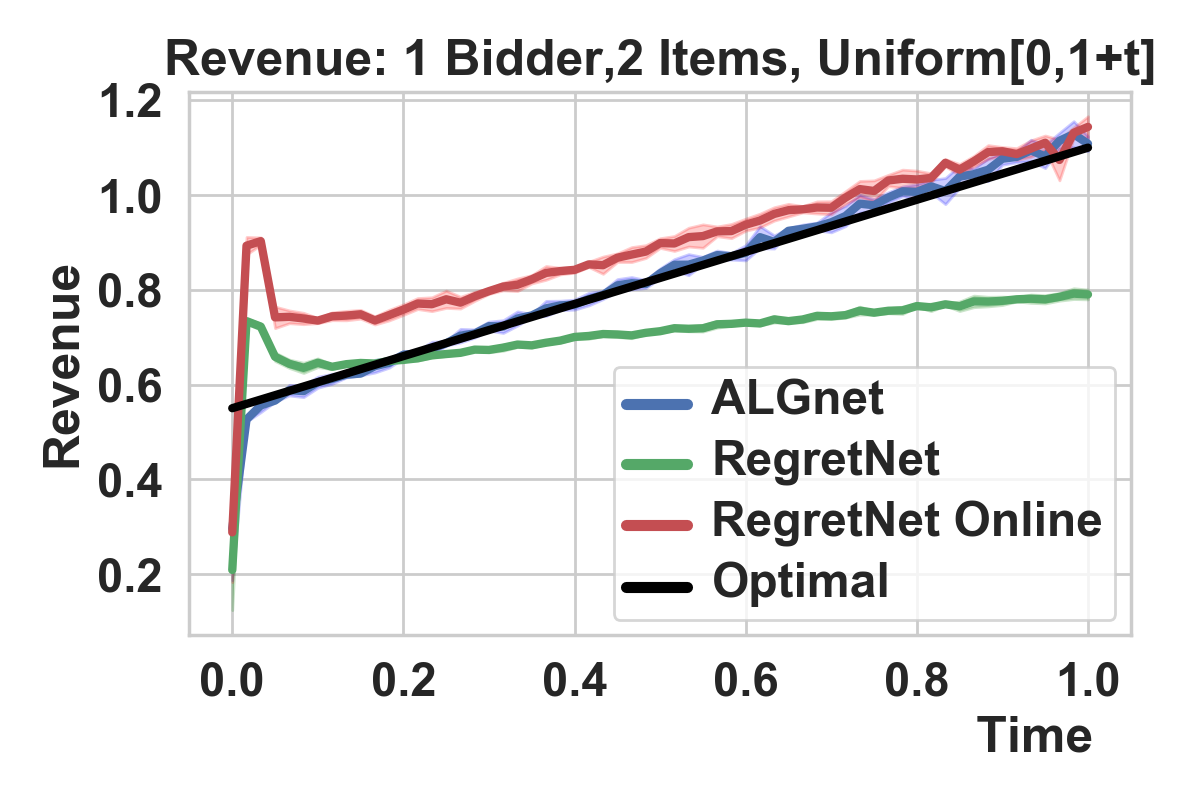}
  \caption{}
  \label{fig:4}
\end{subfigure}
\begin{subfigure}{0.5\textwidth}
  \centering
  \includegraphics[width=1.05\linewidth]{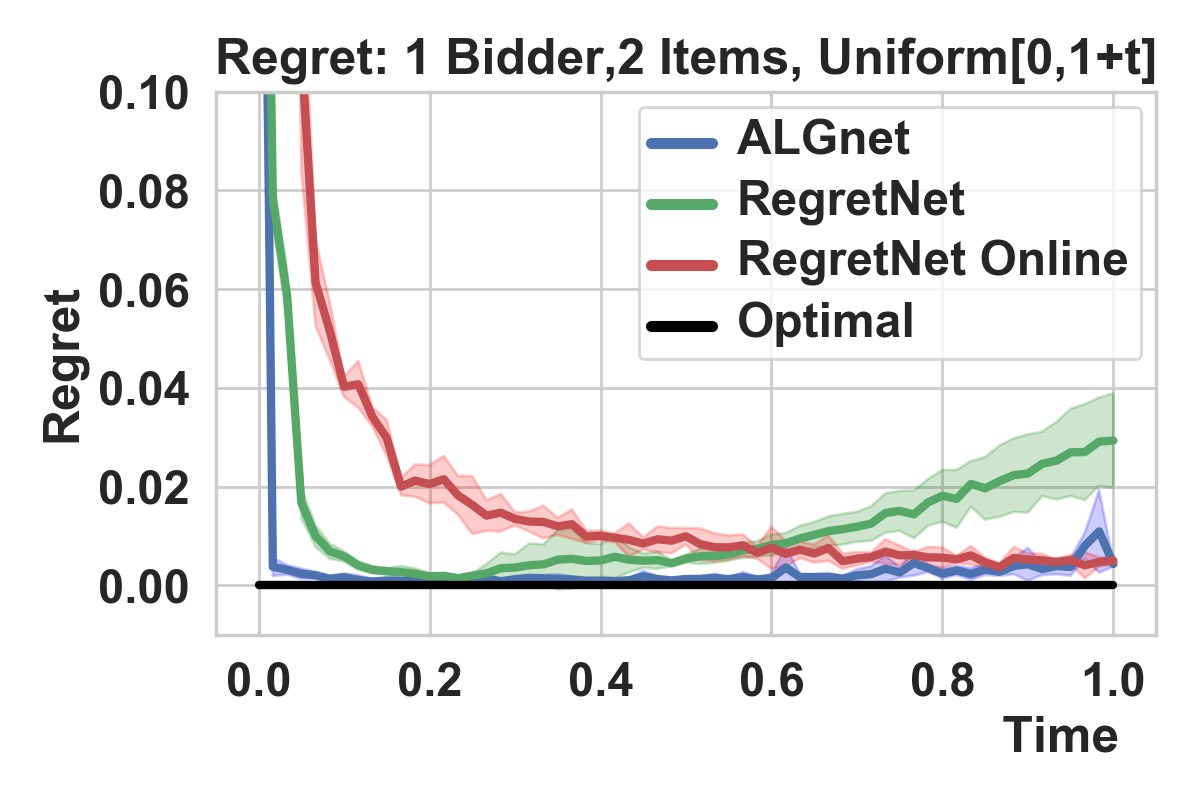}
  \caption{}
  \label{fig:5}
\end{subfigure}
\begin{subfigure}{0.5\textwidth}
  \centering
  \includegraphics[width=1.05\linewidth]{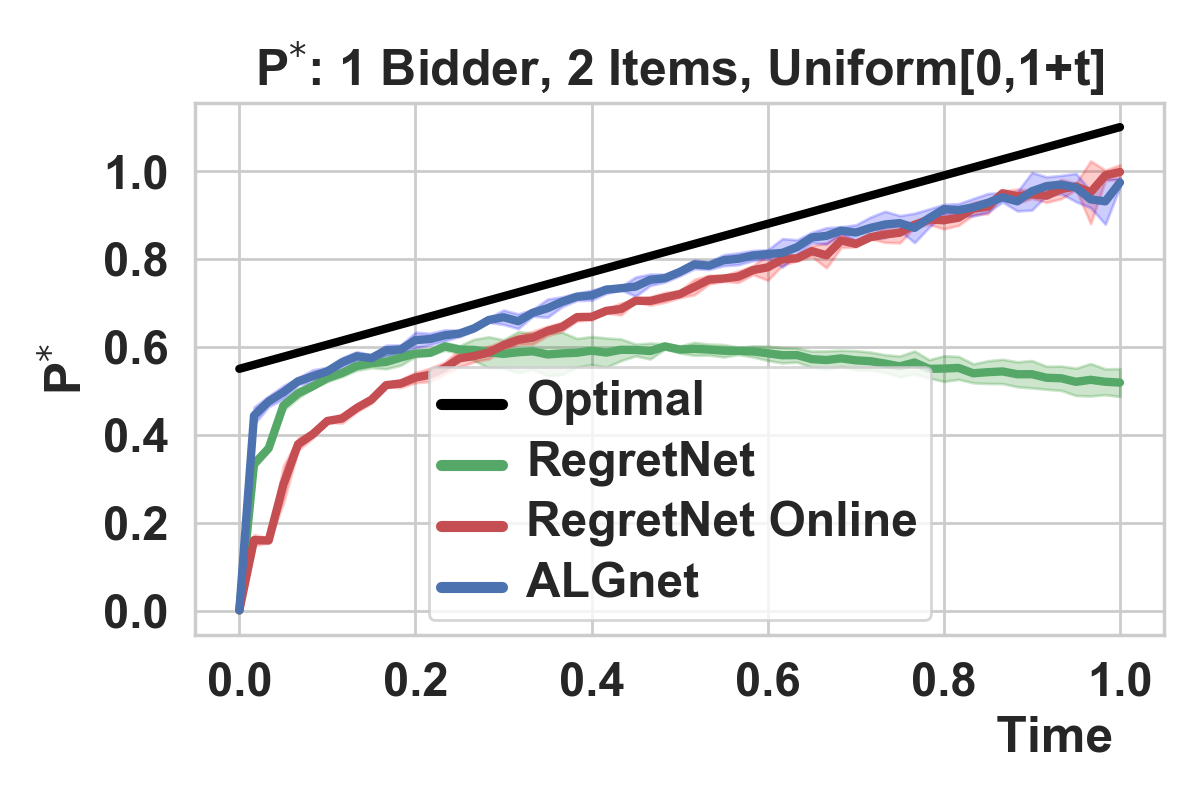}
  \caption{}
  \label{fig:6}
\end{subfigure}

\caption{(a-b-c) compares the evolution of the revenue, regret and $P^{*}$ as a function of the number of epoch for RegretNet and ALGnet for setting (A).  (d-e-f) plots the the revenue, regret and $P^{*}$ as a function of time for ALGnet and (offline \& online) RegretNet for an online auction (Section~\ref{sec:online}).  } 
\label{fig:comparaisonRegretNet}
\end{figure}

\paragraph{Unknown and large-scale auctions.} We now consider settings where the optimal auction is unknown. We look at $n$-bidder $m$-item additive settings where the valuations are sampled i.i.d from $\mathcal{U}[0,1]$ which we will denote  by $n \times m$. In addition to reasonable-scale auctions ($1\times 10$ and $2\times 2$), we investigate large-scale auctions ($3\times 10$ and $5\times 10$) that are much more complex. Only deep learning methods are able to solve them efficiently.  \autoref{fig:table_large_auct} shows that  ALGnet is able to discover auctions that yield comparable or better results than RegretNet.

\begin{table}[h]
\caption{Comparison of RegretNet and ALGnet. The values reported for RegretNet are found in \cite{dutting2017optimal}, the numerical values for $\mathit{rgt}$ and standard deviations are not available.
}\label{fig:table_large_auct}

\begin{subfigure}[b]{0.99\textwidth}
\begin{center}
\begin{tabular}{lccccccc}
\toprule
\multirow{2}{*}{Setting \quad \quad \quad \quad} & 
\multicolumn{2}{c}{RegretNet \quad \quad }& \multicolumn{2}{c}{ALGnet (Ours)}
\\
\cmidrule(lr){2-3}
\cmidrule(lr){4-5}
& $\mathit{rev}$ & $\mathit{rgt}$ & $\mathit{rev}$ & $\mathit{rgt}$
\\
\cmidrule(lr){1-1}
\cmidrule(lr){2-2}
\cmidrule(lr){3-3}
\cmidrule(lr){4-4}
\cmidrule(lr){5-5}
$1 \times 2 $  
& $0.554$ & 
$< 1.0 \cdot 10^{-3}$ & 0.555 $(\pm 0.0019)$ & $0.55 \cdot 10^{-3} (\pm 0.14 \cdot 10^{-3})$ 
\\
% \hline
$1 \times 10 $
& 3.461 & $< 3.0 \cdot 10^{-3}$ & 3.487 $(\pm 0.0135)$ & $1.65 \cdot 10^{-3} (\pm 0.57 \cdot 10^{-3})$  
\\
% \hline

$2 \times 2 $
& 0.878 & $< 1.0 \cdot 10^{-3}$ & 0.879 $(\pm 0.0024)$ & $0.58 \cdot 10^{-3} (\pm 0.23 \cdot 10^{-3})$
\\ 
% \hline

$3 \times 10 $ 
& 5.541 & $<2.0 \cdot 10^{-3}$ & 5.562 $(\pm 0.0308)$ & $ 1.93 \cdot 10^{-3} (\pm 0.33 \cdot 10^{-3})$ 
\\
% \hline

$5 \times 10 $
& 6.778 & $ < 5.0 \cdot 10^{-3}$ & 6.781 $(\pm 0.0504)$ & $3.85 \cdot 10^{-3} (\pm 0.43 \cdot 10^{-3})$\\
\bottomrule
\end{tabular}
\end{center}
\end{subfigure}
\end{table}

\subsection{Online auctions}
\label{sec:online}

 ALGnet is an online algorithm with a time-independent loss function. We would expect it to perform well in settings where the underlying distribution of the valuations  changes over time. We consider a one bidder and  two items additive auction with valuations $v_1$ and $v_2$ sampled i.i.d from $\mathcal{U}[0,1+t]$ where $t$ in increased from $0$ to $1$ at a steady rate. The optimal auction at time $t$ has revenue $0.55\times(1+t)$.
 
 We use ALGnet and two versions of RegretNet, the original offline version (Appendix A) and our own online version (Appendix B) and plot $\mathit{rev}(t)$, $\mathit{rgt}(t)$ and $P^{*}(t)$ (\autoref{fig:comparaisonRegretNet}). The offline version learns from a fixed dataset of valuations sampled  at $t=0$ (i.e. with $V \sim \mathcal{U}[0,1]^{nm}$) while the online versions (as ALGnet) learns from a stream of data at each time $t$.
 
 Overall, ALGnet performs better than the other methods. It learns an optimal auction faster at the initial (especially compared to RegretNet Online) and keep adapting to the distributional shift (contrary to vanilla RegretNet).

\section{Conclusion}

We identified two inefficiencies in previous approaches to deep auction design and propose solutions, building upon recent trends and results from machine learning (amortization) and theoretical auction design (stationary Lagrangian). This resulted in a novel formulation of auction learning as a two-player game between an Auctioneer and a Misreporter and a new architecture ALGnet. ALGnet requires significantly fewer hyperparameters than previous Lagrangian approaches. We demonstrated the effectiveness of ALGnet on a variety of examples by comparing it to the theoretical optimal auction when it is known, and to RegretNet when the optimal solution is not known.

\begin{subappendices}
\section{Proof of \autoref{prop:main}}\label{app:rubinst_wein}
\begin{lemma}
\label{lemma:1bidder}
Let $M$ be a one bidder $m$ item mechanism with expected revenue $P$ and expected regret $R$, then $\forall \epsilon > 0$, there exists a mechanism $M'$ with expected revenue $P' = (1-\epsilon)P - \frac{1-\epsilon}{\epsilon} R$ and zero expected regret, $R' = 0$. 
\end{lemma}

\begin{proof}
For every valuation vector $v \in D$, let $g(v)$ and $p(v)$ denote the allocation vector and price that $M$ assigns to $v$.

We now consider the mechanism $M'$ that does the following:

\begin{itemize} 
    \item $g'(v) = g(v')$ 
    \item $p'(v) = (1-\epsilon) \, p(v') $
\end{itemize}
Where $v'$ is given by : $v' = \mbox{argmax}_{ \tilde{v}\in D} \,\, \langle v \,,\, g(\tilde{v}) \rangle - (1-\epsilon) \, p(\tilde{v}). $
By construction, the mechanism $M'$ has zero regret, all we have to do now is bound its revenue. If we denote by $R(v)$ the regret of the profile $v$ in the mechanism $M$, $R(v) = \mbox{max}_{ \tilde{v}\in D} \,\, \langle v \,,\, g(\tilde{v})-g(v) \rangle -  (p(\tilde{v})-p(v))$ we have.
\begin{align}
\langle v \,,\, g(v') \rangle -  p(v') &=  \langle v \,,\, g(v) \rangle -  p(v) + \langle v \,,\, g(v') -g(v) \rangle - (p(v')-p(v))  \\
& \leq  \langle v \, ,\, g(v) \rangle -  p(v) + R(v) 
\end{align}
Which we will write as: 
\begin{equation}
\langle v \, ,\, g(v) \rangle -  p(v)   \geq  \langle v \,,\, g(v') \rangle -  p(v') - R(v) 
\end{equation}
Second, we have by construction: 
\begin{align}
\langle v \,,\, g(v') \rangle - (1-\epsilon) p(v')  \geq \langle v \,,\, g(v) \rangle - (1-\epsilon) p(v)
\end{align}
By summing these two relations we find :
\begin{equation}
 p(v') \geq p(v) -\frac{R(v)}{\epsilon} 
\end{equation}
Finally we get that: 
\begin{equation}
 p'(v) \geq (1-\epsilon)\,p(v) - \frac{1-\epsilon}{\epsilon} \, R(v)
 \end{equation}
Taking the expectation we get: 
\begin{equation}
 P' \geq  (1-\epsilon) \, P - \frac{1-\epsilon}{\epsilon} \,R 
\end{equation}
\end{proof}

\begin{manualtheorem}{1} 
Let $\mathcal{M}$ be an additive auction with $1$ bidders and $m$ items. Let $P$ and $R$  denote the total expected revenue and regret, $P  = \mathbb{E}_{V \in D}\left[ p(V) \right]$ and $R = \mathbb{E}_{V \in D}\left[ r(V)\right]$.
There exists a mechanism $\mathcal{M}^{*}$ with expected revenue $P^{*} = \left(\sqrt{P }-\sqrt{R }\right)^2$ and zero regret $R^{*} = 0$. 

\end{manualtheorem}

\begin{proof}
From \autoref{lemma:1bidder} we know that $\forall \epsilon > 0 $, we can find a zero regret mechanism with revenue $P' = (1-\epsilon) \, P - \frac{1-\epsilon}{\epsilon} \,R$. 
By optimizing over $\epsilon$ we find that the best mechanism is the one correspond to $\epsilon = \sqrt{\frac{R}{P}}$. The resulting optimal revenue is given by: 
\begin{equation}
    P^{*} = (1-\sqrt{\frac{R}{P}})P - \frac{\sqrt{\frac{R}{P}}}{\sqrt{\frac{R}{P}}} R = P - 2\sqrt{PR} + R = \left(\sqrt{P}-\sqrt{R}\right)^2
\end{equation}
\end{proof}

\section{Training Algorithm for Regret Net}\label{app:dutt_train_alg}

We present the training algorithm for RegretNet, more details can be found in  \citet{dutting2017optimal}.
\begin{algorithm}[h]
\caption{Training Algorithm.} 
\label{alg:duttingAlgorithm}
\begin{algorithmic}[1]
%\State Initialization:
\State \textbf{Input}: Minibatches $\mathcal{S}_1,\dots,\mathcal{S}_T$ of size $B$
\State \textbf{Parameters}: $\gamma >0, \, \eta >0, \, c>0, \, R\in \mathbb{N}, \,  T\in \mathbb{N}, \,T_{\rho}\in \mathbb{N}, \,T_{\lambda}\in \mathbb{N}.$
\State \textbf{Initialize Parameters}: $ \rho^0 \in \mathbb{R}, \, w^0\in \mathbb{R}^d, \, \lambda^0\in\mathbb{R}^n,  $

\State \textbf{Initialize Misreports:} ${v_i'}^{(\ell)}\in \mathcal{D}_i, \,\, \forall \ell \in [B], \, i\in N.$ \\
%\State Main routine:
\For {$t=0,\dots,T$}
        \State Receive minibatch $\mathcal{S}_t = \{V^{(1)},\dots,V^{(B)}\}.$
        \For{$r=0,\dots,R$}
        \State
        \vspace{-1.1cm}
        
        \begin{align*}
            \hspace{.6cm}&\forall \ell \in [B], \, i\in n:\\
            &{v_i'}^{(\ell)} \leftarrow {v_i'}^{(\ell)}+\gamma \nabla_{v_i'}u_i^{w^t} ({v_i}^{(\ell)};({v_i'}^{(\ell)},V_{-i}^{(\ell)}))
        \end{align*}
        \EndFor
        
        \vspace{-.4cm}
        \\
        \State Get Lagrangian gradient 
        % using \eqref{eq:grad_Lag} 
        and update $w^t$:\\
         \hspace{1cm} $w^{t+1}\leftarrow w^t-\eta \nabla_w \mathcal{L}(w^t;\lambda^t;\rho^t). $
         \\
        \State Update $\rho$ once in $T_{\rho}$ iterations:
        \If{$t$ is a multiple of $T_{\rho}$}
        \State $\rho^{t+1} \leftarrow \rho^{t} + c $ 
        \Else
        \State $\rho^{t+1}\leftarrow \rho^t$
        \EndIf
         \\
        \State Update Lagrange multipliers once in $T_{\lambda}$ iterations:
        \If{$t$ is a multiple of $T_{\lambda}$}
        \State $\lambda_i^{t+1}\leftarrow \lambda_i^t + \rho^t \, \widehat{r}_i(w^{t}),\forall i\in N$
        \Else
        \State $\lambda^{t+1}\leftarrow \lambda^t$
        \EndIf
\EndFor
%\EndProcedure
\end{algorithmic}
\end{algorithm}

\clearpage
\section{Training algorithm for Online Regret Net}\label{app:dutt_train_alg_online}

We present an online version of the training algorithm for RegretNet, more details can be found in \citet{dutting2017optimal}. This version in mentioned in the original paper but the algorithm is not explicitly written there. The following code is our own adaptation of the original RegretNet algorithm for online settings.

\begin{algorithm}[h]
\caption{Training Algorithm.} 
\label{alg:duttingAlgorithm2}
\begin{algorithmic}[1]
%\State Initialization:
\State \textbf{Input}: Valuation's Distribution $\mathcal{D}$
\State \textbf{Parameters}: $\gamma >0, \, \eta >0, \, c>0, \, R\in \mathbb{N}, \,  T\in \mathbb{N}, \,T_{\rho}\in \mathbb{N}, \,T_{\lambda}\in \mathbb{N}, B \in \mathbb{N}$
\State \textbf{Initialize Parameters}: $ \rho^0 \in \mathbb{R}, \, w^0\in \mathbb{R}^d, \, \lambda^0\in\mathbb{R}^n,  $

%\State Main routine:
\For {$t=0,\dots,T$}
        \State Sample minibatch $\mathcal{S}_t = \{V^{(1)},\dots,V^{(B)}\}$ from distribution $\mathcal{D}$.
        \State {Initialize Misreports:} ${v_i'}^{(\ell)}\in \mathcal{D}_i, \,\, \forall \ell \in [B], \, i\in N.$ \\
        \For{$r=0,\dots,R$}
        \State
        
        \begin{align*}
            \hspace{.6cm}&\forall \ell \in [B], \, i\in n:\\
            &{v_i'}^{(\ell)} \leftarrow {v_i'}^{(\ell)}+\gamma \nabla_{v_i'}u_i^{w^t} ({v_i}^{(\ell)};({v_i'}^{(\ell)},V_{-i}^{(\ell)}))
        \end{align*}
        \EndFor
        
        \\
        \State Get Lagrangian gradient 
        % using \eqref{eq:grad_Lag} 
        and update $w^t$:\\
         \hspace{1cm} $w^{t+1}\leftarrow w^t-\eta \nabla_w \mathcal{L}(w^t;\lambda^t;\rho^t). $
         \\
        \State Update $\rho$ once in $T_{\rho}$ iterations:
        \If{$t$ is a multiple of $T_{\rho}$}
        \State $\rho^{t+1} \leftarrow \rho^{t} + c $ 
        \Else
        \State $\rho^{t+1}\leftarrow \rho^t$
        \EndIf
         \\
        \State Update Lagrange multipliers once in $T_{\lambda}$ iterations:
        \If{$t$ is a multiple of $T_{\lambda}$}
        \State $\lambda_i^{t+1}\leftarrow \lambda_i^t + \rho^t \, \widehat{r}_i(w^{t}),\forall i\in N$
        \Else
        \State $\lambda^{t+1}\leftarrow \lambda^t$
        \EndIf
\EndFor

\end{algorithmic}
\end{algorithm}

\clearpage

\section{Implementation and Setup}

We implemented ALGnet in PyTorch and all our experiments can be run on Google's Colab platform (with GPU). 
In \autoref{alg:gameAlgorithm}, we used batches of valuation profiles of size $B\in \{500\}$ and set $T \in \{160000, 240000 \}$, $T_{limit} \in \{40000, 60000 \}$, $T_{init} \in \{800, 1600\}$ and $\tau \in \{100\}.$

We used the AdamW optimizer \citep{loshchilov2017decoupled} to train the Auctioneer's and the Misreporter's networks with learning rate $\gamma \in \{0.0005,0.001\}.$ Typical values for the architecture's parameters are $n_a=n_p=n_m \in [3,7]$ and $h_p=h_n=h_m \in \{50,100,200\}$. These networks are similar in size to the ones used for RegretNet in \citet{dutting2017optimal}.

For each experiment, we compute the total revenue $\mathit{rev}:=\mathbb{E}_{V\sim D}[\sum_{i\in N} p_i^{w}(V)]$ and average regret $\mathit{rgt}:=\nicefrac{1}{n}\,\mathbb{E}_{V\sim D}[\sum_{i\in N} r_i^{w}(V)]$ using a test set of $10,000$ valuation profiles. We run each experiment 5 times with different random seeds and report the average and standard deviation of these runs.

\end{subappendices}

\singlespacing
\bibliographystyle{plainnat}

% add the Bibliography to the Table of Contents
\cleardoublepage
\ifdefined\phantomsection
  \phantomsection  % makes hyperref recognize this section properly for pdf link
\else
\fi
\addcontentsline{toc}{chapter}{Bibliography}

% include your .bib file
\bibliography{thesis}

\end{document}